\PassOptionsToPackage{dvipsnames,table}{xcolor}
\documentclass{article}
\usepackage{iclr2027_conference,times}
\iclrfinalcopy

\usepackage[utf8]{inputenc}
\usepackage[T1]{fontenc}
\usepackage{url}
\usepackage{placeins}

\usepackage{graphicx}
\usepackage{booktabs}
\usepackage{makecell}
\usepackage{subcaption}
\usepackage{microtype}
\usepackage{nicefrac}
\usepackage{mathtools}
\usepackage{mdframed}
\usepackage{amsthm,amsfonts,amssymb,mathrsfs}
\usepackage{thmtools}
\usepackage{algorithm}
\usepackage{algpseudocode}
\usepackage[ruled,vlined,linesnumbered,algo2e]{algorithm2e}
\SetAlgoNlRelativeSize{-1}
\usepackage{enumitem}
\usepackage{aliascnt}
\usepackage{wrapfig}
\usepackage{caption}
\usepackage{tabularx}
\usepackage[dvipsnames,table]{xcolor}
\usepackage{tikz}
\usetikzlibrary{arrows.meta,positioning,calc,fit,shapes.geometric}
\usepackage[colorlinks,linktoc=all,pagebackref=true]{hyperref}
\usepackage[all]{hypcap}
\usepackage[nameinlink,capitalise,noabbrev]{cleveref}

\makeatletter
\@ifundefined{theHALG@line}
  {\newcommand{\theHALG@line}{\thealgorithm.\arabic{ALG@line}}}
  {\renewcommand{\theHALG@line}{\thealgorithm.\arabic{ALG@line}}}
\makeatother

\definecolor{darkred}{RGB}{100,0,0}
\definecolor{darkblue}{RGB}{0,0,100}
\definecolor{darkgreen}{RGB}{0,75,0}
\definecolor{remarkgreen}{RGB}{238,248,240}

\hypersetup{
    citecolor=darkgreen,
    linkcolor=darkgreen,
    urlcolor=darkgreen
}

\declaretheoremstyle[
    spaceabove=\parsep,
    spacebelow=\parsep,
    bodyfont=\normalfont\itshape,
]{theoremsty}

\declaretheoremstyle[
    spaceabove=\parsep,
    spacebelow=\parsep,
    bodyfont=\normalfont,
]{normalsty}

\mdfdefinestyle{coloredstyle}{%
    leftline=false,%
    rightline=false,%
    topline=false,%
    bottomline=false,%
    backgroundcolor=gray!10,%
    innertopmargin=1.5ex,%
    innerbottommargin=0.7ex,%
    innerleftmargin=1.ex,%
    innerrightmargin=1.ex,%
    skipabove=\topskip,%
    skipbelow=\parskip%
}

\mdfdefinestyle{remarkstyle}{%
    leftline=false,%
    rightline=false,%
    topline=false,%
    bottomline=false,%
    backgroundcolor=remarkgreen,%
    innertopmargin=1.5ex,%
    innerbottommargin=0.7ex,%
    innerleftmargin=1.ex,%
    innerrightmargin=1.ex,%
    skipabove=\topskip,%
    skipbelow=\parskip%
}

\declaretheorem[name=Theorem,style=theoremsty,mdframed={style=coloredstyle}]{theorem}
\declaretheorem[name=Proposition,style=theoremsty,mdframed={style=coloredstyle}]{proposition}

\declaretheorem[name=Theorem,style=theoremsty,mdframed={style=coloredstyle}]{theorem*}
\declaretheorem[name=Proposition,style=theoremsty,mdframed={style=coloredstyle}]{proposition*}
\declaretheorem[name=Corollary,style=theoremsty,mdframed={style=coloredstyle}]{corollary*}
\declaretheorem[name=Lemma,style=theoremsty,mdframed={style=coloredstyle}]{lemma*}
\declaretheorem[name=Conjecture,style=theoremsty,mdframed={style=coloredstyle}]{conjecture*}

\declaretheorem[name=Theorem,style=theoremsty,mdframed={style=coloredstyle},numberwithin=section]{theorem_ap}
\declaretheorem[name=Proposition,style=theoremsty,mdframed={style=coloredstyle},numberwithin=section]{proposition_ap}
\declaretheorem[name=Corollary,style=theoremsty,mdframed={style=coloredstyle},numberwithin=section]{corollary_ap}
\declaretheorem[name=Lemma,style=theoremsty,mdframed={style=coloredstyle},numberwithin=section]{lemma_ap}

\declaretheorem[name=Remark,style=normalsty,mdframed={style=remarkstyle}]{remark}

\declaretheorem[name=Assumption,style=normalsty,mdframed={style=coloredstyle}]{assumption}

\declaretheorem[name=Remark,style=normalsty,mdframed={style=remarkstyle},numbered=no]{remark*}
\declaretheorem[name=Definition,style=normalsty,mdframed={style=coloredstyle},numbered=no]{definition*}
\declaretheorem[name=Assumption,style=normalsty,mdframed={style=coloredstyle},numbered=no]{assumption*}

\numberwithin{theorem}{section}
\numberwithin{proposition}{section}
\numberwithin{corollary}{section}
\numberwithin{lemma}{section}
\numberwithin{conjecture}{section}
\numberwithin{remark}{section}
\numberwithin{definition}{section}
\numberwithin{assumption}{section}
\numberwithin{example}{section}

\usepackage{bbm}

\newcommand{\calX}{{\mathcal{X}}}

\newcommand{\calZ}{{\mathcal{Z}}}

\newcommand{\bbE}{\mathbb{E}}

\newcommand{\bbP}{\mathbb{P}}

\newcommand{\bbR}{\mathbb{R}}

\newcommand{\mbra}[1]{\left\{#1\right\}}
\newcommand{\sbra}[1]{\left(#1\right)}

\definecolor{blendedblue}{rgb}{0.2,0.2,0.7}
\definecolor{revisionbrick}{HTML}{7A3E45}
\definecolor{emphasisplum}{HTML}{594653}
\definecolor{revisionteal}{HTML}{006D77}
\definecolor{revisionpurple}{HTML}{6F42A6}
\definecolor{revisionorange}{HTML}{A14F00}
\definecolor{revisionforest}{HTML}{28643A}
\newif\ifshowauthorcomments
\showauthorcommentsfalse

\usepackage{arydshln}

\newif\ifshowbibliography
\showbibliographytrue

\title{Imperfect Verifiers, Exact Targets: Inference-Time Scaling for Continuous Generation}
\author{%
Kijung Jeon \\
Georgia Institute of Technology \\
\texttt{kjeon@gatech.edu}
\And
Molei Tao \\
Georgia Institute of Technology \\
\texttt{mtao@gatech.edu}
}

\crefname{maintheorem}{Theorem}{Theorems}
\Crefname{maintheorem}{Theorem}{Theorems}
\crefname{theorem_ap}{Theorem}{Theorems}
\Crefname{theorem_ap}{Theorem}{Theorems}
\crefname{proposition_ap}{Proposition}{Propositions}
\Crefname{proposition_ap}{Proposition}{Propositions}
\crefname{corollary_ap}{Corollary}{Corollaries}
\Crefname{corollary_ap}{Corollary}{Corollaries}
\crefname{lemma_ap}{Lemma}{Lemmas}
\Crefname{lemma_ap}{Lemma}{Lemmas}

\newcommand{\algorithmhrefprefix}{main}
\makeatletter
\patchcmd{\algocf@caption@algo}
  {}
  {}
  {}{\PackageError{main}{Algorithm hyperlink patch failed}{Check algorithm2e's caption definition.}}
\makeatother

\begin{document}

\maketitle
\lhead{Preprint}
\begin{abstract}
Verifier-guided inference-time scaling uses additional computation to score intermediate states and steer generation toward better outputs. However, imperfect intermediate verifiers can distort sampling, and repeated refinement alone does not ensure recovery of the desired reward-tilted distribution. We introduce Continuous-State Verifier-Guided Backtracking (CS-VGB), an inference-time scaling method for sampling from reward-tilted distributions with continuous generative models through repeated denoising and renoising. With correctly paired denoising and renoising transitions, the stationary distribution of CS-VGB at the clean checkpoint equals the desired reward-tilted target, even when intermediate verifiers are inaccurate. Under stated regularity conditions, CS-VGB converges exponentially to the target despite verifier error, and the expected number of updates needed to reach a clean sample grows only linearly with the number of checkpoints.
In analytically tractable settings, we empirically demonstrate that CS-VGB recovers the desired reward-tilted distribution despite verifier error, whereas the compared verifier-based sampling methods fail to recover the target under the same verifier errors. We also instantiate the framework with diffusion models and one-step stochastic posterior maps. Across prompt-guided ImageNet generation, facial-attribute generation on CelebA, and multi-robot planning, recurrent CS-VGB search improves reward and task success as inference budgets grow, with image-generation rewards continuing beyond the observed plateaus of denoising-only baselines.\footnote{Code is available at \url{https://github.com/KraitGit/CS_VGB}.}

\end{abstract}

\vspace{-5pt}
\section{Introduction}
\vspace{-5pt}
\label{sec:main-introduction} 
Inference-time scaling seeks to improve the outputs of pretrained generative models by spending additional computation at sampling time. For generation in continuous spaces, such as images and robot trajectories, this computation can be allocated to Monte Carlo estimation of guidance gradients~\citep{main:song2023lgd,potaptchik2026meta,holderrieth2026diamond}, larger candidate sets for soft value-based decoding (SVDD)~\citep{li2025svdd}, or more particles with intermediate weighting and resampling via sequential Monte Carlo (SMC)~\citep{wu2023tds,singhal2025fk}.

Despite these different uses of computation, these methods follow a one-way denoising trajectory and favor intermediate states that appear promising before the terminal reward is available, often relying on a \emph{process verifier}, or \emph{intermediate verifier}, to estimate the reward of future completions. However, errors in these estimates can lead the sampler to discard promising continuations and spend subsequent computation on less favorable ones, while the one-way trajectory prevents it from revisiting earlier noise levels to reconsider those choices.

To address this limitation, recent inference-time scaling methods use renoising to return to earlier noise levels and explore alternative denoising continuations, as in depth-first search (DFS)~\citep{zhang2025classical}, reward-guided iterative refinement (RERD)~\citep{uehara2025refinement}, and cyclic diffusion search (ABCD)~\citep{lee2025cyclic}. These approaches enable bidirectional exploration across noise levels, but it remains unclear how to preserve a prescribed target distribution despite errors in the process verifier. We therefore seek the \emph{reward-tilted distribution} $\pi_0^\star(dx)\propto r(x)\pi_0(dx)$, where $\pi_0$ is the model's clean distribution and $r$ is a nonnegative terminal reward, and ask:
\begin{mdframed}[style=coloredstyle,backgroundcolor=black!4,innertopmargin=4pt,innerbottommargin=4pt,innerleftmargin=6pt,innerrightmargin=6pt,skipabove=3pt,skipbelow=3pt]
\noindent\begin{minipage}[t]{1.55em}
\vspace{0pt}\textbf{\sffamily Q}
\end{minipage}%
\begin{minipage}[t]{\dimexpr\linewidth-1.55em\relax}
\vspace{0pt}\raggedright
Can we construct a bidirectional sampler for continuous-state generation that revisits earlier decisions while preserving the desired reward-tilted target despite errors in process verifiers?
\end{minipage}
\end{mdframed}

Recent work in discrete generation addresses this question through verifier-guided backtracking: AR-VGB~\citep{rohatgi2026vgb} and MDM-VGB~\citep{jeon2026mdmvgb} separate the role of imperfect process verifiers from the target distribution in autoregressive and masked diffusion models, respectively. Motivated by these results, we rigorously extend this sampling principle to continuous generative models through \textbf{CS-VGB}, an outer Markov chain Monte Carlo (MCMC) framework for models that admit denoising and renoising transitions between checkpoints. CS-VGB turns a pretrained model's one-way sampler into a bidirectional, recurrent sampler that revisits earlier decisions by alternating denoising and renoising. In this construction, the process verifier guides exploration, while the exact terminal reward determines the stationary distribution at the clean checkpoint.

Our contributions are threefold:
\begin{itemize}[leftmargin=*,itemsep=3pt,topsep=4pt]
\item \textbf{CS-VGB with imperfect process verifiers.} We establish that, with correctly paired denoising and renoising transitions, the stationary clean distribution is $\pi_0^\star$ regardless of process-verifier accuracy. To retain this guarantee in practical finite-candidate implementations, we introduce \emph{reverse insertion}, together with \emph{lazy pool evaluation} to reduce the number of process-verifier evaluations without changing the stationary clean target.
\item \textbf{Convergence guarantees and first-hitting times.} Under stated assumptions, CS-VGB converges exponentially fast despite process-verifier error, while \textit{directional persistence} enables a first clean visit in an expected number of updates linear in the checkpoint count. With sufficiently small transition mismatch, convergence remains still exponential to a potentially biased target, whose bias vanishes as the transition mismatch tends to zero.
\item \textbf{Inference-time sampling and search.} In analytically tractable settings, we empirically demonstrate that CS-VGB recovers the reward-tilted target despite process-verifier error, whereas the compared verifier-based sampling methods deviate from this target. Building on the CS-VGB framework, we also propose a finite-budget search algorithm for diffusion models and one-step posterior maps~\citep{potaptchik2026meta,holderrieth2026diamond}, demonstrating that recurrent sampling can improve performance beyond the plateaus of denoising-only methods.
\end{itemize}
We discuss additional related work on reward-guided sampling and search, verifier-guided backtracking, and stochastic transitions and posterior maps in \cref{sec:main-related}.

\section{CS-VGB: Recurrent Sampling and Search}
\label{sec:main-results}
\label{sec:main-preliminaries}

We consider inference-time scaling methods for sampling from the \emph{reward-tilted distribution} of a pretrained generator with clean law $\pi_0$ and terminal reward $r:\mathcal X\to[0,\infty)$, where $\mathcal X\subset\mathbb R^d$:
\begin{equation}
\pi_0^\star(dx):=\frac{r(x)\pi_0(dx)}{Z},
\qquad Z:=\mathbb E_{X_0 \sim \pi_0}[r(X_0)]\in(0,\infty).
\label{eq:main-target-definition}
\end{equation}
\paragraph{Checkpoints and transition kernels.}
On a fine time grid $0=t_0<\cdots<t_T=1$, reference generation starts from $X_{t_T}\sim\pi_{t_T}=\mathcal N(0,I)$ and follows the pretrained denoising process toward $t_0=0$. Along this process, we select \emph{checkpoints} $t_{\kappa_j}$ with $0=\kappa_0<\cdots<\kappa_J\leq T$ and denote the corresponding reference states by $X_j:=X_{t_{\kappa_j}}$. These checkpoints are the noise levels at which CS-VGB chooses whether to continue denoising or return to a noisier state. Between adjacent checkpoints, the pretrained generator induces the \emph{denoising kernel} $K_j^{\mathrm F}(dy\mid x)$, the conditional law of $X_{j-1}$ given $X_j=x$, while the \emph{renoising kernel} $K_{j-1}^{\mathrm B}(dx\mid y)$ proposes a move from checkpoint $j-1$ back to $j$. Together, the starting law $\pi_J$ and the denoising kernels determine the reference marginals $\pi_j=\operatorname{Law}(X_j)$ through
\begin{equation}
\pi_{j-1}=\pi_jK_j^{\mathrm F},\qquad j=J,\ldots,1,
\label{eq:main-reference-marginals}
\end{equation}
where $\pi_J$ is the reference law at the starting checkpoint and equals $\mathcal N(0,I)$ when $\kappa_J=T$. We instantiate these denoising and renoising kernels for diffusion models~\citep{ho2020ddpm,song2021score} and for two one-step posterior maps: Meta Flow Maps (MFM)~\citep{potaptchik2026meta} and Posterior Diamond Maps (PDM)~\citep{holderrieth2026diamond}. The resulting checkpoint transitions are summarized in \cref{tab:main-checkpoint-kernels}.

\paragraph{Process verifier and sampling goal.}
We assume access to a \emph{process verifier} $\widehat V_j(x)$ that estimates the reward of future completions from state $x$ at checkpoint $j$, whose ideal value is $V_j(x):=\mathbb E[r(X_0)\mid X_j=x]$ under the reference denoising process. While $\widehat V_j$ may differ from $V_j$ at noisy checkpoints, we assume that $\widehat V_0=r$ is evaluated exactly. Our goal is to construct an outer MCMC kernel $P$ that uses these estimates to select denoising and renoising moves while preserving $\Pi(dx\mid j=0)=\pi_0^\star(dx)$ under its stationary law $\Pi$, without requiring $\widehat V_j=V_j$.

\begin{table}[!ht]
\centering
\caption{\textbf{Stochastic checkpoint kernels for CS-VGB.} Here $t=t_{\kappa_j}$ and $s=t_{\kappa_{j-1}}$. For the VP SDE with noise-rate schedule $\beta$, $\alpha_{t\mid s}:=\exp(-\frac12\int_s^t\beta(u)\,du)$ and $\sigma_{t\mid s}^2:=1-\alpha_{t\mid s}^2$. For the linear-interpolant constructions, $a_j:=(1-t)/(1-s)$ and $\sigma_j^2:=t^2-a_j^2s^2$. Full transitions and endpoint conventions are given in \cref{app:subsec:reference-model-classes,app:subsec:model-checkpoint-transitions}.}
\label{tab:main-checkpoint-kernels}

\setlength{\tabcolsep}{4pt}
\renewcommand{\arraystretch}{1.2}
\begin{tabular}{@{}m{0.20\linewidth}m{0.45\linewidth}m{0.30\linewidth}@{}}
\toprule
Model & Denoising: $K_j^{\mathrm F}(dy\mid x)$ & Renoising: $K_{j-1}^{\mathrm B}(dx\mid y)$\\
\midrule
Diffusion models\newline (VP SDE) &
$K_j^{\mathrm F}=k_{\kappa_j}^{\mathrm F}\cdots k_{\kappa_{j-1}+1}^{\mathrm F}$,\newline
$k_k^{\mathrm F}$: reverse diffusion step\newline
(\cref{eq:cs-vgb-diffusion-fine-forward}) &
$\mathcal N(dx;\alpha_{t\mid s}y,\sigma_{t\mid s}^2I)$\\
\addlinespace[3pt]
One-step posterior map\newline (MFM/PDM) &
$K_j^{\mathrm F}(dy\mid x)=\operatorname{Law}(Y\in dy\mid x)$,\newline
$Y$: Gaussian bridge sample\newline
(\cref{fig:cs-vgb-posterior-bridge}) &
$\mathcal N(dx;a_jy,\sigma_j^2I)$\\
\bottomrule
\end{tabular}

\end{table}
\FloatBarrier
\vspace{-5pt}
\subsection{Verifier-Guided Sampling and Search Framework}
\label{sec:main-single-chain}
\paragraph{Edge-conductance-based construction.}
To favor states with higher verifier values while preserving the reward-tilted clean target, we construct denoising and renoising transitions from a common edge-conductance measure. Following the weighted-network construction of random walks and its continuous-state counterpart~\citep{doyle1984random,bezuglyi2018markov}, we weight adjacent state pairs and normalize at either endpoint to obtain the two transitions. In this case, the interior edges use process verifiers, whereas the clean edge uses only the terminal reward. Together with the action rule below, this construction preserves the stationary clean target under the exact-kernel conditions of \cref{thm:main-stationarity}.

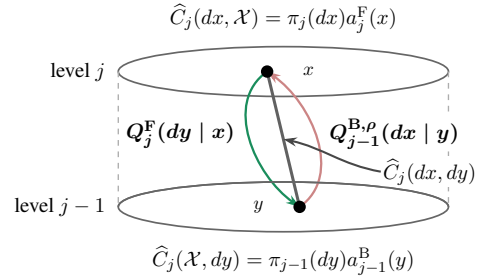
\begin{wrapfigure}[13]{r}{0.45\textwidth}
\vspace{-22pt}
\centering
\resizebox{\linewidth}{!}{\definecolor{conductanceRenoise}{HTML}{CC8585}
\begin{tikzpicture}[text=black,
    x=1cm,
    y=1cm,
    font=\sffamily\small,
    layer/.style={draw=black!60,line width=0.9pt},
    wall/.style={draw=black!32,dashed,line width=0.7pt},
    point/.style={circle,fill=black,inner sep=2.3pt},
    shared/.style={draw=black!62,line width=1.6pt},
    forward/.style={-{Stealth[length=2.5mm,width=1.7mm]},
        draw=ForestGreen!82!black,line width=1.2pt},
    backward/.style={-{Stealth[length=2.5mm,width=1.7mm]},
        draw=conductanceRenoise,line width=1.2pt},
    callout/.style={-{Stealth[length=2.5mm,width=1.7mm]},
        draw=black!72,line width=1.1pt},
    edge-label/.style={fill=white,inner sep=1.7pt,font=\sffamily\small},
    mass/.style={font=\sffamily\small,align=center},
    level-label/.style={font=\normalfont\small},
    conductance-label/.style={font=\sffamily\normalsize}
]
    \tikzset{
        every node/.append style={font=\sffamily\fontsize{11}{12}\selectfont},
        edge-label/.append style={font=\sffamily\fontsize{11}{12}\selectfont,inner sep=1pt},
        mass/.append style={font=\sffamily\fontsize{11}{12}\selectfont},
        conductance-label/.append style={font=\sffamily\fontsize{11}{12}\selectfont},
        level-label/.append style={font=\normalfont\fontsize{11}{12}\selectfont,inner sep=1pt}
    }
    \coordinate (upper) at (-0.30,1.25);
    \coordinate (lower) at (0.30,-1.25);

    \draw[layer] (0,1.25) ellipse (3.05 and 0.46);
    \draw[layer] (0,-1.25) ellipse (3.05 and 0.46);
    \draw[wall] (-3.05,1.25) -- (-3.05,-1.25);
    \draw[wall] (3.05,1.25) -- (3.05,-1.25);

    \node[anchor=east,level-label]
        at (-3.25,1.25) {level ${j}$};
    \node[anchor=east,level-label]
        at (-3.25,-1.25) {level ${j-1}$};

    \node[mass,anchor=south] at (0,1.84)
        {$\widehat C_j(dx,\mathcal X)=\pi_j(dx)a_j^{\mathrm F}(x)$};
    \node[mass,anchor=north] at (0,-1.84)
        {$\widehat C_j(\mathcal X,dy)=\pi_{j-1}(dy)a_{j-1}^{\mathrm B}(y)$};

    \draw[shared] (upper) -- (lower);
    \draw[forward] (upper) to[out=-145,in=145] (lower);
    \draw[backward] (lower) to[out=35,in=-35] (upper);

    \node[point,label={[font=\sffamily, label distance=12pt]1.7:$x$}] at (upper) {};
    \node[point,label={[font=\sffamily, label distance=12pt]180.1:$y$}] at (lower) {};

    \node[edge-label,text=black] at (-1.9,0.1)
        {${\boldsymbol{Q_j^{\mathrm F}(dy\mid x)}}$};
    \node[edge-label,text=black] at (2.05,0.1)
        {${\boldsymbol{Q_{j-1}^{\mathrm B,\rho}(dx\mid y)}}$};
    \node[edge-label,conductance-label,anchor=west] (conductance) at (1.8,-0.6)
        {${\widehat C_j(dx,dy)}$};
    \draw[callout] (conductance.west) to[out=170,in=-12]
        ($(upper)!0.50!(lower)$);
    \draw[layer] (-3.05,-1.25)
        arc[start angle=180,end angle=0,x radius=3.05cm,y radius=0.46cm];
\end{tikzpicture}}
\vspace{-4pt}
\captionsetup{skip=-4pt}
\caption{Normalizing $\widehat C_j$ gives both sampling transitions, while its endpoint marginals determine the directional masses.}
\label{fig:main-common-edge}
\end{wrapfigure}

We first compare the joint distributions of adjacent states $(x,y)$ obtained by denoising from $x\sim\pi_j$ or renoising from $y\sim\pi_{j-1}$, which we call the \emph{edge laws}:
\begin{equation}
\begin{aligned}
J_j^{\mathrm F}(dx,dy)&:=\pi_j(dx)K_j^{\mathrm F}(dy\mid x),\\
J_j^{\mathrm B}(dx,dy)&:=\pi_{j-1}(dy)K_{j-1}^{\mathrm B}(dx\mid y).
\end{aligned}
\label{eq:main-reference-edge-laws}
\end{equation}
To make the two joint distributions agree, we reweight $K_{j-1}^{\mathrm B}$ by the \emph{edge ratio} $\rho_j:=dJ_j^{\mathrm F}/dJ_j^{\mathrm B}$, giving $K_{j-1}^{\mathrm B,\rho}:=\rho_jK_{j-1}^{\mathrm B}$ and
\begin{equation}
\pi_j(dx)K_j^{\mathrm F}(dy\mid x)
=\pi_{j-1}(dy)K_{j-1}^{\mathrm B,\rho}(dx\mid y).
\label{eq:main-paired-transitions}
\end{equation}
When the joint distributions already coincide, $\rho_j=1$ and no correction is needed, since $K_{j-1}^{\mathrm B}(dx\mid y)$ already gives the conditional law of $X_j$ given $X_{j-1}=y$ under reference denoising.
We then weight this common edge law to define the \emph{edge-conductance measure}
\begin{equation}
\widehat C_j(dx,dy):=\pi_j(dx)K_j^{\mathrm F}(dy\mid x)
\begin{cases}
r(y),&j=1,\\
\sqrt{\widehat V_j(x)\widehat V_{j-1}(y)},&2\leq j\leq J.
\end{cases}
\label{eq:main-weighted-edge}
\end{equation}
Integrating out the noisy endpoint $x$ of the clean-level ($j=1$) edge gives $\widehat C_1(\mathcal X,dy)=r(y)\pi_0(dy)$, whose normalization is exactly $\pi_0^\star$, regardless of process verifier errors. At an interior state $x$, integrating over possible destinations at $j-1$ and $j+1$ similarly gives the \emph{directional masses} $a_j^{\mathrm F}(x)$ and $a_j^{\mathrm B}(x)$ relative to $\pi_j$:
\begin{equation}
\pi_j(dx)a_j^{\mathrm F}(x):=\widehat C_j(dx,\mathcal X),
\qquad
\pi_j(dx)a_j^{\mathrm B}(x):=\widehat C_{j+1}(\mathcal X,dx),
\label{eq:main-directional-masses}
\end{equation}
Dividing each weighted reference kernel by this directional mass makes it a probability transition kernel over the next state (\cref{fig:main-common-edge}). Thus, for $2\leq j\leq J$,
\begin{equation}
Q_j^{\mathrm F}(dy\mid x)
=\frac{\sqrt{\widehat V_j(x)\widehat V_{j-1}(y)}}{a_j^{\mathrm F}(x)}K_j^{\mathrm F}(dy\mid x),\quad
Q_{j-1}^{\mathrm B,\rho}(dx\mid y)
=\frac{\sqrt{\widehat V_j(x)\widehat V_{j-1}(y)}}{a_{j-1}^{\mathrm B}(y)}K_{j-1}^{\mathrm B,\rho}(dx\mid y).
\label{eq:main-weighted-transitions}
\end{equation}
On the clean edge, $Q_1^{\mathrm F}(dy\mid x)=r(y)K_1^{\mathrm F}(dy\mid x)/a_1^{\mathrm F}(x)$ and $Q_0^{\mathrm B,\rho}=K_0^{\mathrm B,\rho}$. We can therefore sample the destination of a move using the corresponding $Q$ kernel, and use the relative sizes of $a_j^{\mathrm F}$ and $a_j^{\mathrm B}$ to set the move, stay, and flip probabilities below.

\begin{remark}[When $\rho_j=1$]
For diffusion models, $\rho_j=1$ when the denoising transitions exactly reverse the reference noising process (\cref{rem:exact-pairing-model-classes}). The same identity holds when an exact clean posterior sample is followed by a Gaussian bridge from that noising process (\cref{prop:cs-vgb-posterior-bridge-pairing}).
\end{remark}

\Needspace{13\baselineskip}
\begin{wrapfigure}[11]{r}{0.45\textwidth}
\vspace{-17pt}
\centering
\captionsetup{skip=3pt}
\makebox[\linewidth][l]{\hspace*{-8pt}\resizebox{\linewidth}{!}{\definecolor{actionDenoise}{HTML}{19734F}
\definecolor{actionRenoise}{HTML}{9A4857}
\pgfdeclarelayer{actionBackground}
\pgfsetlayers{actionBackground,main}
\begin{tikzpicture}[x=1cm,y=1cm,font=\fontsize{11}{12}\selectfont,
state/.style={rectangle,rounded corners=3pt,draw=black!70,line width=0.75pt,fill=black!5,inner xsep=6pt,inner ysep=5pt},
action/.style={-{Stealth[length=1.8mm,width=1.3mm]},line width=1.25pt}]
\begin{pgfonlayer}{actionBackground}
\draw[draw=black!30,line width=0.5pt,dash pattern=on 3pt off 2.5pt]
    (1.6,1.91) ellipse (4.6cm and 0.78cm);
\draw[black!30, line width=0.5pt, dash pattern=on 3pt off 2.5pt]
    (1.6,0) ellipse (4.6cm and 1.01cm);
\draw[black!30, line width=0.5pt, dash pattern=on 3pt off 2.5pt]
    (1.6,-2.1) ellipse (4.6cm and 0.78cm);
\end{pgfonlayer}
\node[anchor=center,text=black,inner sep=0pt] (upperLevel) at (-3.58,1.91) {$j+1$};
\node[anchor=center,text=black,inner sep=0pt] (currentLevel) at (-3.58,0) {$j$};
\node[anchor=center,text=black,inner sep=0pt] (lowerLevel) at (-3.58,-2.1) {$j-1$};
\node[state] (forwardState) at (0,0) {$(j,x,\mathrm F)$};
\node[state] (backwardState) at (3.5,0) {$(j,x,\mathrm B)$};
\node[state] (forwardNext) at (-0.66,-2.1) {$(j-1,Y,\mathrm F)$};
\node[state] (backwardNext) at (4.05,1.9) {$(j+1,Z,\mathrm B)$};
\draw[action,draw=actionDenoise] (forwardState.south) -- (forwardNext.north);
\draw[action,draw=actionRenoise] (backwardState.north) -- (backwardNext.south);
\draw[action,draw=black!70] (forwardState.north east) to[out=45,in=135] (backwardState.north west);
\draw[action,draw=black!70] (backwardState.south west) to[out=225,in=-45] (forwardState.south east);
\draw[action,draw=black!70] (forwardState.north west) .. controls (-1.6,0.9) and (-1.6,-0.9) .. (forwardState.south west);
\draw[action,draw=black!70] (backwardState.south east) .. controls (5.1,-0.9) and (5.1,0.9) .. (backwardState.north east);
\node[align=center,anchor=south,fill=white,draw=none,inner sep=2pt] (forwardFlip) at (1.7,0.8) {flip\\$p_j^{\mathrm B}-\chi\ell_j$};
\node[align=center,anchor=north,fill=white,draw=none,inner sep=2pt] (backwardFlip) at (1.75,-0.8) {flip\\$p_j^{\mathrm F}-\chi\ell_j$};
\node[align=center,anchor=east] (forwardStay) at (-1.38,0) {stay\\$\chi\ell_j$};
\node[align=center,anchor=west] (backwardStay) at (4.88,0) {stay\\$\chi\ell_j$};
\node[anchor=west,text=actionDenoise] (forwardMove) at (-2.86,-1.15) {move: $p_j^{\mathrm F}$};
\node[anchor=east,text=actionRenoise] (backwardMove) at (5.9,1.05) {move: $p_j^{\mathrm B}$};
\end{tikzpicture}}}
\vspace{-6pt}
\caption{Interior updates: move, stay, or flip. Here $Y\sim Q_j^{\mathrm F}$ and $Z\sim Q_j^{\mathrm B,\rho}$.}
\label{fig:main-persistent-actions}
\vspace{4pt}
\end{wrapfigure}
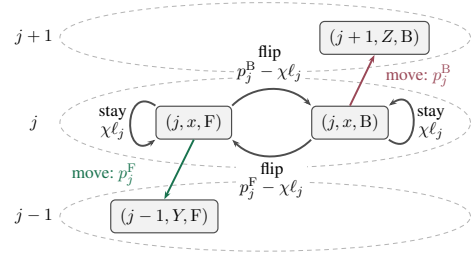
\paragraph{Persistent denoising and renoising.}
We now describe how the CS-VGB sampler moves between checkpoints while retaining its direction $d\in\mbra{\mathrm F,\mathrm B}$ as part of the state $(j,x,d)$, with $\mathrm F$ moving to $j-1$ and $\mathrm B$ to $j+1$. At an interior checkpoint, the relative directional masses $p_j^d=a_j^d/(a_j^{\mathrm F}+a_j^{\mathrm B})$ and their overlap $\ell_j=\min\{p_j^{\mathrm F},p_j^{\mathrm B}\}$ determine the action probabilities in \cref{fig:main-persistent-actions}: a \textbf{move} samples from the corresponding $Q_j^{\mathrm F}$ or $Q_j^{\mathrm B,\rho}$ kernel, a \textbf{stay} leaves the state unchanged, and a \textbf{flip} reverses only $d$. To reduce direction reversals without changing move probabilities, the \emph{directional persistence parameter} $\chi\in[0,1]$ replaces a larger fraction of flips with stays as $\chi$ increases. This follows crossing-flow cancellation in lifted chains~\citep{hayes2010liftings}, as used in discrete VGB~\citep{rohatgi2026vgb,jeon2026mdmvgb}, while boundary moves are forced inward.

Whereas $\chi$ controls directional persistence, we may also wish to control the relative preference for denoising and renoising while preserving the clean target. We therefore introduce a \emph{depth weight} $\alpha>0$ by scaling each edge-conductance measure as $\widehat C_j^\alpha:=\alpha^{j-1}\widehat C_j$ (\cref{eq:cs-vgb-calibrated-conductance}). This multiplies the backward-to-forward mass ratio $a_j^{\mathrm B}/a_j^{\mathrm F}$ by $\alpha$, so $\alpha<1$ favors denoising and $\alpha>1$ favors renoising. We remark that the stationary clean conditional remains still unchanged because the clean edge has factor $\alpha^0=1$.

\paragraph{Finite-candidate implementation.}
As the integrals defining the directional masses in \cref{eq:main-directional-masses} are generally intractable in practice, we approximate them using retained pools of $L_F,L_B\geq2$ candidates, $Y_i\sim G_j^{\mathrm F}(\cdot\mid x)$ and $Z_k\sim G_j^{\mathrm B}(\cdot\mid x)$. Here $G_j^d$ is the \emph{proposal kernel} used to generate candidates in direction $d$. It may replace the reference denoising or renoising kernel $K_j^d$ to propose more promising states, for example through gradient-based guidance in denoising or even renoising.

Combining this depth weighting with the proposal correction gives the candidate weights
\begin{equation}
e_{j,i}^{\mathrm F}:=\alpha^{j-1}\frac{dK_j^{\mathrm F}}{dG_j^{\mathrm F}}(Y_i\mid x)\sqrt{\widehat V_j(x)\widehat V_{j-1}(Y_i)},
\qquad
e_{j,k}^{\mathrm B}:=\alpha^j\frac{dK_j^{\mathrm B,\rho}}{dG_j^{\mathrm B}}(Z_k\mid x)\sqrt{\widehat V_j(x)\widehat V_{j+1}(Z_k)},
\label{eq:main-candidate-weights}
\end{equation}
where $\sqrt{\widehat V_j(x)\widehat V_{j-1}(Y_i)}$ is replaced by $r(Y_i)$ at $j=1$, and $\sqrt{\widehat V_j(x)\widehat V_{j+1}(Z_k)}$ by $r(x)$ at $j=0$. With both pools evaluated, use
\begin{equation}
\widehat a_j^d:=\frac1{L_d}\sum_{i=1}^{L_d}e_{j,i}^d,\qquad
\widehat p_j^d:=\frac{\widehat a_j^d}{\widehat a_j^{\mathrm F}+\widehat a_j^{\mathrm B}},
\qquad \widehat\ell_j:=\min\{\widehat p_j^{\mathrm F},\widehat p_j^{\mathrm B}\}
\label{eq:main-pool-action}
\end{equation}
in the action rule of \cref{fig:main-persistent-actions}, selecting candidate $i$ on a move with probability $e_{j,i}^d/\sum_h e_{j,h}^d$.

\noindent\textit{Reverse insertion.}
However, naively redrawing all candidates after each move can bias the stationary distribution, so we use reverse insertion to preserve the clean target with finite candidate pools.
After a move, insert the previous $x$ into a uniform opposite-pool slot at the destination and draw all remaining entries independently (\cref{fig:cs-vgb-checkpoint-update}). For a flip, no move or reverse insertion occurs: the state $(j,x)$ and both pools are retained, with the roles of the current and opposite pools exchanged as $d$ reverses. In the stay case, the state and pools remain unchanged.

\noindent\textit{Lazy pool evaluation.}
Opposite-direction candidates determine action probabilities but are not destinations of the current move, so frequently evaluating their full pool can add excessive model calls, especially for verifiers based on denoised clean predictions. We therefore evaluate only the current pool and one \emph{marked} opposite candidate, deferring the remaining evaluations until needed. In direction $\mathrm F$, with \emph{marked} $Z_K$ in the $\mathrm B$-candidate pool, draw a temporary $I\sim\operatorname{Unif}(\{1,\ldots,L_F\})$ and use
\begin{equation}
\bbP(\mathsf{move},\mathsf{stay},\mathsf{flip}\mid I,K)
=\frac{(\widehat a_j^{\mathrm F},\ \chi\min\{e_{j,I}^{\mathrm F},e_{j,K}^{\mathrm B}\},\
e_{j,K}^{\mathrm B}-\chi\min\{e_{j,I}^{\mathrm F},e_{j,K}^{\mathrm B}\})}
{\widehat a_j^{\mathrm F}+e_{j,K}^{\mathrm B}}.
\label{eq:main-lazy-action}
\end{equation}
A move independently redraws $Y_i$ with probability $e_{j,i}^{\mathrm F}/\sum_{h=1}^{L_F}e_{j,h}^{\mathrm F}$ and marks the reverse-inserted state at an interior destination. A flip uses the preselected index $I$ to mark $Y_I$ and fills only missing $\mathrm B$-pool entries, while a stay changes nothing. The $\mathrm B$ update is symmetric, with evaluated entries reused across flips and stays (\cref{fig:cs-vgb-lazy-cases}).

These updates form the complete CS-VGB algorithm in \cref{alg:cs-vgb-practical-outer-loop}, which still preserves the stationary clean distribution $\pi_0^\star$ under reverse insertion and lazy pool evaluation when $e_{j,i}^{\mathrm F}$ and $e_{j,k}^{\mathrm B}$ are exact (\cref{thm:c-joint-pool-invariance}).
\vspace{-5pt}

\subsection{Theoretical Guarantees of CS-VGB}
\vspace{-5pt}
\label{sec:main-guarantees}

In this subsection, we study the theoretical guarantees of CS-VGB. In particular, we ask (1) which clean distribution is preserved, (2) how quickly the chain approaches stationarity, and (3) how soon it reaches a clean sample.
\vspace{-5pt}
\paragraph{Stationary clean target.}
Since the clean-edge marginal is $\widehat C_1(\mathcal X,dx)=r(x)\pi_0(dx)$, normalizing it gives the desired target $\pi_0^\star$. The following theorem therefore shows that this target is preserved as the stationary clean conditional of CS-VGB, even when the process verifier is inaccurate.

\vspace{-5pt}
\begin{theorem}[Stationary clean target; \cref{thm:c-exact-stationarity}]
\label{thm:main-stationarity}
Under exact terminal reward evaluation (\cref{ass:terminal-anchoring}) and the paired-transition identity in \cref{eq:main-paired-transitions}, CS-VGB has a stationary law $\Pi(j,dx,d)$ satisfying the following without requiring $\widehat V_j=V_j$:
\begin{equation}
\Pi(dx,\mathrm B\mid j=0)=\Pi(dx\mid j=0)=r(x)\pi_0(dx)/Z=\pi_0^\star(dx).
\qquad Z=\int r(x)\pi_0(dx),
\label{eq:main-clean-target}
\end{equation}
\end{theorem}
\vspace{-10pt}

\paragraph{Convergence under verifier error.}
\label{sec:main-convergence}

Preserving the target does not determine how quickly it is reached. We therefore analyze the convergence speed of CS-VGB kernel $P$ defined by the exact directional masses in \cref{eq:main-directional-masses}, with depth weight $\alpha=1$. For the theoretical analysis, we use lazified kernel $P_{\mathrm L}=(I+P)/2$ to remove periodicity. We measure initialization error by $D_\nu:=\|d\nu/d\Pi-1\|_{L^2(\Pi)}$ and use $\|\cdot\|_{\mathrm{TV}}$ for total variation distance. With $\pi_j^\star(dx):=V_j(x)\pi_j(dx)/Z$, we quantify the \emph{average-state verifier error} by $\varepsilon_{\mathrm{av}}:=\max_{1\leq j\leq J}\|\widehat V_j/V_j-1\|_{L^2(\pi_j^\star)}$.
\vspace{-5pt}
\begin{theorem}[Convergence under verifier error; \cref{thm:c-average-verifier-gap,cor:c-cs-vgb-convergence}]
\label{thm:main-convergence}
Let $J\geq2$, $0<r_{\min}\leq r\leq1$, and $0<v_{\min}\leq\widehat V_j\leq1$. Assume endpoint maximal correlation $\rho_{\mathrm{end}}<1$ (\cref{eq:c-endpoint-correlation}) and a finite exact-verifier transition-density moment $H_p$ for some $p>2$ (\mbox{\cref{eq:c-oracle-density-moment}}). Set $\gamma_V=(1-\rho_{\mathrm{end}})r_{\min}/(32J^2)$. Uniformly over $\chi\in[0,1]$, we can choose
\begin{equation}
\gamma_{\mathrm{CS}}
\geq\max\!\bigl\{\gamma_Vr_{\min}v_{\min}^2,\ \gamma_V-O\!\bigl(H_p\varepsilon_{\mathrm{av}}^{\,1-2/p}\bigr)\bigr\} \geq\gamma_V\Bigl(1+O\!\Bigl(\frac{H_p\varepsilon_{\mathrm{av}}^{\,1-2/p}}{\gamma_V}\Bigr)\Bigr)^{-1}>0,
\label{eq:main-verifier-rate}
\vspace{-5pt}
\end{equation}
where the $O(\cdot)$ hides constants depending only on $r_{\min},v_{\min},p$. For initial law $\nu\ll\Pi$ with $D_\nu<\infty$, $\|\nu P_{\mathrm L}^n-\Pi\|_{\mathrm{TV}}=O(D_\nu e^{-\gamma_{\mathrm{CS}}n/4})$. For $n>\frac{4}{\gamma_{\mathrm{CS}}}\log \sbra{\max\{1,D_\nu/[2\Pi(j=0)]\}}$,
\begin{equation}
\|(\nu P_{\mathrm L}^n)(dx\mid j=0)-\pi_0^\star(dx)\|_{\mathrm{TV}}
=O\!\left(\frac{D_\nu}{\Pi(j=0)}e^{-\gamma_{\mathrm{CS}}n/4}\right).
\label{eq:main-clean-mixing}
\end{equation}
\end{theorem}
\vspace{-5pt}
\begin{remark}[Sufficient conditions for convergence]
\label{rem:main-density-moment-conditions}

For \cref{thm:main-convergence}, exactly paired Gaussian transitions with bounded clean support ensure $\rho_{\mathrm{end}}<1$ and $H_p<\infty$ for some $p>2$ (\cref{lem:c-bounded-gaussian-endpoint-correlation,lem:c-gaussian-higher-moments}). The moment condition also holds under the exponential-tail assumption in \cref{lem:c-gaussian-higher-moments}. Under exact pairing, both linear-interpolant renoising at $t_{\kappa_J}=1$ (\cref{eq:cs-vgb-fm-checkpoint-backward}) and VP renoising with an endpoint-exact schedule (\cref{tab:vp-sde-schedule}) yield an independent noise reset, hence $\rho_{\mathrm{end}}=0$.
\end{remark}

\vspace{-5pt}
\Cref{eq:main-verifier-rate} shows that CS-VGB converges exponentially fast to the correct stationary clean target even with inaccurate process verifiers, although the rate depends on verifier accuracy. As the average-state verifier error $\varepsilon_{\mathrm{av}}$ decreases, the guaranteed rate approaches the exact-verifier bound. We remark that, unlike the uniform multiplicative error bound used in MDM-VGB's mixing speed analysis~\citep{jeon2026mdmvgb}, our bound measures error in $L^2(\pi_j^\star)$, following the average-case perspective also studied for AR-VGB~\citep{rohatgi2026vgb}. 
Importantly, the lower bound rate $\gamma_Vr_{\min}v_{\min}^2$ gives a \textbf{dimension-free} rate guarantee when $1-\rho_{\mathrm{end}},r_{\min},v_{\min}$ remain uniformly bounded away from zero as dimension $d$ increases. 

\vspace{-5pt}
\begin{remark}[Clean stationary mass; \cref{prop:c-sqrt-clean-mass}]
If $J^{-1}\sum_{j=1}^J\bbE_{\pi_j}|\widehat V_j-V_j|\leq\varepsilon Z$, then $\Pi(j=0)\geq(4J-2+4J\varepsilon)^{-1}$. For bounded $\varepsilon$, the clean level therefore retains stationary mass $\Omega(J^{-1})$, paralleling the inverse-sequence-length leaf mass in MDM-VGB~\citep{jeon2026mdmvgb}.
\end{remark}
\vspace{-5pt}
\begin{remark}[Transition mismatch; \cref{thm:c-practical-bias,cor:c-average-rho-lazy-convergence}]
\label{rem:main-transition-mismatch}
Unlike process-verifier error, omitting the edge correction $\rho$ can bias the stationary clean target. Let $\varepsilon_{\rho,\mathrm{av}}:=\max_{1\leq j\leq J}\|\rho_j-1\|_{L^2(J_j^{\mathrm B})}$ quantify this error. Under the assumptions of \cref{thm:c-practical-bias}, let $\Delta_{\rho,\mathrm{av}}=O(\varepsilon_{\rho,\mathrm{av}}^{(p-2)/(p-1)})$ as in \cref{eq:c-average-rho-operator-budget}, for admissible $p>2$. If $\Delta_{\rho,\mathrm{av}}<\gamma_{\mathrm{CS}}$ and $\Delta_{\rho,\mathrm{av}}/(\gamma_{\mathrm{CS}}-\Delta_{\rho,\mathrm{av}})<\Pi(j=0)$, the uncorrected CS-VGB kernel admits a stationary law $\bar\Pi$ with
\begin{equation*}
\|\bar\Pi(dx\mid j=0)-\pi_0^\star\|_{\mathrm{TV}}
\leq\frac{\Delta_{\rho,\mathrm{av}}}{\Pi(j=0)(\gamma_{\mathrm{CS}}-\Delta_{\rho,\mathrm{av}})},
\end{equation*}
The lazified chain still converges exponentially under \cref{cor:c-average-rho-lazy-convergence}, but additional updates alone cannot remove a nonzero stationary bias. Unlike process-verifier error, transition mismatch can therefore affect the limiting clean distribution, with the bias vanishing as $\varepsilon_{\rho,\mathrm{av}}\to0$.
\end{remark}
\vspace{-7pt}
\paragraph{Directional persistence and the first clean visit.}
\label{sec:main-hitting}
While convergence concerns distributional accuracy, finite-budget generation also requires reaching a clean sample within a limited number of updates. To quantify this, let $\tau_0:=\inf\{n\geq0:j_n=0\}$ be the \emph{first-hitting time} for the original kernel. The following bound gives $\Theta(J)$ expected updates to the first clean checkpoint under an average directional-imbalance condition, paralleling the $\Theta(n)$ bound for momentum MDM-VGB under uniform near-balance, where $n$ is sequence length~\citep{jeon2026mdmvgb}. We remark that this dependence on $J$ is optimal, since at least $J$ moves are necessary.
\noindent\begin{minipage}{\linewidth}
\begin{proposition}[First-hitting time; \cref{prop:c-average-momentum-hitting}]
\label{prop:main-hitting}
Assume the exact-kernel conditions of \cref{thm:main-stationarity}, $J\geq2$, $0<v_{\min}\leq\widehat V_j\leq1$, and the noise reset $K_{J-1}^{\mathrm B,\rho}(\cdot\mid y)=\pi_J$ almost surely. Along $Y_J\sim\pi_J$, $Y_{j-1}\sim Q_j^{\mathrm F}(\cdot\mid Y_j)$, suppose
\begin{equation}
\max_{1\leq j<J}\bbE\left[\left(\log\frac{a_j^{\mathrm B}(Y_j)}{a_j^{\mathrm F}(Y_j)}\right)_+\right]
=O(J^{-1}),
\qquad 1-\chi=O(J^{-1}).
\label{eq:main-hitting-assumptions}
\end{equation}
Then $\bbE_{(J,x,\mathrm F)}[\tau_0]=\Theta(J)$ for $\pi_J$-almost every $x$.
\end{proposition}
\end{minipage}
\vspace{-6pt}
\subsection{Combining Sequential and Parallel Search}
\vspace{-2pt}
To use parallel compute alongside recurrent exploration, we propose \textbf{Multi-chain Synchronized CS-VGB} which shares checkpoints and actions across $N$ particles with local candidate selection, preserving the product clean target under the stated assumptions (\cref{app:subsec:synchronized-cs-vgb,thm:sync-finite-pool-stationarity}). For practical one-step posterior-map search, particles share the highest-reward endpoint and its parent, then select the highest-weight local candidate. Detailed updates appear in \cref{eq:search-best-endpoint-sharing,eq:search-local-argmax}. Although this heuristic does not inherit the stationary-target guarantee, we propose it as a practical way to extend CS-VGB's value-guided recurrent bidirectional sampling to finite-budget search.

\vspace{-5pt}
\newpage
\section{Experiments}
\vspace{-5pt}
\label{sec:main-experiments}
We empirically evaluate (1) whether CS-VGB recovers the reward-tilted target despite process-verifier error, (2) whether recurrent, value-guided denoising and renoising improve search performance at comparable computational cost, and (3) how individual design choices contribute to these gains.
\vspace{-5pt}
\subsection{Sampling Accuracy under Imperfect Verifiers}
\label{sec:exp-sampling-fidelity}
\vspace{-5pt}

We first ask whether repeated denoising and renoising recover $\pi_0^\star$ despite process-verifier error. We use a two-dimensional Gaussian mixture with an analytic reward-tilted target and exactly paired transitions. We then vary only the process-verifier error while keeping the reference process and target fixed. To model verifier error that diminishes toward the clean endpoint, motivated by reward estimates from one-step denoising, we set $\log\widehat V_j(x)=\log V_j(x)+\lambda t\,\eta(x,t)-\tfrac12(\lambda t)^2s_\eta^2$ at $t=t_{\kappa_j}$, where $\eta(x,t)$ is a fixed smooth random Fourier error field and $s_\eta^2$ is its variance over random phases (\cref{app:subsec:toy-verifier-error}). In this setup, the scale $\lambda$ controls the distortion strength, and $\widehat V_0=r$ remains exact.

\par
\begin{figure}[H]
\vspace{-5pt}
\centering
\includegraphics[width=\linewidth]{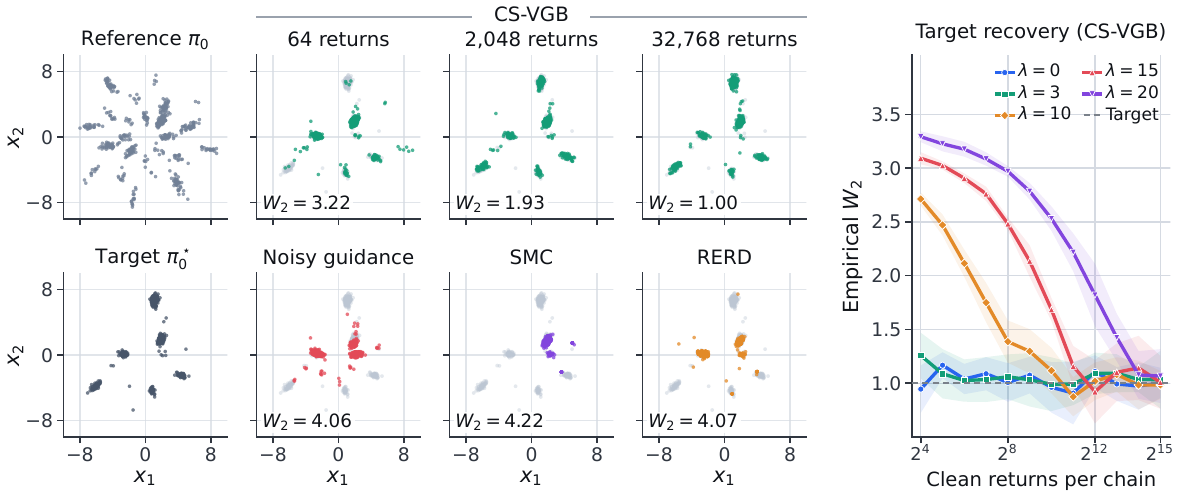}
\captionsetup{skip=5pt}
\caption{\textbf{Target recovery of CS-VGB under process-verifier error.} Left: reference and target samples. Center: CS-VGB at successive clean returns (top) and baseline outputs (bottom), all at $\lambda=20$, each showing 512 samples overlaid on target samples in gray. Right: empirical $W_2$ across 10 random seeds for $\lambda\in\{0,3,10,15,20\}$ (mean $\pm1$ SD). Each return count is per chain. The dashed line averages $W_2$ over 100 independent 512-sample target sets against the fixed target reference.}
\label{fig:main-target-recovery}
\vspace{-10pt}
\end{figure}

\begin{wrapfigure}{r}{0.40\linewidth}
\centering
\includegraphics[width=\linewidth]{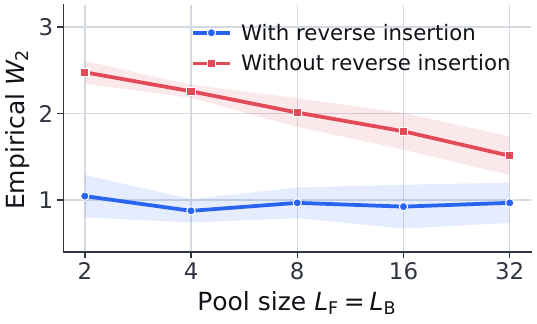}
\captionsetup{skip=4pt}
\caption{\textbf{Small-pool recovery with reverse insertion.} Empirical $W_2$ after $2^{11}$ clean returns per chain at $\lambda=10$.}
\label{fig:main-reverse-insertion}
\vspace{-5pt}
\end{wrapfigure}
\paragraph{Recovery under verifier error.}
CS-VGB recovers comparable target accuracy across verifier-error levels (\cref{fig:main-target-recovery}). At $\lambda=20$, empirical $W_2$ falls from $2.78$ at 512 clean returns to $1.07$ at 32,768, close to $1.00$ with the exact verifier. In contrast, finite-particle SMC and noisy guidance retain distorted mode weights under strong verifier error, as does RERD~\citep{uehara2025refinement} despite repeated renoising and denoising. This comparison highlights the importance of the target-preserving update rule of CS-VGB, rather than renoising and denoising alone. While CS-VGB reaches comparable target accuracy across error levels, larger $\lambda$ requires more clean returns, consistent with \cref{thm:main-convergence}: verifier error can slow convergence without changing the clean target.

\vspace{-0.8\baselineskip}
\paragraph{Why reverse insertion matters.}
This recovery is not obtained simply by repeatedly drawing fresh candidate pools. Removing reverse insertion increases $W_2$ from $0.88$ to $2.26$ at $L_F=L_B=4$, with 2,048 clean returns per chain (\cref{fig:main-reverse-insertion}). With insertion, error remains near one across pool sizes; without it, larger pools reduce the error but leave a gap even at $L_F=L_B=32$. Thus, the ablation supports reverse insertion as a mechanism for accurate finite-pool sampling, rather than relying on large pools to approximate the exact masses. 
\par

\subsection{Compute-Matched Search Performance}
\label{sec:exp-search-performance}
The sampling results illustrate how CS-VGB uses imperfect verifiers to guide exploration while preserving the clean target defined by the exact terminal reward. For practical reward-guided generation, a common goal is to find high-reward samples within a limited computational budget. This motivates a complementary question: \emph{can CS-VGB's recurrent, value-guided exploration also improve finite-budget search?} We evaluate this across image generation with posterior maps and multi-robot planning with diffusion models, using the practical search implementations described in \cref{app:subsec:practical-implementation-details}.

To assess search performance, we compare task-specific reward across computational budgets, measured by the number of function evaluations (NFE). We count each generator forward pass as one NFE and each guidance vector--Jacobian product (VJP) through the generator as two additional NFEs. The suffix \texttt{-G} denotes gradient-guided variants; guided CS-VGB uses a single chain. Further task descriptions, experimental setups, and hyperparameters are provided in \cref{app:sec:experiment-settings}; additional ablations in \cref{app:subsec:search-efficiency-ablations} examine computational savings from lazy pool evaluation and the effect of directional persistence $\chi$ on checkpoint traversal.
\vspace{-5pt}
\begin{figure}[H]
\centering
\includegraphics[width=\linewidth]{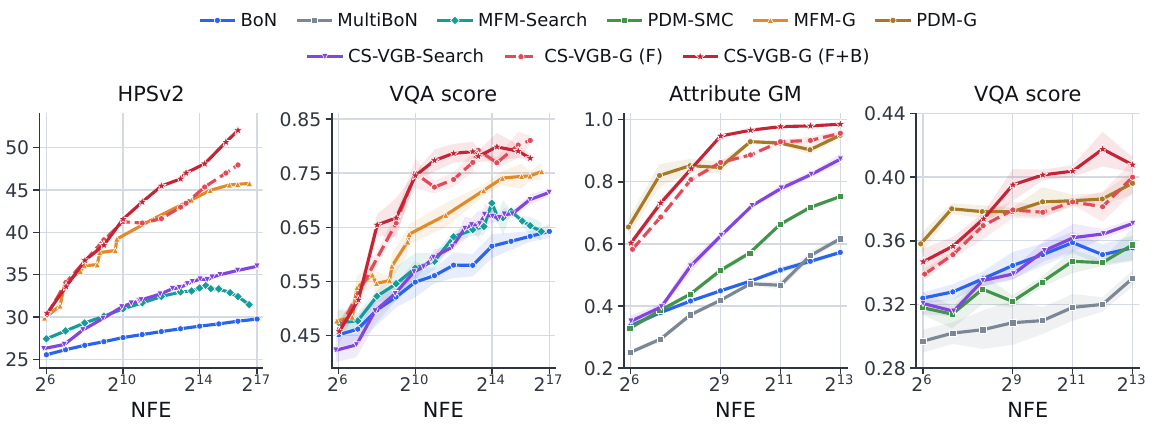}
\par\vspace{2pt}
\makebox[0.5\linewidth][c]{\small (a) ImageNet with MFM}%
\makebox[0.5\linewidth][c]{\small (b) CelebA with PDM}
\par
\caption{\textbf{Reward-guided image generation with posterior maps.} Optimized reward (ImageNet: HPSv2, CelebA: Attribute GM) and independent VQA scores versus actual NFE. Means over three seeds, each with 100 prompts or conditions; bands show $\pm1$ SD where available. F and F+B denote gradient guidance in denoising only and both directions, respectively.}
\label{fig:main-image-search}
\label{fig:main-imagenet-mfm}
\label{fig:main-celeba-pdm}
\vspace{-10pt}
\end{figure}
\vspace{-3pt}
\paragraph{Prompt-guided generation with posterior maps.}
Using Meta Flow Maps (MFM)~\citep{potaptchik2026meta} on ImageNet~\citep{main:deng2009imagenet}, we generate images for 100 class--prompt pairs that we construct by pairing each ImageNet class label with an additional visual condition (\cref{app:subsec:imagenet-settings}). We optimize HPSv2~\citep{exp:wu2023hpsv2} and use VQAScore~\citep{exp:lin2024vqascore} as a held-out metric of prompt alignment. We compare CS-VGB with best-of-$N$ (BoN), MFM-Search, and MFM-G (\cref{fig:main-imagenet-mfm}). CS-VGB-G reaches HPSv2 $45.46$ at 4,104 NFE, exceeding MFM-G's $44.94$ at 23,959 NFE with $5.8\times$ fewer NFEs. It then continues improving beyond MFM-G's plateau, reaching $52.01$ versus $45.66$ at approximately $2^{16}$ NFE, with higher held-out VQAScore ($0.778$ versus $0.745$). Without gradients, CS-VGB-Search also surpasses MFM-Search at larger budgets. Thus, recurrent search both reduces the computation needed for comparable reward and makes productive use of additional computation after one-way guidance saturates. Qualitative comparisons are provided in \cref{tab:imagenet-qualitative}.
\vspace{-3pt}
\paragraph{Facial-attribute control with posterior maps.}
With Posterior Diamond Maps (PDM)~\citep{holderrieth2026diamond} on CelebA $64\times64$~\citep{liu2015faceattributes}, we target 100 conditions specifying ten facial attributes each. We optimize the geometric mean (GM) of the target-attribute probabilities predicted by FaRL~\citep{exp:zheng2022farl} and use VQAScore as a held-out metric of alignment with the specified attributes. Baselines are BoN, MultiBoN~\citep{lee2025cyclic,sabour2025flowmaps}, PDM-SMC, and PDM gradient guidance (\cref{fig:main-celeba-pdm}). CS-VGB-G reaches attribute GM $0.948$ at 514.62 NFE, exceeding PDM-G's $0.929$ at 1,023 NFE while using $1.99\times$ fewer NFEs. Without gradients, CS-VGB-Search reaches GM $0.723$ at 1,066 NFE, compared with PDM-SMC's $0.717$ at 4,068 NFE---a roughly $3.8\times$ reduction in NFE. The gains with and without gradients suggest that recurrent search can reach strongly attribute-aligned outputs more efficiently, rather than benefiting only from gradient guidance. Qualitative comparisons are provided in \cref{tab:celeba-qualitative}.
\begin{wraptable}{r}{0.40\linewidth}
\vspace{-5pt}
\centering\small
\setlength{\tabcolsep}{3pt}
\caption{\textbf{Restart versus recurrence.} ImageNet: HPSv2; CelebA: attribute GM. Three-seed means over 100 conditions; target NFE 4,096.}
\label{tab:main-recurrence}
\begin{tabular}{@{}lrr@{}}
\toprule
Task / variant & Restart & Recurrent\\
\midrule
ImageNet / -G & $39.20$ & $\mathbf{45.57}$\\
ImageNet / Search & $32.23$ & $\mathbf{32.69}$\\
CelebA / -G & $0.86$ & $\mathbf{0.98}$\\
CelebA / Search & $0.80$ & $\mathbf{0.82}$\\
\bottomrule
\vspace{-15pt}
\end{tabular}
\end{wraptable}
\vspace{-15pt}
\paragraph{Why does recurrent bidirectional search help?}
We now examine whether the gains from recurrent search depend on continuing from previously reached clean states, rather than repeatedly restarting from independent noise. Unlike a restart, renoising retains the current clean state as the starting point for exploring alternative continuations, allowing subsequent computation to build on what the search has already found. Consistent with this intuition, recurrence improves both ImageNet HPSv2 and CelebA attribute GM, with larger gains under gradient guidance (\cref{tab:main-recurrence}). This pattern suggests that reusing a promising state is especially useful when guidance can steer the resulting exploration toward higher reward, motivating a closer look at how the backward proposals are generated.
\vspace{-5pt}
\paragraph{Why does backward guidance help?}
The next question is therefore whether guidance should act only during denoising or also shape the noisier states from which denoising resumes. The edge-conductance kernels in \cref{eq:main-weighted-transitions} already favor destinations with larger $\sqrt{\widehat V}$ at interior checkpoints, and the candidate weights in \cref{eq:main-candidate-weights} implement this preference through selection. However, selection can only exploit the candidates available: forward-only guidance (F) retains value-based backward selection but does not steer backward candidate generation. Bidirectional guidance (F+B) addresses this distinction by guiding the backward proposals themselves toward high-value regions. In \cref{fig:main-image-search}, F+B reaches ImageNet HPSv2 $52.01$ versus $47.94$ with F at approximately $2^{16}$ NFE, and improves both attribute reward and VQAScore on CelebA across the evaluated budgets. These results suggest that backward guidance improves recurrent search by steering candidate generation toward high-value regions before value-based selection.

\begin{wrapfigure}{r}{0.40\linewidth}
\vspace{-15pt}
\centering
\includegraphics[width=\linewidth]{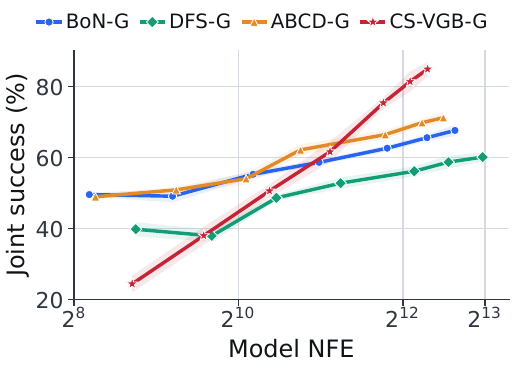}
\captionsetup{skip=4pt}
\caption{\textbf{Multi-robot planning.} Joint success over four tasks with 64 instances each (mean $\pm1$ SD, five seeds) versus actual model NFE.}
\label{fig:main-planning}
\vspace{-8pt}
\end{wrapfigure}
\vspace{-5pt}
\paragraph{Multi-robot planning with diffusion models.}
Using pretrained trajectory diffusion models from MMD~\citep{exp:shaoul2025mmd}, we seek collision-free joint trajectories in four environments: Highways, Circle, and two Conveyor variants. Highways and Circle involve 16 robots, while the Conveyor tasks involve four; instances vary robot start and goal positions and obstacle geometry. We compare against BoN-G and the recurrent diffusion-search baselines ABCD-G~\citep{lee2025cyclic} and DFS-G~\citep{zhang2025classical}. At the largest NFE cap, CS-VGB-G achieves $84.8\%$ joint success, compared with $71.2\%$ for ABCD-G, $67.6\%$ for BoN-G, and $60.1\%$ for DFS-G, using fewer model evaluations on average (\cref{fig:main-planning}). Although CS-VGB-G underperforms ABCD-G and DFS-G at low NFE, it surpasses both at higher budgets. This crossover is consistent with process-verifier evaluation overhead limiting low-budget search, while larger budgets allow value-guided refinement to offset that cost. Environment details and task-wise results are provided in \cref{app:subsec:planning-settings}; \cref{fig:planning-success-trajectories} shows qualitative examples of successful CS-VGB-G trajectories.
\vspace{-15pt}
\section{Conclusion, Limitations, and Future Work}
\vspace{-5pt}
\label{sec:main-conclusion}
We introduced CS-VGB for recurrent, value-guided sampling and search with pretrained generators in continuous domain. Shared edge conductances and terminal anchoring preserve the clean reward-tilted target despite verifier error under stated assumptions, with convergence, first-hitting, and finite-pool guarantees. Experiments demonstrate target recovery and search gains beyond observed denoising-only plateaus in image generation, alongside improved planning success at larger budgets. As limitations, verifier error can slow mixing, evaluation costs can offset low-budget gains, and sampling guarantees do not automatically extend to search heuristics. A natural future direction is to explore alternatives to geometric-mean conductance weighting, including variationally learned conductances, to improve mixing or computational efficiency while retaining terminal anchoring and the conditions required for target preservation.

\newpage

\subsection*{AI use statement}
\label{sec:ai-use-statement}
In this work, we used generative AI tools for editing and polishing the writing, debugging code, and creating TikZ figures. We also used these tools to clarify mathematical background relevant to our derivations and to cross-check proofs. We did not use generative AI tools for developing the core research ideas or drawing scientific conclusions. All AI-assisted text, code, figures, and mathematical arguments were reviewed and verified by the authors. We take responsibility for the final content of this work.

\subsection*{Reproducibility statement}
\label{sec:reproducibility-statement}
We provide the implementation details needed to reproduce our results in the main paper and appendix, including pretrained model information, inference configurations, evaluation protocols, and hyperparameters. Additional details on datasets, preprocessing, and experimental settings are provided in \cref{app:sec:experiment-settings}. The theoretical framework, assumptions, derivations, and complete proofs are presented in \crefrange{app:sec:full-cs-vgb-framework}{app:subsec:synchronized-cs-vgb}. Source code and scripts for reproducing the main experiments are available at \url{https://github.com/KraitGit/CS_VGB}.
\par

\subsubsection*{Acknowledgments}
KJ and MT are grateful for partial supports by NSF Grant DMS-2513699, DOE Grant SC0026274, and Richard Duke Fellowship. MT also thanks Simons Institute of the Theory of Computing. The authors also thank Thuy-Duong `June' Vuong for inspiration.
\par
\ifshowbibliography
    \bibliographystyle{iclr2027_conference}
    \bibliography{references}
\fi

\clearpage
\appendix
\numberwithin{equation}{section}
\numberwithin{figure}{section}
\numberwithin{table}{section}
\setcounter{algocf}{0}
\setcounter{algorithm}{0}
\renewcommand{\algorithmhrefprefix}{appendix}
\renewcommand{\theHequation}{appendix.\thesection.\arabic{equation}}
\renewcommand{\theHfigure}{appendix.\thesection.\arabic{figure}}
\renewcommand{\theHtable}{appendix.\thesection.\arabic{table}}
\raggedbottom
\mdfsetup{nobreak=true}
\AtBeginEnvironment{proof}{%
  \setlength{\abovedisplayskip}{6pt plus1pt minus3pt}%
  \setlength{\belowdisplayskip}{6pt plus1pt minus3pt}%
}
\definecolor{red}{HTML}{8B3038}
\setlength{\textfloatsep}{12pt plus 2pt minus 2pt}
\setlength{\floatsep}{10pt plus 2pt minus 2pt}
\setlength{\intextsep}{10pt plus 2pt minus 2pt}
\renewcommand{\topfraction}{.9}
\renewcommand{\bottomfraction}{.8}
\renewcommand{\textfraction}{.1}
\renewcommand{\floatpagefraction}{.8}
\section{Notation and Reference Transitions}
\label{app:sec:key-notation-baselines}

\begin{table}[!htbp]
\caption{Table of Key Notations}
\label{tab:cs-vgb-key-notation}
\label{tab:cs-vgb-key-notation-chain}
\centering
\small
\setlength{\aboverulesep}{3pt}
\setlength{\belowrulesep}{3pt}
\setlength{\extrarowheight}{0pt}
\setlength{\tabcolsep}{4pt}
\begin{tabular}{@{}c@{\hspace{12pt}}l@{\hspace{12pt}}l@{}}
\toprule
\textbf{Symbol} 
    & \textbf{Definition}
    & \textbf{Mathematical description}\\
\midrule

$t_k$ 
    & Time grid
    & $ 0=t_0<\cdots<t_T=1$\\
\addlinespace[4pt]

${\kappa_j}$
    & Checkpoint indices
    & ${0=\kappa_0<\cdots<\kappa_J\leq T,\quad J\geq2}$\\
\addlinespace[4pt]

${ \pi_j}$ 
    & Reference marginal
    & ${ \pi_j=\pi_{t_{\kappa_j}},\quad \pi_{j-1}=\pi_jK_j^{\mathrm F}}$\\
\addlinespace[4pt]

${ \pi_j^\star}$ 
    & Reward-tilted marginal
    & ${ \pi_j^\star(dx)=V_j(x)\pi_j(dx)/Z}$\\
\addlinespace[4pt]

$r$ 
    & Terminal reward
    & $r:\mathcal X\to[0,\infty)$\\
\addlinespace[4pt]

${ V_j}$ 
    & Ideal process verifier
    & ${ V_j(x):=\mathbb E_{X_0\sim(K_j^{\mathrm F}\cdots K_1^{\mathrm F})(\cdot\mid x)}[r(X_0)]}$\\
\addlinespace[4pt]

${ \widehat V_j}$ 
    & Process verifier
    & ${ \widehat V_j:=\widehat V_{t_{\kappa_j}},\quad\widehat V_0=r}$\\
\addlinespace[4pt]

$ Z,\mathcal Z$ 
    & Target and chain normalizers
    & ${ \left\{\begin{aligned}[c]
        Z&=\int r\,d\pi_0, \\
        \mathcal Z&=\mu_0(\mathcal X)+2\sum_{j=1}^{J-1}\mu_j(\mathcal X)+\mu_J(\mathcal X)
\end{aligned}\right.}$\\
\midrule
${k_k^{\mathrm F},k_{k-1}^{\mathrm B}}$
    & Fine-grid reference kernels
    & ${t_k\to t_{k-1},\quad t_{k-1}\to t_k}$\\
\addlinespace[4pt]
${ K_j^{\mathrm F}}$ 
    & Checkpoint denoising kernel
    & ${ K_j^{\mathrm F}(dy\mid x),\quad j\to j-1}$\\
\addlinespace[4pt]

${ K_{j-1}^{\mathrm B}}$ 
    & Checkpoint renoising kernel
    & ${ K_{j-1}^{\mathrm B}(dx\mid y),\quad j-1\to j}$\\
\addlinespace[4pt]

${ G_j^{\mathrm F},G_j^{\mathrm B}}$ 
    & Candidate proposal kernels
    & ${ Y_i\sim G_j^{\mathrm F}(\cdot\mid x),\quad Z_i\sim G_j^{\mathrm B}(\cdot\mid x)}$\\
\addlinespace[4pt]

${ J_j^{\mathrm F},J_j^{\mathrm B}}$ 
    & Reference edge laws
    & ${ \left\{\begin{aligned}[c]
        J_j^{\mathrm F}(dx,dy)&=\pi_j(dx)K_j^{\mathrm F}(dy\mid x),\\
J_j^{\mathrm B}(dx,dy)&=\pi_{j-1}(dy)K_{j-1}^{\mathrm B}(dx\mid y)
\end{aligned}\right.}$\\
\addlinespace[4pt]

${ \rho_j}$ 
    & Edge correction ratio
    & ${ \rho_j=dJ_j^{\mathrm F}/dJ_j^{\mathrm B}}$\\
\addlinespace[4pt]

${ K_{j-1}^{\mathrm B,\rho}}$ 
    & Corrected renoising kernel
    & ${ \rho_j(x,y)K_{j-1}^{\mathrm B}(dx\mid y)}$\\
\addlinespace[4pt]

${ \widehat C_j}$ 
    & \begin{tabular}[c]{@{}l@{}}Implemented\\edge-conductance measure\end{tabular}
    & ${ \left\{\begin{aligned}[c]
        \widehat C_j(dx,dy)&=J_j^{\mathrm F}(dx,dy){\sqrt{\widehat V_j(x)}}{\sqrt{\widehat V_{j-1}(y)}}, \\
        \widehat C_1(dx,dy)&=J_1^{\mathrm F}(dx,dy)r(y)
\end{aligned}\right.}$\\
\addlinespace[4pt]

${ C_j}$ 
    & \begin{tabular}[c]{@{}l@{}}Comparison\\edge-conductance measure\end{tabular}
    & ${ C_j(dx,dy)=J_j^{\mathrm F}(dx,dy)V_{j-1}(y)}$\\
\midrule
${ F_j,B_j^\rho}$ 
    & Directional weight integrals
    & ${ F_j=K_j^{\mathrm F}{\sqrt{\widehat V_{j-1}}},\quad B_j^\rho=K_j^{\mathrm B,\rho}{\sqrt{\widehat V_{j+1}}}}$\\
\addlinespace[4pt]

${ Q_j^{\mathrm F}}$ 
    & Verifier-weighted denoising
    & ${ K_j^{\mathrm F}(dy\mid x){\sqrt{\widehat V_{j-1}(y)}}/F_j(x)}$\\
\addlinespace[4pt]

${ Q_j^{\mathrm B,\rho}}$ 
    & Verifier-weighted renoising
    & ${ K_j^{\mathrm B,\rho}(dy\mid x){\sqrt{\widehat V_{j+1}(y)}}/B_j^\rho(x)}$\\
\addlinespace[4pt]

${ a_j^{\mathrm F},a_j^{\mathrm B}}$ 
    & Directional incident masses
    & ${ \left\{\begin{aligned}[c]
        a_j^{\mathrm F}&={\sqrt{\widehat V_j}}F_j\ (j\geq2),\quad a_1^{\mathrm F}=F_1, \\
        a_j^{\mathrm B}&={\sqrt{\widehat V_j}}B_j^\rho\ (1\leq j<J)
\end{aligned}\right.}$\\
\addlinespace[4pt]

${ \mu_j}$ 
    & Total incident measure
    & ${ \left\{\begin{aligned}[c]
        \mu_j&=\widehat C_j(\cdot,\mathcal X)+\widehat C_{j+1}(\mathcal X,\cdot), \\
        \mu_0&=r\pi_0,\quad \mu_J=\widehat C_J(\cdot,\mathcal X)
\end{aligned}\right.}$\\
\addlinespace[4pt]

$ \Pi$ 
    & Stationary law of CS-VGB
    & ${ \Pi(j,dx,d)=\mu_j(dx)/\mathcal Z}$\\
\addlinespace[4pt]

$P$ 
    & Exact CS-VGB kernel
    & $\Pi P=\Pi$\\
\addlinespace[4pt]

$\chi$ 
    & Directional persistence parameter
    & $\chi\in[0,1]$\\
\addlinespace[4pt]

$ L_F,L_B$ 
    & Candidate pool sizes
    & $ L_F,L_B\in\mathbb N$\\
\bottomrule
\end{tabular}
\par\vspace{5pt}
\begin{minipage}{\linewidth}
The square-root formulas for $\widehat C_j,F_j,Q_j^{\mathrm F},a_j^{\mathrm F}$ apply to $2\leq j\leq J$, and those for $B_j^\rho,Q_j^{\mathrm B,\rho},a_j^{\mathrm B}$ to $1\leq j<J$.
At the clean edge, $F_1=K_1^{\mathrm F}r$, $Q_1^{\mathrm F}(dy\mid x)=K_1^{\mathrm F}(dy\mid x)r(y)/F_1(x)$, and $Q_0^{\mathrm B,\rho}=K_0^{\mathrm B,\rho}$; see \cref{eq:cs-vgb-new-forward-kernel,eq:cs-vgb-clean-renoising,eq:cs-vgb-directional-masses}.
\end{minipage}
\end{table}

\begin{figure}[!ht]
\centering
\definecolor{roadInk}{HTML}{222B33}
\definecolor{roadBlue}{HTML}{365F83}
\definecolor{roadGreen}{HTML}{42745E}
\definecolor{roadPurple}{HTML}{76506F}
\definecolor{roadFill}{HTML}{F2F2F2}

\begin{tikzpicture}[font=\normalsize, >=Latex]

\node[anchor=north west, rounded corners=4pt, draw=black, fill=roadFill, line width=1.05pt, minimum width=211pt, minimum height=1.35cm, inner sep=0pt]
    (targetPanel) at (2.40,0) {};
\node[anchor=north west, rounded corners=4pt, draw=black, fill=roadFill, line width=1.05pt, minimum width=174pt, minimum height=1.75cm, inner sep=0pt]
    (ratePanel) at (0,-1.85) {};
\node[anchor=north west, rounded corners=4pt, draw=black, fill=roadFill, line width=1.05pt, minimum width=184pt, minimum height=1.75cm, inner sep=0pt]
    (verifierPanel) at (6.95,-1.85) {};
\node[anchor=north west, rounded corners=4pt, draw=black, fill=roadFill, line width=1.05pt, minimum width=175pt, minimum height=1.65cm, inner sep=0pt]
    (correctionPanel) at (0,-4.10) {};
\node[anchor=north west, rounded corners=4pt, draw=black, fill=roadFill, line width=1.05pt, minimum width=184pt, minimum height=1.65cm, inner sep=0pt]
    (poolPanel) at (6.95,-4.10) {};

\draw[-{Latex[length=2mm,width=1.6mm]}, draw=black, line width=1pt]
    (3.175,-1.35) -- (3.175,-1.85);
\draw[-{Latex[length=2mm,width=1.6mm]}, black, line width=1pt]
    (ratePanel.east) -- (verifierPanel.west);
\draw[-{Latex[length=2mm,width=1.6mm]}, draw=black, line width=1pt]
    (3.175,-3.60) -- (3.175,-4.10);
\draw[-{Latex[length=2mm,width=1.6mm]}, black, rounded corners=4pt, line width=1pt]
    (targetPanel.east) -- (13.9,-0.69)
    -- (13.9,-4.94) -- (poolPanel.east);

\node[anchor=north west, text=roadBlue, align=left,
      font=\normalsize\bfseries, inner sep=0pt]
    (targetQuestion) at (2.65,-0.22)
    {Q: What distribution does CS-VGB preserve?};
\node[anchor=north west, text=roadInk, align=left, inner sep=0pt]
    (targetAnswer) at (2.65,-0.78)
    {(\ref{thm:c-exact-stationarity}): the reward-tilted clean target.};

\node[anchor=north west, text=roadBlue, align=left,
      font=\normalsize\bfseries, inner sep=0pt]
    (rateQuestion) at (0.25,-2.07)
    {Q: How fast does CS-VGB converge?};
\node[anchor=north west, text=roadInk, align=left,
      text width=5.85cm, inner sep=0pt]
    (rateAnswer) at (0.25,-2.55)
    {(\ref{thm:c-average-moment-poincare}, \ref{cor:c-cs-vgb-convergence}): exponential convergence.\\[3pt]
     (\ref{prop:c-explicit-lifted-poincare}): quantitative rate lower bound.};

\node[anchor=north west, text=roadBlue, align=left,
      font=\normalsize\bfseries, inner sep=0pt]
    (verifierQuestion) at (7.20,-2.07)
    {Q: What does verifier error change?};
\node[anchor=north west, text=roadInk, align=left,
      text width=5.85cm, inner sep=0pt]
    (verifierAnswer) at (7.20,-2.55)
    {(\ref{thm:c-average-verifier-gap}): correct target preserved;
     potentially weaker mixing-rate guarantee.};

\node[anchor=north west, text=roadBlue, align=left,
      font=\normalsize\bfseries, inner sep=0pt]
    (correctionQuestion) at (0.25,-4.32)
    {Q: What if we omit edge correction?};
\node[anchor=north west, text=roadInk, align=left,
      text width=5.85cm, inner sep=0pt]
    (correctionAnswer) at (0.25,-4.88)
    {(\ref{thm:c-practical-bias}): a stationary law exists, but may be biased.};

\node[anchor=north west, text=roadBlue, align=left,
      font=\normalsize\bfseries, inner sep=0pt]
    (poolQuestion) at (7.20,-4.32)
    {Q: Can finite $L_f,L_b$ preserve the target?};
\node[anchor=north west, text=roadInk, align=left,
      text width=5.85cm, inner sep=0pt]
    (poolAnswer) at (7.20,-4.88)
    {(\ref{thm:c-joint-pool-invariance}): reverse insertion preserves the full stationary law.};

\end{tikzpicture}
\caption{Roadmap of the theoretical analysis of CS-VGB.}
\label{fig:appendix-proof-roadmap}
\end{figure}
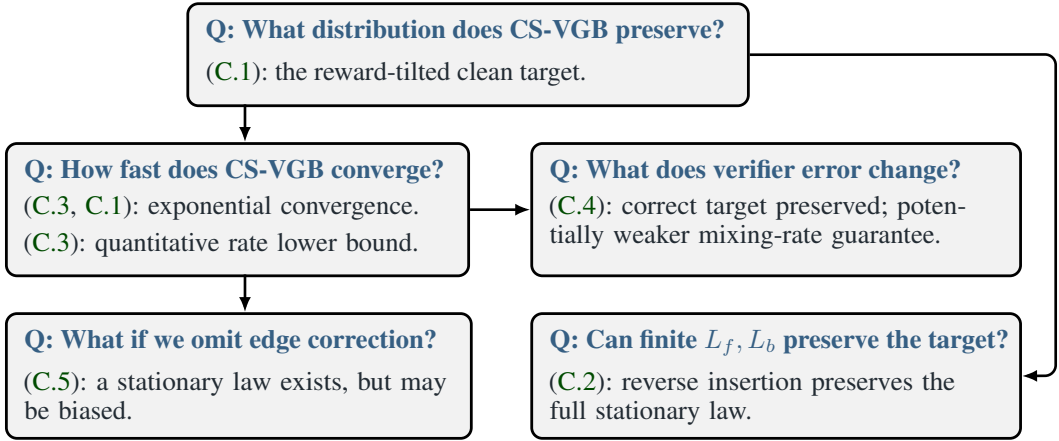

\subsection{Reference Transitions across Continuous Generative Models}
\label{app:subsec:reference-model-classes}
In this subsection, we review the reference processes underlying the checkpoint kernels in \cref{app:subsec:model-checkpoint-transitions}. We frist fix the same time convention for Diffusion models, flow matching, and one-step posterior map: $t=0$ is the data endpoint, $X_0\sim\pi_{\mathrm{data}}$, and $t=1$ is the noise endpoint, $X_1\sim\mathcal N(0,I)$. Generation runs from $1$ to $0$, consistent with the time convention in \cref{app:subsec:cs-vgb-setup}.

\paragraph{Diffusion models.}
Score-based diffusion~\citep{song2021score} transports clean data toward noise through the forward SDE
\begin{equation*}
dX_t=f(X_t,t)dt+g(t)dW_t,
\qquad X_0\sim\pi_{\mathrm{data}},
\end{equation*}
where $W_t$ is standard Brownian motion, $f(x,t)$ is the drift, and $g(t)$ is a scalar diffusion coefficient depending only on time. We write $\pi_t^{\mathrm{ref}}$ for the marginal law of $X_t$ and $\pi_t^{\mathrm{ref}}(x)$ for its density when it exists. These ideal marginals are distinct from the checkpoint marginals $\pi_j$ induced by the implemented reference kernels in \cref{app:subsec:cs-vgb-setup}. The corresponding reverse-time process is given by~\citep{anderson1982reverse}
\begin{equation*}
dX_t=\left[f(X_t,t)-g(t)^2\nabla_x\log\pi_t^{\mathrm{ref}}(X_t)\right]dt+g(t)d\overline W_t,
\end{equation*}
where $\overline W_t$ is Brownian motion in decreasing time. For the VP (variance-preserving) SDE~\citep{song2021score}, the noising marginal can be sampled as
\begin{equation*}
X_t=\alpha_tX_0+\sigma_t\epsilon,
\qquad \epsilon\sim\mathcal N(0,I),
\end{equation*}
where $\alpha_t$ and $\sigma_t$ are the signal and noise scales. \Cref{tab:vp-sde-schedule} gives the VP coefficients and a schedule with an exact standard Gaussian prior at $t=1$.
\begin{center}
\captionof{table}{VP SDE form and an endpoint-exact noise schedule.}
\label{tab:vp-sde-schedule}
\small
\setlength{\tabcolsep}{5pt}
\renewcommand{\arraystretch}{1.3}
\begin{tabular}{@{}p{0.50\linewidth}p{0.40\linewidth}@{}}
\toprule
\textbf{General VP form} & \textbf{Endpoint-exact schedule} \\
\midrule
\makecell[l]{$f(x,t)=-\frac{1}{2}\beta(t)x$, \quad $g(t)=\sqrt{\beta(t)}$\\[4pt]$\alpha_t=\exp\left(-\frac{1}{2}\int_0^t\beta(u)du\right)$\\[4pt]$\sigma_t^2=1-\alpha_t^2$}
& \makecell[l]{$\beta(t)=\frac{2}{1-t}$\\[4pt]$\alpha_t=1-t$\\[4pt]$\sigma_t^2=1-(1-t)^2$} \\
\bottomrule
\end{tabular}
\end{center}
Here $\beta(t)>0$ is the noise rate. An exact Gaussian prior requires $\alpha_1=0$ and hence $\int_0^1\beta(t)dt=\infty$; otherwise some clean signal remains. For the schedule in the table, $\pi_t^{\mathrm{ref}}\to\mathcal N(0,I)$ as $t\uparrow1$, so the prior at $t=1$ is defined by this limit.

A score network $s_\theta(x,t)$ approximates $\nabla_x\log\pi_t^{\mathrm{ref}}(x)$ through denoising score matching~\citep{song2021score}. Substituting it into the reverse process gives
\begin{equation}
dX_t=\left[f(X_t,t)-g(t)^2s_\theta(X_t,t)\right]dt+g(t)d\overline W_t,
\label{eq:score-reverse-sde}
\end{equation}
which is integrated from $t=1$ to $t=0$ using an SDE integrator such as Euler--Maruyama.

For $0\leq t<s\leq1$, let $q_{s\mid t}$ denote the forward-SDE transition from the cleaner state $X_t$ to the noisier state $X_s$. For the VP SDE in \cref{tab:vp-sde-schedule},
\begin{equation*}
q_{s\mid t}(dx_s\mid x_t)=\mathcal N(dx_s;\alpha_{s\mid t}x_t,\sigma_{s\mid t}^2I),
\qquad \alpha_{s\mid t}=\frac{\alpha_s}{\alpha_t},
\qquad \sigma_{s\mid t}^2=\sigma_s^2-\alpha_{s\mid t}^2\sigma_t^2.
\end{equation*}
Its \emph{Bayes reverse} $q_{t\mid s}$ is the conditional law of the cleaner state given the noisier state and is determined by
\begin{equation*}
\pi_t^{\mathrm{ref}}(dx_t)q_{s\mid t}(dx_s\mid x_t)
=\pi_s^{\mathrm{ref}}(dx_s)q_{t\mid s}(dx_t\mid x_s).
\end{equation*}
With an exact score and exact integration of \cref{eq:score-reverse-sde}, starting from $X_s=x_s$ and running to $t<s$ gives $X_t\sim q_{t\mid s}(\cdot\mid x_s)$. Score approximation and numerical discretization generally change this conditional distribution.

\paragraph{Flow matching.}
Flow matching~\citep{lipman2023flowmatching} learns an instantaneous velocity field for a probability path. For the linear interpolant,
\begin{equation*}
X_t=(1-t)X_0+t\epsilon,
\qquad X_0\sim\pi_{\mathrm{data}},\quad \epsilon\sim\mathcal N(0,I).
\end{equation*}
with independent $X_0$ and $\epsilon$, the sample-conditional velocity is $\epsilon-X_0$. Averaging over the endpoints given $X_t=x$ gives the marginal velocity
\begin{equation*}
v_t(x):=\bbE[\partial_tX_t\mid X_t=x]=\bbE[\epsilon-X_0\mid X_t=x],
\end{equation*}
A velocity network $v_\theta(t,x)$ approximates this conditional expectation. With the exact velocity, generation follows the ODE
\begin{equation}
dX_t=v_t(X_t)dt,
\label{eq:flow-matching-ode}
\end{equation}
from $t=1$ to $t=0$. Its transitions are deterministic, but the same marginals $\pi_t^{\mathrm{ref}}$ can also be realized by the following forward and reverse SDEs with diffusion rate $\beta(t)>0$~\citep{albergo2023stochastic}:
\begin{align}
dX_t&=\left[v_t(X_t)+\frac{1}{2}\beta(t)\nabla_x\log\pi_t^{\mathrm{ref}}(X_t)\right]dt+\sqrt{\beta(t)}dW_t,
\tag{\text{FM-Noising SDE}}
\label{eq:flow-matching-forward-sde}\\
dX_t&=\left[v_t(X_t)-\frac{1}{2}\beta(t)\nabla_x\log\pi_t^{\mathrm{ref}}(X_t)\right]dt+\sqrt{\beta(t)}d\overline W_t.
\tag{\text{FM-Denoising SDE}}
\label{eq:flow-matching-reverse-sde}
\end{align}
The first runs in increasing time and the second in decreasing time, with $\overline W_t$ denoting backward-time Brownian motion.

Differentiating the Gaussian conditional density of $X_t\mid X_0\sim\mathcal N((1-t)X_0,t^2I)$ and averaging over $X_0\mid X_t=x$ expresses the score through the conditional mean of $\epsilon$. Using $\epsilon=X_t+(1-t)(\epsilon-X_0)$ then gives, for $t>0$,
{
\begin{equation*}
\nabla_x\log\pi_t^{\mathrm{ref}}(x)=-\frac1t\bbE[\epsilon\mid X_t=x]
=-\frac{x+(1-t)v_t(x)}{t}.
\end{equation*}
}
Substituting this relation into \cref{eq:flow-matching-reverse-sde} gives the velocity-only form
\begin{equation*}
dX_t=\left[\frac{\beta(t)}{2t}X_t+\left(1+\frac{(1-t)\beta(t)}{2t}\right)v_t(X_t)\right]dt+\sqrt{\beta(t)}d\overline W_t.
\end{equation*}
Replacing $v_t$ by $v_\theta$ gives a practical reverse SDE without a separate score model.

\paragraph{One-step posterior sampling.}
{
For the linear interpolant above, a one-step posterior sampler transforms an independent Gaussian input into a clean sample conditional on $X_t=x$. In the ideal case,
\begin{equation}
Z={\Phi_{t,0}^{\theta}}(x,\zeta),\qquad \zeta\sim\mathcal N(0,I),\qquad
\operatorname{Law}(Z\mid x)=\operatorname{Law}(X_0\mid X_t=x).
\label{eq:background-pdm-sampling}
\end{equation}
Meta Flow Maps (MFM) and Posterior Diamond Maps (PDM) learn such posterior samplers~\citep{potaptchik2026meta,holderrieth2026diamond}. Under perfect training at the population level, these methods recover the identity in \cref{eq:background-pdm-sampling}. Under this identity, independent Gaussian inputs give independent posterior draws, so for integrable reward $r$,
\begin{equation}
\frac1N\sum_{i=1}^N r(Z_i)
\xrightarrow[N\to\infty]{\mathrm{a.s.}}
\bbE[r(X_0)\mid X_t=x],\qquad
Z_i\overset{\mathrm{iid}}{\sim}\operatorname{Law}(X_0\mid X_t=x).
\label{eq:background-pdm-reward}
\end{equation}
Thus posterior samples estimate the expected reward rather than the reward of a single deterministic prediction. We combine such a sample with a Gaussian bridge to construct denoising transitions in \cref{eq:cs-vgb-flowmap-one-nfe-step,eq:cs-vgb-flowmap-fine-forward}.
}

\section{Full CS-VGB Framework}
\label{app:sec:full-cs-vgb-framework}

{Many inference-time alignment methods use intermediate reward estimates to guide generation from noise to data, then stop at a clean sample. A common approach evaluates the reward at a clean prediction, using a denoiser or a flow map~\citep{li2025svdd,sabour2025flowmaps}. We call this intermediate reward estimate a \emph{process verifier}. At checkpoint $j$, its ideal value is}
\begin{equation*}
V_j(x):=\bbE_{X_0\sim(K_j^{\mathrm F}\cdots K_1^{\mathrm F})(\cdot\mid x)}[r(X_0)],
\end{equation*}
where $K_j^{\mathrm F}\cdots K_1^{\mathrm F}$ composes the reference denoising transitions from checkpoint $j$ to $0$. The reward at a single clean prediction need not equal this expectation, so errors can affect subsequent choices and the final distribution. CS-VGB instead repeatedly renoises and denoises the current sample, using additional inference compute to revisit earlier choices rather than only generate independent samples.

{Our sampling framework aims to preserve the \emph{reward-tilted terminal distribution}, rather than merely maximize reward. The process verifier guides exploration, while the exact reward determines the target at the clean boundary. Under the conditions of \cref{thm:c-exact-stationarity}, the edge-corrected construction preserves the reward-tilted terminal distribution even with an imperfect process verifier: its accuracy affects transitions and mixing, but not this target distribution. We define the target in \cref{app:subsec:cs-vgb-setup} and the repeated denoising and renoising updates in \cref{app:subsec:coarse-checkpoint-cs-vgb}.}

{
\subsection{Problem Setup and Target Distribution}
\label{app:subsec:cs-vgb-setup}

\paragraph{Reference process and target.}
Let $\mathcal X\subseteq\mathbb R^d$ be the sample space.
We distinguish the \emph{fine denoising time grid} from the \emph{checkpoints} at which CS-VGB chooses an action:
\begin{equation}
0=t_0<\cdots<t_T=1,\qquad
0=\kappa_0<\cdots<\kappa_J\leq T,\qquad 2\leq J\leq T.
\label{eq:cs-vgb-two-grids}
\end{equation}
Here $k$ indexes fine-grid steps and $j$ indexes checkpoints. In the CS-VGB construction and analysis, write $X_j:=X_{t_{\kappa_j}}$. From checkpoint $j$, the kernels $K_j^{\mathrm F}$ and $K_j^{\mathrm B}$ move to checkpoints $j-1$ and $j+1$, respectively. When these moves follow the fine grid, they compose the corresponding steps as
\begin{equation}
K_j^{\mathrm F}
:=k_{\kappa_j}^{\mathrm F}k_{\kappa_j-1}^{\mathrm F}\cdots
k_{\kappa_{j-1}+1}^{\mathrm F},\qquad
K_j^{\mathrm B}:=k_{\kappa_j}^{\mathrm B}k_{\kappa_j+1}^{\mathrm B}\cdots k_{\kappa_{j+1}-1}^{\mathrm B}.
\label{eq:cs-vgb-checkpoint-composition}
\end{equation}
Here $k_k^{\mathrm F}$ is the denoising kernel from $t_k$ to $t_{k-1}$, and $k_k^{\mathrm B}$ is the renoising kernel from $t_k$ to $t_{k+1}$. Their model-specific forms are given in \cref{app:subsec:model-checkpoint-transitions}. The F composition applies for $1\leq j\leq J$ and the B composition for $0\leq j<J$.
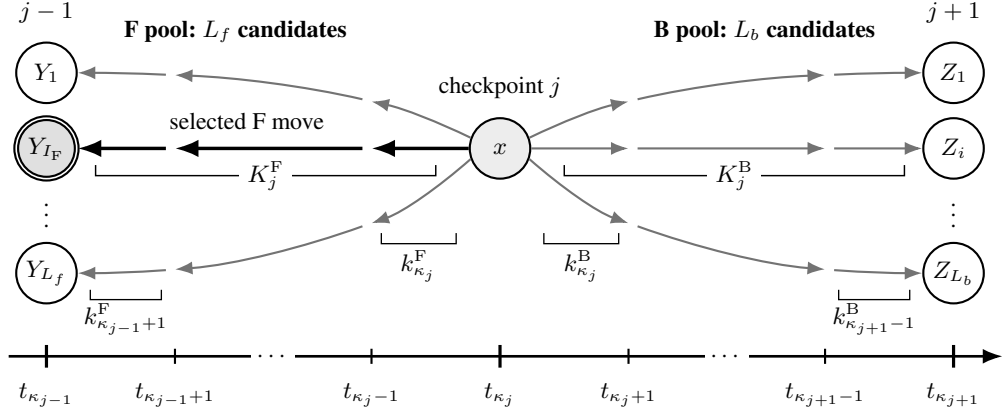
\begin{figure}[tbp]
\centering
\begin{tikzpicture}[font=\small,>=Latex,line width=.8pt]
\node (forwardTitle) at (2.5,3.45) {\textbf{F pool: $L_f$ candidates}};
\node (backwardTitle) at (9.5,3.45) {\textbf{B pool: $L_b$ candidates}};
\node[circle,draw=black,fill=black!7,minimum size=8mm,inner sep=1pt] (currentSample) at (6,1.9) {$x$};
\node[circle,draw=black,fill=white,minimum size=8mm,inner sep=1pt] (forwardFirst) at (0,2.9) {$Y_1$};
\node[circle,draw=black,fill=black!12,double,minimum size=8mm,inner sep=1pt] (forwardSelected) at (0,1.9) {$Y_{I_{\mathrm F}}$};
\node[circle,draw=black,fill=white,minimum size=8mm,inner sep=1pt] (forwardLast) at (0,.25) {$Y_{L_f}$};
\node (forwardDots) at (0,1.12) {$\vdots$};
\node[circle,draw=black,fill=white,minimum size=8mm,inner sep=1pt] (backwardFirst) at (12,2.9) {$Z_1$};
\node[circle,draw=black,fill=white,minimum size=8mm,inner sep=1pt] (backwardMiddle) at (12,1.9) {$Z_i$};
\node[circle,draw=black,fill=white,minimum size=8mm,inner sep=1pt] (backwardLast) at (12,.25) {$Z_{L_b}$};
\node (backwardDots) at (12,1.12) {$\vdots$};
\draw[->,draw=black!55] (currentSample.160) .. controls (5.18,2.25) and (4.72,2.43) .. (4.3,2.52);
\draw[->,draw=black!55] (4.18,2.55) .. controls (3.4,2.73) and (2.48,2.84) .. (1.7,2.87);
\draw[->,draw=black!55] (1.58,2.875) .. controls (1.16,2.895) and (.8,2.9) .. (forwardFirst.east);
\draw[->,draw=black,line width=1.3pt] (currentSample.west) -- (4.3,1.9);
\draw[->,draw=black,line width=1.3pt] (4.18,1.9) -- (1.7,1.9);
\draw[->,draw=black,line width=1.3pt] (1.58,1.9) -- (forwardSelected.east);
\draw[->,draw=black!55] (currentSample.200) .. controls (5.18,1.38) and (4.72,1.06) .. (4.3,.91);
\draw[->,draw=black!55] (4.18,.87) .. controls (3.4,.59) and (2.48,.36) .. (1.7,.3);
\draw[->,draw=black!55] (1.58,.29) .. controls (1.16,.26) and (.8,.25) .. (forwardLast.east);
\draw[->,draw=black!55] (currentSample.20) .. controls (6.82,2.25) and (7.28,2.43) .. (7.7,2.52);
\draw[->,draw=black!55] (7.82,2.55) .. controls (8.6,2.73) and (9.52,2.84) .. (10.3,2.87);
\draw[->,draw=black!55] (10.42,2.875) .. controls (10.84,2.895) and (11.2,2.9) .. (backwardFirst.west);
\draw[->,draw=black!55] (currentSample.east) -- (7.7,1.9);
\draw[->,draw=black!55] (7.82,1.9) -- (10.3,1.9);
\draw[->,draw=black!55] (10.42,1.9) -- (backwardMiddle.west);
\draw[->,draw=black!55] (currentSample.340) .. controls (6.82,1.38) and (7.28,1.06) .. (7.7,.91);
\draw[->,draw=black!55] (7.82,.87) .. controls (8.6,.59) and (9.52,.36) .. (10.3,.3);
\draw[->,draw=black!55] (10.42,.29) .. controls (10.84,.26) and (11.2,.25) .. (backwardLast.west);
\node (selectedLabel) at (2.65,2.26) {selected F move};
\draw[draw=black,line width=.55pt] (.65,1.68) -- (.65,1.54) -- (5.15,1.54) -- (5.15,1.68);
\node[fill=white,inner sep=3pt] (forwardKernelLabel) at (2.9,1.54) {$K_j^{\mathrm F}$};
\draw[draw=black,line width=.55pt] (6.85,1.68) -- (6.85,1.54) -- (11.35,1.54) -- (11.35,1.68);
\node[fill=white,inner sep=3pt] (backwardKernelLabel) at (9.1,1.54) {$K_j^{\mathrm B}$};
\draw[draw=black,line width=.5pt] (4.43,.78) -- (4.43,.64) -- (5.42,.64) -- (5.42,.78);
\node[inner sep=1pt] (forwardFirstStepLabel) at (4.92,.35) {$k_{\kappa_j}^{\mathrm F}$};
\draw[draw=black,line width=.5pt] (.58,.03) -- (.58,-.1) -- (1.52,-.1) -- (1.52,.03);
\node[inner sep=1pt] (forwardLastStepLabel) at (1.05,-.35) {$k_{\kappa_{j-1}+1}^{\mathrm F}$};
\draw[draw=black,line width=.5pt] (6.58,.78) -- (6.58,.64) -- (7.57,.64) -- (7.57,.78);
\node[inner sep=1pt] (backwardFirstStepLabel) at (7.08,.35) {$k_{\kappa_j}^{\mathrm B}$};
\draw[draw=black,line width=.5pt] (10.48,.03) -- (10.48,-.1) -- (11.42,-.1) -- (11.42,.03);
\node[inner sep=1pt] (backwardLastStepLabel) at (10.95,-.35) {$k_{\kappa_{j+1}-1}^{\mathrm B}$};
\node (currentLabel) at (6,2.7) {checkpoint $j$};
\node (forwardCheckpoint) at (0,3.7) {$j-1$};
\node (backwardCheckpoint) at (12,3.7) {$j+1$};
\begin{scope}[yshift=-6mm]
\draw[->,draw=black,line width=1pt] (-.5,-.25) -- (12.65,-.25);
\draw[draw=black,line width=1.2pt] (0,-.4) -- (0,-.1);
\draw[draw=black] (1.7,-.35) -- (1.7,-.15);
\draw[draw=black] (4.3,-.35) -- (4.3,-.15);
\draw[draw=black,line width=1.2pt] (6,-.4) -- (6,-.1);
\draw[draw=black] (7.7,-.35) -- (7.7,-.15);
\draw[draw=black] (10.3,-.35) -- (10.3,-.15);
\draw[draw=black,line width=1.2pt] (12,-.4) -- (12,-.1);
\node[fill=white,inner sep=2pt] (leftGridDots) at (3,-.25) {$\cdots$};
\node[fill=white,inner sep=2pt] (rightGridDots) at (9,-.25) {$\cdots$};
\node[anchor=north] (leftTime) at (0,-.48) {$t_{\kappa_{j-1}}$};
\node[anchor=north] (leftFineFirst) at (1.7,-.48) {$t_{\kappa_{j-1}+1}$};
\node[anchor=north] (leftFineLast) at (4.3,-.48) {$t_{\kappa_j-1}$};
\node[anchor=north] (currentTime) at (6,-.48) {$t_{\kappa_j}$};
\node[anchor=north] (rightFineFirst) at (7.7,-.48) {$t_{\kappa_j+1}$};
\node[anchor=north] (rightFineLast) at (10.3,-.48) {$t_{\kappa_{j+1}-1}$};
\node[anchor=north] (rightTime) at (12,-.48) {$t_{\kappa_{j+1}}$};
\end{scope}
\end{tikzpicture}
\caption{Fine steps and checkpoint moves. Candidate paths connect neighboring checkpoints; the highlighted F move selects one of $L_F$ candidates. The B pool contains $L_B$ candidates in the opposite direction.}
\label{fig:cs-vgb-checkpoint-candidates}
\end{figure}
The composition integrates out all intermediate states, with no direction decision or candidate selection at the intermediate fine-grid points. Starting from the Gaussian prior $\pi_{t_T}:=\mathcal N(0,I)$ at $t_T=1$, denoising to $t_{\kappa_J}$ gives the initial law $\pi_J$. Continuing through the checkpoints then defines the remaining marginals, so for $j=J,\ldots,1$,
\begin{equation}
\pi_J:=\pi_{t_T}k_T^{\mathrm F}\cdots k_{\kappa_J+1}^{\mathrm F},\qquad
\pi_{j-1}(dy):=\int_{\mathcal X}\pi_j(dx)K_j^{\mathrm F}(dy\mid x).
\label{eq:cs-vgb-implemented-marginals}
\end{equation}
The empty kernel product is the identity, so $\pi_J=\mathcal N(0,I)$ when $\kappa_J=T$. For a nonnegative reward $r$, the reward-tilted terminal distribution is
\begin{equation*}
\pi_0^\star(dx):=\frac{r(x)\pi_0(dx)}{Z},\qquad
{Z:=\int_{\mathcal X}r(x)\pi_0(dx)\in(0,\infty).}
\end{equation*}

\paragraph{Process verifier.}
The \emph{ideal process verifier} is the expected reward after denoising to $t_0=0$:
\begin{equation*}
V_j(x):=V_{t_{\kappa_j}}(x)
=\bbE_{X_0\sim(K_j^{\mathrm F}\cdots K_1^{\mathrm F})(\cdot\mid x)}[r(X_0)],
\qquad V_0=r,\qquad V_j=K_j^{\mathrm F}V_{j-1}.
\end{equation*}
The same kernels define the \emph{reward-tilted time marginals},
\begin{equation*}
\pi_j^\star(dx):=\frac{V_j(x)\pi_j(dx)}{Z},\qquad
\int V_j\,d\pi_j=\int V_{j-1}\,d\pi_{j-1}=\cdots=\int r\,d\pi_0=Z.
\end{equation*}
CS-VGB uses a nonnegative, finite approximation ${\widehat V_j:=\widehat V_{t_{\kappa_j}}}$, obtained from a denoised prediction or a learned verifier. At the clean boundary it uses the exact reward:
\begin{assumption}[Terminal anchoring]
\label{ass:terminal-anchoring}
At ${j=0}$, CS-VGB evaluates $r$ exactly and sets $\widehat V_0:=r$.
\end{assumption}
When using a learned verifier, one way to train it is to regress on terminal rewards along reference denoising paths:
\begin{equation*}
\mathcal L_{\mathrm{PV}}(\phi)
:=\bbE\!\left[(\widehat V_\phi(t_{\kappa_j},X_j)-r(X_0))^2\right].
\end{equation*}
Here $j$ is sampled over the supported checkpoints, and $(X_J,\ldots,X_0)$ follows $K_J^{\mathrm F},\ldots,K_1^{\mathrm F}$ from $\pi_J$. In the ideal training, its population minimizer satisfies ${\widehat V_\phi(t_{\kappa_j},x)=\bbE[r(X_0)\mid X_j=x]=V_j(x)}$ for ${\pi_j}$-almost every $x$. The construction below does not require this approximation to be exact.

\subsection{CS-VGB at Selected Checkpoints}
\label{app:subsec:coarse-checkpoint-cs-vgb}

The state ${(j,x,d)}$ records the checkpoint index, sample, and direction. Direction $\mathrm F$ denoises to ${j-1}$, and direction $\mathrm B$ renoises to ${j+1}$. At ${j=0}$ only $\mathrm B$ is available, and at ${j=J}$ only $\mathrm F$ is available. We first define the exact update, then its candidate-pool implementation.

\paragraph{Corrected renoising.}
Let ${K_{j-1}^{\mathrm B}(dx\mid y)}$ be the reference renoising kernel from ${j-1}$ to ${j}$. The reference kernels ${K_j^{\mathrm F}}$ and ${K_{j-1}^{\mathrm B}}$ need not be exact reverses, meaning that the joint edge laws in \cref{eq:cs-vgb-edge-laws} need not satisfy $J_j^{\mathrm F}=J_j^{\mathrm B}$. These laws are defined by
\begin{equation}
J_j^{\mathrm F}(dx,dy):=\pi_j(dx)K_j^{\mathrm F}(dy\mid x),
\qquad
J_j^{\mathrm B}(dx,dy):=\pi_{j-1}(dy)K_{j-1}^{\mathrm B}(dx\mid y).
\label{eq:cs-vgb-edge-laws}
\end{equation}
Assume ${J_j^{\mathrm F}\ll J_j^{\mathrm B}}$. The \emph{edge ratio} ${\rho_j:=dJ_j^{\mathrm F}/dJ_j^{\mathrm B}}$ corrects their mismatch:
\begin{equation}
K_{j-1}^{\mathrm B,\rho}(dx\mid y)
:=\rho_j(x,y)K_{j-1}^{\mathrm B}(dx\mid y).
\label{eq:cs-vgb-corrected-renoising}
\end{equation}
We remark that $K_{j-1}^{\mathrm B,\rho}$ has total mass one for $\pi_{j-1}$-almost every $y$. Specifically, for every measurable $A\subseteq\mathcal X$,
\begin{align}
\int_A\pi_{j-1}(dy)\int_{\mathcal X}K_{j-1}^{\mathrm B,\rho}(dx\mid y)
&=\int_{\mathcal X\times A}\rho_j\,dJ_j^{\mathrm B}
=J_j^{\mathrm F}(\mathcal X,A)=\pi_{j-1}(A).
\label{eq:cs-vgb-corrected-kernel-normalization}
\end{align}
\Cref{eq:cs-vgb-corrected-kernel-normalization} verifies this normalization. Moreover, the correction in \cref{eq:cs-vgb-corrected-renoising} makes the two joint edge laws equal, giving
\begin{equation}
\pi_j(dx)K_j^{\mathrm F}(dy\mid x)
=\pi_{j-1}(dy)K_{j-1}^{\mathrm B,\rho}(dx\mid y).
\label{eq:cs-vgb-corrected-edge-identity}
\end{equation}

\begin{remark}[Exact pairing across model classes]
\label{rem:exact-pairing-model-classes}
In this remark, we clarify when $\rho_j=1$ holds for diffusion models, flow matching, and one-step posterior maps.
For VP diffusion, the noising process and learned denoising process are
\begin{align*}
\text{Noising:}\quad dX_t&=-\tfrac12\beta(t)X_t\,dt+\sqrt{\beta(t)}\,dW_t,\\
\text{Denoising:}\quad dX_t&=\bigl[-\tfrac12\beta(t)X_t-\beta(t)s_\theta(X_t,t)\bigr]dt+\sqrt{\beta(t)}\,d\overline W_t.
\end{align*}
For linear flow matching, converting the ODE to the SDEs in \cref{eq:flow-matching-forward-sde,eq:flow-matching-reverse-sde} with $\beta(t)=2t/(1-t)$ gives
\begin{align*}
\text{Noising:}\quad dX_t&=-\frac{X_t}{1-t}\,dt+\sqrt{\frac{2t}{1-t}}\,dW_t,\\
\text{Denoising:}\quad dX_t&=\left[\frac{X_t}{1-t}+2v_\theta(t,X_t)\right]dt+\sqrt{\frac{2t}{1-t}}\,d\overline W_t.
\end{align*}
In both cases, noising runs in increasing time and denoising in decreasing time. Under ideal training, $s_\theta(x,t)=\nabla_x\log\pi_t^{\mathrm{ref}}(x)$ or $v_\theta(t,x)=v_t(x)$, respectively, so exact integration of the denoising and renoising SDEs yields the paired-transition identity below for $\pi_j=\pi_{t_{\kappa_j}}^{\mathrm{ref}}$:
\begin{equation*}
\pi_j(dx)K_j^{\mathrm F}(dy\mid x)
=\pi_{j-1}(dy)K_{j-1}^{\mathrm B}(dx\mid y),\qquad \rho_j=1.
\end{equation*}
For a one-step posterior map, such as MFM or PDM~\citep{potaptchik2026meta,holderrieth2026diamond}, an exact clean posterior draw combined with the Gaussian bridge in \cref{prop:cs-vgb-posterior-bridge-pairing} yields the conditional law $\operatorname{Law}(X_{j-1}\mid X_j=x)$ of the chosen noising process and hence gives $\rho_j=1$. The implemented kernels are given in \cref{app:subsec:model-checkpoint-transitions}.
\end{remark}

\paragraph{Edge-conductance measures and sampling transitions.}
Starting from the common edge law in \cref{eq:cs-vgb-corrected-edge-identity}, we weight adjacent state pairs by the process verifiers on interior edges and only the terminal reward on the clean edge. This defines an \emph{edge-conductance measure} $\widehat C_j(dx,dy)$, extending the edge weights of a weighted-network random walk to continuous states~\citep{doyle1984random,bezuglyi2018markov}. As illustrated in \cref{fig:cs-vgb-joint-conductance}, normalizing this measure over $y$ with $x$ fixed gives the denoising transition, while normalizing over $x$ with $y$ fixed gives the renoising transition, and the corresponding total weights define the directional masses.
{Thus, for $j=2,\ldots,J$,}
\begin{align}
\widehat C_j(dx,dy)
&:=\pi_j(dx)K_j^{\mathrm F}(dy\mid x){\sqrt{\widehat V_j(x)\widehat V_{j-1}(y)}}=\pi_{j-1}(dy)K_{j-1}^{\mathrm B,\rho}(dx\mid y){\sqrt{\widehat V_{j-1}(y)\widehat V_j(x)}}.
\label{eq:cs-vgb-interior-conductance}
\end{align}
We remark that this geometric-mean weighting is related to square-root balancing in locally balanced MCMC~\citep{zanella2020informed} and square-root approximations of transition rates in Markov state models~\citep{lie2013square}.

On the clean edge, we use only the terminal reward:
\begin{equation}
\widehat C_1(dx,dy)
:=\pi_1(dx)K_1^{\mathrm F}(dy\mid x)r(y)
=\pi_0(dy)K_0^{\mathrm B,\rho}(dx\mid y)r(y).
\label{eq:cs-vgb-clean-conductance}
\end{equation}
In particular, $\widehat C_1(\mathcal X,dy)=r(y)\pi_0(dy)=Z\pi_0^\star(dy)$. This choice fixes the target tilted distribution at the clean boundary without requiring ${\widehat V_j=V_j}$ at intermediate grid points.

To obtain a sampling transition, hold the current endpoint fixed and normalize the weights over the other endpoint. For {${2\leq j\leq J}$}, this gives the denoising transition
\begin{equation}
F_j(x):=\int_{\mathcal X}K_j^{\mathrm F}(dy\mid x){\sqrt{\widehat V_{j-1}(y)}},
\qquad
Q_j^{\mathrm F}(dy\mid x)
:=\frac{K_j^{\mathrm F}(dy\mid x){\sqrt{\widehat V_{j-1}(y)}}}{F_j(x)}.
\label{eq:cs-vgb-new-forward-kernel}
\end{equation}
{On the clean edge, the corresponding formulas use the reward itself:
\begin{equation*}
F_1(x):=\int_{\mathcal X}K_1^{\mathrm F}(dy\mid x)r(y),\qquad
Q_1^{\mathrm F}(dy\mid x):=\frac{K_1^{\mathrm F}(dy\mid x)r(y)}{F_1(x)}.
\end{equation*}}
On the same edge, the corrected renoising transition goes from $y$ at checkpoint index ${j-1}$ to $x$ at checkpoint index ${j}$. For ${2\leq j\leq J}$, define
\begin{equation}
B_{j-1}^\rho(y):=\int_{\mathcal X}K_{j-1}^{\mathrm B,\rho}(dx\mid y){\sqrt{\widehat V_j(x)}},
\qquad
Q_{j-1}^{\mathrm B,\rho}(dx\mid y)
:=\frac{K_{j-1}^{\mathrm B,\rho}(dx\mid y){\sqrt{\widehat V_j(x)}}}{B_{j-1}^\rho(y)}.
\label{eq:cs-vgb-new-backward-kernel}
\end{equation}
Here ${F_j(x)}$ and ${B_{j-1}^\rho(y)}$ are the corresponding total masses; we require them to be positive and finite wherever an update is performed. The clean-edge definition instead gives
\begin{equation}
Q_0^{\mathrm B,\rho}(dx\mid y):=K_0^{\mathrm B,\rho}(dx\mid y).
\label{eq:cs-vgb-clean-renoising}
\end{equation}
Thus the first renoising move from a clean sample has no process-verifier term, but follows the same construction of conditioning the edge measure on the current endpoint. \Cref{fig:cs-vgb-joint-conductance} illustrates how this construction gives the denoising and renoising transitions on an interior edge.
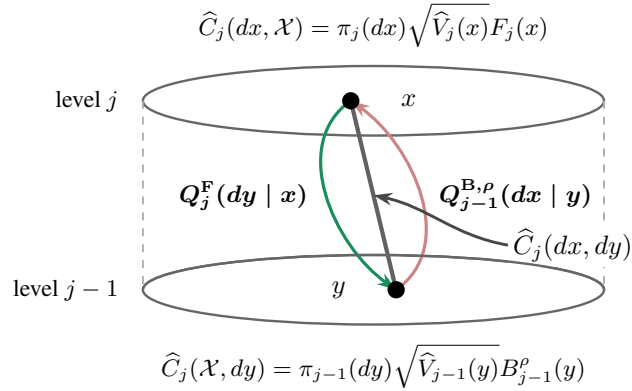
\begin{figure}[h]
\centering
\definecolor{conductanceRenoise}{HTML}{CC8585}
\begin{tikzpicture}[text=black,
    x=1cm,
    y=1cm,
    font=\sffamily\small,
    layer/.style={draw=black!60,line width=0.9pt},
    wall/.style={draw=black!32,dashed,line width=0.7pt},
    point/.style={circle,fill=black,inner sep=2.3pt},
    shared/.style={draw=black!62,line width=1.6pt},
    forward/.style={-{Stealth[length=2.5mm,width=1.7mm]},
        draw=ForestGreen!82!black,line width=1.2pt},
    backward/.style={-{Stealth[length=2.5mm,width=1.7mm]},
        draw=conductanceRenoise,line width=1.2pt},
    callout/.style={-{Stealth[length=2.5mm,width=1.7mm]},
        draw=black!72,line width=1.1pt},
    edge-label/.style={fill=white,inner sep=1.7pt,font=\sffamily\small},
    mass/.style={font=\sffamily\small,align=center},
    level-label/.style={font=\normalfont\small},
    conductance-label/.style={font=\sffamily\normalsize}
]
    \coordinate (upper) at (-0.30,1.25);
    \coordinate (lower) at (0.30,-1.25);

    \draw[layer] (0,1.25) ellipse (3.05 and 0.46);
    \draw[layer] (0,-1.25) ellipse (3.05 and 0.46);
    \draw[wall] (-3.05,1.25) -- (-3.05,-1.25);
    \draw[wall] (3.05,1.25) -- (3.05,-1.25);

    \node[anchor=east,level-label]
        at (-3.25,1.25) {level $j$};
    \node[anchor=east,level-label]
        at (-3.25,-1.25) {level $j-1$};

    \node[mass,anchor=south] at (0,1.84)
        {$\widehat C_j(dx,\mathcal X)=\pi_j(dx)\sqrt{\widehat V_j(x)}F_j(x)$};
    \node[mass,anchor=north] at (0,-1.84)
        {$\widehat C_j(\mathcal X,dy)=\pi_{j-1}(dy)\sqrt{\widehat V_{j-1}(y)}B_{j-1}^{\rho}(y)$};

    \draw[shared] (upper) -- (lower);
    \draw[forward] (upper) to[out=-145,in=145] (lower);
    \draw[backward] (lower) to[out=35,in=-35] (upper);

    \node[point,label={[font=\sffamily, label distance=12pt]1.7:$x$}] at (upper) {};
    \node[point,label={[font=\sffamily, label distance=12pt]180.1:$y$}] at (lower) {};

    \node[edge-label,text=black] at (-1.72,0)
        {$\boldsymbol{Q_j^{\mathrm F}(dy\mid x)}$};
    \node[edge-label,text=black] at (1.88,0)
        {$\boldsymbol{Q_{j-1}^{\mathrm B,\rho}(dx\mid y)}$};
    \node[edge-label,conductance-label,anchor=west] (conductance) at (1.78,-0.66)
        {$\widehat C_j(dx,dy)$};
    \draw[callout] (conductance.west) to[out=170,in=-12]
        ($(upper)!0.50!(lower)$);
    \draw[layer] (-3.05,-1.25)
        arc[start angle=180,end angle=0,x radius=3.05cm,y radius=0.46cm];
\end{tikzpicture}
\caption{
Reward-tilted edge-conductance measure between adjacent checkpoints {(${j\geq2}$)}.
The joint measure ${\widehat C_j(dx,dy)}$ assigns mass to state pairs
$ (x,y)$ at levels ${ j}$ and ${ j-1}$.
Its marginals ${\widehat C_j(dx,\mathcal X)}$ and ${\widehat C_j(\mathcal X,dy)}$
give the incident masses at $x$ and $y$, while conditioning the
same joint measure in the two directions yields the transition kernels
${ Q_j^{\mathrm F}(dy\mid x)}$ and
${ Q_{j-1}^{\mathrm B,\rho}(dx\mid y)}$.
}
\label{fig:cs-vgb-joint-conductance}
\end{figure}

\paragraph{Direction probabilities and directional persistence.}
At an interior checkpoint ${1\leq j<J}$, integrating the adjacent edge-conductance measures gives
\begin{equation*}
\widehat C_j(dx,\mathcal X)=\pi_j(dx)a_j^{\mathrm F}(x),\qquad
\widehat C_{j+1}(\mathcal X,dx)=\pi_j(dx)a_j^{\mathrm B}(x),
\end{equation*}
where
\begin{equation}
a_j^{\mathrm F}(x)
:=\begin{cases}
F_1(x),&j=1,\\
{\sqrt{\widehat V_j(x)}}F_j(x),&j\geq2,
\end{cases}
\qquad
a_j^{\mathrm B}(x):={\sqrt{\widehat V_j(x)}}B_j^\rho(x).
\label{eq:cs-vgb-directional-masses}
\end{equation}
The kernels $Q_j^{\mathrm F}$ and $Q_j^{\mathrm B,\rho}$ specify where to move once a direction is chosen. To choose between them, we use the relative masses of the two adjacent edges, assigning more probability to the direction with larger mass. This gives the probabilities below; their smaller value, ${\ell_j(x)}$, will be used to control direction changes:
\begin{equation}
p_j^{\mathrm F}(x):=\frac{a_j^{\mathrm F}(x)}{a_j^{\mathrm F}(x)+a_j^{\mathrm B}(x)},
\qquad
p_j^{\mathrm B}(x):=\frac{a_j^{\mathrm B}(x)}{a_j^{\mathrm F}(x)+a_j^{\mathrm B}(x)},
\qquad
\ell_j(x):=\min\{p_j^{\mathrm F}(x),p_j^{\mathrm B}(x)\}.
\label{eq:cs-vgb-direction-probabilities}
\end{equation}
These probabilities reflect the local preference for denoising or renoising, but do not account for the previous direction. Choosing a new direction from them at every update can therefore cause frequent reversals, especially when the two masses are comparable. Such reversals repeatedly undo changes in depth and can lead to diffusive progress across checkpoints. Since a practical run has a finite compute budget, the time to reach the clean boundary ${j=0}$ matters alongside convergence to stationarity. For example, using the same momentum structure adopted below, MDM-VGB reduces the expected first-leaf hitting time from $\Theta(n^2)$ to $O(n)$ when the forward/backward probability imbalance is $O(1/n)$ and $1-\chi=O(1/n)$, where $n$ is the number of positions to reveal in masked diffusion models~\citep[Theorem~4.5]{jeon2026mdmvgb}.
{ In our framework, \Cref{prop:c-average-momentum-hitting} bounds the first hitting time under an average directional-imbalance condition and an exact reset at the top checkpoint.}

{Motivated by this behavior, CS-VGB retains the current direction $d$ and uses \emph{crossing-flow cancellation}~\citep{hayes2010liftings}, as in VGB-Momentum~\citep{rohatgi2026vgb,jeon2026mdmvgb}. The update in \cref{tab:cs-vgb-momentum-update} removes the common probability ${\chi\ell_j(x)}$ from both opposing flips and adds it to the respective stays, leaving the move probabilities unchanged. Increasing the \emph{directional persistence parameter} $\chi\in[0,1]$ strengthens this cancellation and suppresses more direction changes, while preserving the stationary law established in \cref{thm:c-exact-stationarity}.}
\begin{table}[H]
\centering
\caption{Momentum update at an interior checkpoint.}
\label{tab:cs-vgb-momentum-update}
{
\small
\setlength{\tabcolsep}{4.5pt}
\renewcommand{\arraystretch}{1.18}
\begin{tabular}{@{}ccll@{}}
\toprule
Current direction & Action & Probability & Next state \\
\midrule
$ \mathrm F$
& Move
& ${ p_j^{\mathrm F}(x)}$
& ${ (j-1,Y,\mathrm F)}$, ${ Y\sim Q_j^{\mathrm F}(\cdot\mid x)}$ \\
$ \mathrm F$
& Flip
& ${ p_j^{\mathrm B}(x)-\chi\ell_j(x)}$
& ${ (j,x,\mathrm B)}$ \\
$ \mathrm F$
& Stay
& ${ \chi\ell_j(x)}$
& ${ (j,x,\mathrm F)}$ \\
\midrule
$ \mathrm B$
& Move
& ${ p_j^{\mathrm B}(x)}$
& ${ (j+1,Y,\mathrm B)}$, ${ Y\sim Q_j^{\mathrm B,\rho}(\cdot\mid x)}$ \\
$ \mathrm B$
& Flip
& ${ p_j^{\mathrm F}(x)-\chi\ell_j(x)}$
& ${ (j,x,\mathrm F)}$ \\
$ \mathrm B$
& Stay
& ${ \chi\ell_j(x)}$
& ${ (j,x,\mathrm B)}$ \\
\bottomrule
\end{tabular}
}
\end{table}
A move changes the checkpoint index and sample, a flip changes only the direction, and a stay changes neither. The default $\chi=1$ gives
\begin{equation}
\bbP(\mathrm F\to\mathrm B)
=\frac{(a_j^{\mathrm B}-a_j^{\mathrm F})_+}{a_j^{\mathrm F}+a_j^{\mathrm B}},
\qquad
\bbP(\mathrm B\to\mathrm F)
=\frac{(a_j^{\mathrm F}-a_j^{\mathrm B})_+}{a_j^{\mathrm F}+a_j^{\mathrm B}}.
\label{eq:cs-vgb-maximal-persistence}
\end{equation}
At the boundaries, exit with probability one using $Q_0^{\mathrm B,\rho}$ at ${j=0}$ and ${Q_J^{\mathrm F}}$ at ${j=J}$. On arrival at ${j=0}$, set $d=\mathrm B$; on arrival at ${j=J}$, set $d=\mathrm F$. These rules define the \emph{exact one-step kernel} $P$ when the masses and conditional draws are evaluated exactly. Under the conditions of \cref{thm:c-exact-stationarity}, $P$ preserves a stationary law $\Pi$ whose distribution at the clean boundary is $\pi_0^\star$.

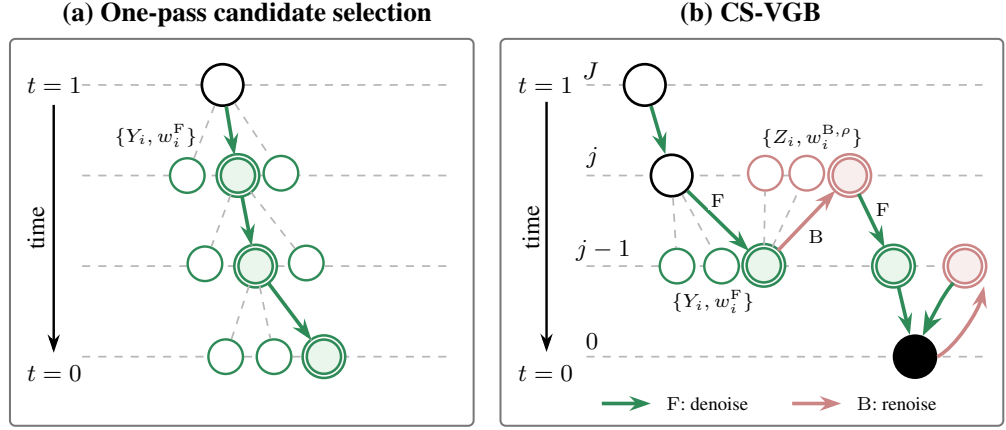
\begin{figure}[!tbp]
\centering
\resizebox{\linewidth}{!}{\definecolor{figTwoGreenFill}{HTML}{EAF6EC}
\definecolor{figTwoGreenBorder}{HTML}{2E8B57}
\definecolor{figTwoRedFill}{HTML}{F8EEEE}
\definecolor{figTwoRedBorder}{HTML}{CC8585}
\definecolor{figTwoGuide}{HTML}{B8B8B8}
\definecolor{figTwoPanel}{HTML}{777777}
\definecolor{figTwoTextGray}{HTML}{666666}
\definecolor{figTwoBrick}{HTML}{8B3F45}

\begin{tikzpicture}[text=black,
    x=1cm,
    y=1cm,
    font=\normalfont\small,
    stateCircle/.style={circle,minimum size=5.4mm,inner sep=0pt},
    candidateCircle/.style={circle,minimum size=4.5mm,inner sep=0pt},
    chosenCircle/.style={circle,double,minimum size=5.0mm,inner sep=0pt},
    arrowLine/.style={-{Stealth[length=2.8mm,width=1.9mm]},line width=1.35pt}
]

\draw[draw=figTwoPanel,line width=0.75pt,rounded corners=2pt]
    (-6.49,-0.42) rectangle (-0.45,4.62);
\draw[draw=figTwoPanel,line width=0.75pt,rounded corners=2pt]
    (-0.1,-0.42) rectangle (6.45,4.62);

\node[font=\normalfont\bfseries] (onePassTitle) at (-3.40,4.94)
    {(a) One-pass candidate selection};
\node[font=\normalfont\bfseries] (persistentTitle) at (3.16,4.94)
    {(b) CS-VGB};

\node[anchor=east] (onePassNoiseTime) at (-5.43,4.02) {$t=1$};
\node[anchor=east] (onePassCleanTime) at (-5.43,0.26) {$t=0$};
\draw[-{Stealth[length=2.4mm,width=1.6mm]},draw=black,line width=0.9pt]
    (-5.92,3.76) -- (-5.92,0.53);
\node[rotate=90,font=\normalfont\footnotesize] (onePassDenoiseLabel) at (-6.15,2.16)
    {time};

\node[anchor=east] (persistentNoiseTime) at (0.97,4.02) {$t=1$};
\node[anchor=east] (persistentCleanTime) at (0.97,0.26) {$t=0$};
\draw[-{Stealth[length=2.4mm,width=1.6mm]}, black, line width=0.9pt]
    (0.5,3.8) -- (0.5,0.54);
\node[rotate=90,font=\normalfont\footnotesize] (persistentDenoiseLabel) at (0.3,2.17)
    {time};

\draw[draw=figTwoGuide,line width=0.65pt,dashed] (-5.55,4.02) -- (-0.72,4.02);
\draw[draw=figTwoGuide,line width=0.65pt,dashed] (-5.55,2.84) -- (-0.72,2.84);
\draw[draw=figTwoGuide,line width=0.65pt,dashed] (-5.55,1.66) -- (-0.72,1.66);
\draw[draw=figTwoGuide,line width=0.65pt,dashed] (-5.55,0.48) -- (-0.72,0.48);

\node[stateCircle,draw=black,line width=1.0pt,fill=white] (onePassNoise) at (-3.72,4.02) {};
\node[candidateCircle,draw=figTwoGreenBorder,line width=0.9pt,fill=white] (onePassCandidateTwoLeft) at (-4.18,2.84) {};
\node[chosenCircle,draw=figTwoGreenBorder,line width=0.95pt,fill=figTwoGreenFill] (onePassChosenTwo) at (-3.52,2.84) {};
\node[candidateCircle,draw=figTwoGreenBorder,line width=0.9pt,fill=white] (onePassCandidateTwoRight) at (-2.96,2.87) {};

\node[candidateCircle,draw=figTwoGreenBorder,line width=0.9pt,fill=white] (onePassCandidateOneLeft) at (-3.95,1.69) {};
\node[chosenCircle,draw=figTwoGreenBorder,line width=0.95pt,fill=figTwoGreenFill] (onePassChosenOne) at (-3.29,1.66) {};
\node[candidateCircle,draw=figTwoGreenBorder,line width=0.9pt,fill=white] (onePassCandidateOneRight) at (-2.63,1.70) {};

\node[candidateCircle,draw=figTwoGreenBorder,line width=0.9pt,fill=white] (onePassCandidateCleanLeft) at (-3.68,0.48) {};
\node[candidateCircle,draw=figTwoGreenBorder,line width=0.9pt,fill=white] (onePassCandidateCleanMiddle) at (-3.05,0.48) {};
\node[chosenCircle,draw=figTwoGreenBorder,line width=0.95pt,fill=figTwoGreenFill] (onePassChosenClean) at (-2.4,0.48) {};

\draw[draw=figTwoGuide,line width=0.65pt,dashed] (onePassNoise) -- (onePassCandidateTwoLeft);
\draw[draw=figTwoGuide,line width=0.65pt,dashed] (onePassNoise) -- (onePassChosenTwo);
\draw[draw=figTwoGuide,line width=0.65pt,dashed] (onePassNoise) -- (onePassCandidateTwoRight);
\draw[draw=figTwoGuide,line width=0.65pt,dashed] (onePassChosenTwo) -- (onePassCandidateOneLeft);
\draw[draw=figTwoGuide,line width=0.65pt,dashed] (onePassChosenTwo) -- (onePassChosenOne);
\draw[draw=figTwoGuide,line width=0.65pt,dashed] (onePassChosenTwo) -- (onePassCandidateOneRight);
\draw[draw=figTwoGuide,line width=0.65pt,dashed] (onePassChosenOne) -- (onePassCandidateCleanLeft);
\draw[draw=figTwoGuide,line width=0.65pt,dashed] (onePassChosenOne) -- (onePassCandidateCleanMiddle);
\draw[draw=figTwoGuide,line width=0.65pt,dashed] (onePassChosenOne) -- (onePassChosenClean);

\draw[arrowLine,draw=figTwoGreenBorder] (onePassNoise) -- (onePassChosenTwo);
\draw[arrowLine,draw=figTwoGreenBorder] (onePassChosenTwo) -- (onePassChosenOne);
\draw[arrowLine,draw=figTwoGreenBorder] (onePassChosenOne) -- (onePassChosenClean);
\node[text=black,font=\scriptsize,anchor=west] (onePassPoolLabel) at (-5.25,3.36)
    {$\{Y_i,w_i^{\mathrm F}\}$};

\draw[draw=figTwoGuide,line width=0.65pt,dashed] (1.02,4.02) -- (6.18,4.02);
\draw[draw=figTwoGuide,line width=0.65pt,dashed] (1.02,2.84) -- (6.18,2.84);
\draw[draw=figTwoGuide,line width=0.65pt,dashed] (1.02,1.66) -- (6.18,1.66);
\draw[draw=figTwoGuide,line width=0.65pt,dashed] (1.02,0.48) -- (6.18,0.48);
\node[anchor=east,font=\footnotesize,text=black] (persistentLevelJ) at (1.3,4.2) {${\color{black}J}$};
\node[anchor=east,font=\footnotesize,text=black] (persistentLevelMid) at (1.3,3.08) {${\color{black}j}$};
\node[anchor=east,font=\footnotesize,text=black] (persistentLevelLower) at (1.72,1.86) {${\color{black}j-1}$};
\node[anchor=east,font=\footnotesize,text=black] (persistentLevelZero) at (1.3,0.68) {$0$};

\node[stateCircle,draw=black,line width=1.0pt,fill=white] (persistentNoise) at (1.78,4.02) {};
\node[stateCircle,draw=black,line width=1.0pt,fill=white] (persistentCheckpointJ) at (2.13,2.84) {};
\node[chosenCircle,draw=figTwoGreenBorder,line width=0.95pt,fill=figTwoGreenFill] (persistentForwardLower) at (3.33,1.67) {};
\node[chosenCircle,draw=figTwoRedBorder,line width=0.95pt,fill=figTwoRedFill] (persistentBackwardUpper) at (4.45,2.84) {};
\node[chosenCircle,draw=figTwoGreenBorder,line width=0.95pt,fill=figTwoGreenFill] (persistentForwardFinal) at (5.02,1.66) {};
\node[stateCircle,draw=black,line width=1.15pt,fill=black] (persistentClean) at (5.30,0.48) {};
\node[chosenCircle,draw=figTwoRedBorder,line width=0.95pt,fill=figTwoRedFill] (persistentRepeatState) at (5.95,1.66) {};

\draw[arrowLine,draw=figTwoGreenBorder] (persistentNoise) -- (persistentCheckpointJ);
\draw[arrowLine,draw=figTwoGreenBorder] (persistentCheckpointJ) -- (persistentForwardLower);
\draw[arrowLine,draw=figTwoRedBorder] (persistentForwardLower) -- (persistentBackwardUpper);
\draw[arrowLine,draw=figTwoGreenBorder] (persistentBackwardUpper) -- (persistentForwardFinal);
\draw[arrowLine,draw=figTwoGreenBorder] (persistentForwardFinal) -- (persistentClean);
\node[fill=white,inner sep=1pt,font=\scriptsize] (forwardLabelOne) at (2.73,2.52) {$\mathrm F$};
\node[fill=white,inner sep=1pt,font=\scriptsize] (backwardLabel) at (4.01,2.05) {$\mathrm B$};
\node[fill=white,inner sep=1pt,font=\scriptsize] (forwardLabelTwo) at (4.88,2.4) {$\mathrm F$};

\node[candidateCircle,draw=figTwoGreenBorder,line width=0.9pt,fill=white] (persistentForwardCandidateLeft) at (2.20,1.66) {};
\node[candidateCircle,draw=figTwoGreenBorder,line width=0.9pt,fill=white] (persistentForwardCandidateMiddle) at (2.78,1.66) {};
\draw[draw=figTwoGuide,line width=0.65pt,dashed] (persistentCheckpointJ) -- (persistentForwardCandidateLeft);
\draw[draw=figTwoGuide,line width=0.65pt,dashed] (persistentCheckpointJ) -- (persistentForwardCandidateMiddle);
\node[text=black,font=\scriptsize,anchor=east] (persistentForwardPoolLabel) at (3.33,1.2)
    {$\{Y_i,w_i^{\mathrm F}\}$};

\node[candidateCircle,draw=figTwoRedBorder,line width=0.9pt,fill=white] (persistentBackwardCandidateLeft) at (3.35,2.87) {};
\node[candidateCircle,draw=figTwoRedBorder,line width=0.9pt,fill=white] (persistentBackwardCandidateMiddle) at (3.89,2.87) {};
\draw[draw=figTwoGuide,line width=0.65pt,dashed] (persistentForwardLower) -- (persistentBackwardCandidateLeft);
\draw[draw=figTwoGuide,line width=0.65pt,dashed] (persistentForwardLower) -- (persistentBackwardCandidateMiddle);
\node[text=black,font=\scriptsize] (persistentBackwardPoolLabel) at (3.96,3.34)
    {$\{Z_i,w_i^{\mathrm B,\rho}\}$};

\draw[arrowLine,draw=figTwoRedBorder]
    (persistentClean.east) .. controls (5.83,0.60) and (6.13,1.08) .. (6.2,1.4);
\draw[arrowLine, draw=figTwoGreenBorder, arrows=-Stealth]
    (5.8,1.44) .. controls (5.67,1.31) and (5.49,0.93) .. (5.4,0.75);

\draw[arrowLine,draw=figTwoGreenBorder] (1.22,-0.13) -- (1.86,-0.13);
\node[anchor=west,font=\scriptsize] (forwardLegend) at (1.94,-0.13) {$\mathrm F$: denoise};
\draw[arrowLine,draw=figTwoRedBorder] (3.70,-0.13) -- (4.34,-0.13);
\node[anchor=west,font=\scriptsize] (backwardLegend) at (4.42,-0.13) {$\mathrm B$: renoise};

\end{tikzpicture}}
\caption{Algorithm comparison of CS-VGB. (a) One-pass candidate selection, following SVDD~\citep{li2025svdd}, selects one candidate at each denoising step and stops at the first clean output. (b) CS-VGB may instead renoise and repeat the search, while the exact terminal reward anchors every visit to $t=0$.}
\label{fig:cs-vgb-persistent-search}
\vspace{-10pt}
\end{figure}

}

\mdfsetup{fontcolor=black}
\subsection{Finite-Candidate Implementation and Practical Refinements}
\label{app:subsec:cs-vgb-calibrated-lazy}

The preceding subsection assumes exact directional masses $a_j^{\mathrm F}(x),a_j^{\mathrm B}(x)$ and exact draws from the conditional transitions $Q_j^{\mathrm F}(\cdot\mid x),Q_j^{\mathrm B,\rho}(\cdot\mid x)$. In practice, we approximate these quantities using $L_F$ denoising and $L_B$ renoising candidates at each checkpoint. To avoid the stationary bias that this approximation can introduce, we retain the pools and use \emph{reverse insertion} after each move. Building on the retained pools, \emph{lazy pool evaluation} evaluates candidates in the direction opposite to the current one only as needed, reusing previously computed values to reduce the number of function evaluations (NFE) while remaining compatible with momentum. Beyond this per-update cost, practical performance also depends on the clean samples collected within a finite update budget. Since excessive backtracking can limit clean visits, we introduce \emph{depth-dependent weights} on edge-conductance measures to favor denoising. Together, these refinements preserve the stationary clean target under \cref{thm:c-joint-pool-invariance}.

\paragraph{Candidate weights.}
The state retains the \emph{candidate pools} $\mathbf u=(Y_{1:L_F},Z_{1:L_B})$ alongside $(j,x,d)$, with fixed $L_F,L_B\geq2$. Fresh candidates are drawn independently from $G_j^{\mathrm F}(\cdot\mid x)$ and $G_j^{\mathrm B}(\cdot\mid x)$, with weights
\begin{align}
w_{j,i}^{\mathrm F}
&:=\frac{dK_j^{\mathrm F}(\cdot\mid x)}{dG_j^{\mathrm F}(\cdot\mid x)}(Y_i)
\begin{cases}r(Y_i),&j=1,\\ \sqrt{\widehat V_{j-1}(Y_i)},&j\geq2,\end{cases}\notag\\
w_{j,i}^{\mathrm B,\rho}
&:=\rho_{j+1}(Z_i,x)
\frac{dK_j^{\mathrm B}(\cdot\mid x)}{dG_j^{\mathrm B}(\cdot\mid x)}(Z_i)
\begin{cases}1,&j=0,\\ \sqrt{\widehat V_{j+1}(Z_i)},&j\geq1.\end{cases}
\label{eq:cs-vgb-new-candidate-weights}
\end{align}
At $j=0$, only the B pool is retained; at $j=J$, only the F pool is retained.

\paragraph{Depth weighting and direction probabilities.}
For \emph{depth weighting}, fix $\alpha>0$ and adjust the relative frequency of denoising and renoising through
\begin{equation}
\widehat C_j^\alpha:=\alpha^{j-1}\widehat C_j,\qquad j=1,\ldots,J.
\label{eq:cs-vgb-calibrated-conductance}
\end{equation}
This leaves the conditional transitions unchanged, but reweights the directional masses. Choosing $\alpha<1$ favors F and $\alpha>1$ favors B. Including the current-state factors gives the \emph{candidate contributions}
\begin{equation}
e_{j,i}^{\mathrm F}
:=\alpha^{j-1}w_{j,i}^{\mathrm F}
\begin{cases}1,&j=1,\\ \sqrt{\widehat V_j(x)},&j\geq2,\end{cases}
\qquad
e_{j,i}^{\mathrm B}
:=\alpha^jw_{j,i}^{\mathrm B,\rho}
\begin{cases}r(x),&j=0,\\ \sqrt{\widehat V_j(x)},&j\geq1.\end{cases}
\label{eq:cs-vgb-depth-weighted-candidate-contributions}
\end{equation}
Their pool averages determine the direction probabilities:
\begin{equation}
\widehat a_j^{d}:=\frac1{L_d}\sum_{i=1}^{L_d}e_{j,i}^{d},
\qquad
\widehat p_j^{d}:=
\frac{\widehat a_j^{d}}{\widehat a_j^{\mathrm F}+\widehat a_j^{\mathrm B}},
\qquad
\widehat\ell_j:=\min\{\widehat p_j^{\mathrm F},\widehat p_j^{\mathrm B}\},
\label{eq:cs-vgb-practical-direction-probabilities}
\end{equation}
Unavailable masses are zero. Evaluating both pools and substituting these probabilities into \cref{tab:cs-vgb-momentum-update} determines the move, flip, and stay probabilities. 

\paragraph{Reverse insertion.}
With \emph{reverse insertion}, a move selects a candidate in direction $d$ with probability $e_{j,i}^d/\sum_{h=1}^{L_d} e_{j,h}^d$, then inserts the previous $x$ uniformly into the opposite pool at the destination:
\begin{equation}
\begin{aligned}
\mathrm F:\quad &(j,x)\longrightarrow(j-1,Y_I),
&&K\sim\operatorname{Unif}([L_B]),\quad Z'_K=x,\\
\mathrm B:\quad &(j,x)\longrightarrow(j+1,Z_I),
&&K\sim\operatorname{Unif}([L_F]),\quad Y'_K=x.
\end{aligned}
\label{eq:c-pool-reverse-insertion}
\end{equation}
All other destination entries are drawn independently from their proposals (\cref{fig:cs-vgb-checkpoint-update}). When combined with lazy pool evaluation, the inserted candidate is also designated as the opposite \emph{mark}.

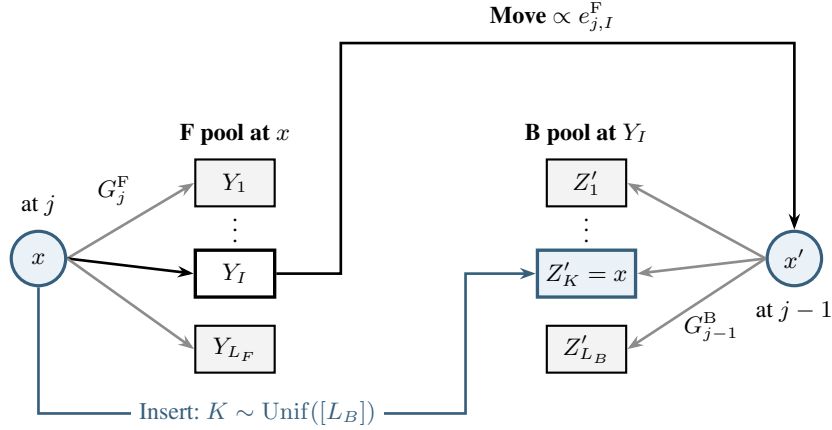
\begin{figure}[!htb]
\centering

\definecolor{poolBlue}{HTML}{365F7D}
\definecolor{poolPaleBlue}{HTML}{EAF2F7}
\definecolor{poolPaleGray}{HTML}{F3F3F3}
\begin{tikzpicture}[
 x=1cm,y=1cm,font=\normalfont\small,text=black,
 slot/.style={draw=black,line width=.8pt,minimum width=1.05cm,minimum height=.62cm,inner sep=3pt},
 flowArrow/.style={-{Stealth[length=2.1mm,width=1.5mm]},line width=1pt}
]
\node[circle,draw=poolBlue,fill=poolPaleBlue,line width=1.2pt,minimum size=.72cm]
 (previousState) at (0,0) {$x$};
\node[slot,fill=poolPaleGray] (sourceCandidateOne) at (2.6,1) {$Y_1$};
\node (sourceEllipsis) at (2.6,.48) {$\vdots$};
\node[slot,fill=white,line width=1.25pt] (selectedCandidate) at (2.6,-.2) {$Y_I$};
\node[slot,fill=poolPaleGray] (sourceCandidateLast) at (2.6,-1.2) {$Y_{L_F}$};
\draw[flowArrow,draw=black!45] (previousState.east) -- (sourceCandidateOne.west);
\draw[flowArrow,draw=black] (previousState.east) -- (selectedCandidate.west);
\draw[flowArrow,draw=black!45] (previousState.east) -- (sourceCandidateLast.west);
\node[circle,draw=poolBlue,fill=poolPaleBlue,line width=1.2pt,minimum size=.72cm]
 (destinationState) at (10,0) {$x'$};
\node[slot,fill=poolPaleGray] (destinationCandidateOne) at (7.25,1) {$Z'_1$};
\node (destinationEllipsis) at (7.25,.48) {$\vdots$};
\node[slot,fill=poolPaleBlue,draw=poolBlue,line width=1.25pt]
 (insertedSource) at (7.25,-.2) {$Z'_K=x$};
\node[slot,fill=poolPaleGray] (destinationCandidateLast) at (7.25,-1.2) {$Z'_{L_B}$};
\draw[flowArrow,draw=black!45] (destinationState.west) -- (destinationCandidateOne.east);
\draw[flowArrow,draw=black!45] (destinationState.west) -- (insertedSource.east);
\draw[flowArrow,draw=black!45] (destinationState.west) -- (destinationCandidateLast.east);
\node[font=\normalfont\small] (forwardProposalLabel) at (1,0.9) {$G_j^{\mathrm F}$};
\node[font=\normalfont\small] (backwardProposalLabel) at (8.9,-0.9) {$G_{j-1}^{\mathrm B}$};
\draw[flowArrow,draw=black] (selectedCandidate.east)
 -- (4,-.2) -- (4,2.85) -- (10,2.85) -- (destinationState.north);
\node[font=\normalfont\small] (moveHeading) at (6.8,3.2)
 {\textbf{Move} $\propto e_{j,I}^{\mathrm F}$};
\draw[flowArrow,draw=poolBlue] (previousState.south)
 -- (0,-2.05) -- (5.65,-2.05) -- (5.65,-.2) -- (insertedSource.west);
\node[fill=white,inner sep=3pt,font=\normalfont\footnotesize,text=poolBlue]
 (insertionLabel) at (2.9,-2.05) {Insert: $K\sim\operatorname{Unif}([L_B])$};
\node[font=\normalfont\small\bfseries] (selectHeading) at (2.6,1.65) {F pool at $x$};
\node[font=\normalfont\small\bfseries] (backwardHeading) at (7.25,1.65) {B pool at $Y_I$};
\node[font=\normalfont\footnotesize] (sourceStateLabel) at (0,0.7) {at $j$};
\node[font=\normalfont\footnotesize] (newStateLabel) at (10,-.7) {at $j-1$};
\end{tikzpicture}
\caption{Reverse insertion for an F move. Select $Y_I$ and insert the previous $x$ into a uniform B slot. All remaining entries are freshly drawn.}
\label{fig:cs-vgb-checkpoint-update}
\end{figure}

\paragraph{Lazy pool evaluation.}
With \emph{lazy pool evaluation}, an F update uses the full F pool and its B \emph{mark} $Z_K$; a B update uses the full B pool and its F \emph{mark} $Y_I$.
In direction F, draw $I\sim\operatorname{Unif}([L_F])$ with the B mark $K$ fixed; in direction B, draw $K\sim\operatorname{Unif}([L_B])$ with the F mark $I$ fixed. Conditional on these indices, the probabilities of $(\mathsf{move},\mathsf{stay},\mathsf{flip})$ are
\begin{equation}
\begin{aligned}
\mathrm F:\quad&
\frac{(\widehat a_j^{\mathrm F},\,\chi\min\{e_{j,I}^{\mathrm F},e_{j,K}^{\mathrm B}\},\,e_{j,K}^{\mathrm B}-\chi\min\{e_{j,I}^{\mathrm F},e_{j,K}^{\mathrm B}\})}
{\widehat a_j^{\mathrm F}+e_{j,K}^{\mathrm B}},\\
\mathrm B:\quad&
\frac{(\widehat a_j^{\mathrm B},\,\chi\min\{e_{j,K}^{\mathrm B},e_{j,I}^{\mathrm F}\},\,e_{j,I}^{\mathrm F}-\chi\min\{e_{j,K}^{\mathrm B},e_{j,I}^{\mathrm F}\})}
{\widehat a_j^{\mathrm B}+e_{j,I}^{\mathrm F}}.
\end{aligned}
\label{eq:cs-vgb-marked-momentum-gate}
\end{equation}
On an F-to-B flip, mark $Y_I$ in the retained F pool and complete the B pool by drawing only missing entries from $G_j^{\mathrm B}(\cdot\mid x)$, keeping $Z_K$ and all previously generated candidates and values. The B-to-F case is symmetric. On a move, discard the uniform index, independently select a candidate by weight, and apply reverse insertion (\cref{eq:c-pool-reverse-insertion}), marking the inserted state at an interior destination.
A stay changes nothing. \Cref{fig:cs-vgb-lazy-cases} illustrates the move and flip cases.

\begin{figure}[!htb]
\centering

\definecolor{lazyBlue}{HTML}{365F7D}
\definecolor{lazyPale}{HTML}{EAF2F7}
\definecolor{lazyGray}{HTML}{F3F3F3}
\begin{tikzpicture}[x=1cm,y=1cm,font=\normalfont\small,text=black,
 pool/.style={draw=black,line width=.7pt,minimum width=2.7cm,minimum height=.62cm,inner sep=4pt},
 state/.style={circle,draw=lazyBlue,fill=lazyPale,line width=1pt,minimum size=.65cm,inner sep=2pt},
 arrow/.style={-{Stealth[length=2mm]},line width=.9pt}]
\node[anchor=west,font=\normalfont\small\bfseries] (moveTitle) at (-0.45,3.25) {(a)  F $\to$ F  move};
\node[state] (moveOld) at (0,1.6) {$x$};
\node[anchor=east,font=\normalfont\footnotesize] (moveOldLevel) at (0.33,2.22) {$j,\mathrm F$};
\node[pool,fill=lazyGray] (moveOldF) at (2.15,2.15) {$Y_1,\ldots,Y_I,\ldots,Y_{L_F}$};
\node[pool,draw=lazyBlue,fill=lazyPale] (moveOldB) at (2.15,1.1) {$Z_K$ \quad\emph{mark}};
\node[font=\normalfont\footnotesize] (moveOldFLabel) at (2.15,2.7) {F pool: full};
\node[font=\normalfont\footnotesize] (moveOldBLabel) at (2.15,0.55) {B pool: others deferred};
\draw[arrow] (moveOld.east) -- (moveOldF.west);
\draw[arrow] (moveOldF.east) -- (5.5,2.15) -- (5.5,1.6) -- (6.15,1.6);
\node[fill=white,font=\normalfont\footnotesize,text=red] (moveSelectLabel) at (5.1,2.5) {redraw $Y_I\propto e_{j,I}^{\mathrm F}$};
\node[state] (moveNew) at (6.5,1.6) {$x'$};
\node[font=\normalfont\footnotesize] (moveNewLevel) at (6.5,.95) {$j-1,\mathrm F$};
\node[pool,fill=lazyGray] (moveNewF) at (9,2.15) {$Y'_1,\ldots,Y'_{L_F}$};
\node[pool,draw=lazyBlue,fill=lazyPale] (moveNewB) at (9,1.1) {$Z'_{K'}=x$ \quad\emph{mark}};
\node[font=\normalfont\footnotesize] (moveNewFLabel) at (9,2.7) {F pool: fresh};
\node[font=\normalfont\footnotesize] (moveNewBLabel) at (9.05,0.55) {B pool: others deferred};
\draw[arrow] (moveNew.east) -- (moveNewF.west);
\draw[arrow,draw=lazyBlue] (moveOld.south) -- (0,.1) -- (7.35,.1) -- (7.35,1.1) -- (moveNewB.west);
\node[fill=white,text=lazyBlue,font=\normalfont\footnotesize,inner sep=3pt] (insertLabel) at (4,0.1)
 {reverse insert: $K'\sim\operatorname{Unif}([L_B])$};
\draw[black!20] (-0.45,-0.35) -- (10.4,-0.35);
\node[anchor=west,font=\normalfont\small\bfseries] (flipTitle) at (-0.45,-0.8) {(b) F $\to$ B flip};
\node[state] (flipOld) at (0,-2.4) {$x$};
\node[font=\normalfont\footnotesize] (flipOldLevel) at (0,-3.05) {$j,\mathrm F$};
\node[pool,fill=lazyGray] (flipOldF) at (2.15,-1.85) {$Y_1,\ldots,Y_{L_F}$};
\node[pool,draw=lazyBlue,fill=lazyPale] (flipOldB) at (2.15,-2.9) {$Z_K$ \quad\emph{mark}};
\node[font=\normalfont\footnotesize] (flipOldFLabel) at (2.15,-1.3) {F pool: full};
\node[font=\normalfont\footnotesize] (flipOldBLabel) at (2.2,-3.45) {B pool: others deferred};
\draw[arrow] (flipOld.east) -- (flipOldF.west);
\node[state] (flipNew) at (6.5,-2.4) {$x$};
\node[font=\normalfont\footnotesize] (flipNewLevel) at (6.5,-3.02) {$j,\mathrm B$};
\node[pool,draw=lazyBlue,fill=lazyPale] (flipNewF) at (9,-1.85) {$Y_1,\ldots,\underbrace{Y_I}_{\text{\emph{mark}}},\ldots,Y_{L_F}$};
\node[pool,fill=lazyGray] (flipNewB) at (9,-2.9) {$Z_1,\ldots,Z_K,\ldots,Z_{L_B}$};
\node[font=\normalfont\footnotesize] (flipNewFLabel) at (9,-1.15) {F pool: retained};
\node[font=\normalfont\footnotesize] (flipNewBLabel) at (9,-3.5) {B pool: full; keep $Z_K$};
\draw[arrow] (flipNew.east) -- (flipNewB.west);
\draw[arrow,draw=lazyBlue] (flipOldF.east) -- (flipNewF.west);
\node[fill=white,font=\normalfont\footnotesize,text=red] (newMarkLabel) at (5.4,-1.55) {mark preselected $I$};
\draw[arrow] (flipOldB.east) -- (5.3,-2.9) -- (5.3,-3.85) -- (11,-3.87) -- (11,-2.9) -- (flipNewB.east);
\node[fill=white,font=\normalfont\footnotesize,inner sep=3pt] (fillLabel) at (7.27,-4.23) {draw missing B entries};
\end{tikzpicture}
\caption{Lazy pool evaluation from direction F, shown when only the B mark has been generated. First draw $I\sim\operatorname{Unif}(\{1,\ldots,L_F\})$, then choose move, stay, or flip using the probabilities in \cref{eq:cs-vgb-marked-momentum-gate}, conditioned on $I$ and the index $K$ of the existing B mark. (a) An F move discards that index and independently selects $Y_I$ by weight, creates the destination pools, and reverse-inserts $x$ as the new B mark. Panel (a) assumes an interior destination ($j\ge2$); for $j=1$, arrival at checkpoint $0$ resets the direction to B and retains only the B pool, with no separate mark. (b) An F-to-B flip leaves $(j,x)$ unchanged, fills missing B entries, and marks the preselected $Y_I$ in the retained F pool. Already generated candidates and values are reused. Stay changes nothing and reuses the same pools and mark.}
\label{fig:cs-vgb-lazy-cases}
\end{figure}

\begin{algorithm2e}[!htbp]
\DontPrintSemicolon
\SetAlgoLined
\small
\KwIn{starting law $\pi_J$; checkpoint count $J$; proposal kernels $G_j^d$; pool sizes $L_F,L_B\ge2$; update budget $U$; depth weight $\alpha>0$; directional persistence $\chi\in[0,1]$; lazy evaluation on/off}
$x\sim\pi_J$; $(j,d)\leftarrow(J,\mathrm F)$; $\mathcal S\leftarrow\varnothing$\;
\For{$h=0,\ldots,U-1$}{
    \lIf{$j=0$}{Append $x$ to $\mathcal S$}
    Complete the current-direction pool if lazy, or all available pools otherwise; evaluate only new candidates\;
    \uIf{$j\in\{0,J\}$}{$\mathsf a\leftarrow\mathsf{move}$\;}
    \uElseIf{lazy}{
        $I\sim\operatorname{Unif}([L_F])$ if $d=\mathrm F$; $K\sim\operatorname{Unif}([L_B])$ if $d=\mathrm B$\;
        Sample $\mathsf a$ via \cref{eq:cs-vgb-marked-momentum-gate} given $(I,K)$; opposite mark fixed\;
    }
    \Else{
        Sample $\mathsf a$ via \cref{tab:cs-vgb-momentum-update} with $(\widehat p_j^d,\widehat\ell_j)$ from \cref{eq:cs-vgb-practical-direction-probabilities}\;
    }
    Apply $\mathsf a$:\;
    
    \expandafter\patchcmd\csname algocf@Vsline\endcsname
      {\strut\par\nointerlineskip}{\par\nointerlineskip}{}{}
    \csname algocf@group\endcsname{\begingroup
\normalfont\small\color{darkblue}
\setlength{\dashlinedash}{2pt}
\setlength{\dashlinegap}{2pt}
\setlength{\arrayrulewidth}{.5pt}
\setlength{\tabcolsep}{6pt}
\setlength{\fboxsep}{4pt}
\colorbox{SkyBlue!12}{%
\begin{tabular}[t]{@{}p{\dimexpr.5\linewidth-49pt\relax}:p{\dimexpr.5\linewidth-49pt\relax}@{\hspace{2pt}}}
\textbf{Stay}
 & \textbf{Flip}\\[3pt]
No change.
 & \textbf{if lazy}: mark $Y_I$ (F) or $Z_K$ (B)\\
 & using the preselected index;\\[3pt]
 & \textcolor{darkblue}{fill missing opposite-pool entries.}\\[2pt]
 & $d:\mathrm F\leftrightarrow\mathrm B$; retain pools.\\[5pt]
\hdashline
\rule{0pt}{15pt}\textbf{F move}
 & \rule{0pt}{15pt}\textbf{B move}\\[5pt]
Independently reselect $I \in[L_F]$ using
 & Independently reselect $K \in [L_B]$ using\\[3pt]
$\bbP(I=i)=e_{j,i}^{\mathrm F}/\sum_h e_{j,h}^{\mathrm F}$
 & $\bbP(K=k)=e_{j,k}^{\mathrm B}/\sum_h e_{j,h}^{\mathrm B}$\\[6pt]
$(j',x')\leftarrow(j-1,Y_I)$
 & $(j',x')\leftarrow(j+1,Z_K)$\\[3pt]
$K'\sim\operatorname{Unif}([L_B])$
 & $I'\sim\operatorname{Unif}([L_F])$\\[3pt]
$Z'_{K'}\leftarrow x$
 & $Y'_{I'}\leftarrow x$\\[3pt]
\textbf{if lazy and $j'>0$}: mark $Z'_{K'}$
 & \textbf{if lazy and $j'<J$}: mark $Y'_{I'}$\\[3pt]
\end{tabular}
}
\endgroup
\;
    }
    \expandafter\patchcmd\csname algocf@Vsline\endcsname
      {\par\nointerlineskip}{\strut\par\nointerlineskip}{}{}
    
    \If{$\mathsf a=\mathsf{move}$}{
        $(j,x)\leftarrow(j',x')$; reset pools, retaining only the inserted candidate; evaluate its weight\;
        \lIf{$j\in\{0,J\}$}{$d\leftarrow\mathrm B$ at $0$, $\mathrm F$ at $J$; inward pool only, no mark}
    }
}
\lIf{$j=0$}{Append $x$ to $\mathcal S$}
\Return{$\mathcal S$}\;
\caption{Finite-candidate single-chain CS-VGB.
}
\label{alg:cs-vgb-practical-outer-loop}
\label{alg:cs-vgb-interior-checkpoint-update}
\end{algorithm2e}

\paragraph{Omitting the edge correction.}
When ${\rho_j}$ is unavailable, an implementation may set ${\rho_j=1}$. We measure its effect through the average reference-edge error in \cref{eq:c-edge-error}. Under the reward and verifier bounds in \cref{eq:c-quantitative-weight-bounds} and the transition-density condition in \cref{eq:c-edge-moving-density-moment}, \cref{thm:c-practical-bias} bounds the resulting stationary bias.

\subsection{Model-Specific Checkpoint Transitions}
\label{app:subsec:model-checkpoint-transitions}

We now describe how to construct the reference kernels $K_j^{\mathrm F}$ and $K_j^{\mathrm B}$ for diffusion models, flow matching, and one-step posterior samplers. For diffusion and flow matching, denoising composes fine-grid steps as in \cref{eq:cs-vgb-checkpoint-composition}, while renoising uses Gaussian transitions across each checkpoint interval. A one-step posterior sampler instead provides a direct denoising transition through a Gaussian bridge.

The reference steps below also serve as unguided proposals, $g_k^{\mathrm F}=k_k^{\mathrm F}$. For a Gaussian reference step $k_k^{\mathrm F}(\cdot\mid x)=\mathcal N(m_k(x),\Sigma_k)$, a positive differentiable verifier gives the mean-shift proposal
\begin{equation}
g_k^{\mathrm F}(dy\mid x)
=\mathcal N\!\left(dy;m_k(x)+s_k\Sigma_k\nabla\log\widehat V_{t_{k-1}}(m_k(x)),\Sigma_k\right),
\label{eq:cs-vgb-practical-gaussian-guidance}
\end{equation}
where $s_k\geq0$ is the guidance strength. The reference kernel remains unchanged; the importance weights account for the proposal change.

\paragraph{{Diffusion models.}}
Consider the VP SDE with noise-rate schedule $\beta(t)>0$, for which $f(x,t)=-\beta(t)x/2$ and $g(t)=\sqrt{\beta(t)}$. Since the coefficients may diverge at $t=1$, we can choose $\kappa_J<T$ and restrict the checkpoint updates to $[0,t_{\kappa_J}]$, below the fully noisy endpoint. The initial law $\pi_J$ is then the reference marginal at $t_{\kappa_J}$ defined in \cref{eq:cs-vgb-implemented-marginals}. The learned reverse drift in \cref{eq:score-reverse-sde} is $-\beta(t)[x/2+s_\theta(x,t)]$. Let $\Delta_k:=t_k-t_{k-1}$. For $1\leq k\leq\kappa_J$, one Euler--Maruyama step from $t_k$ to $t_{k-1}$ gives the fine-grid reference denoising kernel
\begin{equation}
k_k^{\mathrm F}(dx_{k-1}\mid x_k)
:=\mathcal N\left(dx_{k-1};
x_k+\Delta_k\beta(t_k)\left[\frac{x_k}{2}+s_\theta(x_k,t_k)\right],
\Delta_k\beta(t_k)I\right).
\label{eq:cs-vgb-diffusion-fine-forward}
\end{equation}
Recalling the conditional VP scales $\alpha_{t\mid s}:=\exp(-\frac{1}{2}\int_s^t\beta(u)du)$ and $\sigma_{t\mid s}^2:=1-\alpha_{t\mid s}^2$ for $0\leq s<t\leq1$, the reference renoising kernel over the same time interval is
\begin{equation}
k_{k-1}^{\mathrm B}(dx_k\mid x_{k-1})
:=\mathcal N\left(dx_k;
\alpha_{t_k\mid t_{k-1}}x_{k-1},
\sigma_{t_k\mid t_{k-1}}^2I\right).
\label{eq:cs-vgb-diffusion-fine-backward}
\end{equation}
Composing the fine-grid noising steps gives the checkpoint renoising kernel
\begin{equation}
K_{j-1}^{\mathrm B}(dx\mid y)
=\mathcal N\left(dx;\alpha_{t_{\kappa_j}\mid t_{\kappa_{j-1}}}y,
\sigma_{t_{\kappa_j}\mid t_{\kappa_{j-1}}}^2I\right).
\label{eq:cs-vgb-diffusion-checkpoint-backward}
\end{equation}
We remark that composing Gaussian denoising steps generally gives a non-Gaussian checkpoint kernel $K_j^{\mathrm F}$. The edge correction $\rho_j$ accounts for any mismatch between the denoising and renoising edge laws.

\paragraph{{Flow matching.}}
The flow-matching ODE gives a deterministic transition between grid points, whose matched reverse is the exact inverse solution map. When this inverse is not available as an exact finite-step transition, the two edge laws can have different supports and their density ratio can be singular. We therefore use the stochastic SDE pair in \cref{eq:flow-matching-forward-sde,eq:flow-matching-reverse-sde}. For the linear interpolant $X_t=(1-t)X_0+t\epsilon$, choosing $\beta(t)=2t/(1-t)$ gives the renoising SDE (increasing $t$) and the denoising SDE (decreasing $t$), respectively,
\begin{equation}
\resizebox{0.88\linewidth}{!}{$ \displaystyle
{dX_t=-\frac{X_t}{1-t}dt+\sqrt{\frac{2t}{1-t}}dW_t,
\qquad
dX_t=\left[\frac{X_t}{1-t}+2v_t(X_t)\right]dt+\sqrt{\frac{2t}{1-t}}d\overline W_t.}
$}
\label{eq:cs-vgb-fm-checkpoint-sdes}
\end{equation}
Replacing $v_t$ by the pretrained velocity $v_\theta$, one Euler--Maruyama discretization step of the denoising SDE from $t_k$ to $t_{k-1}$ gives the reference denoising kernel
\begin{equation}
k_k^{\mathrm F}(dx_{k-1}\mid x_k)
:=\mathcal N\left(dx_{k-1};
x_k-\Delta_k\left[\frac{x_k}{1-t_k}+2v_\theta(t_k,x_k)\right],
\frac{2t_k\Delta_k}{1-t_k}I\right).
\label{eq:cs-vgb-fm-fine-forward}
\end{equation}
The increasing-time SDE is linear, and its exact one-step noising kernel is
\begin{equation}
k_{k-1}^{\mathrm B}(dx_k\mid x_{k-1})
:=\mathcal N\left(dx_k;
\frac{1-t_k}{1-t_{k-1}}x_{k-1},
\left[t_k^2-\left(\frac{1-t_k}{1-t_{k-1}}\right)^2t_{k-1}^2\right]I\right).
\label{eq:cs-vgb-fm-fine-backward}
\end{equation}
A checkpoint F move composes the fine steps in \cref{eq:cs-vgb-fm-fine-forward}; the B reference kernel over the whole interval is
\begin{equation}
K_{j-1}^{\mathrm B}(dx\mid y)
=\mathcal N\left(dx;\frac{1-t_{\kappa_j}}{1-t_{\kappa_{j-1}}}y,
\left[t_{\kappa_j}^2-\frac{(1-t_{\kappa_j})^2}{(1-t_{\kappa_{j-1}})^2}t_{\kappa_{j-1}}^2\right]I\right).
\label{eq:cs-vgb-fm-checkpoint-backward}
\end{equation}
Thus an F proposal can require several model evaluations, while this B reference draw does not require numerical substeps.

\paragraph{One-step posterior sampling.}
{
We use the Gaussian reference renoising kernel $K_{j-1}^{\mathrm B}$ in \cref{eq:cs-vgb-fm-checkpoint-backward} and construct its corresponding denoising kernel $K_j^{\mathrm F}$ using a one-step posterior sampler. To denoise from ${t_{\kappa_j}}$ to ${t_{\kappa_{j-1}}}$, we first draw a clean endpoint $Z$ conditional on ${X_j=x}$ using MFM or PDM~\citep{potaptchik2026meta,holderrieth2026diamond}. We then sample ${X_{j-1}}$ through a Gaussian bridge conditioned on both endpoints (\cref{fig:cs-vgb-posterior-bridge}), without evaluating intermediate fine-grid states. For ${0\leq t_{\kappa_{j-1}}<t_{\kappa_j}\leq1}$, define
\begin{equation}
a_j:=\frac{1-t_{\kappa_j}}{1-t_{\kappa_{j-1}}},\qquad
\sigma_j^2:=t_{\kappa_j}^2-a_j^2t_{\kappa_{j-1}}^2,\qquad
m_j(x,z):=\frac{a_jt_{\kappa_{j-1}}^2}{t_{\kappa_j}^2}x
+\frac{(1-t_{\kappa_{j-1}})\sigma_j^2}{t_{\kappa_j}^2}z.
\label{eq:cs-vgb-flowmap-noise-variances}
\end{equation}

\begin{figure}[!htb]
\centering

\definecolor{bridgeBlue}{HTML}{365F7D}
\definecolor{bridgePaleBlue}{HTML}{EAF2F7}
\begin{tikzpicture}[
 x=1cm,y=1cm,font=\normalfont\small,text=black,
 endpoint/.style={draw=black,line width=.9pt,minimum width=1.1cm,
 minimum height=.7cm,inner sep=4pt},
 flowArrow/.style={-{Stealth[length=2.1mm,width=1.5mm]},line width=1pt}
]
\node[endpoint,fill=bridgePaleBlue,draw=bridgeBlue]
 (currentNoisyState) at (1,0) {$x$};
\node[endpoint,fill=white] (bridgeCandidate) at (4.3,0) {$Y$};
\node[endpoint,fill=bridgePaleBlue,draw=bridgeBlue]
 (cleanPosteriorDraw) at (11,0) {$Z$};
\node[font=\normalfont\footnotesize] (sourceTime)
 at (1,.65) {$t=t_{\kappa_j}$};
\node[font=\normalfont\footnotesize] (destinationTime)
 at (4.3,.65) {$t=t_{\kappa_{j-1}}$};
\node[font=\normalfont\footnotesize] (cleanTime)
 at (11,.65) {$t=0$};
\draw[flowArrow,draw=black] (currentNoisyState.north east)
 .. controls (3,1.9) and (9,1.9) .. (cleanPosteriorDraw.north west);
\node[font=\normalfont\small] (posteriorMapLabel)
 at (6,1.95) {$Z=\Phi_{t_{\kappa_j},0}^{\theta}(x,\zeta)$};
\node[draw=black,fill=white,line width=.9pt,rounded corners=2pt,
 minimum width=7cm,minimum height=1.25cm,inner sep=0pt]
 (bridgeBox) at (6,-1.55) {};
\draw[flowArrow,draw=bridgeBlue] (currentNoisyState.south)
 -- (1,-1.55) -- (bridgeBox.west);
\draw[flowArrow,draw=bridgeBlue] (cleanPosteriorDraw.south)
 -- (11,-1.55) -- (bridgeBox.east);
\draw[flowArrow,draw=bridgeBlue] (4.3,-.925) -- (bridgeCandidate.south);
\node[font=\normalfont\small\bfseries]
 (bridgeHeading) at (6,-1.18) {Gaussian bridge};
\node[font=\normalfont\small] (bridgeFormula)
 at (6,-1.75) {$Y=m_j(x,Z)+\dfrac{t_{\kappa_{j-1}}\sigma_j}{t_{\kappa_j}}\xi,\qquad \xi\sim\mathcal N(0,I)$};
\end{tikzpicture}
\caption{Posterior sampling followed by a Gaussian bridge conditioned on both endpoints. For an exact posterior sampler, ${Y\sim K_j^{\mathrm F}(\cdot\mid x)}$ (\cref{prop:cs-vgb-posterior-bridge-pairing}).}
\label{fig:cs-vgb-posterior-bridge}
\end{figure}
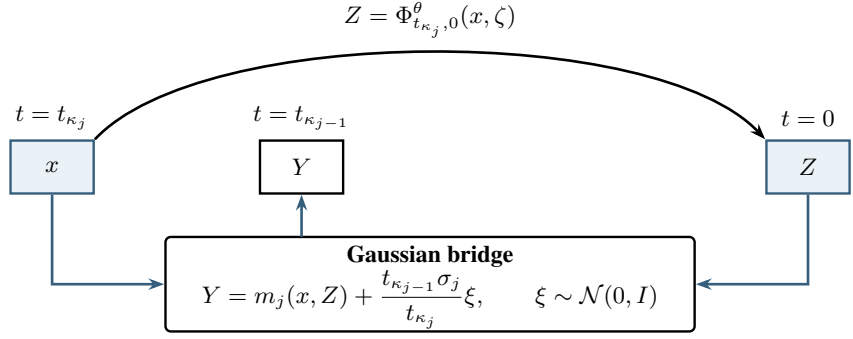

\begin{proposition_ap}[Exact pairing from posterior sampling]
\label{prop:cs-vgb-posterior-bridge-pairing}
Let $(X_0,\ldots,X_J)$ be the Markov chain initialized from $X_0\sim\pi_0$ with transitions $X_j\mid X_{0:j-1}\sim K_{j-1}^{\mathrm B}(\cdot\mid X_{j-1})$, $j=1,\ldots,J$, given in \cref{eq:cs-vgb-fm-checkpoint-backward}, and let $\pi_j$ denote its marginals. Using an exact sampler of $\operatorname{Law}(X_0\mid X_j=x)$, draw
\begin{equation}
Z=\Phi_{t_{\kappa_j},0}^{\theta}(x,\zeta),\qquad
Y=m_j(x,Z)+\frac{t_{\kappa_{j-1}}\sigma_j}{t_{\kappa_j}}\xi,\qquad \zeta,\xi\overset{\mathrm{iid}}{\sim}\mathcal N(0,I).
\label{eq:cs-vgb-flowmap-one-nfe-step}
\end{equation}
At $j=1$, $Y=Z$. The kernel $K_j^{\mathrm F}(\cdot\mid x):=\operatorname{Law}(Y\mid x)$ satisfies
\begin{equation}
\pi_j(dx)K_j^{\mathrm F}(dy\mid x)
=\pi_{j-1}(dy)K_{j-1}^{\mathrm B}(dx\mid y),\qquad \rho_j=1.
\label{eq:cs-vgb-flowmap-exact-pairing}
\end{equation}

\end{proposition_ap}

\begin{proof}
The reference process satisfies ${X_j=a_jX_{j-1}+\sigma_j\epsilon_j}$, with independent Gaussian noise ${\epsilon_j}$, and ${X_j\mid X_0=z\sim\mathcal N((1-t_{\kappa_j})z,t_{\kappa_j}^2I)}$. Bayes' rule therefore gives the posterior sampled by $Z$,
\begin{equation}
\operatorname{Law}(X_0\in dz\mid X_j=x)
=\frac{\varphi_{t_{\kappa_j}^2}(x-(1-t_{\kappa_j})z)\pi_0(dz)}{\pi_j(x)},
\label{eq:cs-vgb-flowmap-posterior}
\end{equation}
where $\varphi_v$ is the density of $\mathcal N(0,vI)$ and ${\pi_j(x)}$ is the density of the noisy marginal ${\pi_j}$. Given $X_0=z$, the pair ${(X_{j-1},X_j)}$ is jointly Gaussian with covariance ${a_jt_{\kappa_{j-1}}^2I}$. Conditioning further on ${X_j=x}$ gives
\begin{align*}
\bbE[X_{j-1}\mid X_j=x,X_0=z]
&=(1-t_{\kappa_{j-1}})z+\frac{a_jt_{\kappa_{j-1}}^2}{t_{\kappa_j}^2}\bigl[x-(1-t_{\kappa_j})z\bigr]=m_j(x,z),\\
\operatorname{Var}(X_{j-1}\mid X_j=x,X_0=z)
&=\left(t_{\kappa_{j-1}}^2-\frac{a_j^2t_{\kappa_{j-1}}^4}{t_{\kappa_j}^2}\right)I
=\frac{t_{\kappa_{j-1}}^2\sigma_j^2}{t_{\kappa_j}^2}I.
\end{align*}
Hence the Gaussian draw in \cref{eq:cs-vgb-flowmap-one-nfe-step} has precisely the conditional law
\begin{equation}
\operatorname{Law}(X_{j-1}\mid X_j=x,X_0=z)=\mathcal N\left(m_j(x,z),\frac{t_{\kappa_{j-1}}^2\sigma_j^2}{t_{\kappa_j}^2}I\right).
\label{eq:cs-vgb-flowmap-center}
\end{equation}
Averaging over $Z$ removes the conditioning on the clean endpoint,
\begin{align}
K_j^{\mathrm F}(dy\mid x)
&=\int\mathcal N\left(dy;m_j(x,z),\frac{t_{\kappa_{j-1}}^2\sigma_j^2}{t_{\kappa_j}^2}I\right)\,\operatorname{Law}(X_0\in dz\mid X_j=x)\notag\\
&=\operatorname{Law}(X_{j-1}\in dy\mid X_j=x).
\label{eq:cs-vgb-flowmap-fine-forward}
\end{align}
Multiplying by ${\pi_j(dx)}$ now gives the required pairing,
\begin{equation*}
\pi_j(dx)K_j^{\mathrm F}(dy\mid x)
=\operatorname{Law}(X_j\in dx,X_{j-1}\in dy)
=\pi_{j-1}(dy)K_{j-1}^{\mathrm B}(dx\mid y).
\end{equation*}
Finally, $t_0=0$ gives $m_1(x,Z)=Z$ and zero bridge variance, hence $Y=Z$.
\end{proof}

For $j\geq2$, one possible guided proposal keeps the posterior draw $Z$
unchanged and shifts the Gaussian bridge mean toward higher verifier values:
\begin{equation}
Y=m_j(x,Z)+s_j\Sigma_j\nabla\log\widehat V_{j-1}(m_j(x,Z))
+\Sigma_j^{1/2}\xi,
\qquad
\Sigma_j:=\frac{t_{\kappa_{j-1}}^2\sigma_j^2}{t_{\kappa_j}^2}I,
\label{eq:posterior-bridge-covariance-guidance}
\end{equation}
where $s_j\geq0$ is the guidance strength, $\xi\sim\mathcal N(0,I)$ independently of $Z$,
and $\widehat V_{j-1}$ is positive and differentiable.
At $j=1$, $\Sigma_1=0$ and $Y=Z$, so this bridge guidance has no effect.
The ImageNet and CelebA experiments instead use the normalized mean shifts in \cref{eq:search-guided-forward,eq:search-guided-backward}, with gradients evaluated at the current state.

\section{CS-VGB Theoretical Guarantees}
\label{app:sec:cs-vgb-theoretical-guarantees}

We first show that the exact checkpoint update in \cref{app:subsec:coarse-checkpoint-cs-vgb} has the desired stationary clean distribution and that reverse insertion preserves this target with finite candidate pools. We then study how the exact chain approaches stationarity and reaches the clean boundary, including the effect of process-verifier error on convergence. Finally, we quantify the bias and convergence error introduced by omitting the edge correction $\rho_j$.

We retain the notation of \cref{app:subsec:cs-vgb-setup,app:subsec:coarse-checkpoint-cs-vgb}: ${j}$ indexes time ${t_{\kappa_j}}$, ${\pi_j}$ is its model marginal, and $\widehat V_0=r$. Throughout, the state space is $\mathcal X\subseteq\bbR^d$, and ${J\geq2}$. The process verifiers are nonnegative and finite. In particular, the corrected kernels are taken in versions satisfying \cref{eq:cs-vgb-corrected-edge-identity}.

\subsection{Stationarity Guarantee for CS-VGB}
\label{app:subsec:cs-vgb-exact-stationarity}

Let $P$ denote the CS-VGB one-step kernel with the exact edge-correction ratios ${\rho_j(x,y)}$ from \cref{eq:cs-vgb-edge-laws,eq:cs-vgb-corrected-renoising}, exact masses ${F_j,B_j^\rho}$, and exact draws from ${Q_j^{\mathrm F},Q_j^{\mathrm B,\rho}}$. It uses \cref{tab:cs-vgb-momentum-update} with $\chi\in[0,1]$ and forced exits at the boundaries.
{
The state of this kernel records the checkpoint index, sample, and direction. Since the direction is fixed at the two boundaries, $P$ acts on
\begin{equation}
\mathsf S
:=\bigl(\{0\}\times\mathcal X\times\{\mathrm B\}\bigr)
\cup\!\bigcup_{j=1}^{J-1}\!\bigl(\{j\}\times\mathcal X\times\{\mathrm F,\mathrm B\}\bigr)
\cup\bigl(\{J\}\times\mathcal X\times\{\mathrm F\}\bigr).
\label{eq:c-mixing-state-space}
\end{equation}
}
To describe the stationary law of the transitions in \cref{eq:cs-vgb-new-forward-kernel,eq:cs-vgb-new-backward-kernel}, we define the \emph{incident measure} $\mu_j$ at checkpoint $j$ by summing the marginals of its adjacent edge-conductance measures from \cref{eq:cs-vgb-interior-conductance,eq:cs-vgb-clean-conductance}. At each boundary only one edge contributes, so for $1\leq j<J$,
{
\begin{equation}
\mu_0(dy):=\widehat C_1(\mathcal X,dy), \qquad \mu_J(dx):=\widehat C_J(dx,\mathcal X), \qquad \mu_j(dx):=\pi_j(dx)\sbra{a_j^{\mathrm F}(x)+a_j^{\mathrm B}(x)}.
\label{eq:c-incident-measures}
\end{equation}
}
Assuming the edge-conductance measures have positive finite total mass, we define the \emph{normalizing constant} $\calZ$ by summing these measures over all checkpoint-direction states, with two directions at each interior checkpoint and one at each boundary,
\begin{equation}
\calZ:=\mu_0(\mathcal X)+2\sum_{j=1}^{J-1}\mu_j(\mathcal X)+\mu_J(\mathcal X)
\in(0,\infty).
\label{eq:c-lift-normalizer}
\end{equation}

\begin{theorem_ap}[Stationary law of exact CS-VGB]
\label{thm:c-exact-stationarity}
Under terminal anchoring (\cref{ass:terminal-anchoring}) and the corrected edge identity in \cref{eq:cs-vgb-corrected-edge-identity}, the CS-VGB kernel $P$ has stationary distribution $\Pi$, given by the following expressions with ${1\leq j<J}$ and $d\in\mbra{\mathrm F,\mathrm B}$ for the interior checkpoints,
{
\begin{equation}
\Pi(0,dx,\mathrm B):=\frac{\mu_0(dx)}\calZ, \quad
\Pi(J,dx,\mathrm F):=\frac{\mu_J(dx)}{\calZ}, \quad
\Pi(j,dx,d):=\frac{\mu_j(dx)}{\calZ}.
\label{eq:c-stationary-law}
\end{equation}
}
At the clean boundary ${j=0}$,
{
\begin{equation}
\Pi(0,\mathcal X,\mathrm B)=\frac{\mu_0(\mathcal X)}{\calZ},\qquad
\frac{\Pi(0,dx,\mathrm B)}{\Pi(0,\mathcal X,\mathrm B)}=\frac{r(x)\pi_0(dx)}{Z}=\pi_0^\star(dx).
\label{eq:c-clean-conditional}
\end{equation}
}
\end{theorem_ap}
\Cref{fig:cs-vgb-measure-sets} summarizes how the same edge-conductance measures determine the sampling transitions and stationary law. The factor $ 2$ in $ \calZ$ counts the two interior direction states. At the clean boundary, conditioning cancels $ \calZ$ and gives $ r\pi_0/Z=\pi_0^\star$.
\begin{figure}[t]
\centering
\begingroup
\begin{tikzpicture}[text=black,>=Latex,font=\normalfont\small,
 draw=black,line width=1pt,every node/.style={inner xsep=7pt,inner ysep=6pt}]
\node[draw=black,fill=black!3,rounded corners=3pt,align=center]
 (edge) at (6.27,0.56)
 {\textbf{Edge measure between endpoints $dx,dy$}\\[5pt]
 $\widehat C_j(dx,dy)=\pi_j(dx)K_j^{\mathrm F}(dy\mid x)
 \sqrt{\widehat V_j(x)}\sqrt{\widehat V_{j-1}(y)},\quad j\geq2$\\[4pt]
 $\widehat C_1(dx,dy)=\pi_1(dx)K_1^{\mathrm F}(dy\mid x)r(y)$};
\node[draw=black,fill=black!3,rounded corners=3pt,align=center] (conditionals) at (3.15,-2.1)
 {\textbf{Normalize at either endpoint}\\[6pt]
 $\widehat C_j(dx,dy)=\widehat C_j(dx,\mathcal X)Q_j^{\mathrm F}(dy\mid x)$\\[6pt]
 $\widehat C_j(dx,dy)=\widehat C_j(\mathcal X,dy)Q_{j-1}^{\mathrm B,\rho}(dx\mid y)$};
\node[draw=black,fill=black!3,rounded corners=3pt,align=center]
 (marginals) at (9.85,-2.1)
 {\textbf{Add adjacent edge masses}\\[6pt]
 $\mu_j(dx)=\widehat C_j(dx,\mathcal X)+\widehat C_{j+1}(\mathcal X,dx)$\\[6pt]
 $\mu_j(dx)=\pi_j(dx)\bigl[a_j^{\mathrm F}(x)+a_j^{\mathrm B}(x)\bigr]$};
\node[draw=black,fill=black!3,rounded corners=3pt,align=center]
 (stationary) at (6.5,-4.9)
 {\textbf{Normalize over time $t_{\kappa_j}$ and direction $d$}\\[5pt]
 $\displaystyle\Pi(j,dx,d)=\frac{\mu_j(dx)}{\mathcal Z},
 \qquad\mathcal Z=\mu_0(\mathcal X)+2\sum_{j=1}^{J-1}\mu_j(\mathcal X)+\mu_J(\mathcal X)$};
\draw[-{Latex[length=2.5mm,width=1.8mm]},line width=1.1pt]
 (3.77,-0.46) -- (conditionals.north);
\draw[-{Latex[length=2.5mm,width=1.8mm]},line width=1.1pt]
 (9,-0.46) -- (marginals.north);
\draw[-{Latex[length=2.5mm,width=1.8mm]},line width=1.1pt]
 (marginals.south) -- (9,-3.95);
\end{tikzpicture}
\endgroup
\caption{Edge-conductance measures define both the CS-VGB update and its stationary law. For ${1\leq j<J}$, $\mu_j$ sums the marginals of the two adjacent edge-conductance measures; at the boundaries, $\mu_0=\widehat C_1(\mathcal X,\cdot)$ and ${\mu_J=\widehat C_J(\cdot,\mathcal X)}$. The resulting law and clean target are established in \cref{thm:c-exact-stationarity}.}
\label{fig:cs-vgb-measure-sets}
\label{fig:cs-vgb-conductance-structure}
\end{figure}
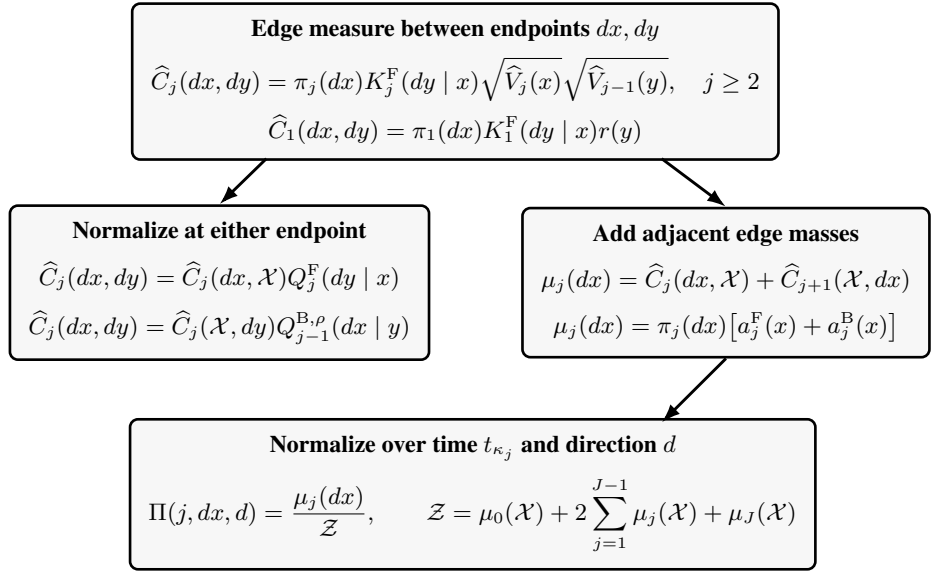
\begin{proof}
For ${2\leq j<J}$, the definitions of ${a_j^{\mathrm F},a_j^{\mathrm B}}$ and ${p_j^{\mathrm F},p_j^{\mathrm B}}$ give
\begin{equation}
\begin{aligned}
\mu_j(dx)p_j^{\mathrm F}(x)
&=\pi_j(dx)a_j^{\mathrm F}(x)
=\pi_j(dx){\sqrt{\widehat V_j(x)}}F_j(x)
=\int_{\mathcal X}\widehat C_j(dx,dy)
=\widehat C_j(dx,\mathcal X),\\
\mu_j(dy)p_j^{\mathrm B}(y)
&=\pi_j(dy)a_j^{\mathrm B}(y)
=\pi_j(dy){\sqrt{\widehat V_j(y)}}B_j^\rho(y)
=\int_{\mathcal X}\widehat C_{j+1}(dx,dy)
=\widehat C_{j+1}(\mathcal X,dy).
\end{aligned}
\label{eq:c-incident-direction}
\end{equation}
For ${j=1}$, the same identities follow by using $a_1^{\mathrm F}=F_1$ and the clean-edge definition of $\widehat C_1$ in the forward relation.
For an interior edge, the two moving fluxes, before division by ${\calZ}$, are therefore
\begin{equation}
\mu_j(dx)p_j^{\mathrm F}(x)Q_j^{\mathrm F}(dy\mid x)
=\widehat C_j(dx,dy)
=\mu_{j-1}(dy)p_{j-1}^{\mathrm B}(y)Q_{j-1}^{\mathrm B,\rho}(dx\mid y),
\qquad 2\leq j<J.
\label{eq:c-moving-flux}
\end{equation}
At the clean boundary, the right-hand side is instead $\mu_0(dy)Q_0^{\mathrm B,\rho}(dx\mid y)=\widehat C_1(dx,dy)$. At the noise boundary, the left-hand side is ${\mu_J(dx)Q_J^{\mathrm F}(dy\mid x)=\widehat C_J(dx,dy)}$.

Fix an interior checkpoint ${j}$ and recall from \cref{eq:cs-vgb-direction-probabilities} that ${\ell_j(x):=\min\mbra{p_j^{\mathrm F}(x),p_j^{\mathrm B}(x)}}$. Moving arrivals in direction $\mathrm F$ have measure ${\widehat C_{j+1}(\mathcal X,dx)=\mu_j(dx)p_j^{\mathrm B}(x)}$. The other arrivals are flips from $\mathrm B$ and stays in $\mathrm F$. Their sum is
\begin{equation}
\mu_j(dx)p_j^{\mathrm B}(x)
+\mu_j(dx)\sbra{p_j^{\mathrm F}(x)-\chi\ell_j(x)}
+\mu_j(dx)\chi\ell_j(x)=\mu_j(dx).
\label{eq:c-forward-balance}
\end{equation}
Similarly, the incoming measure in direction $\mathrm B$ is
\begin{equation}
\mu_j(dx)p_j^{\mathrm F}(x)
+\mu_j(dx)\sbra{p_j^{\mathrm B}(x)-\chi\ell_j(x)}
+\mu_j(dx)\chi\ell_j(x)=\mu_j(dx).
\label{eq:c-backward-balance}
\end{equation}
Equations~\eqref{eq:c-forward-balance} and~\eqref{eq:c-backward-balance} show that the unnormalized incoming measure at each interior lifted state ${(j,dx,d)}$ is ${\mu_j(dx)}$. The boundary relations give the corresponding incoming measures $\mu_0(dy)$ and ${\mu_J(dx)}$. Dividing each incoming measure by $\calZ$ therefore recovers the components in \cref{eq:c-stationary-law}, which proves $\Pi P=\Pi$. Finally, terminal anchoring and normalization of $K_0^{\mathrm B,\rho}$ give
{
\begin{equation*}
\begin{aligned}
\mu_0(dy)=r(y)\pi_0(dy)\int_{\mathcal X}K_0^{\mathrm B,\rho}(dx\mid y)
=r(y)\pi_0(dy), \qquad \mu_0(\mathcal X)&=\int_{\mathcal X}r(y)\pi_0(dy)=Z.
\end{aligned}
\end{equation*}
}
Hence $\Pi(0,dy,\mathrm B)/\Pi(0,\mathcal X,\mathrm B)=r(y)\pi_0(dy)/Z=\pi_0^\star(dy)$.
\end{proof}
\par\removelastskip\vspace{4pt}
Thus, changing the interior verifier changes the invariant law over intermediate states, but not its clean conditional. {We next bound the stationary probability of being at the clean checkpoint in terms of the average verifier error under the reference marginals.}
{
\begin{proposition_ap}[Clean stationary mass under average verifier error]
\label{prop:c-sqrt-clean-mass}

Under the assumptions of \cref{thm:c-exact-stationarity}, let $0<Z:=\int_{\mathcal X}r(x)\pi_0(dx)<\infty$ and $J\geq2$. Suppose the nonnegative process verifiers satisfy, for some $\varepsilon\geq0$,
\begin{equation}
\frac1J\sum_{j=1}^{J}\bbE_{\pi_j}\!\left[|\widehat V_j-V_j|\right]\leq\varepsilon Z.
\label{eq:c-sqrt-average-verifier-error}
\end{equation}
Then
\begin{equation}
\Pi(j=0):=\Pi(0,\mathcal X,\mathrm B)=\frac{Z}{\calZ}\geq\frac{1}{4J-2+4J\varepsilon}.
\label{eq:c-sqrt-clean-mass-bound}
\end{equation}
For exact verifiers, this recovers $\Pi(j=0)\geq1/(4J-2)$, with equality for any positive constant reward.
\end{proposition_ap}
\begin{proof}
Write $M_j:=\widehat C_j(\mathcal X,\mathcal X)$. Terminal anchoring gives $M_1=Z$. For $j\geq2$, taking expectations under $\pi_j(dx)K_j^{\mathrm F}(dy\mid x)$ yields
\begin{align*}
M_j
&=\bbE\sqrt{\widehat V_j(X_j)\widehat V_{j-1}(X_{j-1})}
\overset{(\circ)}{\leq}\frac12\left(\bbE_{\pi_j}\widehat V_j+\bbE_{\pi_{j-1}}\widehat V_{j-1}\right)\\
&\overset{(\triangle)}{\leq}Z+\frac12\left(\bbE_{\pi_j}|\widehat V_j-V_j|+\bbE_{\pi_{j-1}}|\widehat V_{j-1}-V_{j-1}|\right),
\end{align*}
where $(\circ)$ uses $\sqrt{ab}\leq(a+b)/2$, and $(\triangle)$ uses $\bbE_{\pi_j}V_j=Z$. By \cref{eq:c-incident-measures,eq:c-lift-normalizer}, each verifier error enters the normalizer with coefficient at most $4$, so
\begin{align*}
\calZ
&=3Z+4\sum_{j=2}^{J-1}M_j+3M_J \leq(4J-2)Z+4\sum_{j=1}^{J}\bbE_{\pi_j}|\widehat V_j-V_j|
\leq(4J-2+4J\varepsilon)Z.
\end{align*}
Since $\Pi(j=0)=Z/\calZ$, this proves the bound. If $\widehat V_j=V_j$ and $r\equiv c>0$, then $V_j\equiv c$ and $M_j=Z=c$ for every edge, giving equality in the exact-verifier bound.
\end{proof}
}
{For bounded average verifier error $\varepsilon=O(1)$, the clean checkpoint retains stationary probability $\Omega(J^{-1})$, ruling out exponential dilution of its mass as the number of checkpoints grows. Therefore, this implies clean samples are not exponentially rare under the stationary law.}

\mdfsetup{fontcolor=black}
\subsection{Exact Finite-Candidate Updates}
\label{app:subsec:cs-vgb-finite-candidate-exactness}
\label{app:subsec:cs-vgb-calibrated-lazy-guarantees}

In this subsection, we verify that reverse insertion and lazy pool evaluation preserve the clean target despite using finite pools. We keep the weights $e_{j,i}^{\mathrm F},e_{j,k}^{\mathrm B}$ and averages $\widehat a_j^{\mathrm F},\widehat a_j^{\mathrm B}$ from \cref{app:subsec:cs-vgb-calibrated-lazy}. For $\mathbf u=(Y_{1:L_F},Z_{1:L_B})$, the law of fresh pools is
\begin{equation*}
\mathsf G_j(d\mathbf u\mid x)
:=\prod_{i=1}^{L_F}G_j^{\mathrm F}(dY_i\mid x)
\prod_{k=1}^{L_B}G_j^{\mathrm B}(dZ_k\mid x),
\end{equation*}
omitting the unavailable pool at each boundary. Let $\mathcal Z_\alpha\in(0,\infty)$ be the normalizer in \cref{eq:c-lift-normalizer} with $\widehat C_j$ replaced by $\alpha^{j-1}\widehat C_j$. We assume the updates below are well-defined.

For the CS-VGB kernel in \cref{app:subsec:cs-vgb-calibrated-lazy} with reverse insertion, lazy pool evaluation, depth weighting, and exact edge corrections $\rho_j$, define $\Pi_{\mathrm{mark}}^\alpha$ at $1\leq j<J$ by
\begin{equation*}
\begin{aligned}
\Pi_{\mathrm{mark}}^\alpha(j,dx,d\mathbf u,\mathrm F,k)
&:=\frac{\pi_j(dx)\mathsf G_j(d\mathbf u\mid x)}{\mathcal Z_\alpha}
\frac{\widehat a_j^{\mathrm F}+e_{j,k}^{\mathrm B}}{L_B},\\
\Pi_{\mathrm{mark}}^\alpha(j,dx,d\mathbf u,\mathrm B,i)
&:=\frac{\pi_j(dx)\mathsf G_j(d\mathbf u\mid x)}{\mathcal Z_\alpha}
\frac{\widehat a_j^{\mathrm B}+e_{j,i}^{\mathrm F}}{L_F}.
\end{aligned}
\end{equation*}
At the boundaries, use \cref{eq:c-unified-pool-laws} without a mark.

\Needspace{5\baselineskip}
\begin{theorem_ap}[Stationary clean target with finite candidates]
\label{thm:c-joint-pool-invariance}
For $L_F,L_B\geq2$, the above kernel has stationary law $\Pi_{\mathrm{mark}}^\alpha$, with marginal
\begin{equation}
\Pi_{\mathrm{pool}}^\alpha(j,dx,d\mathbf u,d)
=\frac{\pi_j(dx)\mathsf G_j(d\mathbf u\mid x)}{\mathcal Z_\alpha}
(\widehat a_j^{\mathrm F}+\widehat a_j^{\mathrm B}),
\label{eq:c-unified-pool-laws}
\end{equation}
after summing over the mark. Without lazy pool evaluation (i.e., using only reverse insertion and depth weighting), the CS-VGB kernel instead has stationary law $\Pi_{\mathrm{pool}}^\alpha$. Both satisfy
\begin{equation*}
\Pi_{\mathrm{pool}}^\alpha(dx\mid j=0)
=\Pi_{\mathrm{mark}}^\alpha(dx\mid j=0)=\pi_0^\star(dx).
\end{equation*}
For $\alpha=1$, both stationary laws have marginal $\Pi$ on $(j,x,d)$, as defined in \cref{thm:c-exact-stationarity}.
\end{theorem_ap}

\begin{proof}
We first establish an identity for the incoming move measures used below.
Let $\mathsf T_{j,i}^{\mathrm F}$ and $\mathsf T_{j,k}^{\mathrm B}$ fix one candidate and draw all remaining entries independently:
\begin{equation*}
\begin{aligned}
\mathsf T_{j,i}^{\mathrm F}(d\mathbf u\mid x;y)
&:=\delta_y(dY_i)\prod_{h\ne i}G_j^{\mathrm F}(dY_h\mid x)
\prod_{\ell=1}^{L_B}G_j^{\mathrm B}(dZ_\ell\mid x),\\
\mathsf T_{j,k}^{\mathrm B}(d\mathbf u\mid x;y)
&:=\delta_y(dZ_k)\prod_{h=1}^{L_F}G_j^{\mathrm F}(dY_h\mid x)
\prod_{\ell\ne k}G_j^{\mathrm B}(dZ_\ell\mid x).
\end{aligned}
\end{equation*}
Unavailable pools are omitted. For fixed $i,k$, we compute the destination measure after moving to $Y_i$ and inserting the previous state $x$ into B slot $k$. The exact weights in \cref{eq:cs-vgb-new-candidate-weights} give
\begin{equation}
\begin{aligned}
&\int \pi_j(dx)\mathsf G_j(d\mathbf u\mid x)e_{j,i}^{\mathrm F}
\delta_{Y_i}(dy)\mathsf T_{j-1,k}^{\mathrm B}(d\mathbf u'\mid y;x)\\
&\qquad=\alpha^{j-1}\int\widehat C_j(dx,dy)
\mathsf T_{j-1,k}^{\mathrm B}(d\mathbf u'\mid y;x)
=\pi_{j-1}(dy)\mathsf G_{j-1}(d\mathbf u'\mid y)e_{j-1,k}^{\mathrm B}.
\end{aligned}
\label{eq:c-indexed-edge-criterion}
\end{equation}
On the left, $\pi_j(dx)\mathsf G_j(d\mathbf u\mid x)$ is the source reference measure, weighted by $e_{j,i}^{\mathrm F}$; $\delta_{Y_i}(dy)$ sets $y=Y_i$, and $\mathsf T_{j-1,k}^{\mathrm B}$ inserts $x$ into slot $k$ and generates the remaining destination entries. Integrating out $(x,\mathbf u)$ leaves a measure on $(y,\mathbf u')$. The first equality integrates the source pools $\mathbf u$; the second uses the exact B weight at $Z'_k=x$. The B move is symmetric.

We first verify \cref{eq:c-unified-pool-laws} without lazy pool evaluation by summing the measures entering an interior F state $(j,x,\mathbf u,\mathrm F)$. At the source of an incoming F move,
\begin{equation*}
\Pi_{\mathrm{pool}}^\alpha(j+1,dx',d\mathbf u',\mathrm F)
\widehat p_{j+1}^{\mathrm F}
=\frac{\pi_{j+1}(dx')\mathsf G_{j+1}(d\mathbf u'\mid x')}{\mathcal Z_\alpha}
\widehat a_{j+1}^{\mathrm F}.
\end{equation*}
Selecting candidate $i$ with probability $e_{j+1,i}^{\mathrm F}/(L_F\widehat a_{j+1}^{\mathrm F})$ and insertion slot $k$ uniformly, then applying \cref{eq:c-indexed-edge-criterion}, gives
\begin{align*}
&\int\Pi_{\mathrm{pool}}^\alpha(j+1,dx',d\mathbf u',\mathrm F)
\bbP(\mathsf{move},dx,d\mathbf u\mid j+1,x',\mathbf u',\mathrm F)\\
&\qquad=\frac{\pi_j(dx)\mathsf G_j(d\mathbf u\mid x)}{\mathcal Z_\alpha}
\sum_{i=1}^{L_F}\sum_{k=1}^{L_B}\frac{e_{j,k}^{\mathrm B}}{L_FL_B}
=\frac{\pi_j(dx)\mathsf G_j(d\mathbf u\mid x)}{\mathcal Z_\alpha}\widehat a_j^{\mathrm B}.
\end{align*}
A B-to-F flip and an F stay leave $(j,x,\mathbf u)$ unchanged, so their incoming measures are
\begin{equation*}
\begin{aligned}
&\Pi_{\mathrm{pool}}^\alpha(j,dx,d\mathbf u,\mathrm B)
(\widehat p_j^{\mathrm F}-\chi\widehat\ell_j)
=\frac{\pi_j(dx)\mathsf G_j(d\mathbf u\mid x)}{\mathcal Z_\alpha}
\left(\widehat a_j^{\mathrm F}-\chi\min\{\widehat a_j^{\mathrm F},\widehat a_j^{\mathrm B}\}\right),\\
&\Pi_{\mathrm{pool}}^\alpha(j,dx,d\mathbf u,\mathrm F)\chi\widehat\ell_j
=\frac{\pi_j(dx)\mathsf G_j(d\mathbf u\mid x)}{\mathcal Z_\alpha}
\chi\min\{\widehat a_j^{\mathrm F},\widehat a_j^{\mathrm B}\}.
\end{aligned}
\end{equation*}
Thus their sum is
\begin{equation*}
\frac{\pi_j(dx)\mathsf G_j(d\mathbf u\mid x)}{\mathcal Z_\alpha}
\left[
\widehat a_j^{\mathrm B}
+\widehat a_j^{\mathrm F}-\chi\min\{\widehat a_j^{\mathrm F},\widehat a_j^{\mathrm B}\}
+\chi\min\{\widehat a_j^{\mathrm F},\widehat a_j^{\mathrm B}\}
\right]
=\Pi_{\mathrm{pool}}^\alpha(j,dx,d\mathbf u,\mathrm F).
\end{equation*}
Interchanging F and B proves stationarity of the B copy.

With lazy pool evaluation, fix the destination state $(j,x,\mathbf u,\mathrm F,k)$. We first compute the total outgoing F-move measure from checkpoint $j+1<J$, summing over its possible B marks $\ell$:
\begin{equation*}
\sum_{\ell=1}^{L_B}
\Pi_{\mathrm{mark}}^\alpha(j+1,dx',d\mathbf u',\mathrm F,\ell)
\frac{\widehat a_{j+1}^{\mathrm F}}{\widehat a_{j+1}^{\mathrm F}+e_{j+1,\ell}^{\mathrm B}}
=\frac{\pi_{j+1}(dx')\mathsf G_{j+1}(d\mathbf u'\mid x')}{\mathcal Z_\alpha}
\widehat a_{j+1}^{\mathrm F}.
\end{equation*}
The same selection and insertion rule, keeping destination mark $k$ fixed, gives by \cref{eq:c-indexed-edge-criterion}
\begin{align*}
&\sum_{\ell=1}^{L_B}\int
\Pi_{\mathrm{mark}}^\alpha(j+1,dx',d\mathbf u',\mathrm F,\ell)
\bbP(\mathsf{move},dx,d\mathbf u,K=k\mid j+1,x',\mathbf u',\mathrm F,\ell)\\
&\qquad=\frac{\pi_j(dx)\mathsf G_j(d\mathbf u\mid x)}{\mathcal Z_\alpha}
\sum_{i=1}^{L_F}\frac{e_{j,k}^{\mathrm B}}{L_FL_B}
=\frac{\pi_j(dx)\mathsf G_j(d\mathbf u\mid x)}{\mathcal Z_\alpha}
\frac{e_{j,k}^{\mathrm B}}{L_B}.
\end{align*}
For a B-to-F flip, the uniform draw of $K=k$ and the conditional flip probability in \cref{eq:cs-vgb-marked-momentum-gate} give
\begin{align*}
&\sum_{i=1}^{L_F}
\Pi_{\mathrm{mark}}^\alpha(j,dx,d\mathbf u,\mathrm B,i)
\bbP(\mathrm B\to\mathrm F,K=k\mid j,x,\mathbf u,I=i)\\
&\qquad=
\frac{\pi_j(dx)\mathsf G_j(d\mathbf u\mid x)}{\mathcal Z_\alpha}
\sum_{i=1}^{L_F}
\frac{\widehat a_j^{\mathrm B}+e_{j,i}^{\mathrm F}}{L_F}
\frac1{L_B}
\frac{e_{j,i}^{\mathrm F}-\chi\min\{e_{j,k}^{\mathrm B},e_{j,i}^{\mathrm F}\}}
{\widehat a_j^{\mathrm B}+e_{j,i}^{\mathrm F}}\\
&\qquad=
\frac{\pi_j(dx)\mathsf G_j(d\mathbf u\mid x)}{\mathcal Z_\alpha}
\frac1{L_FL_B}\sum_{i=1}^{L_F}
\left[e_{j,i}^{\mathrm F}-\chi\min\{e_{j,i}^{\mathrm F},e_{j,k}^{\mathrm B}\}\right].
\end{align*}
An F stay retains mark $k$; averaging over its temporary uniform index gives
\begin{equation*}
\Pi_{\mathrm{mark}}^\alpha(j,dx,d\mathbf u,\mathrm F,k)
\bbP(\mathsf{stay}\mid j,x,\mathbf u,\mathrm F,k)
=
\frac{\pi_j(dx)\mathsf G_j(d\mathbf u\mid x)}{\mathcal Z_\alpha}
\frac1{L_FL_B}\sum_{i=1}^{L_F}
\chi\min\{e_{j,i}^{\mathrm F},e_{j,k}^{\mathrm B}\}.
\end{equation*}
The minimum terms cancel, so adding the incoming move measure gives
\begin{equation}
\frac{\pi_j(dx)\mathsf G_j(d\mathbf u\mid x)}{\mathcal Z_\alpha}
\left[
\frac{e_{j,k}^{\mathrm B}}{L_B}
+\frac1{L_FL_B}\sum_{i=1}^{L_F}e_{j,i}^{\mathrm F}
\right]
=
\Pi_{\mathrm{mark}}^\alpha(j,dx,d\mathbf u,\mathrm F,k).
\label{eq:c-marked-local-balance}
\end{equation}
The B case is symmetric. At a boundary, the forced move and the same insertion calculation give \cref{eq:c-unified-pool-laws} without a mark. Summing over the mark gives
\begin{equation*}
\sum_{k=1}^{L_B}\Pi_{\mathrm{mark}}^\alpha(j,dx,d\mathbf u,\mathrm F,k)
=\frac{\pi_j(dx)\mathsf G_j(d\mathbf u\mid x)}{\mathcal Z_\alpha}
\left(\widehat a_j^{\mathrm F}+\frac1{L_B}\sum_{k=1}^{L_B}e_{j,k}^{\mathrm B}\right)
=\Pi_{\mathrm{pool}}^\alpha(j,dx,d\mathbf u,\mathrm F).
\end{equation*}
The B case is symmetric. Finally, the exact weights satisfy
\begin{equation*}
\begin{aligned}
\pi_j(dx)\int\mathsf G_j(d\mathbf u\mid x)\widehat a_j^{\mathrm F}
&=\alpha^{j-1}\widehat C_j(dx,\mathcal X), \qquad \pi_j(dx)\int\mathsf G_j(d\mathbf u\mid x)\widehat a_j^{\mathrm B} =\alpha^j\widehat C_{j+1}(\mathcal X,dx),
\end{aligned}
\end{equation*}
for each available direction. Thus integrating out the pools recovers the depth-weighted law with normalizer $\mathcal Z_\alpha$. Since the clean edge has weight $\alpha^0=1$,
\begin{equation*}
\Pi_{\mathrm{pool}}^\alpha(0,dx,\mathrm B)
=\frac{\widehat C_1(\mathcal X,dx)}{\mathcal Z_\alpha}
=\frac{r(x)\pi_0(dx)}{\mathcal Z_\alpha},
\qquad
\Pi_{\mathrm{pool}}^\alpha(dx\mid j=0)
=\frac{r(x)\pi_0(dx)}Z=\pi_0^\star(dx).
\end{equation*}
The marked law has the same clean conditional. For $\alpha=1$, integrating out the pools gives $\Pi$ from \cref{thm:c-exact-stationarity}.
\end{proof}

\subsection{Convergence Guarantee for CS-VGB}
\label{app:subsec:cs-vgb-convergence}

Having identified the stationary law, we now ask how quickly the chain approaches it. We relate movement between states to the decay of distributional error and derive rate bounds for the exact kernel $P$, then quantify how process-verifier error affects these bounds. The first hitting time of the clean boundary is treated separately in \cref{app:subsec:cs-vgb-momentum-hitting}.

\paragraph{From a Poincar\'e inequality to convergence.}
For probability measures $\nu\ll\nu'$ on $\mathsf S$, define their \emph{total variation distance} by
\begin{equation}
\|\nu-\nu'\|_{\mathrm{TV}}
:=\sup_{E\subseteq\mathsf S}|\nu(E)-\nu'(E)|
=\frac12\int_{\mathsf S}\left|\frac{d\nu}{d\nu'}-1\right|d\nu'.
\label{eq:c-total-variation}
\end{equation}
For $f\in L^2(\Pi)$, let $\Pi(f)$, $\|f\|_{2,\Pi}$, and $\operatorname{Var}_{\Pi}(f)$ denote its mean, $L^2(\Pi)$ norm, and variance under $\Pi$, respectively. The \emph{Dirichlet form} $\mathcal E_P(f,f)$ measures its mean squared change over one stationary transition:
\begin{equation}
\begin{aligned}
\Pi(f)&:=\int f\,d\Pi,
&\|f\|_{2,\Pi}^2&:=\Pi(f^2),\\
\operatorname{Var}_{\Pi}(f)&:=\|f-\Pi(f)\|_{2,\Pi}^2,
&\mathcal E_P(f,f)&:=\frac12\iint[f(z')-f(z)]^2\,\Pi(dz)P(z,dz').
\end{aligned}
\label{eq:c-outer-dirichlet-form}
\end{equation}
We say that $P$ satisfies a \emph{Poincar\'e inequality} with constant $\gamma>0$ if
\begin{equation}
\mathcal E_P(f,f)\geq\gamma\operatorname{Var}_{\Pi}(f),
\qquad f\in L^2(\Pi).
\label{eq:c-outer-poincare}
\end{equation}
This condition requires the one-step variation to control the stationary variance uniformly over all $f\in L^2(\Pi)$. For the theoretical convergence analysis, we introduce a \emph{lazy kernel} to remove possible periodicity and turn this inequality into a contraction bound:
\begin{equation}
P_{\mathrm L,\eta}:=\frac{\eta I+P}{1+\eta},\qquad \eta>0,
\qquad P_{\mathrm L}:=P_{\mathrm L,1}.
\label{eq:c-mixing-lazy-kernel}
\end{equation}
Here $I$ is the identity kernel, so the chain stays at its current state with probability $\eta/(1+\eta)$ and otherwise applies $P$. This additional self-loop preserves $\Pi$ and is separate from the momentum update's stay action. The mixing bounds below apply to the $1/2$-lazy kernel $P_{\mathrm L}$; this analytical modification does not require adding lazification to the practical algorithm.

\begin{corollary_ap}[Exponential convergence of the outer MCMC]
\label{cor:c-cs-vgb-convergence}
Suppose \cref{eq:c-outer-poincare} holds, and let $\nu$ be an initial distribution satisfying $\nu\ll\Pi$ and $D_\nu:=\|d\nu/d\Pi-1\|_{2,\Pi}<\infty$. For every integer $n\geq0$,
\begin{equation}
\left\|\frac{d(\nu P_{\mathrm L}^n)}{d\Pi}-1\right\|_{2,\Pi}^2
\leq(1-\gamma/2)^nD_\nu^2,\qquad
\|\nu P_{\mathrm L}^n-\Pi\|_{\mathrm{TV}}
\leq\frac{D_\nu}{2}(1-\gamma/2)^{n/2}
\leq\frac{D_\nu}{2}e^{-\gamma n/4}.
\label{eq:c-cs-vgb-convergence}
\end{equation}
Whenever $D_\nu(1-\gamma/2)^{n/2}/2<\Pi(0,\mathcal X,\mathrm B)$, the clean conditional is well-defined and satisfies
\begin{equation}
\left\|
\frac{(\nu P_{\mathrm L}^n)(0,dx,\mathrm B)}
     {(\nu P_{\mathrm L}^n)(0,\mathcal X,\mathrm B)}
-\pi_0^\star(dx)
\right\|_{\mathrm{TV}}
\leq\frac{D_\nu(1-\gamma/2)^{n/2}}{2\Pi(0,\mathcal X,\mathrm B)}
\leq\frac{D_\nu e^{-\gamma n/4}}{2\Pi(0,\mathcal X,\mathrm B)}.
\label{eq:c-cs-vgb-clean-convergence}
\end{equation}
Thus clean-conditional error at most $\delta\in(0,1)$ is guaranteed by
\begin{equation}
n\geq\left\lceil\frac4\gamma
\log\left(\max\left\{1,\frac{D_\nu}{2\Pi(0,\mathcal X,\mathrm B)\delta}\right\}\right)\right\rceil.
\label{eq:c-clean-mixing-budget}
\end{equation}
\end{corollary_ap}

\begin{proof}
Write $Pf(z):=\int f(z')P(z,dz')$, and let $P^*$ be the adjoint of $P$ with respect to the $L^2(\Pi)$ inner product $\langle f,g\rangle_\Pi:=\Pi(fg)$, so that $\langle f,Pg\rangle_\Pi=\langle P^*f,g\rangle_\Pi$. Stationarity and Jensen's inequality give
\begin{align}
\mathcal E_P(f,f)&=\|f\|_{2,\Pi}^2-\langle f,Pf\rangle_\Pi
=\langle f,(I-P)f\rangle_\Pi, \qquad\|Pf\|_{2,\Pi}^2\leq\Pi(P(f^2))=\|f\|_{2,\Pi}^2.
\label{eq:c-markov-l2-contraction}
\end{align}
The adjoint $P^*$ corresponds to the time-reversed kernel under $\Pi$, characterized by $\Pi(dz)P(z,dz')=\Pi(dz')P^*(z',dz)$. It therefore also preserves $\Pi$, and the same Jensen argument gives $\|P^*g\|_{2,\Pi}\leq\|g\|_{2,\Pi}$. Using $P_{\mathrm L}^*=(I+P^*)/2$ and $\langle g,P^*g\rangle_\Pi=\langle g,Pg\rangle_\Pi$, the lazy-chain argument~\citep{fill1991eigenvalue} gives, for $\Pi(g)=0$,
\begin{align}
\|P_{\mathrm L}^*g\|_{2,\Pi}^2
&=\frac14\bigl(\|g\|_{2,\Pi}^2+2\langle g,P^*g\rangle_\Pi+\|P^*g\|_{2,\Pi}^2\bigr)\leq\|g\|_{2,\Pi}^2-\frac12\mathcal E_P(g,g)
\overset{(\circ)}{\leq}(1-\gamma/2)\|g\|_{2,\Pi}^2.
\label{eq:c-adjoint-contraction}
\end{align}
Here $(\circ)$ follows from the Poincar\'e inequality in \cref{eq:c-outer-poincare}, since $\Pi(g)=0$ gives $\operatorname{Var}_\Pi(g)=\|g\|_{2,\Pi}^2$.
By duality, $\int f\,d(\nu P_{\mathrm L})=\langle P_{\mathrm L}^*(d\nu/d\Pi),f\rangle_\Pi$. Since $P_{\mathrm L}^*\mathbf1=\mathbf1$, iteration yields
\begin{equation}
\frac{d(\nu P_{\mathrm L}^n)}{d\Pi}-1
=(P_{\mathrm L}^*)^n\left(\frac{d\nu}{d\Pi}-1\right).
\label{eq:c-density-evolution}
\end{equation}
Applying \cref{eq:c-adjoint-contraction} and Cauchy--Schwarz gives
\begin{equation}
\|\nu P_{\mathrm L}^n-\Pi\|_{\mathrm{TV}}
\leq\frac12\left\|\frac{d(\nu P_{\mathrm L}^n)}{d\Pi}-1\right\|_{2,\Pi}
\leq\frac{D_\nu}{2}(1-\gamma/2)^{n/2}
\leq\frac{D_\nu}{2}e^{-\gamma n/4}.
\label{eq:c-tv-from-l2}
\end{equation}
For the clean bound, put $a:=\Pi(0,\mathcal X,\mathrm B)$ and $b:=(\nu P_{\mathrm L}^n)(0,\mathcal X,\mathrm B)$. The stated condition gives $|a-b|<a$, hence $b>0$. For any two laws with these positive boundary masses,
\begin{equation}
\left\|
\frac{(\nu P_{\mathrm L}^n)(0,dx,\mathrm B)}b
-\frac{\Pi(0,dx,\mathrm B)}a
\right\|_{\mathrm{TV}}
\leq\frac{\|\nu P_{\mathrm L}^n-\Pi\|_{\mathrm{TV}}}{\max\{a,b\}}.
\label{eq:c-clean-conditioning-transfer}
\end{equation}
Indeed, when $b\geq a$ and the conditional difference on $E\subseteq\mathcal X$ is nonnegative,
\begin{align*}
b\left(\frac{(\nu P_{\mathrm L}^n)(0,E,\mathrm B)}b-\frac{\Pi(0,E,\mathrm B)}a\right)
&=(\nu P_{\mathrm L}^n)(0,E,\mathrm B)-\Pi(0,E,\mathrm B)
-\frac{b-a}{a}\Pi(0,E,\mathrm B)\\
&\leq\|\nu P_{\mathrm L}^n-\Pi\|_{\mathrm{TV}}.
\end{align*}
Use $E^c$ for a negative difference and exchange the two laws when $a>b$. Combining \cref{eq:c-clean-conditioning-transfer} with \cref{eq:c-clean-conditional} proves \cref{eq:c-cs-vgb-clean-convergence}; \cref{eq:c-clean-mixing-budget} follows by solving the exponential bound.
\end{proof}

{
\paragraph{A positive CS-VGB Poincar\'e constant.}
\Cref{cor:c-cs-vgb-convergence} expresses the convergence rate in terms of the Poincar\'e constant. We now show when this constant is positive for CS-VGB and how to bound it using the model transitions and process-verifier weights.

We compare the edge-conductance measures ${\widehat C_j}$ used by CS-VGB with the reward-tilted reference edges ${C_j(dx,dy):=\pi_j(dx)K_j^{\mathrm F}(dy\mid x)V_{j-1}(y)}$, $j=1,\ldots,J$. Both use the same denoising kernels ${K_j^{\mathrm F}}$. Recall that the exact continuation values ${V_j:=V_{t_{\kappa_j}}}$ satisfy
\begin{equation}
V_0=r,\qquad V_j=K_j^{\mathrm F}V_{j-1},
\qquad j=1,\ldots,J.
\label{eq:c-mixing-verifier-recursion}
\end{equation}
Weighting a reference trajectory by its terminal reward and normalizing gives the path law $\pi_{0:J}^\star$ and its marginal $\overline C_j$ for the pair $(X_j,X_{j-1})$ at neighboring checkpoints:
\begin{align}
\pi_{0:J}^\star(dx_{0:J})
&:=\frac{r(x_0)}{Z}\pi_J(dx_J)
   \prod_{j=1}^J K_j^{\mathrm F}(dx_{j-1}\mid x_j),
\label{eq:c-mixing-reference-path}\\
\overline C_j(dx,dy)
&:=\frac{C_j(dx,dy)}{Z}
 =\frac{\pi_j(dx)K_j^{\mathrm F}(dy\mid x)V_{j-1}(y)}{Z}.
\label{eq:c-balanced-comparison-edges}
\end{align}
To see why ${\overline C_j}$ is the joint law of ${(X_j,X_{j-1})}$, integrate the path law in \cref{eq:c-mixing-reference-path} over all other states. The part from noise to checkpoint index ${j}$ and the reward-weighted part from checkpoint index ${j-1}$ to clean give, respectively,
\begin{align}
\int_{\mathcal X^{J-j}}\mkern-5mu\pi_J(dx_J)\mkern-3mu\prod_{\ell=j+1}^J\mkern-5mu K_\ell^{\mathrm F}(dx_{\ell-1}\mid x_\ell)&=\pi_j(dx_j),\mkern-2mu
\int_{\mathcal X^{j-1}}\mkern-5mu r(x_0)\mkern-3mu\prod_{\ell=1}^{j-1}\mkern-5mu K_\ell^{\mathrm F}(dx_{\ell-1}\mid x_\ell)=V_{j-1}(x_{j-1}).
\label{eq:c-path-marginalization-steps}
\end{align}
At ${j=J}$ the first expression is simply ${\pi_J(dx_J)}$; at ${j=1}$ the second is $r(x_0)=V_0(x_0)$. The remaining transition is ${K_j^{\mathrm F}(dy\mid x)}$, so the joint marginal is ${\pi_j(dx)K_j^{\mathrm F}(dy\mid x)V_{j-1}(y)/Z=\overline C_j(dx,dy)}$.

To use $C_j$ as a reference for the convergence analysis, we verify that its two marginals are the reward-tilted measures $Z\pi_j^\star$ and $Z\pi_{j-1}^\star$. Integrating out either endpoint and using ${V_j=K_j^{\mathrm F}V_{j-1}}$ and ${\pi_jK_j^{\mathrm F}=\pi_{j-1}}$ gives
\begin{align}
C_j(dx,\mathcal X)&=\pi_j(dx)\int_{\mathcal X}K_j^{\mathrm F}(dy\mid x)V_{j-1}(y)
=\pi_j(dx)V_j(x)=Z\pi_j^\star(dx),\notag\\
C_j(\mathcal X,dy)&=V_{j-1}(y)\int_{\mathcal X}\pi_j(dx)K_j^{\mathrm F}(dy\mid x)
=V_{j-1}(y)\pi_{j-1}(dy)=Z\pi_{j-1}^\star(dy).
\label{eq:c-ideal-edge-marginals}
\end{align}
To measure how the verifier changes these edge weights, assume ${\widehat C_j\ll C_j}$ and define the \emph{edge-weight ratio}
\begin{equation}
R_1:=1,\qquad R_j(x,y):=\frac{d\widehat C_j}{dC_j}(x,y),
\qquad j=2,\ldots,J.
\label{eq:c-average-edge-equivalence}
\end{equation}
Terminal anchoring gives $R_1=1$. For ${j\geq2}$, recall that
\begin{align*}
\widehat C_j(dx,dy)&=\pi_j(dx)K_j^{\mathrm F}(dy\mid x){\sqrt{\widehat V_j(x)\widehat V_{j-1}(y)}}, \qquad C_j(dx,dy) =\pi_j(dx)K_j^{\mathrm F}(dy\mid x)V_{j-1}(y).
\end{align*}
Both measures have the same factor ${\pi_j(dx)K_j^{\mathrm F}(dy\mid x)}$, so their density ratio is the ratio of the remaining weights. Writing the verifier ratio as ${\xi_j(x):=\widehat V_j(x)/V_j(x)}$ wherever ${V_j(x)>0}$ gives
\begin{equation}
{R_j(x,y)=\frac{\sqrt{\widehat V_j(x)}\sqrt{\widehat V_{j-1}(y)}}{V_{j-1}(y)}
=\sqrt{\frac{V_j(x)}{V_{j-1}(y)}}\sqrt{\xi_j(x)\xi_{j-1}(y)},\qquad j\geq2.}
\label{eq:c-local-edge-distortion}
\end{equation}

\paragraph{From edge weights to the stationary measure.}
To compare the stationary law with the reward-tilted reference marginals, we translate the edge ratios $R_j$ into weights on each checkpoint. We first verify that these ratios give a finite normalizing constant. For this purpose, assume
\begin{equation}
0\leq R_j<\infty\quad\overline C_j\text{-almost surely},\qquad
\max_{1\leq j\leq J}\mathbb E_{\overline C_j}[R_j]<\infty.
\label{eq:c-average-edge-first-moment}
\end{equation}
Recall from \cref{eq:c-incident-measures} that ${\mu_j(dx):=\pi_j(dx)\sbra{a_j^{\mathrm F}(x)+a_j^{\mathrm B}(x)}}$ for ${1\leq j<J}$: it adds the incident masses in the two directions. By \cref{eq:c-incident-direction},
\begin{equation*}
\pi_j(dx)a_j^{\mathrm F}(x)=\widehat C_j(dx,\mathcal X),\qquad
\pi_j(dy)a_j^{\mathrm B}(y)=\widehat C_{j+1}(\mathcal X,dy).
\end{equation*}
The first term collects edges toward checkpoint index ${j-1}$, and the second collects edges toward ${j+1}$. Adding them, with only one incident edge at each boundary, gives
\begin{equation}
\mu_0=\widehat C_1(\mathcal X,\cdot),\qquad
\mu_J=\widehat C_J(\cdot,\mathcal X),\qquad
\mu_j=\widehat C_j(\cdot,\mathcal X)+\widehat C_{j+1}(\mathcal X,\cdot),\quad 1\leq j<J.
\label{eq:c-mixing-incident-measures}
\end{equation}
Recall that $Z=\int_{\mathcal X}r(x)\pi_0(dx)$ normalizes the reward-tilted target, while $\calZ$ normalizes the full stationary measure $\Pi$. Since ${\widehat C_j=ZR_j\overline C_j}$, each edge has total mass
\begin{equation*}
\widehat C_j(\mathcal X,\mathcal X)=Z\int R_j\,d\overline C_j
=Z\mathbb E_{\overline C_j}[R_j]
\leq Z\max_{1\leq \ell\leq J}\mathbb E_{\overline C_\ell}[R_\ell].
\end{equation*}
Each boundary has one incident edge; each interior checkpoint has two, counted once for each direction. There are therefore ${2+4(J-1)=4J-2}$ edge contributions, giving
\begin{equation}
0<\calZ=\mu_0(\mathcal X)+2\sum_{j=1}^{J-1}\mu_j(\mathcal X)+\mu_J(\mathcal X)
\leq(4J-2)Z\max_{1\leq j\leq J}\mathbb E_{\overline C_j}[R_j]<\infty.
\label{eq:c-average-normalizer-bound}
\end{equation}
This proves that $\Pi$ has a finite normalizing constant. We now express the stationary measure $\Pi$ as a weighted version of the checkpoint index targets ${\pi_j^\star}$. This allows us to analyze $\Pi$ using the reward-tilted reference path. To do so, define the density ${w_j}$ of ${\mu_j}$ relative to ${Z\pi_j^\star}$:
\begin{equation}
w_j:=\frac{d\mu_j}{Z\,d\pi_j^\star}.
\label{eq:c-average-incident-density}
\end{equation}
To compute ${w_j}$, recall that ${\overline C_j}$ is the joint law of ${(X_j,X_{j-1})}$ and has marginals ${\pi_j^\star}$ and ${\pi_{j-1}^\star}$. Thus, writing ${R_j=R_j(X_j,X_{j-1})}$ and taking conditional expectations under ${\pi_{0:J}^\star}$,
\begin{align*}
\widehat C_j(dx,\mathcal X)&=Z\int_{\mathcal X}R_j(x,y)\overline C_j(dx,dy)
=Z\pi_j^\star(dx)\bbE_{\overline C_j}[R_j\mid X_j=x],\\
\widehat C_{j+1}(\mathcal X,dy)&=Z\int_{\mathcal X}R_{j+1}(x,y)\overline C_{j+1}(dx,dy)
=Z\pi_j^\star(dy)\bbE_{\overline C_{j+1}}[R_{j+1}\mid X_j=y].
\end{align*}
At an interior checkpoint ${1\leq j<J}$, ${\mu_j}$ is the sum of these two marginals. Dividing by ${Z\pi_j^\star}$ adds their conditional expectations. At the boundaries there is only one term, and $R_1=1$, giving
\begin{equation}
w_0=1,\quad w_J(X_J)=\bbE_{\overline C_J}[R_J\mid X_J],\quad w_j(X_j)=\bbE_{\overline C_j}[R_j\mid X_j]+\bbE_{\overline C_{j+1}}[R_{j+1}\mid X_j].
\label{eq:c-average-incident-identities}
\end{equation}
Substituting ${\mu_j=Zw_j\pi_j^\star}$ into ${\Pi(j,dx,d)=\mu_j(dx)/\calZ}$ gives
\[
{\Pi(j,dx,d)=\frac{Zw_j(x)}{\calZ}\pi_j^\star(dx).}
\]
This expresses $\Pi$ as a weighted version of the reward-tilted marginals $\pi_j^\star$, with weights $w_j$. We use this representation below to relate the variance under $\Pi$ to the reward-tilted reference path in the Poincar\'e analysis.

\paragraph{A condition ensuring that states are connected.}
Finite total mass ensures that $\Pi$ is a probability measure, but does not ensure convergence to it. If the chain cannot leave a region $A$ with $0<\Pi(A)<1$, a chain started in $A$ cannot approach $\Pi$. We therefore require the denoising and renoising moves to cross every such region. For a measurable set $A\subseteq\mathsf S$, let ${A_j^d:=\{x:(j,x,d)\in A\}}$, using $A_0^{\mathrm F}=A_0^{\mathrm B}$ and ${A_J^{\mathrm B}=A_J^{\mathrm F}}$ at the boundaries. The condition is
\begin{equation}
0<\Pi(A)<1\quad\Longrightarrow\quad
\sum_{j=1}^J\sum_{d\in\{\mathrm F,\mathrm B\}}
\iint [\mathbf1_{A_j^d}(x)-\mathbf1_{A_{j-1}^d}(y)]^2\,\widehat C_j(dx,dy)>0.
\label{eq:c-movement-connectedness}
\end{equation}
The integral counts the edge mass crossing between $A$ and its complement. The following lemma gives a sufficient condition in terms of the reference path.

\begin{lemma_ap}[Connectedness from the reference endpoints]
\label{lem:c-endpoint-connectedness}
Suppose
\begin{equation}
0<R_j<\infty\qquad\overline C_j\text{-almost surely},\qquad j=1,\ldots,J,
\label{eq:c-positive-edge-distortion}
\end{equation}
and, for every $u\in L^2(\pi_0^\star)$ and ${v\in L^2(\pi_J^\star)}$,
\begin{equation}
u(X_0)=v(X_J)\quad\pi_{0:J}^\star\text{-almost surely}
\quad\Longrightarrow\quad\operatorname{Var}_{\pi_0^\star}(u)=0.
\label{eq:c-no-common-endpoint-function}
\end{equation}
Then \cref{eq:c-movement-connectedness} holds for every measurable $A\subseteq\mathsf S$ with $0<\Pi(A)<1$.
\end{lemma_ap}

\begin{proof}
Suppose the sum in \cref{eq:c-movement-connectedness} vanishes. Each nonnegative term then vanishes, so ${\widehat C_j=ZR_j\overline C_j}$ gives
\begin{equation*}
0=\int\sbra{\mathbf1_{A_j^d}(x)-\mathbf1_{A_{j-1}^d}(y)}^2\widehat C_j(dx,dy)
=Z\int\sbra{\mathbf1_{A_j^d}(x)-\mathbf1_{A_{j-1}^d}(y)}^2R_j(x,y)\overline C_j(dx,dy).
\end{equation*}
Since ${R_j>0}$ ${\overline C_j}$-almost surely by \cref{eq:c-positive-edge-distortion}, the two indicators agree ${\overline C_j}$-almost surely. Moreover, ${\overline C_j}$ is the law of ${(X_j,X_{j-1})}$ under ${\pi_{0:J}^\star}$, hence
\begin{equation*}
\pi_{0:J}^\star\sbra{\mathbf1_{A_j^d}(X_j)\neq\mathbf1_{A_{j-1}^d}(X_{j-1})}
=\int\mathbf1_{\{\mathbf1_{A_j^d}(x)\neq\mathbf1_{A_{j-1}^d}(y)\}}\overline C_j(dx,dy)=0.
\end{equation*}
Therefore,
\begin{equation*}
\mathbf1_{A_j^d}(X_j)=\mathbf1_{A_{j-1}^d}(X_{j-1})
\quad\pi_{0:J}^\star\text{-almost surely}
\end{equation*}
for every ${j}$ and $d$. Following these equalities along the path shows that ${\mathbf1_{A_0}(X_0)=\mathbf1_{A_J}(X_J)}$ almost surely. By \cref{eq:c-no-common-endpoint-function}, the variance of $\mathbf1_{A_0}(X_0)$ is zero. This indicator is therefore either zero almost surely or one almost surely. Equality along each edge gives that same value at every checkpoint index and in both directions. Since ${\Pi(j,dx,d)=Zw_j(x)\pi_j^\star(dx)/\calZ}$, a set of zero ${\pi_j^\star}$-measure also has zero stationary mass at that checkpoint index. Thus, in the first case $\Pi(A)=0$, and in the second case $\Pi(\mathsf S\setminus A)=0$, so $\Pi(A)=1$. Consequently, \cref{eq:c-movement-connectedness} holds for every measurable $A\subseteq\mathsf S$ with $0<\Pi(A)<1$.
\end{proof}

The next lemma gives a sufficient condition for \cref{eq:c-no-common-endpoint-function} in terms of the maximal correlation between the clean and noise endpoints.
\begin{lemma_ap}[An endpoint correlation condition]
\label{lem:c-endpoint-correlation-condition}
Suppose that, for all $u\in L^2(\pi_0^\star)$ and ${v\in L^2(\pi_J^\star)}$,
\begin{equation}
\left|\operatorname{Cov}_{\pi_{0:J}^\star}
\bigl(u(X_0),v(X_J)\bigr)\right|
\leq\rho_{\mathrm{end}}
\sqrt{\operatorname{Var}_{\pi_0^\star}(u)\operatorname{Var}_{\pi_J^\star}(v)},
\qquad 0\leq\rho_{\mathrm{end}}<1.
\label{eq:c-endpoint-correlation}
\end{equation}
Then \cref{eq:c-no-common-endpoint-function} holds. In particular, \cref{eq:c-endpoint-correlation} holds with $\rho_{\mathrm{end}}=0$ if the clean and noise endpoints are independent under the reference path ${\pi_J(dx_J)\prod_{j=1}^J K_j^{\mathrm F}(dx_{j-1}\mid x_j)}$.
\end{lemma_ap}
\begin{proof}
Let ${u(X_0)=v(X_J)}$ ${\pi_{0:J}^\star}$-almost surely. The two random variables then have the same mean and variance, and
\begin{equation*}
\operatorname{Cov}_{\pi_{0:J}^\star}\bigl(u(X_0),v(X_J)\bigr)
=\operatorname{Var}_{\pi_0^\star}(u)=\operatorname{Var}_{\pi_J^\star}(v).
\end{equation*}
Substituting these identities into \cref{eq:c-endpoint-correlation} gives
\begin{equation*}
\operatorname{Var}_{\pi_0^\star}(u)\leq\rho_{\mathrm{end}}\operatorname{Var}_{\pi_0^\star}(u),\qquad
(1-\rho_{\mathrm{end}})\operatorname{Var}_{\pi_0^\star}(u)\leq0.
\end{equation*}
Since $\rho_{\mathrm{end}}<1$ and variance is nonnegative, $\operatorname{Var}_{\pi_0^\star}(u)=0$, proving \cref{eq:c-no-common-endpoint-function}.

For the last claim, independence under this reference path gives the endpoint law ${\pi_0(dx_0)\pi_J(dx_J)}$. Weighting by $r(x_0)/Z$ therefore gives
\begin{align*}
\pi_{0:J}^\star(X_0\in dx_0,X_J\in dx_J)
&=\frac{r(x_0)}{Z}\int_{\mathcal X^{J-1}}\pi_J(dx_J)\prod_{j=1}^J K_j^{\mathrm F}(dx_{j-1}\mid x_j)\\
&=\frac{r(x_0)}{Z}\pi_0(dx_0)\pi_J(dx_J)
=\pi_0^\star(dx_0)\pi_J(dx_J).
\end{align*}
Thus the endpoints remain independent under ${\pi_{0:J}^\star}$, so their covariance is zero for every such $u,v$.
\end{proof}

Condition \cref{eq:c-endpoint-correlation} bounds the correlation between all square-integrable functions of the clean state $X_0$ and the top-checkpoint state $X_J$ strictly below one under $\pi_{0:J}^\star$. In a noising interpretation, this rules out preserving any nonconstant square-integrable function of the clean state without error at the top checkpoint.

Finally, we require the squared edge density to have a finite integral after accounting for the mass at both endpoints. For ${j=1,\ldots,J}$, assume
\begin{equation}
\widehat C_j\ll\mu_j\otimes\mu_{j-1},\qquad
h_j:=\frac{d\widehat C_j}{d(\mu_j\otimes\mu_{j-1})},\qquad
\mathcal I_j:=\iint h_j(x,y)^2\,\mu_j(dx)\mu_{j-1}(dy)<\infty.
\label{eq:c-normalized-edge-integrability}
\end{equation}
In the proof below, write ${f_j^d(x):=f(j,x,d)}$, with the boundary convention
\begin{equation*}
f_0^{\mathrm F}=f_0^{\mathrm B}=f_0:=f(0,\cdot,\mathrm B),\qquad
f_J^{\mathrm F}=f_J^{\mathrm B}=f_J:=f(J,\cdot,\mathrm F).
\end{equation*}
The next lemma is a consequence of the Fredholm alternative~\citep{brezis2011functional}; we give a direct proof for completeness.
\begin{lemma_ap}[From compactness to a positive Poincar\'e constant]
\label{lem:c-compact-movement-gap}
Let $M$ be a bounded, self-adjoint, compact operator on $L^2(\Pi)$, and let $c$ satisfy $0<c_0\leq c\leq c_1<\infty$ $\Pi$-almost surely. Define $Lf:=cf-Mf$. Suppose
\begin{equation*}
\langle f,Lf\rangle_\Pi\geq0\quad\text{for every }f\in L^2(\Pi),
\qquad \ker L=\operatorname{span}\{\mathbf1\}.
\end{equation*}
Then $\langle f,Lf\rangle_\Pi\geq\gamma\operatorname{Var}_\Pi(f)$ for some $\gamma>0$ and every $f\in L^2(\Pi)$.
\end{lemma_ap}
\begin{proof}
Since $L\mathbf1=0$, it suffices to consider mean-zero unit vectors $f$, i.e., $\Pi(f)=0$ and $\|f\|_{2,\Pi}=1$. If the conclusion failed, there would be $f_n$ with
\begin{equation*}
\Pi(f_n)=0,\qquad \|f_n\|_{2,\Pi}=1,\qquad
\langle f_n,Lf_n\rangle_\Pi\longrightarrow0.
\end{equation*}
Positivity of $\langle h+tf,L(h+tf)\rangle_\Pi$ for every real $t$ gives Cauchy--Schwarz for the form $\langle h,Lf\rangle_\Pi$, i.e., $|\langle h,Lf\rangle_\Pi|^2\leq\langle h,Lh\rangle_\Pi\langle f,Lf\rangle_\Pi$. Taking $h=Lf$ yields
\begin{equation*}
\|Lf\|_{2,\Pi}^4
\leq\langle Lf,L(Lf)\rangle_\Pi\langle f,Lf\rangle_\Pi
\leq\|L\|\|Lf\|_{2,\Pi}^2\langle f,Lf\rangle_\Pi,
\qquad \|Lf_n\|_{2,\Pi}\longrightarrow0.
\end{equation*}
Now, we recall that compactness gives a subsequence with ${Mf_{n_h}\to v}$. Since $c^{-1}\leq c_0^{-1}$,
\begin{equation*}
f_{n_h}=c^{-1}(Mf_{n_h}+Lf_{n_h})
\longrightarrow f:=c^{-1}v
\quad\text{in }L^2(\Pi).
\end{equation*}
By continuity, $\Pi(f)=0$, $\|f\|_{2,\Pi}=1$, and $Lf=0$. The first two identities imply that $f$ is nonconstant, since a mean-zero constant has norm zero. Together with $Lf=0$, this contradicts $\ker L=\operatorname{span}\{\mathbf1\}$.
\end{proof}

\Needspace{5\baselineskip}
We now apply this lemma to the exact CS-VGB kernel.
\begin{theorem_ap}[A positive Poincar\'e constant for CS-VGB]
\label{thm:c-average-moment-poincare}
Let $P$ be the exact CS-VGB kernel with $0<\calZ<\infty$. Suppose \cref{eq:c-movement-connectedness,eq:c-normalized-edge-integrability} hold. Then there is a $\gamma_{\mathrm{CS}}\in(0,1]$, independent of $\chi$, such that
\begin{equation}
\mathcal E_P(f,f)\geq\gamma_{\mathrm{CS}}\operatorname{Var}_\Pi(f),
\qquad f\in L^2(\Pi),\qquad\chi\in[0,1].
\label{eq:c-average-positive-poincare}
\end{equation}
Consequently, the full-state and clean-conditional convergence bounds in \cref{cor:c-cs-vgb-convergence} hold with $\gamma=\gamma_{\mathrm{CS}}$.
\end{theorem_ap}

\begin{proof}

We apply \cref{lem:c-compact-movement-gap} to the movement part of $P$. Let $P_{\mathrm{move}}$ retain only checkpoint moves with their original probabilities, and let $P_{\mathrm{move}}^*$ be its adjoint in $L^2(\Pi)$. Define
\begin{equation}
M:=\frac{P_{\mathrm{move}}+P_{\mathrm{move}}^*}{2},\qquad
{c:=M\mathbf1,\qquad}
\mathcal L_{\mathrm{move}}f:=cf-Mf.
\label{eq:c-average-movement-operator}
\end{equation}
For $1\leq j<J$, the outgoing movement mass is
\begin{equation*}
(P_{\mathrm{move}}\mathbf1)(j,x,d)
=P_{\mathrm{move}}((j,x,d),\mathsf S)=p_j^d(x).
\end{equation*}
By the adjoint identity $\langle P_{\mathrm{move}}^*\mathbf1,f\rangle_\Pi=\langle\mathbf1,P_{\mathrm{move}}f\rangle_\Pi$, the incoming mass at an interior state $z=(j,x,d)$, $1\leq j<J$, can be written as
\begin{equation*}
(P_{\mathrm{move}}^*\mathbf1)(z)\Pi(dz)
=(\Pi P_{\mathrm{move}})(dz)
=\int\Pi(dz')P_{\mathrm{move}}(z',dz)
=\bigl(1-p_j^d(x)\bigr)\Pi(dz).
\end{equation*}
Indeed, arrivals in directions $\mathrm F$ and $\mathrm B$ come from $j+1$ and $j-1$, respectively. With $p_J^{\mathrm F}=p_0^{\mathrm B}=1$ for forced boundary moves, \cref{eq:c-moving-flux,eq:c-incident-direction} gives
\begin{align*}
(P_{\mathrm{move}}^*\mathbf1)(j,x,\mathrm F)\Pi(j,dx,\mathrm F)
&=\int_{\mathcal X}
\underbrace{\Pi(j+1,dy,\mathrm F)}_{\mu_{j+1}(dy)/\calZ}
\underbrace{p_{j+1}^{\mathrm F}(y)}_{\text{move probability}}
\underbrace{Q_{j+1}^{\mathrm F}(dx\mid y)}_{\text{destination law}}\\
&=\frac1{\calZ}\int_{\mathcal X}\mu_{j+1}(dy)p_{j+1}^{\mathrm F}(y)Q_{j+1}^{\mathrm F}(dx\mid y)\\
&=\frac{\widehat C_{j+1}(\mathcal X,dx)}{\calZ}
=p_j^{\mathrm B}(x)\frac{\mu_j(dx)}{\calZ},\\
(P_{\mathrm{move}}^*\mathbf1)(j,x,\mathrm B)\Pi(j,dx,\mathrm B)
&=\int_{\mathcal X}
\underbrace{\Pi(j-1,dy,\mathrm B)}_{\mu_{j-1}(dy)/\calZ}
\underbrace{p_{j-1}^{\mathrm B}(y)}_{\text{move probability}}
\underbrace{Q_{j-1}^{\mathrm B,\rho}(dx\mid y)}_{\text{destination law}}\\
&=\frac1{\calZ}\int_{\mathcal X}\mu_{j-1}(dy)p_{j-1}^{\mathrm B}(y)Q_{j-1}^{\mathrm B,\rho}(dx\mid y)\\
&=\frac{\widehat C_j(dx,\mathcal X)}{\calZ}
=p_j^{\mathrm F}(x)\frac{\mu_j(dx)}{\calZ}.
\end{align*}
Since $\Pi(j,dx,d)=\mu_j(dx)/\calZ$, these identities give, almost surely on each stationary component,
\begin{equation*}
(P_{\mathrm{move}}\mathbf1)(j,x,d)=p_j^d(x),\qquad
(P_{\mathrm{move}}^*\mathbf1)(j,x,d)=1-p_j^d(x).
\end{equation*}
At the boundaries both quantities equal one, with the direction reset inward after arrival. Therefore
\begin{equation}
c(j,x,d)=
\begin{cases}
1,&j\in\{0,J\},\\
1/2,&1\leq j<J.
\end{cases}
\label{eq:c-average-movement-degree}
\end{equation}
The kernel $M$ is nonnegative and symmetric with respect to $\Pi$, with $M\mathbf1=c$. Thus
\begin{equation*}
|Mf(z)|^2=\left|\int f(z')M(z,dz')\right|^2
\overset{(\circ)}{\leq}\left(\int M(z,dz')\right)\left(\int f(z')^2M(z,dz')\right)
=c(z)M(f^2)(z),
\end{equation*}
where $(\circ)$ is Cauchy--Schwarz. Integrating with respect to $\Pi(dz)$ gives
\begin{equation*}
\|Mf\|_{2,\Pi}^2\leq\Pi\bigl(cM(f^2)\bigr)
\overset{(\triangle)}{\leq}\Pi\bigl(M(f^2)\bigr)
\overset{(\square)}{=}\Pi(cf^2)
\overset{(\triangle)}{\leq}\Pi(f^2)=\|f\|_{2,\Pi}^2,
\end{equation*}
where $(\triangle)$ uses $c\leq1$, and $(\square)$ follows from the symmetry of $M$ and $M\mathbf1=c$.
Thus $M$ is bounded and self-adjoint. To compute its movement energy, use $M\mathbf1=c$:
\begin{equation*}
\iint f(z)^2\,\Pi(dz)M(z,dz')
=\int f(z)^2c(z)\Pi(dz)=\Pi(cf^2).
\end{equation*}
Symmetry and the definition of $Mf$ give the other two terms:
\begin{align*}
\iint f(z')^2\,\Pi(dz)M(z,dz')
&=\iint f(z)^2\,\Pi(dz)M(z,dz')=\Pi(cf^2),\\
\iint f(z)f(z')\,\Pi(dz)M(z,dz')
&=\int f(z)Mf(z)\Pi(dz)=\langle f,Mf\rangle_\Pi.
\end{align*}
On each edge, \cref{eq:c-moving-flux} gives
\begin{align*}
\Pi(j,dx,\mathrm F)p_j^{\mathrm F}(x)Q_j^{\mathrm F}(dy\mid x)
&=\frac{\mu_j(dx)p_j^{\mathrm F}(x)Q_j^{\mathrm F}(dy\mid x)}{\calZ}
=\frac{\widehat C_j(dx,dy)}{\calZ},\\
\Pi(j-1,dy,\mathrm B)p_{j-1}^{\mathrm B}(y)Q_{j-1}^{\mathrm B,\rho}(dx\mid y)
&=\frac{\mu_{j-1}(dy)p_{j-1}^{\mathrm B}(y)Q_{j-1}^{\mathrm B,\rho}(dx\mid y)}{\calZ}
=\frac{\widehat C_j(dx,dy)}{\calZ}.
\end{align*}
Since $[f(z)-f(z')]^2$ is symmetric, its integrals against the stationary edge measures of $P_{\mathrm{move}}$ and $P_{\mathrm{move}}^*$ agree. Hence
\begin{align}
\langle f,\mathcal L_{\mathrm{move}}f\rangle_\Pi
&=\Pi(cf^2)-\langle f,Mf\rangle_\Pi
=\frac12\iint\sbra{f(z)^2+f(z')^2-2f(z)f(z')}\,\Pi(dz)M(z,dz')\notag\\
&=\frac12\iint\sbra{f(z)-f(z')}^2\,\Pi(dz)M(z,dz')\notag\\
&=\frac12\iint\sbra{f(z)-f(z')}^2\,\Pi(dz)P_{\mathrm{move}}(z,dz')\notag\\
&=\frac{1}{2\calZ}\sum_{j=1}^J\sum_{d\in\{\mathrm F,\mathrm B\}}
\iint\sbra{f_j^d(x)-f_{j-1}^d(y)}^2\,\widehat C_j(dx,dy)\notag\\
&\overset{(\circ)}{\leq}\mathcal E_P(f,f),
\label{eq:c-movement-energy-lower-bound}
\end{align}
where $(\circ)$ drops the nonnegative flip contributions from \cref{eq:c-outer-dirichlet-form}; stays contribute zero.

To establish compactness, write $M(z,dz')=m(z,z')\Pi(dz')$. For $z=(j,x,d)$ and $z'=(j-1,y,d)$, averaging $P_{\mathrm{move}}$ and its adjoint gives
\begin{equation}
\Pi(dz)M(z,dz')\mkern-6mu=\mkern-6mu\frac{\widehat C_j(dx,dy)}{2\calZ}\mkern-6mu=\mkern-6mu m(z,z')\Pi(dz)\Pi(dz')\mkern-6mu=\mkern-6mu m\bigl((j,x,d),(j-1,y,d)\bigr)\frac{\mu_j(dx)\mu_{j-1}(dy)}{\calZ^2}.
\label{eq:c-symmetric-movement-edge-law}
\end{equation}
Using the definition of $h_j$ in \cref{eq:c-normalized-edge-integrability}, we obtain
\begin{align*}
m\bigl((j,x,d),(j-1,y,d)\bigr)\frac{\mu_j(dx)\mu_{j-1}(dy)}{\calZ^2}
&=\frac{\widehat C_j(dx,dy)}{2\calZ} =\frac{h_j(x,y)}{2\calZ}\mu_j(dx)\mu_{j-1}(dy).
\end{align*}
Comparing densities with respect to $\mu_j\otimes\mu_{j-1}$, and using the symmetry of $M$, gives almost everywhere
\begin{equation*}
m\bigl((j,x,d),(j-1,y,d)\bigr)=\frac{\calZ}{2}h_j(x,y),
\qquad m(z,z')=m(z',z),
\end{equation*}
with the boundary directions identified as before. Using $\mathcal I_j$ defined in \cref{eq:c-normalized-edge-integrability}, each direction and orientation contributes
\begin{equation*}
\iint\left(\frac{\calZ}{2}h_j(x,y)\right)^2
\frac{\mu_j(dx)}{\calZ}\frac{\mu_{j-1}(dy)}{\calZ}
=\frac14\mathcal I_j.
\end{equation*}
Summing over the two direction labels and both orientations of each edge in $M$ gives
\begin{equation}
\iint m(z,z')^2\,\Pi(dz)\Pi(dz')
=\sum_{j=1}^J4\left(\frac{\mathcal I_j}{4}\right)=\sum_{j=1}^J\mathcal I_j<\infty.
\label{eq:c-movement-kernel-square-integrability}
\end{equation}
Finite sums of products are dense in $L^2(\Pi\otimes\Pi)$~\citep{tao2010epsilon}, so we can choose finite-rank approximations
\begin{align*}
m_N(z,z')&=\sum_{\ell=1}^{r_N}a_{N,\ell}(z)b_{N,\ell}(z'),\qquad
\|m-m_N\|_{L^2(\Pi\otimes\Pi)}\longrightarrow0,\\
M_Nf(z)&:=\int m_N(z,z')f(z')\Pi(dz')=\sum_{\ell=1}^{r_N}a_{N,\ell}(z)\int b_{N,\ell}(z')f(z')\Pi(dz').
\end{align*}
The range of $M_N$ lies in $\operatorname{span}\{a_{N,1},\ldots,a_{N,r_N}\}$, so $M_N$ has finite rank. For every $f\in L^2(\Pi)$, Cauchy--Schwarz gives
\begin{align*}
\|(M-M_N)f\|_{2,\Pi}^2
&\mkern-5mu=\mkern-5mu\int\mkern-5mu\left|\int\mkern-5mu\sbra{m(z,z')\mkern-3mu-\mkern-3mu m_N(z,z')}\mkern-3mu f(z')\mkern-1mu\Pi(dz')\right|^2\mkern-5mu\Pi(dz)\mkern-5mu\leq\mkern-5mu\|m\mkern-3mu-\mkern-3mu m_N\|_{L^2(\Pi\otimes\Pi)}^2\mkern-3mu\|f\|_{2,\Pi}^2.
\end{align*}
Hence $\|M-M_N\|\to0$. For any sequence with $\sup_n\|f_n\|_{2,\Pi}\leq B$, finite dimensionality gives a convergent subsequence of $(M_Nf_n)$ for each $N$~\citep{kreyszig1978functional}. A diagonal subsequence makes $M_Nf_{n_h}$ converge for every fixed $N$, and
\begin{equation*}
\|Mf_{n_h}-Mf_{n_\ell}\|_{2,\Pi}
\leq2B\|M-M_N\|
+\|M_Nf_{n_h}-M_Nf_{n_\ell}\|_{2,\Pi}.
\end{equation*}
Choosing $N$ large and then $h,\ell$ large makes both terms arbitrarily small. Thus $(Mf_{n_h})$ is Cauchy in $L^2(\Pi)$, which is complete, and hence converges, proving compactness.

It remains to show that $\mathcal L_{\mathrm{move}}f=0$ holds only for constant functions. Suppose $\mathcal L_{\mathrm{move}}f=0$. Then $\langle f,\mathcal L_{\mathrm{move}}f\rangle_\Pi=0$, so \cref{eq:c-movement-energy-lower-bound} implies
\begin{equation}
f_j^d(x)=f_{j-1}^d(y)\qquad\widehat C_j\text{-almost everywhere}
\label{eq:c-average-zero-energy}
\end{equation}
for every $j,d$. If $f$ is not constant $\Pi$-almost surely, there is an $a\in\mathbb R$ such that $A:=\{z:f(z)>a\}$ satisfies $0<\Pi(A)<1$. But \cref{eq:c-average-zero-energy} then gives
\begin{equation*}
\sum_{j=1}^J\sum_{d\in\{\mathrm F,\mathrm B\}}
\iint\sbra{\mathbf1_{A_j^d}(x)-\mathbf1_{A_{j-1}^d}(y)}^2\widehat C_j(dx,dy)=0,
\end{equation*}
contradicting \cref{eq:c-movement-connectedness}. Since $\mathcal L_{\mathrm{move}}\mathbf1=c-M\mathbf1=0$, we conclude
\begin{equation*}
\ker\mathcal L_{\mathrm{move}}=\operatorname{span}\{\mathbf1\}.
\end{equation*}
Now \cref{lem:c-compact-movement-gap} applies with $1/2\leq c\leq1$. Hence there exists $\gamma_{\mathrm{CS}}>0$, reduced to at most one if necessary, such that
\begin{equation*}
\mathcal E_P(f,f)\geq\langle f,\mathcal L_{\mathrm{move}}f\rangle_\Pi
\geq\gamma_{\mathrm{CS}}\operatorname{Var}_\Pi(f).
\end{equation*}
Neither $P_{\mathrm{move}}$ nor $\Pi$ depends on $\chi$, so the same $\gamma_{\mathrm{CS}}$ works for every $\chi\in[0,1]$. Applying \cref{cor:c-cs-vgb-convergence} completes the proof.

\end{proof}

\paragraph{Sufficient conditions for a positive Poincar\'e constant.}
\Cref{thm:c-average-moment-poincare} guarantees convergence under the connectedness and integrability conditions in \cref{eq:c-movement-connectedness,eq:c-normalized-edge-integrability}. To see when these conditions hold for the kernels used by CS-VGB, we consider Gaussian renoising paired with its exact denoising reversal. The following result verifies both conditions using an exponential-moment bound on the clean distribution and terminal reward, together with bounds on the verifier values.

\begin{corollary_ap}[Ideal paired Gaussian transitions]
\label{cor:c-ideal-paired-gaussian-mixing}
Take $\mathcal X=\mathbb R^d$ and a fixed grid with ${J\geq2}$. Suppose the reference kernels and their marginals satisfy
\begin{equation}
\pi_j(dx)K_j^{\mathrm F}(dy\mid x)=\pi_{j-1}(dy)K_{j-1}^{\mathrm B}(dx\mid y),\qquad
K_{j-1}^{\mathrm B}(dx\mid y)=\mathcal N(dx;a_jy,\sigma_j^2I),
\label{eq:c-ideal-gaussian-pairing}
\end{equation}
where ${a_j\in\mathbb R}$ and ${0<\sigma_j<\infty}$ for ${j=1,\ldots,J}$. Write ${A_j:=\prod_{\ell=1}^j a_\ell}$ for the cumulative signal scale. Assume $r\geq0$ and $Z=\int r\,d\pi_0\in(0,\infty)$, and suppose that, for some $c>0$,
\begin{equation}
\mathbb E_{\pi_0}\sbra{(1+r(X_0))e^{c\|X_0\|^2}}<\infty,\qquad c\geq\max_{1\leq j\leq J}\frac{2A_j^2}{\sigma_j^2}.
\label{eq:c-ideal-gaussian-tail-moment}
\end{equation}
This permits unbounded clean support and controls the tails of both $\pi_0$ and $\pi_0^\star$. For some $v_{\min}\in(0,1]$, also assume
\begin{equation}
v_{\min}\leq\widehat V_j(x)\leq1\quad\pi_j\text{-almost surely},\qquad j=1,\ldots,J,\qquad \widehat V_0=r.
\label{eq:c-ideal-gaussian-verifier-range}
\end{equation}
Then the conditions in \cref{eq:c-movement-connectedness,eq:c-normalized-edge-integrability} hold. Therefore, there exists $\gamma_{\mathrm{CS}}\in(0,1]$, independent of the directional persistence parameter $\chi$, such that
\begin{equation}
\mathcal E_P(f,f)\geq\gamma_{\mathrm{CS}}\operatorname{Var}_\Pi(f),\qquad
\|\nu P_{\mathrm L}^n-\Pi\|_{\mathrm{TV}}\leq\frac{D_\nu}{2}e^{-\gamma_{\mathrm{CS}}n/4},
\label{eq:c-ideal-gaussian-exponential-mixing}
\end{equation}
for $f\in L^2(\Pi)$, $\chi\in[0,1]$, and every initial law $\nu\ll\Pi$ with $D_\nu:=\|d\nu/d\Pi-1\|_{2,\Pi}<\infty$. In particular, the initialization ${\nu_J=\delta_J\otimes\pi_J\otimes\delta_{\mathrm F}}$ satisfies
\begin{equation}
D_{\nu_J}^2\leq\frac{\calZ}{{v_{\min}}}-1<\infty.
\label{eq:c-ideal-gaussian-noise-start}
\end{equation}

\end{corollary_ap}

\begin{proof}
We verify \cref{eq:c-normalized-edge-integrability,eq:c-movement-connectedness}. By exact pairing, the reference path can be generated from $X_0\sim\pi_0$ by
\begin{equation*}
X_j=a_jX_{j-1}+\sigma_j\varepsilon_j.
\end{equation*}
Here the ${\varepsilon_j\sim\mathcal N(0,I)}$ are independent of one another and of $X_0$. Writing the cumulative signal scale ${A_j}$ and noise scale ${\tau_j}$ as
\begin{equation}
A_0=1,\qquad \tau_0^2=0,\qquad
A_j=a_jA_{j-1},\qquad \tau_j^2=a_j^2\tau_{j-1}^2+\sigma_j^2.
\label{eq:c-ideal-gaussian-cumulative-scales}
\end{equation}
gives ${X_j\mid X_0=u\sim\mathcal N(A_ju,\tau_j^2I)}$. Let $\varphi_v$ denote the density of $\mathcal N(0,vI)$. For ${j\geq1}$, define the marginal density of ${\pi_j}$ by
\begin{equation}
\pi_j(x):=\int\varphi_{\tau_j^2}(x-A_ju)\pi_0(du).
\label{eq:c-ideal-gaussian-mixture-density}
\end{equation}
To bound ${\mathcal I_j}$, we use ${K_{j-1}^{\mathrm B}(x\mid y):=\varphi_{\sigma_j^2}(x-a_jy)}$ for the density of ${K_{j-1}^{\mathrm B}(dx\mid y)}$. We show that the following integral ${H_j}$ is finite by the Gaussian calculation that follows:
\begin{equation}
H_j:=\int\pi_{j-1}(dy)\int\frac{K_{j-1}^{\mathrm B}(x\mid y)^2}{\pi_j(x)}\,dx
\leq\frac1{b_{j,v}}\int\pi_{j-1}(dy)\int\frac{\varphi_{\sigma_j^2}(x-a_jy)^2}{\varphi_v(x)}\,dx.
\label{eq:c-ideal-gaussian-reference-integral}
\end{equation}
The inequality uses the lower bound ${\pi_j(x)\geq b_{j,v}\varphi_v(x)}$ established in \cref{eq:c-ideal-gaussian-lower-envelope}.
Choose a finite $B$ such that $\pi_0(\{u:\|u\|\leq B\})>0$; such a ball exists because $\pi_0$ is a probability measure. For ${0<v<\tau_j^2}$ and $\|u\|\leq B$, expanding the square gives
\begin{align*}
\|x-A_ju\|^2
&=\|x\|^2-2A_j\langle x,u\rangle+A_j^2\|u\|^2 \overset{(\circ)}{\leq}\|x\|^2+\frac{\tau_j^2-v}{v}\|x\|^2
+\frac{vA_j^2}{\tau_j^2-v}\|u\|^2+A_j^2\|u\|^2\\
&=\frac{\tau_j^2}{v}\|x\|^2+\frac{\tau_j^2A_j^2}{\tau_j^2-v}\|u\|^2
\leq\frac{\tau_j^2}{v}\|x\|^2+\frac{\tau_j^2A_j^2B^2}{\tau_j^2-v}.
\end{align*}
Here $(\circ)$ uses ${2|A_j\langle x,u\rangle|\leq\epsilon\|x\|^2+\epsilon^{-1}A_j^2\|u\|^2}$ with ${\epsilon=(\tau_j^2-v)/v>0}$.
Restricting the mixture integral in \cref{eq:c-ideal-gaussian-mixture-density} to this ball gives
\begin{align}
\pi_j(x)&\geq\int_{\|u\|\leq B}\varphi_{\tau_j^2}(x-A_ju)\pi_0(du)\geq b_{j,v}\varphi_v(x),\notag\\
b_{j,v}&:=\pi_0(\{u:\|u\|\leq B\})\left(\frac{v}{\tau_j^2}\right)^{d/2}
\exp\left(-\frac{A_j^2B^2}{2(\tau_j^2-v)}\right)>0.
\label{eq:c-ideal-gaussian-lower-envelope}
\end{align}
We now evaluate the inner Gaussian integral in \cref{eq:c-ideal-gaussian-reference-integral}. For ${v>\sigma_j^2/2}$,
\begin{align*}
\int\frac{\varphi_{\sigma_j^2}(x-a_jy)^2}{\varphi_v(x)}\,dx
&=\frac{(2\pi v)^{d/2}}{(2\pi\sigma_j^2)^d}\int\exp\left(-\frac{\|x-a_jy\|^2}{\sigma_j^2}+\frac{\|x\|^2}{2v}\right)dx\\
&=\left(\frac{v}{\sigma_j\sqrt{2v-\sigma_j^2}}\right)^d
\exp\left(\frac{a_j^2\|y\|^2}{2v-\sigma_j^2}\right).
\end{align*}
Choose
\begin{equation}
v_j:=\tau_j^2-\frac{\sigma_j^2}{4},\qquad
\eta_j:=\frac{a_j^2}{2v_j-\sigma_j^2}.
\label{eq:c-ideal-gaussian-envelope-choice}
\end{equation}
The variance recursion implies
\begin{equation}
\frac{\sigma_j^2}{2}<v_j<\tau_j^2,\qquad
2v_j-\sigma_j^2=2a_j^2\tau_{j-1}^2+\frac{\sigma_j^2}{2},\qquad
2\eta_j\tau_{j-1}^2=\frac{2a_j^2\tau_{j-1}^2}{2a_j^2\tau_{j-1}^2+\sigma_j^2/2}<1.
\label{eq:c-ideal-gaussian-positive-denominators}
\end{equation}
Consequently,
\begin{align}
H_j
&\leq\frac{1}{b_{j,v_j}}
\left(\frac{v_j}{\sigma_j\sqrt{2v_j-\sigma_j^2}}\right)^d
\mathbb E_{\pi_{j-1}}[e^{\eta_j\|X_{j-1}\|^2}].
\label{eq:c-ideal-gaussian-integral-bound}
\end{align}
Recall that the moment-generating function of the noncentral chi-squared distribution gives, for $Y\sim\mathcal N(m,s^2I)$ and $2\eta s^2<1$,
\begin{equation*}
\mathbb E[e^{\eta\|Y\|^2}]=(1-2\eta s^2)^{-d/2}\exp\left(\frac{\eta\|m\|^2}{1-2\eta s^2}\right).
\end{equation*}
Applying this formula to ${X_{j-1}\mid X_0=u\sim\mathcal N(A_{j-1}u,\tau_{j-1}^2I)}$ and averaging over $X_0$ gives
\begin{align}
\mathbb E_{\pi_{j-1}}[e^{\eta_j\|X_{j-1}\|^2}]
&=\int\mathbb E\left[e^{\eta_j\|X_{j-1}\|^2}\mid X_0=u\right]\pi_0(du)\notag\\
&=(1-2\eta_j\tau_{j-1}^2)^{-d/2}
\mathbb E_{\pi_0}\left[
\exp\left(\frac{\eta_jA_{j-1}^2\|X_0\|^2}
{1-2\eta_j\tau_{j-1}^2}\right)\right]\notag\\
&\overset{(\circ)}{\leq}(1-2\eta_j\tau_{j-1}^2)^{-d/2}
\mathbb E_{\pi_0}[e^{c\|X_0\|^2}]\overset{(\triangle)}{<}\infty.
\label{eq:c-ideal-gaussian-exponential-moment}
\end{align}
Here $(\circ)$ follows from \cref{eq:c-ideal-gaussian-envelope-choice} and the variance recursion, which give
\begin{equation*}
\frac{\eta_jA_{j-1}^2}{1-2\eta_j\tau_{j-1}^2}
=\frac{a_j^2A_{j-1}^2}{2v_j-\sigma_j^2-2a_j^2\tau_{j-1}^2}
=\frac{2A_j^2}{\sigma_j^2}\leq c,
\end{equation*}
and $(\triangle)$ uses $2\eta_j\tau_{j-1}^2<1$ and \cref{eq:c-ideal-gaussian-tail-moment}. This also holds for ${j=1}$, with $\tau_0=0$, and proves ${H_j<\infty}$. For the reward-tilted clean law,
\begin{equation*}
\mathbb E_{\pi_0^\star}[e^{c\|X_0\|^2}]=\frac1Z\mathbb E_{\pi_0}[r(X_0)e^{c\|X_0\|^2}]<\infty.
\end{equation*}
Repeating the calculation with $\pi_0^\star$ as the initial law gives $\pi_j^\star(x)\geq b_{j,v}^\star\varphi_v(x)$ for some $b_{j,v}^\star>0$. Together with the exponential moment above, this proves the same integrability bound with $\pi_{j-1}$ and $\pi_j$ replaced by $\pi_{j-1}^\star$ and $\pi_j^\star$, respectively.

Recalling the renoising joint edge law ${J_j^{\mathrm B}}$ from \cref{eq:cs-vgb-edge-laws}, exact pairing and the verifier bounds give, for ${j\geq2}$,
\begin{equation}
{v_{\min}}J_j^{\mathrm B}\leq\widehat C_j\leq J_j^{\mathrm B},\qquad
J_j^{\mathrm B}(dx,dy)=\pi_{j-1}(dy)K_{j-1}^{\mathrm B}(x\mid y)dx.
\label{eq:c-ideal-gaussian-edge-comparison}
\end{equation}
Since $\widehat C_1(\mathcal X,\mathcal X)=Z$ and ${\widehat C_j(\mathcal X,\mathcal X)\leq1}$ for ${j\geq2}$,
\begin{equation*}
0<\calZ=3Z+4\sum_{j=2}^{J-1}\widehat C_j(\mathcal X,\mathcal X)+3\widehat C_J(\mathcal X,\mathcal X)\leq(4J-2)\max\{1,Z\}<\infty.
\end{equation*}
Each incident measure contains the corresponding marginal of ${\widehat C_j}$, so \cref{eq:c-ideal-gaussian-edge-comparison} gives
\begin{align*}
\mu_j(dx)&\geq\widehat C_j(dx,\mathcal X)\geq {v_{\min}}J_j^{\mathrm B}(dx,\mathcal X)={v_{\min}}\pi_j(dx),\\
\mu_{j-1}(dy)&\geq\widehat C_j(\mathcal X,dy)\geq {v_{\min}}J_j^{\mathrm B}(\mathcal X,dy)={v_{\min}}\pi_{j-1}(dy).
\end{align*}
For ${j\geq2}$, we use ${\widehat C_j(x,y)}$ and ${\mu_j(x)}$ for the corresponding Lebesgue densities. By definition,
\begin{equation*}
h_j(x,y)\mkern-3mu=\mkern-3mu\frac{\widehat C_j(x,y)}{\mu_j(x)\mu_{j-1}(y)},\mkern8mu
\mathcal I_j\mkern-3mu=\mkern-3mu\iint h_j(x,y)^2\mu_j(x)\mu_{j-1}(y)\,dx\,dy
\mkern-3mu=\mkern-3mu\iint\frac{\widehat C_j(x,y)^2}{\mu_j(x)\mu_{j-1}(y)}\,dx\,dy.
\end{equation*}
For ${j\geq2}$, using ${\widehat C_j(x,y)=\pi_{j-1}(y)K_{j-1}^{\mathrm B}(x\mid y){\sqrt{\widehat V_j(x)}}{\sqrt{\widehat V_{j-1}(y)}}\leq\pi_{j-1}(y)K_{j-1}^{\mathrm B}(x\mid y)}$ and the two marginal bounds gives
\begin{equation}
\mathcal I_j
\leq\mkern-5mu\iint\frac{\pi_{j-1}(y)^2K_{j-1}^{\mathrm B}(x\mid y)^2}{{v_{\min}^2}\pi_j(x)\pi_{j-1}(y)}\,dx\,dy
\mkern-3mu=\mkern-3mu {v_{\min}^{-2}}\int\pi_{j-1}(dy)\int\frac{K_{j-1}^{\mathrm B}(x\mid y)^2}{\pi_j(x)}\,dx
\mkern-3mu=\mkern-3mu {v_{\min}^{-2}}H_j<\infty.
\label{eq:c-ideal-gaussian-interior-integrability}
\end{equation}
For the clean edge, let $\pi_1^\star(x):=\int K_0^{\mathrm B}(x\mid y)\pi_0^\star(dy)$.
Then
\begin{equation*}
\widehat C_1(dx,dy)=Z\pi_0^\star(dy)K_0^{\mathrm B}(x\mid y)dx,\qquad
\mu_0=Z\pi_0^\star,\qquad \mu_1(dx)\geq Z\pi_1^\star(x)dx.
\end{equation*}
Substituting these identities in the same normalized integral yields
\begin{equation}
\mathcal I_1\leq
\int\pi_0^\star(dy)\int\frac{K_0^{\mathrm B}(x\mid y)^2}{\pi_1^\star(x)}\,dx<\infty.
\label{eq:c-ideal-gaussian-clean-integrability}
\end{equation}
The last inequality follows from the same Gaussian integral bound with clean law $\pi_0^\star$. Thus \cref{eq:c-normalized-edge-integrability} holds for every edge.

To verify \cref{eq:c-movement-connectedness}, note that reward tilting replaces $\pi_0$ by $\pi_0^\star$ without changing the noising kernels. Hence
\begin{equation}
\pi_{0:J}^\star(X_0\in du,X_J\in dx)
=\pi_0^\star(du)\varphi_{\tau_J^2}(x-{A_Ju})dx.
\label{eq:c-ideal-gaussian-endpoint-joint}
\end{equation}
The Gaussian density $\varphi_{\tau_J^2}(x-A_Ju)$ is positive for every $u,x$, and hence so is the marginal density $\pi_J^\star(x)$. Thus this joint law and ${\pi_0^\star\otimes\pi_J^\star}$ have the same null sets. If ${u_0(X_0)=u_J(X_J)}$ under the joint law, the equality also holds under the product law. For square-integrable ${u_0,u_J}$, integration under that product gives
\begin{align*}
0=\mathbb E_{\pi_0^\star\otimes\pi_J^\star}\sbra{(u_0(X_0)-u_J(X_J))^2}&=\mathbb E_{\pi_0^\star}[u_0^2]+\mathbb E_{\pi_J^\star}[u_J^2]-2\mathbb E_{\pi_0^\star}[u_0]\mathbb E_{\pi_J^\star}[u_J]\\
&=\operatorname{Var}_{\pi_0^\star}(u_0)+\operatorname{Var}_{\pi_J^\star}(u_J)+\bigl(\mathbb E_{\pi_0^\star}[u_0]-\mathbb E_{\pi_J^\star}[u_J]\bigr)^2.
\end{align*}
Since all three terms are nonnegative, this implies $\operatorname{Var}_{\pi_0^\star}(u_0)=0$, so \cref{eq:c-no-common-endpoint-function} holds.

Moreover, ${V_j(x)=Z\pi_j^\star(x)/\pi_j(x)\in(0,\infty)}$ for ${j\geq1}$, where ${\pi_j^\star(x)}$ is the Gaussian-mixture density obtained from $\pi_0^\star$. The verifier bounds therefore imply \cref{eq:c-positive-edge-distortion}, and \cref{lem:c-endpoint-connectedness} gives \cref{eq:c-movement-connectedness}.

Both required conditions now hold, so \cref{thm:c-average-moment-poincare} gives \cref{eq:c-ideal-gaussian-exponential-mixing}. Finally, recalling the incident measures and stationary law from \cref{eq:c-incident-measures,eq:c-stationary-law}, at the noise boundary we have
\begin{equation*}
\Pi(J,dx,\mathrm F)=\frac{\mu_J(dx)}{\calZ}=\frac{\widehat C_J(dx,\mathcal X)}{\calZ}=\frac{{\sqrt{\widehat V_J(x)}}F_J(x)}{\calZ}\pi_J(dx).
\end{equation*}
Since ${J\geq2}$, both factors in the numerator are at least {$\sqrt{v_{\min}}$}. The density ${d\nu_J/d\Pi}$ equals ${\calZ/({\sqrt{\widehat V_J}}F_J)}$ at ${(J,x,\mathrm F)}$ and zero elsewhere. Since ${\int(d\nu_J/d\Pi)\,d\Pi=1}$,
\begin{align}
D_{\nu_J}^2
&=\int\left(\frac{d\nu_J}{d\Pi}-1\right)^2d\Pi
=\int\left(\frac{d\nu_J}{d\Pi}\right)^2d\Pi-1\notag\\
&=\int\mkern-5mu\left(\frac{\calZ}{{\sqrt{\widehat V_J(x)}}F_J(x)}\right)^2\frac{{\sqrt{\widehat V_J(x)}}F_J(x)}{\calZ}\pi_J(dx)-1
\notag\\
&=\calZ\int\frac{\pi_J(dx)}{{\sqrt{\widehat V_J(x)}}F_J(x)}-1
\mkern-4mu\leq\mkern-4mu\frac{\calZ}{{v_{\min}}}-1<\infty.
\label{eq:c-ideal-gaussian-noise-start-calculation}
\end{align}

\end{proof}

We remark that \cref{cor:c-ideal-paired-gaussian-mixing} establishes connectedness under Gaussian pairing. Here we show that bounded clean support also ensures the endpoint correlation bound in \cref{eq:c-endpoint-correlation}, which is used in the explicit rate bound below.
\begin{lemma_ap}[Endpoint correlation with bounded clean support]
\label{lem:c-bounded-gaussian-endpoint-correlation}

Assume the Gaussian pairing in \cref{eq:c-ideal-gaussian-pairing}, $\|X_0\|\le R<\infty$ almost surely, and $0<Z=\mathbb E_{\pi_0}r(X_0)<\infty$ for $r\ge0$. Then the clean and noise endpoints have maximal correlation $\rho_{\mathrm{end}}<1$ under $\pi_{0:J}^\star$ (\cref{eq:c-endpoint-correlation}).
\end{lemma_ap}
\begin{proof}
Exact pairing gives $X_J=A_JX_0+\tau_J\xi$, where $\xi\sim\mathcal N(0,I)$ is independent of $X_0$, $A_J=\prod_{j=1}^J a_j$, and $\tau_J^2=\sum_{k=1}^J\sigma_k^2 \sbra{\prod_{\ell=k+1}^J a_\ell^2}>0$. Weighting by $r(X_0)/Z$ changes only the clean law, so this conditional Gaussian law also holds under $\pi_{0:J}^\star$. For every $\|x\|\le R$, its density satisfies
\begin{equation*}
\varphi_{\tau_J^2}(z-A_Jx)\ge h(z):=(2\pi\tau_J^2)^{-d/2}\exp\!\left[-\frac{(\|z\|+|A_J|R)^2}{2\tau_J^2}\right].
\end{equation*}
Since $(\|z\|+|A_J|R)^2\ge\|z\|^2$, we have $0<h(z)\le\varphi_{\tau_J^2}(z)$, so $\epsilon:=\int h(z)\,dz\in(0,1]$. By centering and rescaling, we may assume without loss of generality that $U=u(X_0)$ and $V=v(X_J)$ have mean zero and variance one under $\pi_{0:J}^\star$; zero-variance cases satisfy the correlation bound trivially. The density bound gives
\begin{align*}
2-2\mathbb E[UV]
&\mkern-3mu=\mkern-3mu\mathbb E[(U-V)^2]
\mkern-3mu\ge\mkern-3mu\int h(z)\int (u(x)-v(z))^2\,\pi_0^\star(dx)\,dz
\mkern-3mu\overset{(\circ)}{=}\mkern-3mu\int h(z)(1+v(z)^2)\,dz\mkern-3mu\ge\mkern-3mu\epsilon,
\end{align*}
where $(\circ)$ uses $\int u\,d\pi_0^\star=0$ and $\int u^2\,d\pi_0^\star=1$. Applying the same bound to $-V$ yields $|\mathbb E[UV]|\le1-\epsilon/2$. Taking the supremum over these normalized functions gives $\rho_{\mathrm{end}}\le1-\epsilon/2<1$.
\end{proof}

\paragraph{An explicit bound under bounded reward and verifier weights.}
\Cref{thm:c-average-moment-poincare} establishes a positive Poincar\'e constant under its connectivity and integrability assumptions. We now derive an explicit constant by additionally bounding the reward and implemented verifier weights above and away from zero. These bounds concern the ranges of the weights, not the accuracy of ${ \widehat V_j/V_j}$.

\begin{corollary_ap}[An explicit Poincar\'e bound under bounded weights]
\label{prop:c-explicit-lifted-poincare}
Let $P$ be the exact CS-VGB one-step kernel in \cref{thm:c-average-moment-poincare}, with ${J\geq2}$, terminal anchoring $\widehat V_0=r$, and the continuation recursion in \cref{eq:c-mixing-verifier-recursion}. Suppose that, for constants $r_{\min},v_{\min}\in(0,1]$ and ${j=1,\ldots,J}$,
\begin{equation}
r_{\min}\leq r(x)\leq1\quad\pi_0\text{-almost surely},\qquad
v_{\min}\leq\widehat V_j(x)\leq1\quad\pi_j\text{-almost surely}.
\label{eq:c-quantitative-weight-bounds}
\end{equation}
Assume also that the reward-tilted reference path satisfies
\begin{equation*}
\left|\operatorname{Cov}_{\pi_{0:J}^\star}
\bigl(u(X_0),v(X_J)\bigr)\right|
\leq\rho_{\mathrm{end}}
\sqrt{\operatorname{Var}_{\pi_0^\star}(u)
\operatorname{Var}_{\pi_J^\star}(v)},
\qquad 0\leq\rho_{\mathrm{end}}<1,
\end{equation*}
for all $u\in L^2(\pi_0^\star)$ and ${v\in L^2(\pi_J^\star)}$, as in \cref{eq:c-endpoint-correlation}. Under these assumptions, we can choose the same Poincar\'e constant for every directional persistence parameter $\chi\in[0,1]$ as follows:
\begin{equation}
\mathcal E_P(f,f)\geq\gamma_{\mathrm{CS}}\operatorname{Var}_\Pi(f),
\qquad
\gamma_{\mathrm{CS}}
:=\frac{(1-\rho_{\mathrm{end}})r_{\min}^2{v_{\min}^2}}{32J^2}>0,
\qquad f\in L^2(\Pi).
\label{eq:c-explicit-lifted-rate}
\end{equation}
Consequently, for every $\nu\ll\Pi$ with
$ D_\nu:=\|d\nu/d\Pi-1\|_{2,\Pi}<\infty$,
\begin{equation}
\|\nu P_{\mathrm L}^n-\Pi\|_{\mathrm{TV}}
\leq\frac{D_\nu}{2}
\exp\left[-\frac{(1-\rho_{\mathrm{end}})r_{\min}^2{v_{\min}^2}}{128J^2}n\right].
\label{eq:c-average-explicit-exponential-rate}
\end{equation}
The corresponding clean-conditional bound follows from \cref{eq:c-cs-vgb-clean-convergence}.
\end{corollary_ap}

\begin{proof}
Recall that ${\pi_j^\star(dx)=V_j(x)\pi_j(dx)/Z}$ is the reward-tilted marginal at checkpoint index ${j}$. We use these marginals to define a comparison law $\Pi_\star$ on the full state ${(j,x,d)}$. Write ${f_j^d(x):=f(j,x,d)}$, identifying the boundary copies as
\begin{equation*}
f_0^{\mathrm F}=f_0^{\mathrm B}=f_0:=f(0,\cdot,\mathrm B),\qquad
f_J^{\mathrm F}=f_J^{\mathrm B}=f_J:=f(J,\cdot,\mathrm F).
\end{equation*}
Recall that ${\Pi(j,dx,d)=\mu_j(dx)/\calZ}$, where ${\mu_j}$ sums the incident marginals of ${\widehat C_j}$. Similarly, we define the reference incident measures using ${C_j}$, whose marginals are ${Z\pi_j^\star}$ and ${Z\pi_{j-1}^\star}$:
\begin{align*}
\mu_{\star,0}(dx)&:=C_1(\mathcal X,dx)=Z\pi_0^\star(dx),\qquad
\mu_{\star,J}(dx):=C_J(dx,\mathcal X)=Z\pi_J^\star(dx),\\
\mu_{\star,j}(dx)&:=C_j(dx,\mathcal X)+C_{j+1}(\mathcal X,dx)
=2Z\pi_j^\star(dx),\qquad 1\leq j<J.
\end{align*}
Thus ${\mu_{\star,j}=Z\omega_j\pi_j^\star}$, where ${\omega_0=\omega_J=1}$ and ${\omega_j=2}$ in the interior. Counting both interior directions gives
\begin{equation*}
\calZ_\star:=\mu_{\star,0}(\mathcal X)+2\sum_{j=1}^{J-1}\mu_{\star,j}(\mathcal X)+\mu_{\star,J}(\mathcal X)
=Z+4(J-1)Z+Z=(4J-2)Z.
\end{equation*}
Normalizing these incident measures, with one direction at each boundary, yields
\begin{equation}
\Pi_\star(j,dx,d):=\frac{\mu_{\star,j}(dx)}{\calZ_\star}
=\frac{Z\omega_j\pi_j^\star(dx)}{(4J-2)Z}
=\frac{\omega_j}{4J-2}\pi_j^\star(dx).
\label{eq:c-balanced-movement-reference}
\end{equation}
Recall from \cref{eq:c-movement-energy-lower-bound,eq:c-outer-dirichlet-form} that the time-step contribution to $\mathcal E_P$ sums squared changes across ${\widehat C_j}$ with factor $1/(2\calZ)$. Similarly, we define $\mathcal E_\star$ using ${C_j}$ and $\calZ_\star$. With ${\overline C_j=C_j/Z}$, write
\begin{equation}
S(f):=\sum_{j=1}^J\sum_{d\in\{\mathrm F,\mathrm B\}}
\mathbb E_{\overline C_j}\left[(f_j^d(X_j)-f_{j-1}^d(X_{j-1}))^2\right],
\label{eq:c-path-energy}
\end{equation}
so that
\begin{equation}
\mathcal E_\star(f,f):=\frac1{2\calZ_\star}\sum_{j=1}^J\sum_{d\in\{\mathrm F,\mathrm B\}}
\iint[f_j^d(x)-f_{j-1}^d(y)]^2C_j(dx,dy)
=\frac{Z}{2\calZ_\star}S(f)=\frac{S(f)}{2(4J-2)}.
\label{eq:c-reference-edge-energy}
\end{equation}
First compare the measures. The recursion ${V_j=K_j^{\mathrm F}V_{j-1}}$ with $V_0=r$ gives ${r_{\min}\leq V_j\leq1}$. For ${j\geq2}$,
\begin{equation*}
{v_{\min}}\leq
\frac{d\widehat C_j}{dC_j}(x,y)
=\frac{{\sqrt{\widehat V_j(x)}}{\sqrt{\widehat V_{j-1}(y)}}}{V_{j-1}(y)}
\leq r_{\min}^{-1}.
\end{equation*}
Set $a:={v_{\min}}$ and $b:=r_{\min}^{-1}$. Since $\widehat C_1=C_1$ and $a\leq1\leq b$,
\begin{equation}
aC_j\leq\widehat C_j\leq bC_j,\qquad 1\leq j\leq J.
\label{eq:c-bounded-weight-edge-comparison}
\end{equation}
The marginals of ${C_j}$ are ${Z\pi_j^\star}$ and ${Z\pi_{j-1}^\star}$. For ${1\leq j<J}$, summing the two incident marginals gives
\begin{equation*}
\mu_j(dx)\mkern-2mu=\mkern-2mu\widehat C_j(dx,\mathcal X)+\widehat C_{j+1}(\mathcal X,dx)
\mkern-2mu\geq\mkern-2mu a\bigl[C_j(dx,\mathcal X)+C_{j+1}(\mathcal X,dx)\bigr]
\mkern-2mu=\mkern-2mu2aZ\pi_j^\star(dx).
\end{equation*}
The upper bound follows in the same way with $b$. Each boundary has only one incident edge. Thus, for ${0\leq j\leq J}$, these bounds and their sums over time-direction states yield
\begin{equation}
aZ\omega_j\pi_j^\star\mkern-2mu\leq\mkern-2mu\mu_j\mkern-2mu\leq\mkern-2mu bZ\omega_j\pi_j^\star,\qquad
{a\calZ_\star\mkern-2mu\leq\mkern-2mu\calZ\mkern-2mu\leq\mkern-2mu b\calZ_\star,\qquad}
\frac ab\Pi_\star\mkern-2mu\leq\mkern-2mu\Pi\mkern-2mu\leq\mkern-2mu\frac ba\Pi_\star.
\label{eq:c-bounded-weight-stationary-comparison}
\end{equation}
Thus every $f\in L^2(\Pi)$ also belongs to $L^2(\Pi_\star)$, and $\operatorname{Var}_\Pi(f)\leq(b/a)\operatorname{Var}_{\Pi_\star}(f)$.

We now bound the reference variance. All following expectations, variances, and covariances are under ${\pi_{0:J}^\star}$. The endpoint correlation bound gives
\begin{align*}
\mathbb E[(f_J(X_J)-f_0(X_0))^2]
&=\operatorname{Var}(f_0(X_0))+\operatorname{Var}(f_J(X_J))\mkern-3mu-\mkern-3mu2\operatorname{Cov}(f_0(X_0),f_J(X_J))\\
&\quad+(\mathbb E[f_J(X_J)]-\mathbb E[f_0(X_0)])^2\\
&\geq\operatorname{Var}(f_0(X_0))+\operatorname{Var}(f_J(X_J))
-\mkern-3mu 2\rho_{\mathrm{end}}\sqrt{\operatorname{Var}(f_0(X_0))\operatorname{Var}(f_J(X_J))}\\
&\geq(1-\rho_{\mathrm{end}})[\operatorname{Var}(f_0(X_0))+\operatorname{Var}(f_J(X_J))]
\mkern-3mu\geq\mkern-3mu(1-\rho_{\mathrm{end}})\operatorname{Var}(f_0(X_0)).
\end{align*}
The last step uses $2\sqrt{uv}\leq u+v$. For either direction, telescoping along the path and applying Cauchy--Schwarz yields
\begin{align*}
\mathbb E[(f_j^d(X_j)-f_0(X_0))^2]
&=\mathbb E\left[\left(\sum_{\ell=1}^j[f_\ell^d(X_\ell)-f_{\ell-1}^d(X_{\ell-1})]\right)^2\right]\\
&\leq j\sum_{\ell=1}^j\mathbb E[(f_\ell^d(X_\ell)-f_{\ell-1}^d(X_{\ell-1}))^2]\leq JS(f).
\end{align*}
At ${j=J}$, this gives ${\operatorname{Var}(f_0(X_0))\leq\mathbb E[(f_J(X_J)-f_0(X_0))^2]/(1-\rho_{\mathrm{end}})\leq JS(f)/(1-\rho_{\mathrm{end}})}$. Consequently, for every available ${(j,d)}$,
\begin{align}
\mathbb E[(f_j^d(X_j)-\mathbb E[f_0(X_0)])^2]
&\leq2\mathbb E[(f_j^d(X_j)-f_0(X_0))^2]+2\operatorname{Var}(f_0(X_0))\notag\\
&\leq2JS(f)+\frac{2JS(f)}{1-\rho_{\mathrm{end}}}
\leq\frac{4JS(f)}{1-\rho_{\mathrm{end}}}.
\label{eq:c-slice-variance-control}
\end{align}
Observing that
\begin{equation*}
\Pi_\star[(f-\mathbb E[f_0(X_0)])^2]=\frac1{4J-2}\sum_{j=0}^J\omega_j\sum_d\mathbb E_{\pi_j^\star}\!\left[(f_j^d(X_j)-\mathbb E[f_0(X_0)])^2\right],
\end{equation*}
where $d$ ranges over both directions at interior checkpoints and only the prescribed direction at each boundary, we apply \cref{eq:c-slice-variance-control} to each term. The weights sum to ${1+4(J-1)+1=4J-2}$, so
\begin{equation}
\operatorname{Var}_{\Pi_\star}(f)\leq\Pi_\star[(f-\mathbb E[f_0(X_0)])^2]
\leq\frac{4JS(f)}{1-\rho_{\mathrm{end}}}
=\frac{8J(4J-2)}{1-\rho_{\mathrm{end}}}\mathcal E_\star(f,f)
\leq\frac{32J^2}{1-\rho_{\mathrm{end}}}\mathcal E_\star(f,f).
\label{eq:c-balanced-reference-gap}
\end{equation}
Finally, retain only the moves between adjacent checkpoints in $\mathcal E_P$. By \cref{eq:c-bounded-weight-edge-comparison,eq:c-bounded-weight-stationary-comparison},
\begin{align}
\mathcal E_P(f,f)
&\overset{(\circ)}{\geq}\frac1{2\calZ}\sum_{j=1}^J\sum_{d\in\{\mathrm F,\mathrm B\}}
\iint[f_j^d(x)-f_{j-1}^d(y)]^2\widehat C_j(dx,dy)\notag\\
&\geq\frac{aZ}{2\calZ}S(f)
=\frac{a\calZ_\star}{\calZ}\mathcal E_\star(f,f)
\geq\frac ab\mathcal E_\star(f,f).
\label{eq:c-bounded-weight-energy-comparison}
\end{align}
Here $(\circ)$ follows from \cref{eq:c-movement-energy-lower-bound} by dropping nonnegative flip terms; stays contribute zero. Combining this with \cref{eq:c-balanced-reference-gap} and the variance bound above gives
\begin{equation*}
\mathcal E_P(f,f)\geq
\frac ab\frac{1-\rho_{\mathrm{end}}}{32J^2}\operatorname{Var}_{\Pi_\star}(f)
\geq\left(\frac ab\right)^2\frac{1-\rho_{\mathrm{end}}}{32J^2}\operatorname{Var}_\Pi(f)
=\frac{(1-\rho_{\mathrm{end}})r_{\min}^2{v_{\min}^2}}{32J^2}\operatorname{Var}_\Pi(f).
\end{equation*}
The bound is independent of $\chi$. Since $0<\gamma_{\mathrm{CS}}\leq1$, \cref{cor:c-cs-vgb-convergence} gives \cref{eq:c-average-explicit-exponential-rate}.
\end{proof}

If the clean and noise endpoints are independent under the unweighted reference path, weighting by $r(X_0)/Z$ preserves their independence, so $\rho_{\mathrm{end}}=0$. In this case, \cref{eq:c-explicit-lifted-rate} gives ${\gamma_{\mathrm{CS}}=r_{\min}^2{v_{\min}^2}/(32J^2)}$.

The convergence bound in \cref{eq:c-average-explicit-exponential-rate} also applies to the standard noise initialization ${\nu_J:=\delta_J\otimes\pi_J\otimes\delta_{\mathrm F}}$. Indeed, \cref{eq:c-quantitative-weight-bounds} gives {${\sqrt{\widehat V_J}\geq\sqrt{v_{\min}}}$ and ${F_J=K_J^{\mathrm F}\sqrt{\widehat V_{J-1}}\geq\sqrt{v_{\min}}}$} for ${J\geq2}$, so the same calculation as in \cref{eq:c-ideal-gaussian-noise-start-calculation} yields
\begin{equation}
D_{\nu_J}^2
=\int\left(\frac{d\nu_J}{d\Pi}-1\right)^2d\Pi
=\calZ\int_{\mathcal X}
 \frac{\pi_J(dx)}{{\sqrt{\widehat V_J(x)}}F_J(x)}-1
\leq\frac{\calZ}{{v_{\min}}}-1<\infty.
\label{eq:c-noise-start-l2-bound}
\end{equation}

\paragraph{Effect of process-verifier error.}
\label{app:subsec:cs-vgb-verifier-error}

The preceding convergence bounds do not explicitly quantify the error in $\widehat V_j$ relative to the true process verifier $V_j$. We now incorporate this error into a lower bound on the Poincar\'e constant under an average verifier-error assumption.

Let $P^V$ be the exact CS-VGB kernel with ${\widehat V_j}$ replaced by the true process verifier ${V_j}$ in \cref{eq:c-mixing-verifier-recursion}, and let $\Pi^V$ be its stationary law. Both chains use the same reference transitions and terminal anchoring $\widehat V_0=V_0=r$.
The corresponding exact-verifier edge-conductance measures are $C_1^V:=C_1$ and, for ${2\leq j\leq J}$,
\begin{equation}
{C_j^V(dx,dy):=\pi_j(dx)K_j^{\mathrm F}(dy\mid x)\sqrt{V_j(x)V_{j-1}(y)}
=\sqrt{\frac{V_j(x)}{V_{j-1}(y)}}C_j(dx,dy).}
\label{eq:c-exact-verifier-conductances}
\end{equation}
To quantify the discrepancy from this exact-verifier construction, we measure the relative error in ${\widehat V_j}$ in mean square under the reward-tilted marginal ${\pi_j^\star}$:
\begin{equation}
\varepsilon_{\mathrm{av}}:=\max_{1\leq j\leq J}
\left(\mathbb E_{\pi_j^\star}\left[\left(\frac{\widehat V_j(X_j)}{V_j(X_j)}-1\right)^2\right]\right)^{1/2}.
\label{eq:c-average-verifier-error}
\end{equation}
To relate this average error to the convergence rate, we also impose an integrability condition on the exact-verifier transitions. As in \cref{eq:c-average-movement-operator}, let $P_{\mathrm{move}}^V$ retain only moves between adjacent checkpoints and define
\begin{equation}
M^V:=\frac{P_{\mathrm{move}}^V+(P_{\mathrm{move}}^V)^*}{2},\qquad
M^V(z,dz')=m^V(z,z')\Pi^V(dz'),
\label{eq:c-oracle-movement-density}
\end{equation}
where the adjoint is taken in $L^2(\Pi^V)$. For each finite $p>2$, define
\begin{equation}
H_p:=\left(\iint m^V(z,z')^p\,\Pi^V(dz)\Pi^V(dz')\right)^{1/p}.
\label{eq:c-oracle-density-moment}
\end{equation}
\Cref{lem:c-gaussian-higher-moments} gives a sufficient condition for $H_p<\infty$ under paired Gaussian transitions. We assume this moment is finite for at least one such $p$ and minimize the error bound over all exponents satisfying this condition:
\begin{equation}
\Delta_{\mathrm{av}}:=\frac{6}{{r_{\min}v_{\min}^2}}
\inf_{\substack{2<p<\infty\\H_p<\infty}}
\left\{\frac{H_p}{r_{\min}^{1/p}}\varepsilon_{\mathrm{av}}^{\,1-2/p}\right\}.
\label{eq:c-optimized-average-verifier-error}
\end{equation}

\begin{theorem_ap}[Average verifier error and the Poincar\'e constant]
\label{thm:c-average-verifier-gap}
Let $P$ be the exact CS-VGB one-step kernel with ${J\geq2}$. Assume the reward and verifier bounds in \cref{eq:c-quantitative-weight-bounds}, the endpoint correlation bound in \cref{eq:c-endpoint-correlation}, and $H_p<\infty$ in \cref{eq:c-oracle-density-moment} for at least one finite $p>2$. For every directional persistence parameter $\chi\in[0,1]$ and every $f\in L^2(\Pi)$, we have
\begin{equation*}
\mathcal E_P(f,f)\geq\gamma_{\mathrm{CS}}\operatorname{Var}_\Pi(f),\qquad
\gamma_V:=\frac{(1-\rho_{\mathrm{end}})r_{\min}}{32J^2}.
\end{equation*}
Here $\gamma_V$ is the exact-verifier lower bound, and $\gamma_{\mathrm{CS}}$ can be chosen independently of $\chi$ to satisfy
\begin{equation}
\gamma_{\mathrm{CS}}
\geq\max\left\{\gamma_Vr_{\min}{v_{\min}^2},\ \gamma_V-\Delta_{\mathrm{av}}\right\}
\geq\gamma_V\left(1+\frac{\Delta_{\mathrm{av}}}{\gamma_Vr_{\min}{v_{\min}^2}}\right)^{-1}
>0.
\label{eq:c-average-verifier-gap}
\end{equation}
The convergence and clean-conditional bounds in \cref{cor:c-cs-vgb-convergence} apply with $\gamma=\gamma_{\mathrm{CS}}$; the clean stationary law remains $\pi_0^\star$.
\end{theorem_ap}

\begin{proof}
{
Use a superscript $V$ also for the exact-verifier incident measures and normalizer. The reward bounds and the recursion for ${V_j}$ give ${r_{\min}\leq V_j\leq1}$, so
\begin{equation*}
\sqrt{r_{\min}}\,C_j\leq C_j^V\leq r_{\min}^{-1/2}C_j,\qquad 1\leq j\leq J.
\end{equation*}
Summing incident marginals and normalizing therefore gives
\begin{equation*}
\sqrt{r_{\min}}\,\calZ_\star\leq\calZ^V\leq r_{\min}^{-1/2}\calZ_\star,
\qquad
\Pi^V\leq\frac{\calZ_\star}{\sqrt{r_{\min}}\,\calZ^V}\Pi_\star.
\end{equation*}
Since the mean minimizes the expected squared deviation,
\begin{equation*}
\operatorname{Var}_{\Pi^V}(f)
\leq\frac{\calZ_\star}{\sqrt{r_{\min}}\,\calZ^V}\operatorname{Var}_{\Pi_\star}(f).
\end{equation*}
Let $\mathcal L_{\mathrm{move}}^V=cI-M^V$, where $c=M^V\mathbf1$ equals $1/2$ in the interior and $1$ at the boundaries. Applying the lower edge-conductance measure bound to the movement energy gives
\begin{align*}
\langle f,\mathcal L_{\mathrm{move}}^Vf\rangle_{\Pi^V}
&=\frac1{2\calZ^V}\sum_{j=1}^J\sum_{d\in\{\mathrm F,\mathrm B\}}
\iint[f_j^d(x)-f_{j-1}^d(y)]^2\,C_j^V(dx,dy)\\
&\geq\frac{\sqrt{r_{\min}}\,\calZ_\star}{\calZ^V}\mathcal E_\star(f,f)
\geq\frac{(1-\rho_{\mathrm{end}})r_{\min}}{32J^2}\operatorname{Var}_{\Pi^V}(f),
\end{align*}
where the last step uses \cref{eq:c-balanced-reference-gap} and the preceding variance comparison. Thus
\begin{equation}
\langle f,\mathcal L_{\mathrm{move}}^Vf\rangle_{\Pi^V}
\geq\gamma_V\operatorname{Var}_{\Pi^V}(f),
\qquad
\Pi^V\leq r_{\min}^{-1}\Pi_\star.
\label{eq:c-exact-verifier-movement-gap}
\end{equation}
}

Write ${\Delta(j,x,d):=\widehat V_j(x)/V_j(x)-1}$, with $\Delta=0$ at ${j=0}$, and fix $p>2$ with $H_p<\infty$. Since $\Pi_\star$ averages the time marginals ${\pi_j^\star}$, \cref{eq:c-exact-verifier-movement-gap,eq:c-average-verifier-error} gives
\begin{equation*}
\|\Delta\|_{2,\Pi^V}^2
\leq r_{\min}^{-1}\|\Delta\|_{2,\Pi_\star}^2
\leq r_{\min}^{-1}\max_{1\leq j\leq J}
\mathbb E_{\pi_j^\star}\!\left[\left(\frac{\widehat V_j}{V_j}-1\right)^2\right]
=r_{\min}^{-1}\varepsilon_{\mathrm{av}}^2.
\end{equation*}
Also, ${0<\widehat V_j\leq1}$ and ${V_j\geq r_{\min}}$ imply
\begin{equation*}
0<\frac{\widehat V_j}{V_j}\leq r_{\min}^{-1}
\quad\Longrightarrow\quad
|\Delta|\leq\max\{1,r_{\min}^{-1}-1\}\leq r_{\min}^{-1}.
\end{equation*}
Thus, for $q:=2p/(p-2)$,
\begin{equation*}
\|\Delta\|_{2,\Pi^V}\leq r_{\min}^{-1/2}\varepsilon_{\mathrm{av}},\quad
\|\Delta\|_\infty\leq r_{\min}^{-1}
\ \Longrightarrow\
\|\Delta\|_{q,\Pi^V}\leq
\|\Delta\|_\infty^{1-2/q}\|\Delta\|_{2,\Pi^V}^{2/q}
\leq r_{\min}^{-(1-1/q)}\varepsilon_{\mathrm{av}}^{\,2/q}.
\end{equation*}
For each edge, define its ratio relative to the exact-verifier edge-conductance measure by ${J_1:=1}$ and
\begin{equation*}
{J_j(x,y):=\frac{d\widehat C_j}{dC_j^V}(x,y)
=\sqrt{\frac{\widehat V_j(x)\widehat V_{j-1}(y)}{V_j(x)V_{j-1}(y)}}
=\sqrt{\xi_j(x)\xi_{j-1}(y)},\qquad j\geq2,\qquad \xi_j:=\widehat V_j/V_j.}
\end{equation*}
For lifted states $z,z'$ on edge ${j}$, write ${J_j(z,z')}$ for this ratio evaluated at the noisier sample $x$ and cleaner sample $y$, in either orientation. The notation ${\iint_{j\leftrightarrow j-1}}$ means integrating over ${z=(j,x,d)}$ and ${z'=(j-1,y,d)}$ and also over the reversed pair, for $d\in\{\mathrm F,\mathrm B\}$. At the boundaries, both direction labels refer to the single state $(0,x,\mathrm B)$ or ${(J,x,\mathrm F)}$. The verifier bounds give
{${v_{\min}\leq J_j\leq r_{\min}^{-1}}$. For ${j\geq2}$, the identity $|\sqrt{\xi}-1|=|\xi-1|/(\sqrt{\xi}+1)\leq|\xi-1|$ gives
\begin{equation*}
|J_j(z,z')-1|
\leq\sqrt{\xi_j(x)}\,|\sqrt{\xi_{j-1}(y)}-1|
+|\sqrt{\xi_j(x)}-1|
\leq r_{\min}^{-1/2}\bigl(|\Delta(z)|+|\Delta(z')|\bigr).
\end{equation*}
}
The same inequality holds on the clean edge since ${J_1=1}$.

Define the incident-mass ratio ${s(j,x,d):=d\mu_j/d\mu_j^V(x)}$. For an interior state $z=(\ell,x,d)$, ${1\leq \ell<J}$, the two incident edges give
\begin{align*}
\mu_\ell^V(dx)&=C_\ell^V(dx,\mathcal X)+C_{\ell+1}^V(\mathcal X,dx),\\
\mu_\ell(dx)&=\int J_\ell(x,y)C_\ell^V(dx,dy)+\int J_{\ell+1}(y,x)C_{\ell+1}^V(dy,dx).
\end{align*}
Applying the symmetric edge identity in \cref{eq:c-symmetric-movement-edge-law}, established in the proof of \cref{thm:c-average-moment-poincare}, to the exact-verifier chain gives
\begin{equation*}
\mu_\ell^V(dx)M^V\bigl(z,(\ell-1,dy,d)\bigr)\mkern-2mu=\mkern-2mu\tfrac12 C_\ell^V(dx,dy),\quad
\mu_\ell^V(dx)M^V\bigl(z,(\ell+1,dy,d)\bigr)\mkern-2mu=\mkern-2mu\tfrac12 C_{\ell+1}^V(dy,dx).
\end{equation*}
Multiplying by the corresponding ${J_j}$ and adding therefore yields
\begin{equation*}
\mu_\ell^V(dx)\sum_{j=1}^J\int_{j\leftrightarrow j-1}J_j(z,z')M^V(z,dz')
=\tfrac12\mu_\ell(dx)=\tfrac12s(z)\mu_\ell^V(dx).
\end{equation*}
Comparing densities with respect to $\mu_\ell^V$ gives the sum as $c(z)s(z)$, since $c(z)=1/2$ in the interior. At a boundary, the two direction branches of its sole incident edge together give the same identity with $c(z)=1$. Dividing by $c(z)$ and normalizing the incident measures gives
\begin{equation}
s(z)=\frac1{c(z)}\sum_{j=1}^J\int_{j\leftrightarrow j-1}J_j(z,z')M^V(z,dz'),\qquad
\frac{d\Pi}{d\Pi^V}=\frac{\calZ^V}{\calZ}s.
\label{eq:c-verifier-incident-ratio}
\end{equation}
{In particular, since $s(z)$ is a weighted average of the edge ratios ${J_j}$ satisfying ${v_{\min}\leq J_j\leq r_{\min}^{-1}}$, we have $v_{\min}\leq s\leq r_{\min}^{-1}$. Since $M^V(z,\cdot)/c(z)$ is a probability kernel,}
\begin{align*}
\|s-1\|_{q,\Pi^V}^q
&\overset{(\circ)}{\leq}2\sum_{j=1}^J\iint_{j\leftrightarrow j-1}|J_j(z,z')-1|^q\,\Pi^V(dz)M^V(z,dz')\\
&\overset{(\triangle)}{\leq}2^{q+1}{r_{\min}^{-q/2}}\int c(z)|\Delta(z)|^q\Pi^V(dz)
\leq2^{q+1}{r_{\min}^{-q/2}}\|\Delta\|_{q,\Pi^V}^q.
\end{align*}
Here $(\circ)$ uses Jensen's inequality and $1/c(z)\leq2$; $(\triangle)$ uses the bound on ${|J_j-1|}$, $(a+b)^q\leq2^{q-1}(a^q+b^q)$ for $a,b\geq0$, and the symmetry of $M^V$ with respect to $\Pi^V$; the last step uses $c\leq1$. Taking $q$th roots gives $\|s-1\|_{q,\Pi^V}\leq{4r_{\min}^{-1/2}}\|\Delta\|_{q,\Pi^V}$.

To compare the operators despite their different stationary laws, use the norm-preserving map $Uf:=\sqrt{d\Pi/d\Pi^V}\,f$. On edge ${j}$, the identities ${\widehat C_j=J_jC_j^V}$ and $\mu_\ell=s\mu_\ell^V$ give ${M(z,dz')=[J_j(z,z')/s(z)]M^V(z,dz')}$ by \cref{eq:c-symmetric-movement-edge-law}. Substituting $d\Pi/d\Pi^V=(\calZ^V/\calZ)s$ from \cref{eq:c-verifier-incident-ratio}, the constant $\calZ^V/\calZ$ cancels between $U$ and $U^{-1}$:
\begin{align*}
&(UMU^{-1}f)(z)
=\sqrt{\frac{d\Pi}{d\Pi^V}(z)}
\int\frac{f(z')}{\sqrt{\frac{d\Pi}{d\Pi^V}(z')}}M(z,dz')
=\int\frac{\sqrt{s(z)}}{\sqrt{s(z')}}f(z')M(z,dz')\\
&=\sum_{j=1}^J\int_{j\leftrightarrow j-1}
\frac{J_j(z,z')f(z')}{\sqrt{s(z)s(z')}}M^V(z,dz')
=\sum_{j=1}^J\int_{j\leftrightarrow j-1}
\frac{J_j(z,z')m^V(z,z')f(z')}{\sqrt{s(z)s(z')}}\Pi^V(dz').
\end{align*}
Thus the density on edge ${j}$ is ${m^V(z,z')J_j(z,z')/\sqrt{s(z)s(z')}}$. For $F(a,b):=(ab)^{-1/2}$, both partial derivatives have absolute value at most $1/(2{v_{\min}^2})$ when $a,b\geq {v_{\min}}$. Since $s(z),s(z')\geq {v_{\min}}$ and $v_{\min}\leq1$,
\begin{align*}
\left|\frac{J_j(z,z')}{\sqrt{s(z)s(z')}}-1\right|
&\leq\frac{|J_j(z,z')-1|}{\sqrt{s(z)s(z')}}
+\left|\frac1{\sqrt{s(z)s(z')}}-1\right|\\
&\overset{(\circ)}{\leq}{v_{\min}^{-1}}|J_j(z,z')-1|
+\frac{|s(z)-1|+|s(z')-1|}{2{v_{\min}^2}}.
\end{align*}
Here $(\circ)$ applies the mean value theorem to $F$ between $(s(z),s(z'))$ and $(1,1)$.
The edge domains are disjoint: each state pair $(z,z')$ belongs to at most one checkpoint index edge. Thus, summing a nonnegative integrand over these domains gives at most its integral over all state pairs. In particular,
\begin{equation*}
\sum_{j=1}^J\iint_{j\leftrightarrow j-1}|\Delta(z)|^q\,\Pi^V(dz)\Pi^V(dz')
\leq\iint|\Delta(z)|^q\,\Pi^V(dz)\Pi^V(dz')
=\|\Delta\|_{q,\Pi^V}^q.
\end{equation*}
Using the same bound for $\Delta(z')$ and $s-1$, the preceding estimates give
\begin{equation*}
\left(\sum_{j=1}^J\iint_{j\leftrightarrow j-1}
\left|\frac{J_j(z,z')}{\sqrt{s(z)s(z')}}-1\right|^q
\Pi^V(dz)\Pi^V(dz')\right)^{1/q}
\overset{(\circ)}{\leq}\frac{6}{{\sqrt{r_{\min}}}{v_{\min}^2}}\|\Delta\|_{q,\Pi^V},
\end{equation*}
where $(\circ)$ uses the triangle inequality, ${|J_j-1|\leq {r_{\min}^{-1/2}}(|\Delta(z)|+|\Delta(z')|)}$, and $\|s-1\|_{q,\Pi^V}\leq4{r_{\min}^{-1/2}}\|\Delta\|_{q,\Pi^V}$, giving a coefficient $2/({\sqrt{r_{\min}}}{v_{\min}})+4/({\sqrt{r_{\min}}}{v_{\min}^2})\leq6/({\sqrt{r_{\min}}}{v_{\min}^2})$ since $v_{\min}\leq1$.
Since $1/2=1/p+1/q$, recalling the operator norm $\|A\|_{2\to2}:=\sup_{\|f\|_{2,\Pi^V}=1}\|Af\|_{2,\Pi^V}$ on $L^2(\Pi^V)$, we obtain
\begin{align}
\|UMU^{-1}-M^V\|_{2\to2}
&\overset{(\circ)}{\leq}
\left(\sum_{j=1}^J\iint_{j\leftrightarrow j-1}m^V(z,z')^2
\left|\frac{J_j(z,z')}{\sqrt{s(z)s(z')}}-1\right|^2
\Pi^V(dz)\Pi^V(dz')\right)^{1/2}\notag\\
&\overset{(\triangle)}{\leq}\frac{6H_p}{{\sqrt{r_{\min}}}{v_{\min}^2}}\|\Delta\|_{q,\Pi^V}
\leq\frac{6H_p}{{r_{\min}^{1+1/p}}{v_{\min}^2}}\varepsilon_{\mathrm{av}}^{\,1-2/p}.
\label{eq:c-average-verifier-operator-bound}
\end{align}
Here $(\circ)$ is Cauchy--Schwarz for integral kernels and $(\triangle)$ applies H\"older's inequality over the disjoint edge domains. Taking the infimum over admissible $p$ gives $\|UMU^{-1}-M^V\|_{2\to2}\leq\Delta_{\mathrm{av}}$.

Finally, set $L:=U\mathcal L_{\mathrm{move}}U^{-1}$ and $\psi:=U\mathbf1$. Then $L$ is nonnegative and self-adjoint, $L\psi=0$, and $\|\psi\|_{2,\Pi^V}=1$. Multiplication by $c$ commutes with $U$, so
\begin{equation*}
L=U(cI-M)U^{-1}=cI-UMU^{-1},\qquad \mathcal L_{\mathrm{move}}^V=cI-M^V.
\end{equation*}
Consequently,
\begin{equation*}
L-\mathcal L_{\mathrm{move}}^V=M^V-UMU^{-1},\qquad
\|L-\mathcal L_{\mathrm{move}}^V\|_{2\to2}=\|UMU^{-1}-M^V\|_{2\to2}\leq\Delta_{\mathrm{av}}.
\end{equation*}
For a function $g\in L^2(\Pi^V)$ such that $g\perp\psi$ and $\|g\|_{2,\Pi^V}=1$, use the orthogonality of $g$ and $\psi$ to choose scalars $\alpha,\beta$ so that $h=\alpha\psi+\beta g$ satisfies $\Pi^V(h)=0$ and $\|h\|_{2,\Pi^V}=1$. Then $\alpha^2+\beta^2=1$, and
\begin{equation*}
\gamma_V\overset{(\circ)}{\leq}\langle h,\mathcal L_{\mathrm{move}}^Vh\rangle_{\Pi^V}\overset{(\triangle)}{\leq}\langle h,Lh\rangle_{\Pi^V}+\Delta_{\mathrm{av}}\overset{(\square)}{=}\beta^2\langle g,Lg\rangle_{\Pi^V}+\Delta_{\mathrm{av}}\leq\langle g,Lg\rangle_{\Pi^V}+\Delta_{\mathrm{av}},
\end{equation*}
where $(\circ)$ uses \cref{eq:c-exact-verifier-movement-gap} and $\operatorname{Var}_{\Pi^V}(h)=1$; $(\triangle)$ uses $|\langle h,(\mathcal L_{\mathrm{move}}^V-L)h\rangle_{\Pi^V}|\leq\|\mathcal L_{\mathrm{move}}^V-L\|_{2\to2}\|h\|_{2,\Pi^V}^2\leq\Delta_{\mathrm{av}}$; and $(\square)$ uses $L\psi=0$ and self-adjointness, so $\langle\psi,Lg\rangle_{\Pi^V}=\langle L\psi,g\rangle_{\Pi^V}=0$. The last inequality uses $\beta^2\leq1$ and $\langle g,Lg\rangle_{\Pi^V}\geq0$.

For nonconstant $f$, take $g=U(f-\Pi(f))/\sqrt{\operatorname{Var}_\Pi(f)}$. Since $U$ preserves inner products, $g\perp\psi$ and $\|g\|_{2,\Pi^V}=1$. Moreover, $L=U\mathcal L_{\mathrm{move}}U^{-1}$ and the self-adjoint operator $\mathcal L_{\mathrm{move}}$ satisfies $\mathcal L_{\mathrm{move}}\mathbf1=0$, so
\begin{equation*}
\langle g,Lg\rangle_{\Pi^V}
=\frac{\langle f-\Pi(f),\mathcal L_{\mathrm{move}}(f-\Pi(f))\rangle_\Pi}{\operatorname{Var}_\Pi(f)}
=\frac{\langle f,\mathcal L_{\mathrm{move}}f\rangle_\Pi}{\operatorname{Var}_\Pi(f)}.
\end{equation*}
Multiplying the preceding bound by $\operatorname{Var}_\Pi(f)$ and using \cref{eq:c-movement-energy-lower-bound} gives
\begin{equation*}
\mathcal E_P(f,f)\geq\langle f,\mathcal L_{\mathrm{move}}f\rangle_\Pi
\geq(\gamma_V-\Delta_{\mathrm{av}})\operatorname{Var}_\Pi(f).
\end{equation*}
For constant $f$, both sides vanish. Combining this with the error-independent bound $\gamma_Vr_{\min}{v_{\min}^2}$ from \cref{eq:c-explicit-lifted-rate} gives the maximum in \cref{eq:c-average-verifier-gap}, uniformly in $\chi$. The multiplicative bound follows from $\max\{a,1-x\}\geq a/(a+x)$ for $0<a\leq1$ and $x\geq0$, with $a=r_{\min}{v_{\min}^2}$ and $x=\Delta_{\mathrm{av}}/\gamma_V$.
\end{proof}

As $\varepsilon_{\mathrm{av}}\to0$, the lower bound in \cref{eq:c-average-verifier-gap} recovers $\gamma_V$; for larger errors, it retains the positive bound from \cref{prop:c-explicit-lifted-poincare}.

Lastly, we remark that the moment assumption $H_p<\infty$ in \cref{eq:c-oracle-density-moment} holds for paired Gaussian transitions under a slightly stronger clean-tail condition than in \cref{cor:c-ideal-paired-gaussian-mixing}, as stated below.

\begin{lemma_ap}[Higher-order moments for paired Gaussian transitions]
\label{lem:c-gaussian-higher-moments}
Assume the Gaussian pairing in \cref{eq:c-ideal-gaussian-pairing} and the reward and verifier bounds in \cref{eq:c-quantitative-weight-bounds}. If
\begin{equation}
\mathbb E_{\pi_0}e^{c\|X_0\|^2}<\infty
\quad\text{for some}\quad c>\max_{1\leq j\leq J}\frac{2A_j^2}{\sigma_j^2},
\label{eq:c-gaussian-strict-tail-moment}
\end{equation}
where $A_j$ and $\sigma_j$ are the cumulative signal scale and renoising noise scale from \cref{cor:c-ideal-paired-gaussian-mixing}, respectively,
then $H_p<\infty$ in \cref{eq:c-oracle-density-moment} for some finite $p>2$. In particular, bounded support of $\pi_0$ ensures this tail condition.
\end{lemma_ap}

\begin{proof}
We repeat the Gaussian calculation in the proof of \cref{cor:c-ideal-paired-gaussian-mixing} for the $p$th moment of $\frac{dJ_j^{\mathrm F}}{d(\pi_j\otimes\pi_{j-1})}(x,y)=\frac{\varphi_{\sigma_j^2}(x-a_jy)}{\pi_j(x)}$. Using $\pi_j\geq b_{j,v_j}\varphi_{v_j}$ and $v_j=\tau_j^2-\sigma_j^2/4$ from \cref{eq:c-ideal-gaussian-lower-envelope,eq:c-ideal-gaussian-envelope-choice}, the integrations in \cref{eq:c-ideal-gaussian-integral-bound,eq:c-ideal-gaussian-exponential-moment} bound this moment by a finite multiple of $\mathbb E_{\pi_0}e^{c_j(p)\|X_0\|^2}$, where
\begin{equation*}
\eta_j(p):=\frac{p(p-1)a_j^2}{2[pv_j-(p-1)\sigma_j^2]},\qquad
c_j(p):=\frac{\eta_j(p)A_{j-1}^2}{1-2\eta_j(p)\tau_{j-1}^2},
\end{equation*}
provided both denominators are positive. At $p=2$, we have $\eta_j(2)=\eta_j$, so both denominators are strictly positive by \cref{eq:c-ideal-gaussian-positive-denominators}, and $c_j(2)=2A_j^2/\sigma_j^2<c$ by \cref{eq:c-gaussian-strict-tail-moment}. Continuity with respect to $p$ and the finiteness of $J$ therefore give a common $p>2$ with
\begin{equation*}
\max_{1\leq j\leq J}\left\|\frac{dJ_j^{\mathrm F}}{d(\pi_j\otimes\pi_{j-1})}\right\|_{L^p(\pi_j\otimes\pi_{j-1})}<\infty.
\end{equation*}
Since $J_j^{\mathrm B}=J_j^{\mathrm F}$ and $r_{\min}\leq V_j\leq1$ by \cref{eq:c-mixing-verifier-recursion}, applying \cref{lem:c-edge-density-moment-condition} to the exact-verifier kernel $P^V$ and its stationary law $\Pi^V$ gives $p_{\mathrm{move}}^V\in L^p(\Pi^V\otimes\Pi^V)$, where $p_{\mathrm{move}}^V$ is the density of $P_{\mathrm{move}}^V$ relative to $\Pi^V$. By \cref{eq:c-oracle-movement-density,eq:c-oracle-density-moment},
\begin{equation*}
m^V(z,z')=\frac{p_{\mathrm{move}}^V(z,z')+p_{\mathrm{move}}^V(z',z)}2,
\qquad
H_p^p\leq\iint p_{\mathrm{move}}^V(z,z')^p\,\Pi^V(dz)\Pi^V(dz')<\infty,
\end{equation*}
where the inequality uses $((a+b)/2)^p\leq(a^p+b^p)/2$ for $a,b\geq0$ and exchanges $z,z'$ in the second integral.
\end{proof}

{
\subsection{Momentum and First Hitting Time of the Clean Boundary}
\label{app:subsec:cs-vgb-momentum-hitting}
In practice, CS-VGB runs under a finite computational budget and collects samples at the clean checkpoint. Visits to this checkpoint are therefore a practical concern distinct from mixing time. We analyze the \emph{first hitting time} $\tau_0:=\inf\{n\geq0:j_n=0\}$, where $j_n$ is the checkpoint index after $n$ updates, under an average directional-imbalance condition.

Let $Y_J\sim\pi_J$ and $Y_{j-1}\mid Y_j\sim Q_j^{\mathrm F}(\cdot\mid Y_j)$ for $j=J,\ldots,1$. The expectations below are taken along this denoising path. Write $(u)_+:=\max\{u,0\}$.

\begin{proposition_ap}[First hitting time under average directional imbalance]
\label{prop:c-average-momentum-hitting}
Assume the conditions of \cref{thm:c-exact-stationarity}, with $J\geq2$, and use the exact CS-VGB kernel $P$. Suppose
\begin{equation}
K_{J-1}^{\mathrm B,\rho}(dx\mid y)=\pi_J(dx)
\quad\text{for }\pi_{J-1}\text{-almost every }y.
\label{eq:c-average-hitting-noise-reset}
\end{equation}
Assume that, for some $C_{\mathrm{HT}}\geq0$,
\begin{equation}
\max_{1\leq j<J}\bbE\left[
\left(\log\frac{a_j^{\mathrm B}(Y_j)}{a_j^{\mathrm F}(Y_j)}\right)_+
\right]\leq\frac{C_{\mathrm{HT}}}{J-1},
\qquad
1-\chi\leq\frac1{J-1},\quad \chi\in[0,1].
\label{eq:c-average-hitting-imbalance}
\end{equation}
If $v_{\min}\leq\widehat V_j\leq1$ $\pi_j$-almost surely for $1\leq j\leq J$, with $v_{\min}>0$, then, for $\pi_J$-almost every $x$,
\begin{equation}
J\leq\bbE_{(J,x,\mathrm F)}[\tau_0]
\leq e^{C_{\mathrm{HT}}+1}\frac{3Z+4J-5}{v_{\min}}.
\label{eq:c-average-hitting-bounded-verifier}
\end{equation}
For fixed $C_{\mathrm{HT}}$, $Z>0$, and $v_{\min}>0$, this gives $\bbE_{(J,x,\mathrm F)}[\tau_0]=\Theta(J)$.
\end{proposition_ap}

\begin{proof}
Write $M_j:=\widehat C_j(\mathcal X,\mathcal X)$. We first bound the ratio $\calZ/M_J$. Terminal anchoring and \cref{eq:c-incident-measures,eq:c-lift-normalizer} give
\begin{equation*}
M_1=Z,\qquad
\calZ=3Z+4\sum_{j=2}^{J-1}M_j+3M_J.
\end{equation*}
The verifier bounds give $v_{\min}\leq\sqrt{\widehat V_j\widehat V_{j-1}}\leq1$, hence $v_{\min}\leq M_j\leq1$ for $j\geq2$ and
\begin{equation*}
\frac{\calZ}{M_J}\leq\frac{3Z+4J-5}{v_{\min}}.
\end{equation*}
In particular, $M_J\geq v_{\min}>0$. We next bound the time needed to reach checkpoint $0$ by considering repeated attempts from checkpoint $J$.

The corrected edge identity and \cref{eq:c-average-hitting-noise-reset} give $\pi_J(dx)K_J^{\mathrm F}(dy\mid x)=\pi_J(dx)\pi_{J-1}(dy)$, hence $K_J^{\mathrm F}(dy\mid x)=\pi_{J-1}(dy)$ for $\pi_J$-almost every $x$. Substituting this and \cref{eq:c-average-hitting-noise-reset} into the definitions of the normalized transitions gives
\begin{equation}
Q_J^{\mathrm F}(dy\mid x)
=\frac{\sqrt{\widehat V_{J-1}(y)}\,\pi_{J-1}(dy)}
{\int\sqrt{\widehat V_{J-1}}\,d\pi_{J-1}},\qquad
Q_{J-1}^{\mathrm B,\rho}(dx\mid y)
=\frac{\sqrt{\widehat V_J(x)}\,\pi_J(dx)}
{\int\sqrt{\widehat V_J}\,d\pi_J}.
\label{eq:c-average-hitting-top-departure}
\end{equation}
The F transition does not depend on the current sample $x$. Thus the first update from $J$ always produces the same distribution of states $(J-1,Y,\mathrm F)$; denote it by $\zeta$. Since the returning B transition has a density with respect to $\pi_J$, the next F update follows the same distribution $\zeta$ after every return to $J$, independently of the returned sample.

We next bound the mean time between visits to $J$. Starting from checkpoint $J$, let $\tau_J^+:=\inf\{n\geq1:j_n=J\}$ be the first return time to $J$. Write $D:=\{(j,x,d)\in\mathsf S:j<J\}$ for the states below $J$. We use $\Pi_D(A):=\Pi(A\cap D)$ and $P_D(z,A):=P(z,A\cap D)$ for $z\in D$. Thus $P_D$ keeps the original transition probabilities but excludes any move reaching $J$.
For any measurable $A\subseteq D$, stationarity splits the probability of entering $A$ into arrivals from $J$ and arrivals from $D$:
\begin{equation*}
\Pi_D(A)=\int_{\{j=J\}}P(z,A)\Pi(dz)+\int_D P(z,A)\Pi(dz)
=\Pi(j=J)\zeta(A)+(\Pi_DP_D)(A).
\end{equation*}
Applying this identity repeatedly gives
\begin{equation*}
\Pi_D=\Pi(j=J)\zeta+\Pi_DP_D
=\Pi(j=J)\sum_{m=0}^{n-1}\zeta P_D^m+\Pi_DP_D^n.
\end{equation*}
Taking total masses and dropping the nonnegative remainder $(\Pi_DP_D^n)(D)$ gives
\begin{equation*}
\Pi(D)\geq\Pi(j=J)\sum_{m=0}^{n-1}(\zeta P_D^m)(D),
\qquad
\sum_{m\geq0}(\zeta P_D^m)(D)\leq\frac{\Pi(D)}{\Pi(j=J)},
\end{equation*}
After the first update from $J$, the state has distribution $\zeta$. For $m\geq1$, remaining in $D$ for the next $m$ updates has probability
\begin{align*}
(\zeta P_D^m)(D)
&=\int_D\zeta(dz_0)\int_D P(z_0,dz_1)\cdots\int_D P(z_{m-1},dz_m)
=\bbP_{(J,x,\mathrm F)}(\tau_J^+>m+1).
\end{align*}
Every intermediate state is integrated over $D$, so a path that reaches $J$ and then leaves it is not counted. Using $\bbE[T]=\sum_{n\geq0}\bbP(T>n)$ for a nonnegative integer-valued time $T$, we obtain
\begin{equation}
\bbE_{(J,x,\mathrm F)}[\tau_J^+]
=1+\sum_{m\geq0}(\zeta P_D^m)(D)
\leq1+\frac{\Pi(D)}{\Pi(j=J)}
=\frac1{\Pi(j=J)}=\frac{\calZ}{M_J}<\infty.
\label{eq:c-average-hitting-noise-return}
\end{equation}
Having bounded the return time, we now bound the probability $\bbP_{(J,x,\mathrm F)}(\tau_0<\tau_J^+)$ of visiting $0$ before returning to $J$. We obtain a lower bound by counting only paths that move in direction F without a flip; stays are allowed.
At an interior state $(j,x,\mathrm F)$, each stay leaves the state unchanged, so
\begin{equation*}
\bbP_{(j,x,\mathrm F)}(\text{move before flip})
=\sum_{n=0}^{\infty}(\chi\ell_j(x))^n p_j^{\mathrm F}(x)
=\frac{p_j^{\mathrm F}(x)}{1-\chi\ell_j(x)},
\end{equation*}
For $u=a_j^{\mathrm B}(x)/a_j^{\mathrm F}(x)$, the directional probabilities satisfy
\begin{equation*}
p_j^{\mathrm F}(x)=\frac1{1+u},\qquad
\ell_j(x)=\frac{\min\{1,u\}}{1+u}.
\end{equation*}
Substituting these expressions gives
\begin{equation*}
\frac{p_j^{\mathrm F}(x)}{1-\chi\ell_j(x)}
=\frac1{1+u-\chi\min\{1,u\}}
\geq\frac{e^{-(\log u)_+}}{2-\chi},
\end{equation*}
because $1+u-\chi\min\{1,u\}\leq(2-\chi)\max\{1,u\}=(2-\chi)e^{(\log u)_+}$.
Along the denoising path generated by $Y_J\sim\pi_J$ and $Y_{j-1}\mid Y_j\sim Q_j^{\mathrm F}(\cdot\mid Y_j)$ for $j=J,\ldots,1$, multiply these move-before-flip probabilities and average over the samples. This counts paths reaching $0$ without a flip and therefore gives
\begin{align*}
\bbP_{(J,x,\mathrm F)}(\tau_0<\tau_J^+)
&\geq\bbE\left[\prod_{j=1}^{J-1}
\frac{p_j^{\mathrm F}(Y_j)}{1-\chi\ell_j(Y_j)}\right]\\
&\geq(2-\chi)^{-(J-1)}
\bbE\left[\exp\left\{-\sum_{j=1}^{J-1}
\left(\log\frac{a_j^{\mathrm B}(Y_j)}{a_j^{\mathrm F}(Y_j)}\right)_+\right\}\right]
\\
&\overset{(\circ)}{\geq}(2-\chi)^{-(J-1)}
\exp\left[-\sum_{j=1}^{J-1}\bbE\left[
\left(\log\frac{a_j^{\mathrm B}(Y_j)}{a_j^{\mathrm F}(Y_j)}\right)_+
\right]\right]
\overset{(\triangle)}{\geq}e^{-(C_{\mathrm{HT}}+1)},
\end{align*}
where $(\circ)$ is Jensen's inequality, and $(\triangle)$ uses \cref{eq:c-average-hitting-imbalance}: the sum is at most $C_{\mathrm{HT}}$, and $(J-1)\log(2-\chi)\leq(J-1)(1-\chi)\leq1$.

We combine these two bounds by repeating the attempt after each return to $J$. Let $L_i$ be the number of updates from the start of attempt $i$ to the next visit to $J$, and let $N$ be the first attempt that visits $0$. Let $\mathcal F_{i-1}$ be the sigma-algebra generated by the chain up to the start of attempt $i$. By the strong Markov property, conditional on $\mathcal F_{i-1}$, the next attempt succeeds with the same probability $\bbP_{(J,x,\mathrm F)}(\tau_0<\tau_J^+)$, so $\bbP(N\geq i)=[1-\bbP_{(J,x,\mathrm F)}(\tau_0<\tau_J^+)]^{i-1}$. The same property and \cref{eq:c-average-hitting-noise-return} give
\begin{equation*}
\bbE[L_i\mid\mathcal F_{i-1}]\leq\frac{\calZ}{M_J}\qquad\text{almost surely}.
\end{equation*}
The event $\{N\geq i\}$ is determined by the previous attempts and is therefore $\mathcal F_{i-1}$-measurable. Therefore, the tower property yields
\begin{align*}
\bbE[\mathbf1_{\{N\geq i\}}L_i]
&=\bbE\left[\bbE[\mathbf1_{\{N\geq i\}}L_i\mid\mathcal F_{i-1}]\right]
=\bbE\left[\mathbf1_{\{N\geq i\}}\bbE[L_i\mid\mathcal F_{i-1}]\right]\\
&\leq\bbE\left[\mathbf1_{\{N\geq i\}}\frac{\calZ}{M_J}\right]
=\bbP(N\geq i)\frac{\calZ}{M_J}.
\end{align*}
Finally, the first visit to $0$ occurs before the successful attempt ends, giving
\begin{align}
\bbE_{(J,x,\mathrm F)}[\tau_0]
&\leq\sum_{i\geq1}\bbE[\mathbf1_{\{N\geq i\}}L_i]
\leq\frac{\calZ}{M_J}\sum_{i\geq1}\bigl[1-\bbP_{(J,x,\mathrm F)}(\tau_0<\tau_J^+)\bigr]^{i-1}\notag\\
&=\frac{\calZ}{M_J\,\bbP_{(J,x,\mathrm F)}(\tau_0<\tau_J^+)}
\leq e^{C_{\mathrm{HT}}+1}\frac{\calZ}{M_J}.
\label{eq:c-average-momentum-hitting-bound}
\end{align}
Substituting the bound on $\calZ/M_J$ and using $\tau_0\geq J$ proves the claim.
\end{proof}

\paragraph{First hitting time under exact balance.}
If ${p_j^{\mathrm F}=p_j^{\mathrm B}=1/2}$ at every reachable interior state, then
\begin{equation}
\bbE_{(J,x,\mathrm F)}[\tau_0]
=(2J-1)\frac{1+(1-\chi)J}{2-\chi}.
\label{eq:c-balanced-hitting-formula}
\end{equation}
At $\chi=1$, the mean is $2J-1$, compared with $J^2$ when the direction is resampled at each update (\cref{fig:cs-vgb-balanced-hitting}).
To obtain this formula, define $f_j:=\bbE_{(j,x,\mathrm F)}[\tau_0]$ and $b_j:=\bbE_{(j,x,\mathrm B)}[\tau_0]$, which are independent of $x$ under exact balance. Conditioning on the first update gives, for $1\leq j<J$,
\begin{equation*}
\begin{aligned}
f_j&=1+\frac12f_{j-1}+\frac{1-\chi}{2}b_j+\frac\chi2f_j, \qquad b_j&=1+\frac12b_{j+1}+\frac{1-\chi}{2}f_j+\frac\chi2b_j.
\end{aligned}
\end{equation*}
Solving these recurrences with $f_0=0$ and $b_J=f_J=1+f_{J-1}$ yields \cref{eq:c-balanced-hitting-formula}.

\begin{figure}[!htb]
\centering

\definecolor{hittingBlue}{HTML}{365F7D}
\definecolor{hittingGray}{HTML}{777777}
\definecolor{hittingGrid}{HTML}{DDDDDD}
\begin{tikzpicture}[x=1cm,y=.86cm,font=\normalfont\footnotesize,text=black]
\draw[hittingGrid,line width=.35pt] (0,1) -- (7.8,1);
\draw[black,line width=.6pt] (-.08,1) -- (0,1);
\node[anchor=east] at (-.12,1) {$10$};
\draw[hittingGrid,line width=.35pt] (0,2) -- (7.8,2);
\draw[black,line width=.6pt] (-.08,2) -- (0,2);
\node[anchor=east] at (-.12,2) {$10^{2}$};
\draw[hittingGrid,line width=.35pt] (0,3) -- (7.8,3);
\draw[black,line width=.6pt] (-.08,3) -- (0,3);
\node[anchor=east] at (-.12,3) {$10^{3}$};
\draw[hittingGrid,line width=.35pt] (0,4) -- (7.8,4);
\draw[black,line width=.6pt] (-.08,4) -- (0,4);
\node[anchor=east] at (-.12,4) {$10^{4}$};
\draw[black,line width=.7pt] (0,4.5) -- (0,.38) -- (7.95,.38);
\draw[black,line width=.6pt] (0,.38) -- (0,.30);
\node[anchor=north] at (0,.24) {$2$};
\draw[black,line width=.6pt] (1.3,.38) -- (1.3,.30);
\node[anchor=north] at (1.3,.24) {$4$};
\draw[black,line width=.6pt] (2.6,.38) -- (2.6,.30);
\node[anchor=north] at (2.6,.24) {$8$};
\draw[black,line width=.6pt] (3.9000000000000004,.38) -- (3.9000000000000004,.30);
\node[anchor=north] at (3.9000000000000004,.24) {$16$};
\draw[black,line width=.6pt] (5.2,.38) -- (5.2,.30);
\node[anchor=north] at (5.2,.24) {$32$};
\draw[black,line width=.6pt] (6.5,.38) -- (6.5,.30);
\node[anchor=north] at (6.5,.24) {$64$};
\draw[black,line width=.6pt] (7.800000000000001,.38) -- (7.800000000000001,.30);
\node[anchor=north] at (7.800000000000001,.24) {$128$};
\draw[hittingGray,line width=1pt,dash pattern=on 3pt off 2pt] plot coordinates {
(0.00000,0.60206)
(0.76045,0.95424)
(1.30000,1.20412)
(2.06045,1.55630)
(2.60000,1.80618)
(3.36045,2.15836)
(3.90000,2.40824)
(4.66045,2.76042)
(5.20000,3.01030)
(5.96045,3.36248)
(6.50000,3.61236)
(7.26045,3.96454)
(7.80000,4.21442)
};
\draw[black,line width=1pt] plot coordinates {
(0.00000,0.51491)
(0.76045,0.77152)
(1.30000,0.94983)
(2.06045,1.20412)
(2.60000,1.38997)
(3.36045,1.66276)
(3.90000,1.86494)
(4.66045,2.16218)
(5.20000,2.38120)
(5.96045,2.69976)
(6.50000,2.93164)
(7.26045,3.26495)
(7.80000,3.50503)
};
\draw[hittingBlue,line width=1.5pt] plot coordinates {
(0.00000,0.47712)
(0.76045,0.69897)
(1.30000,0.84510)
(2.06045,1.04139)
(2.60000,1.17609)
(3.36045,1.36173)
(3.90000,1.49136)
(4.66045,1.67210)
(5.20000,1.79934)
(5.96045,1.97772)
(6.50000,2.10380)
(7.26045,2.28103)
(7.80000,2.40654)
};
\node[anchor=west,align=left] at (8.05,4.21442) {Resample direction\\$J^2$};
\node[anchor=west,align=left] at (8.05,3.50503) {CS-VGB, $\chi=0.9$};
\node[anchor=west,align=left,text=hittingBlue] at (8.05,2.40654) {CS-VGB, $\chi=1$\\${\color{black}2J-1}$};
\node[rotate=90,font=\normalfont\small] at (-1.04,2.52) {$\mathbb E[\tau_0]$ (updates)};
\node[font=\normalfont\small] at (3.9,-.48) {Checkpoint intervals $J$};
\end{tikzpicture}
\caption{Exact expected first hitting times from $j=J$ under balanced masses (\cref{eq:c-balanced-hitting-formula}). Resample direction moves directly to $j-1$ or $j+1$ with probabilities $p_j^{\mathrm F}$ or $p_j^{\mathrm B}$, respectively; here both equal $1/2$ at interior checkpoints, and the move from $J$ is forced inward. Its mean is $J^2$, compared with $2J-1$ for CS-VGB at $\chi=1$.}
\label{fig:cs-vgb-balanced-hitting}
\end{figure}
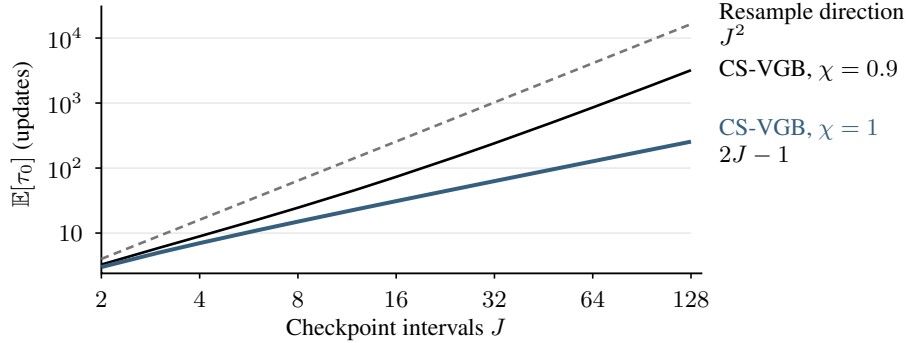

\subsection{Stationarity and Convergence under Edge-Correction Error}
\label{app:subsec:cs-vgb-practical-errors}

We now study CS-VGB with the edge correction omitted, establishing conditions for a stationary distribution, bounding its deviation from the exact stationary law, and analyzing convergence to this potentially biased limit.
\paragraph{Average edge-correction error.}
Let $P$ be the exact CS-VGB kernel and $\bar P$ the generally biased kernel obtained by replacing the edge correction $\rho_j=dJ_j^{\mathrm F}/dJ_j^{\mathrm B}$ with $1$, keeping $\widehat V$ and $\chi$ fixed. Assume the reward and verifier bounds in \cref{eq:c-quantitative-weight-bounds}, and measure the correction error by
\begin{equation}
\varepsilon_{\rho,\mathrm{av}}
:=\max_{1\leq j\leq J}
\left(\bbE_{J_j^{\mathrm B}}[(\rho_j-1)^2]\right)^{1/2}<\infty.
\label{eq:c-edge-error}
\end{equation}
To turn this average edge error into an operator bound, let $P_{\mathrm{move}}$ and $\bar P_{\mathrm{move}}$ retain only checkpoint
moves, with their original probabilities. Assume they have densities
with respect to the exact stationary law $\Pi$:
\begin{equation}
P_{\mathrm{move}}(z,dz')=p_{\mathrm{move}}(z,z')\Pi(dz'),\qquad
\bar P_{\mathrm{move}}(z,dz')=\bar p_{\mathrm{move}}(z,z')\Pi(dz'),
\label{eq:c-edge-moving-densities}
\end{equation}
for $\Pi$-almost every $z$. For at least one finite $p>2$, suppose
\begin{equation}
H_{\rho,p}:=\left(\iint[p_{\mathrm{move}}(z,z')+\bar p_{\mathrm{move}}(z,z')]^p
\,\Pi(dz)\Pi(dz')\right)^{1/p}<\infty.
\label{eq:c-edge-moving-density-moment}
\end{equation}
Combining this density moment with the average correction error $\varepsilon_{\rho,\mathrm{av}}$, define
\begin{equation}
\Delta_{\rho,\mathrm{av}}
:=\inf_{\substack{2<p<\infty\\H_{\rho,p}<\infty}}
H_{\rho,p}^{\,p/[2(p-1)]}
\left(\frac{\varepsilon_{\rho,\mathrm{av}}}{\sqrt{\min\{r_{\min},v_{\min}\}}}\right)^{(p-2)/(p-1)}.
\label{eq:c-average-rho-operator-budget}
\end{equation}
where $r_{\min}$ and $v_{\min}$ are the positive lower bounds on the reward and process verifier, respectively.

Here, $\|\cdot\|_{2\to2}$ and all adjoints are taken in $L^2(\Pi)$.

\begin{lemma_ap}[Average edge error controls the movement operator]
\label{lem:c-practical-operator-bound}
Under \cref{eq:c-quantitative-weight-bounds,eq:c-edge-error,eq:c-edge-moving-density-moment},
$\bar P$ is bounded on $L^2(\Pi)$ and
\begin{equation}
\|\bar P_{\mathrm{move}}-P_{\mathrm{move}}\|_{2\to2}
\leq\|\bar p_{\mathrm{move}}-p_{\mathrm{move}}\|_{L^2(\Pi\otimes\Pi)}
\leq\Delta_{\rho,\mathrm{av}}.
\label{eq:c-practical-operator-error}
\end{equation}
The same bound holds for
$\|\bar M-M\|_{2\to2}$, where
$M=(P_{\mathrm{move}}+P_{\mathrm{move}}^*)/2$ and
$\bar M=(\bar P_{\mathrm{move}}+\bar P_{\mathrm{move}}^*)/2$.
\end{lemma_ap}

\begin{proof}
Write $b_*:=\min\{r_{\min},v_{\min}\}>0$ and define the uncorrected weighted edge measures
\begin{equation}
\widehat C_1^{\,\rho=1}(dx,dy):=r(y)J_1^{\mathrm B}(dx,dy),\quad
\widehat C_j^{\,\rho=1}(dx,dy)
:=\sqrt{\widehat V_j(x)\widehat V_{j-1}(y)}J_j^{\mathrm B}(dx,dy),
\; 2\leq j\leq J.
\label{eq:c-uncorrected-weighted-edges}
\end{equation}
Then $\widehat C_j=\rho_j\widehat C_j^{\,\rho=1}$ and
$\widehat C_j^{\,\rho=1}\leq J_j^{\mathrm B}$.
Writing $\mathcal E_{\rho,j}:=\int(\rho_j-1)^2\,d\widehat C_j^{\,\rho=1}$,
the reward and verifier bounds give
\begin{equation*}
\mathcal E_{\rho,j}\leq \varepsilon_{\rho,\mathrm{av}}^2,\qquad
\calZ=3\widehat C_1(\mathcal X,\mathcal X)
+4\sum_{j=2}^{J-1}\widehat C_j(\mathcal X,\mathcal X)
+3\widehat C_J(\mathcal X,\mathcal X)\geq(4J-2)b_*.
\end{equation*}
Hence
\begin{equation}
\frac{\mathcal E_{\rho,1}+4\sum_{j=2}^J\mathcal E_{\rho,j}}{\calZ}
\leq\frac{4J-3}{(4J-2)b_*}\,\varepsilon_{\rho,\mathrm{av}}^2
\leq\frac{\varepsilon_{\rho,\mathrm{av}}^2}{b_*}.
\label{eq:c-reference-to-weighted-rho-error}
\end{equation}
For probability measures $R,S$, define the \emph{triangular discrimination} and \emph{squared Hellinger distance}, respectively, by
\begin{equation*}
\mathcal T(R,S):=\int\frac{(dR-dS)^2}{dR+dS},\qquad
\mathcal H^2(R,S):=\int(\sqrt{dR}-\sqrt{dS})^2,
\end{equation*}
For densities $u,v$ with respect to a common dominating measure, $(u-v)^2/(u+v)=(\sqrt u-\sqrt v)^2(\sqrt u+\sqrt v)^2/(u+v)\leq2(\sqrt u-\sqrt v)^2$, since $(\sqrt u+\sqrt v)^2\leq2(u+v)$. Integrating gives $\mathcal T(R,S)\leq2\mathcal H^2(R,S)$.
For finite nonzero measures $A,B$ with total masses $a,b$, respectively, let $c_0:=\int\sqrt{d(A/a)d(B/b)}\in[0,1]$. Then
\begin{equation}
a\,\mathcal T(A/a,B/b)
\overset{(\circ)}{\leq}2a\mathcal H^2(A/a,B/b)
=4a(1-c_0)
\overset{(\triangle)}{\leq}4\int(\sqrt{dA}-\sqrt{dB})^2,
\label{eq:c-normalized-triangular-bound}
\end{equation}
where $(\circ)$ uses $\mathcal T\leq2\mathcal H^2$, the equality follows from $\mathcal H^2(A/a,B/b)=2(1-c_0)$, and $(\triangle)$ uses $\int(\sqrt{dA}-\sqrt{dB})^2=a+b-2\sqrt{ab}c_0=(\sqrt b-\sqrt a\,c_0)^2+a(1-c_0^2)\geq a(1-c_0)$.

Fix an interior source $x$ at checkpoint $j$. On disjoint destination copies labelled F and B, define
\begin{align*}
A^\rho_{j,x}(dD,dy)
&:=a_j^{\mathrm F}(x)\delta_{\mathrm F}(dD)Q_j^{\mathrm F}(dy\mid x)
 +\sqrt{\widehat V_j(x)}\delta_{\mathrm B}(dD)
 K_j^{\mathrm B}(dy\mid x)\sqrt{\widehat V_{j+1}(y)}\rho_{j+1}(y,x),\\
A^{\rho=1}_{j,x}(dD,dy)
&:=a_j^{\mathrm F}(x)\delta_{\mathrm F}(dD)Q_j^{\mathrm F}(dy\mid x)
 +\sqrt{\widehat V_j(x)}\delta_{\mathrm B}(dD)
 K_j^{\mathrm B}(dy\mid x)\sqrt{\widehat V_{j+1}(y)}.
\end{align*}
For this fixed $(j,x)$, write $a:=A^{\rho}_{j,x}(\{\mathrm F,\mathrm B\}\times\calX)=a_j^{\mathrm F}(x)+a_j^{\mathrm B}(x)$ for the total mass of $A^\rho_{j,x}$ and $b:=A^{\rho=1}_{j,x}(\{\mathrm F,\mathrm B\}\times\calX)$. Dividing by these total masses gives the probabilities of moving in each current direction:
\begin{align*}
P_{\mathrm{move}}((j,x,\mathrm F),(j-1,dy,\mathrm F))
&=\frac{A^\rho_{j,x}(\mathrm F,dy)}{a},\\
P_{\mathrm{move}}((j,x,\mathrm B),(j+1,dy,\mathrm B))
&=\frac{A^\rho_{j,x}(\mathrm B,dy)}{a}.
\end{align*}
The corresponding probabilities for $\bar P$ are obtained by replacing $A^\rho_{j,x}/a$ with $A^{\rho=1}_{j,x}/b$. Moreover, each current direction has stationary measure
\begin{equation*}
\Pi(j,dx,\mathrm F)=\Pi(j,dx,\mathrm B)=\frac{a\,\pi_j(dx)}{\calZ}.
\end{equation*}
For each fixed $(j,x)$, set $R:=A^\rho_{j,x}/a$ and $S:=A^{\rho=1}_{j,x}/b$. Splitting the defining integral of $\mathcal T(R,S)$ over the F and B destinations gives
\begin{align*}
\mathcal T(R,S)
&=\sum_{d\in\{\mathrm F,\mathrm B\}}\int_{\calX}
\frac{[dR(d,y)-dS(d,y)]^2}{dR(d,y)+dS(d,y)}\\
&=\sum_{d\in\{\mathrm F,\mathrm B\}}\int_{\mathsf S}
\frac{[\bar p_{\mathrm{move}}((j,x,d),z')-p_{\mathrm{move}}((j,x,d),z')]^2}
{\bar p_{\mathrm{move}}((j,x,d),z')+p_{\mathrm{move}}((j,x,d),z')}\,\Pi(dz'),
\end{align*}
where the second equality uses the moving densities $p_{\mathrm{move}}$ and $\bar p_{\mathrm{move}}$ with respect to the common dominating measure $\Pi$, with $0/0:=0$. Multiplying by $\Pi(j,dx,d)=a\,\pi_j(dx)/\calZ$, which is the same for both directions, and integrating over $x$ therefore gives
\begin{align*}
&\sum_{d\in\{\mathrm F,\mathrm B\}}\int_{\calX}\Pi(j,dx,d)
\int_{\mathsf S}\frac{[\bar p_{\mathrm{move}}((j,x,d),z')-p_{\mathrm{move}}((j,x,d),z')]^2}
{\bar p_{\mathrm{move}}((j,x,d),z')+p_{\mathrm{move}}((j,x,d),z')}\,\Pi(dz')\\
&\qquad=\frac1{\calZ}\int_{\calX}a\,\mathcal T(R,S)\,\pi_j(dx)
=\frac1{\calZ}\int_{\calX}a\,\mathcal T(A^\rho_{j,x}/a,A^{\rho=1}_{j,x}/b)\,\pi_j(dx)\\
&\qquad\leq\frac4{\calZ}\int_{\calX}\pi_j(dx)
\int_{\{\mathrm F,\mathrm B\}\times\calX}
\left(\sqrt{dA^\rho_{j,x}}-\sqrt{dA^{\rho=1}_{j,x}}\right)^2,
\end{align*}
where the inequality is \cref{eq:c-normalized-triangular-bound}. The F terms agree, so only the B terms remain in the last integral. Since $(\sqrt\rho-1)^2\leq(\rho-1)^2$,
\begin{align*}
&\frac4{\calZ}\int_{\calX}\pi_j(dx)
\int_{\{\mathrm F,\mathrm B\}\times\calX}
\left(\sqrt{dA^\rho_{j,x}}-\sqrt{dA^{\rho=1}_{j,x}}\right)^2\\
&\qquad=\frac4{\calZ}\iint\sqrt{\widehat V_j(x)\widehat V_{j+1}(y)}
\left(\sqrt{\rho_{j+1}(y,x)}-1\right)^2\pi_j(dx)K_j^{\mathrm B}(dy\mid x)\\
&\qquad\leq\frac4{\calZ}\iint[\rho_{j+1}(y,x)-1]^2
\,\widehat C_{j+1}^{\,\rho=1}(dy,dx)
=\frac{4\mathcal E_{\rho,j+1}}{\calZ}.
\end{align*}
At the clean boundary, fix a clean sample $y$ and compare the two renoising distributions $R:=K_0^{\mathrm B,\rho}(\cdot\mid y)$ and $S:=K_0^{\mathrm B}(\cdot\mid y)$. Both are probability measures, and $dR/dS=\rho_1(\cdot,y)$ gives
\begin{equation*}
\mathcal T(R,S)
=\int_{\calX}\frac{[\rho_1(x,y)-1]^2}{\rho_1(x,y)+1}\,K_0^{\mathrm B}(dx\mid y)
\leq\int_{\calX}[\rho_1(x,y)-1]^2\,K_0^{\mathrm B}(dx\mid y).
\end{equation*}
Integrating against $\Pi(0,dy,\mathrm B)=r(y)\pi_0(dy)/\calZ$ gives
\begin{equation*}
\int\mathcal T(R,S)\,\Pi(0,dy,\mathrm B)
\begin{aligned}[t]
&\leq\frac1{\calZ}\iint[\rho_1(x,y)-1]^2r(y)\pi_0(dy)K_0^{\mathrm B}(dx\mid y)\\
&=\frac1{\calZ}\int(\rho_1-1)^2\,d\widehat C_1^{\,\rho=1}
=\frac{\mathcal E_{\rho,1}}{\calZ}.
\end{aligned}
\end{equation*}
At checkpoint $J$, both kernels force an F move with the same transition $Q_J^{\mathrm F}$, so this boundary contributes no error.
Consequently,
\begin{equation}
\iint\frac{[\bar p_{\mathrm{move}}(z,z')-p_{\mathrm{move}}(z,z')]^2}{\bar p_{\mathrm{move}}(z,z')+p_{\mathrm{move}}(z,z')}
\,\Pi(dz)\Pi(dz')
\leq\frac{\mathcal E_{\rho,1}+4\sum_{j=2}^J\mathcal E_{\rho,j}}{\calZ}
\leq\frac{\varepsilon_{\rho,\mathrm{av}}^2}{b_*}.
\label{eq:c-moving-triangular-error}
\end{equation}
Since $|\bar p_{\mathrm{move}}-p_{\mathrm{move}}|\leq\bar p_{\mathrm{move}}+p_{\mathrm{move}}$, H\"older's inequality gives, for $p>2$,
\begin{align*}
&\iint|\bar p_{\mathrm{move}}-p_{\mathrm{move}}|^2\,d(\Pi\otimes\Pi)\\
&\quad\leq\iint\left(\frac{|\bar p_{\mathrm{move}}-p_{\mathrm{move}}|^2}{\bar p_{\mathrm{move}}+p_{\mathrm{move}}}\right)^{\frac{p-2}{p-1}}
(\bar p_{\mathrm{move}}+p_{\mathrm{move}})^{\frac{p}{p-1}}
\,d(\Pi\otimes\Pi)\\
&\quad\leq\left(\iint\frac{|\bar p_{\mathrm{move}}-p_{\mathrm{move}}|^2}{\bar p_{\mathrm{move}}+p_{\mathrm{move}}}\,d(\Pi\otimes\Pi)\right)^{\frac{p-2}{p-1}}
\left(\iint(\bar p_{\mathrm{move}}+p_{\mathrm{move}})^p\,d(\Pi\otimes\Pi)\right)^{\frac1{p-1}}.
\end{align*}
Taking square roots proves
$\|\bar p_{\mathrm{move}}-p_{\mathrm{move}}\|_2\leq
H_{\rho,p}^{p/[2(p-1)]}(\varepsilon_{\rho,\mathrm{av}}/\sqrt{b_*})^{(p-2)/(p-1)}$.
Applying Cauchy--Schwarz to the integral defining $(\bar P_{\mathrm{move}}-P_{\mathrm{move}})f$ gives
$\|(\bar P_{\mathrm{move}}-P_{\mathrm{move}})f\|_{2,\Pi}\leq\|\bar p_{\mathrm{move}}-p_{\mathrm{move}}\|_{L^2(\Pi\otimes\Pi)}\|f\|_{2,\Pi}$, which bounds the operator norm by the same $L^2$ norm. Taking the infimum over admissible $p$ gives
\cref{eq:c-practical-operator-error}. Since $\|A^*\|_{2\to2}=\|A\|_{2\to2}$, symmetrization gives $\|\bar M-M\|_{2\to2}\leq\tfrac12(\|\bar P_{\mathrm{move}}-P_{\mathrm{move}}\|_{2\to2}+\|(\bar P_{\mathrm{move}}-P_{\mathrm{move}})^*\|_{2\to2})=\|\bar P_{\mathrm{move}}-P_{\mathrm{move}}\|_{2\to2}\leq\Delta_{\rho,\mathrm{av}}$, so the same bound holds for $\bar M-M$.

Finally, write $\bar P=\bar P_{\mathrm{move}}+\bar S$, where $\bar S$ contains flips and stays. At each interior $(j,x)$, it acts on the two directional values of $f$ through
\begin{equation}
\bar S_j(x)=
\begin{pmatrix}
\chi\bar\ell_j(x)&\bar p_j^{\mathrm B}(x)-\chi\bar\ell_j(x)\\
\bar p_j^{\mathrm F}(x)-\chi\bar\ell_j(x)&\chi\bar\ell_j(x)
\end{pmatrix},\qquad
\bar\ell_j(x):=\min\{\bar p_j^{\mathrm F}(x),\bar p_j^{\mathrm B}(x)\},
\label{eq:c-uncorrected-local-block}
\end{equation}
where $\bar p_j^d$ denotes the directional probability for $\bar P$. All entries are nonnegative, and the row sums are $(\bar p_j^{\mathrm B},\bar p_j^{\mathrm F})$, while the column sums are $(\bar p_j^{\mathrm F},\bar p_j^{\mathrm B})$. Thus the Euclidean operator norm satisfies $\|\bar S_j(x)\|_2\leq\sqrt{\|\bar S_j(x)\|_1\|\bar S_j(x)\|_\infty}\leq1$. Since $\Pi(j,dx,\mathrm F)=\Pi(j,dx,\mathrm B)$ and $\bar S$ vanishes at the boundaries, integrating this inequality gives $\|\bar S f\|_{2,\Pi}\leq\|f\|_{2,\Pi}$. Consequently, $\|\bar P\|_{2\to2}\leq\|\bar P_{\mathrm{move}}\|_{2\to2}+1<\infty$, which proves the boundedness of $\bar{P}$.
\end{proof}

\paragraph{Stationary bias under average edge error.}
We now use the movement error bound in \cref{lem:c-practical-operator-bound} to control the stationary bias caused by omitting $\rho_j$. To do so, we compare $\Delta_{\rho,\mathrm{av}}$ with the lower bound on the exact chain's movement form established in the proof of \cref{thm:c-average-moment-poincare}:
\begin{equation}
\langle f,(cI-M)f\rangle_\Pi
\geq\gamma_{\mathrm{CS}}\operatorname{Var}_\Pi(f),
\qquad
c(j,x,d):=\begin{cases}1,&j\in\{0,J\},\\1/2,&1\leq j<J,\end{cases}
\label{eq:c-practical-movement-gap}
\end{equation}
Here $M=(P_{\mathrm{move}}+P_{\mathrm{move}}^*)/2$, as above, and we take $0<\gamma_{\mathrm{CS}}\leq1/2$ by reducing the lower bound if necessary. Under their respective assumptions, \cref{prop:c-explicit-lifted-poincare,thm:c-average-verifier-gap} provide alternative lower bounds for $\gamma_{\mathrm{CS}}$ in \cref{eq:c-practical-movement-gap}. The result below shows that if the movement error is smaller than this lower bound, the uncorrected chain has a stationary law close to $\Pi$.

\begin{theorem_ap}[Stability under average omitted-edge error]
\label{thm:c-practical-bias}
Assume \cref{eq:c-quantitative-weight-bounds,eq:c-edge-error,eq:c-edge-moving-density-moment,eq:c-practical-movement-gap} and $\Delta_{\rho,\mathrm{av}}<\gamma_{\mathrm{CS}}$.
Then $\bar P$ has a unique invariant probability $\bar\Pi\ll\Pi$ among those with an $L^2(\Pi)$ density, and
\begin{equation}
\|\bar\Pi-\Pi\|_{\mathrm{TV}}
\leq\frac{\Delta_{\rho,\mathrm{av}}}{\gamma_{\mathrm{CS}}-\Delta_{\rho,\mathrm{av}}}.
\label{eq:c-practical-bias}
\end{equation}
If $\Delta_{\rho,\mathrm{av}}/(\gamma_{\mathrm{CS}}-\Delta_{\rho,\mathrm{av}})<\Pi(j=0)=Z/\calZ$, then the stationary clean conditional
$\bar\pi_0$ is well-defined and
\begin{equation}
\|\bar\pi_0-\pi_0^\star\|_{\mathrm{TV}}
\leq\frac{\Delta_{\rho,\mathrm{av}}}{\Pi(j=0)\,(\gamma_{\mathrm{CS}}-\Delta_{\rho,\mathrm{av}})}.
\label{eq:c-average-rho-clean-bias}
\end{equation}
These bias bounds hold for every $\chi\in[0,1]$, with the same error budget
and movement-gap constant.
\end{theorem_ap}

\begin{proof}
For the proof, write $g_\rho:=\gamma_{\mathrm{CS}}-\Delta_{\rho,\mathrm{av}}>0$. Recall $\bar P=\bar P_{\mathrm{move}}+\bar S$ and the matrix in \cref{eq:c-uncorrected-local-block}. Since the F and B states at each $(j,x)$ have equal $\Pi$ weight, the adjoint $\bar S^*$ acts there through $\bar S_j(x)^{\mathsf T}$. By the definition of the inner product,
\begin{equation*}
\operatorname{Re}\langle f,(I-\bar S)f\rangle_\Pi
=\int_{\mathsf S}\operatorname{Re}\!\left[\overline{f(z)}\bigl(f(z)-(\bar Sf)(z)\bigr)\right]\Pi(dz).
\end{equation*}
At an interior $(j,x)$, recall $f_j^d(x)=f(j,x,d)$. With the argument $x$ suppressed, the flip/stay probabilities give
\begin{align*}
(\bar Sf)_j^{\mathrm F}
&=\chi\bar\ell_j f_j^{\mathrm F}+(\bar p_j^{\mathrm B}-\chi\bar\ell_j)f_j^{\mathrm B}, \qquad (\bar Sf)_j^{\mathrm B}
=(\bar p_j^{\mathrm F}-\chi\bar\ell_j)f_j^{\mathrm F}+\chi\bar\ell_j f_j^{\mathrm B}.
\end{align*}
Substituting these expressions into the inner-product integrand and using $\bar p_j^{\mathrm F}+\bar p_j^{\mathrm B}=1$ yields
\begin{align*}
\sum_{d\in\{\mathrm F,\mathrm B\}}\operatorname{Re}\!\left[\overline{f_j^d}\bigl(f_j^d-(\bar Sf)_j^d\bigr)\right] &=(1-\chi\bar\ell_j)(|f_j^{\mathrm F}|^2+|f_j^{\mathrm B}|^2)
-(1-2\chi\bar\ell_j)\operatorname{Re}(\overline{f_j^{\mathrm F}}f_j^{\mathrm B})\\
&=\tfrac12(|f_j^{\mathrm F}|^2+|f_j^{\mathrm B}|^2)
+(\tfrac12-\chi\bar\ell_j)|f_j^{\mathrm F}-f_j^{\mathrm B}|^2.
\end{align*}
At the boundaries, $\bar S=0$ leaves $|f|^2$. Thus the first terms combine into $\langle f,cf\rangle_\Pi$, with $c=1/2$ in the interior and $c=1$ at the boundaries. Integrating the two interior directional terms against their common measure $\mu_j(dx)/\calZ$ yields
\begin{align}
\operatorname{Re}\langle f,(I-\bar S)f\rangle_\Pi
&\mkern-3mu=\mkern-3mu\langle f,cf\rangle_\Pi
\mkern-2mu+\mkern-2mu\sum_{j=1}^{J-1}\mkern-3mu\int_{\calX}\mkern-3mu\left(\tfrac12-\chi\bar\ell_j(x)\right)
\mkern-2mu|f(j,x,\mathrm F)-f(j,x,\mathrm B)|^2\,\frac{\mu_j(dx)}{\calZ}\mkern-3mu\geq\mkern-3mu\langle f,cf\rangle_\Pi,
\label{eq:c-positive-direction-block}
\end{align}
since $\chi\bar\ell_j\leq1/2$. Thus \cref{eq:c-practical-movement-gap,eq:c-practical-operator-error} give, for $\Pi(f)=0$,
\begin{align}
\operatorname{Re}\langle f,(I-\bar P)f\rangle_\Pi
&\geq\langle f,cf\rangle_\Pi-\operatorname{Re}\langle f,\bar P_{\mathrm{move}}f\rangle_\Pi\notag\\
&=\langle f,cf\rangle_\Pi-\left\langle f,\frac{\bar P_{\mathrm{move}}+\bar P_{\mathrm{move}}^*}{2}f\right\rangle_\Pi
=\langle f,(cI-\bar M)f\rangle_\Pi\notag\\
&=\langle f,(cI-M)f\rangle_\Pi
 -\langle f,(\bar M-M)f\rangle_\Pi\notag\\
&\geq(\gamma_{\mathrm{CS}}-\Delta_{\rho,\mathrm{av}})\|f\|_{2,\Pi}^2
=g_\rho\|f\|_{2,\Pi}^2.
\label{eq:c-average-rho-coercivity}
\end{align}
The same inequality holds with $\bar P^*$ in place of $\bar P$.

Also, $\bar S^*\mathbf1=\bar P_{\mathrm{move}}\mathbf1+\mathbf1-2c$ and $M\mathbf1=c$, so
\begin{align}
b_\rho:=\bar P^*\mathbf1-\mathbf1
&=\bar P_{\mathrm{move}}^*\mathbf1+\bar S^*\mathbf1-\mathbf1=2(\bar M-M)\mathbf1, \quad \Pi(b_\rho)=0,\quad \|b_\rho\|_{2,\Pi}\leq2\Delta_{\rho,\mathrm{av}}.
\label{eq:c-average-rho-stationarity-defect}
\end{align}
To construct an invariant density, set
\begin{equation*}
T_t:=e^{t(\bar P-I)}
=e^{-t}\sum_{n=0}^{\infty}\frac{t^n}{n!}\bar P^n.
\end{equation*}
Because $T_t$ is positive and $T_t\mathbf1=\mathbf1$, $T_t^*$ maps the closed set of nonnegative, unit-integral $L^2(\Pi)$ densities into itself, i.e.,
\begin{equation*}
h\geq0,\quad \Pi(h)=1,\quad h\in L^2(\Pi)
\quad\Longrightarrow\quad
T_t^*h\geq0,\quad \Pi(T_t^*h)=1,\quad T_t^*h\in L^2(\Pi).
\end{equation*}
The set of these densities is complete as a closed subset of $L^2(\Pi)$. To show contraction, consider any $f\in L^2(\Pi)$ with $\Pi(f)=0$ and put $f_t:=T_t^*f$. Then $\Pi(f_t)=0$, and \cref{eq:c-average-rho-coercivity} gives
\begin{equation}
\frac{d}{dt}\|f_t\|_{2,\Pi}^2
=2\operatorname{Re}\langle f_t,(\bar P^*-I)f_t\rangle_\Pi
\leq-2g_\rho\|f_t\|_{2,\Pi}^2,
\quad \|T_t^*f\|_{2,\Pi}=\|f_t\|_{2,\Pi}\leq e^{-g_\rho t}\|f\|_{2,\Pi}.
\label{eq:c-average-rho-semigroup-contraction}
\end{equation}
Applying this bound to the difference of two probability densities shows that, for fixed $t>0$, $T_t^*$ is a strict contraction on the complete set of densities above. The contraction mapping theorem~\citep[Section~5.1]{kreyszig1978functional} therefore gives a unique fixed density $\bar h_t$. For every $s\geq0$, commutativity gives
\begin{equation*}
T_t^*(T_s^*\bar h_t)=T_s^*(T_t^*\bar h_t)=T_s^*\bar h_t.
\end{equation*}
Thus $T_s^*\bar h_t$ is also a fixed density of $T_t^*$, so uniqueness gives $T_s^*\bar h_t=\bar h_t$. For $s>0$, uniqueness for $T_s^*$ then gives $\bar h_t=\bar h_s$. Denote this common density by $\bar h$. Differentiating $T_s^*\bar h=\bar h$ at $s=0$ yields $(\bar P^*-I)\bar h=0$, hence $\bar P^*\bar h=\bar h$.

Define $\bar\Pi(dz):=\bar h(z)\Pi(dz)$. Then $\bar P^*\bar h=\bar h$ implies that $\bar\Pi$ is stationary for $\bar P$, with $d\bar\Pi/d\Pi=\bar h$.

Set $f:=\bar h-1$. The invariant equation gives
$(I-\bar P^*)f=b_\rho$. Since $\Pi(f)=0$, applying \cref{eq:c-average-rho-coercivity} and then Cauchy--Schwarz with \cref{eq:c-average-rho-stationarity-defect} gives
\begin{equation*}
 g_\rho\|f\|_{2,\Pi}^2
 \leq\operatorname{Re}\langle f,b_\rho\rangle_\Pi
 \leq2\Delta_{\rho,\mathrm{av}}\|f\|_{2,\Pi}.
\end{equation*}
Consequently,
\begin{equation*}
\|\bar h-1\|_{2,\Pi}\leq\frac{2\Delta_{\rho,\mathrm{av}}}{g_\rho},\qquad
\|\bar\Pi-\Pi\|_{\mathrm{TV}}=\frac12\Pi(|\bar h-1|)
\leq\frac12\|\bar h-1\|_{2,\Pi}\leq\frac{\Delta_{\rho,\mathrm{av}}}{g_\rho}.
\end{equation*}
The condition $\Delta_{\rho,\mathrm{av}}/g_\rho<\Pi(j=0)$ ensures $\bar\Pi(j=0)\geq\Pi(j=0)-\Delta_{\rho,\mathrm{av}}/g_\rho>0$, and \cref{eq:c-clean-conditioning-transfer} yields
\begin{equation*}
\|\bar\pi_0-\pi_0^\star\|_{\mathrm{TV}}
\leq\frac{\|\bar\Pi-\Pi\|_{\mathrm{TV}}}{\Pi(j=0)}
\leq\frac{\Delta_{\rho,\mathrm{av}}}{\Pi(j=0)\,g_\rho}.
\end{equation*}
\end{proof}

\paragraph{Convergence under edge-correction omission.}
Following the lazification in \cref{eq:c-mixing-lazy-kernel}, we denote by $\bar P_{\mathrm L}:=(I+\bar P)/2$ the $1/2$-lazy version of $\bar P$. We now show geometric convergence to $\bar\Pi$ and combine it with the stationary-bias bound to control the total error relative to $\Pi$.

\begin{corollary_ap}[Exponential convergence under average edge error]
\label{cor:c-average-rho-lazy-convergence}
Under the assumptions of \cref{thm:c-practical-bias}, there is a finite constant $C_{\rho,\chi}$ such that, for every initial law $\nu\ll\Pi$ with $d\nu/d\Pi\in L^2(\Pi)$ and $n\geq0$,
\begin{equation}
\|\nu\bar P_{\mathrm L}^n-\bar\Pi\|_{\mathrm{TV}}
\leq\frac{C_{\rho,\chi}}2e^{-(\gamma_{\mathrm{CS}}-\Delta_{\rho,\mathrm{av}})n/8}
\left\|\frac{d\nu}{d\Pi}-\frac{d\bar\Pi}{d\Pi}\right\|_{2,\Pi}.
\label{eq:c-average-rho-lazy-convergence}
\end{equation}
In particular, for $D_\nu:=\|d\nu/d\Pi-1\|_{2,\Pi}<\infty$,
\begin{equation}
\|\nu\bar P_{\mathrm L}^n-\Pi\|_{\mathrm{TV}}
\leq\underbrace{\frac{C_{\rho,\chi}}2e^{-(\gamma_{\mathrm{CS}}-\Delta_{\rho,\mathrm{av}})n/8}
\left(D_\nu+\frac{2\Delta_{\rho,\mathrm{av}}}{\gamma_{\mathrm{CS}}-\Delta_{\rho,\mathrm{av}}}\right)
+\frac{\Delta_{\rho,\mathrm{av}}}{\gamma_{\mathrm{CS}}-\Delta_{\rho,\mathrm{av}}}}_{(\star)}.
\label{eq:c-practical-total-error}
\end{equation}
If $(\star)<\Pi(j=0)$, the clean conditional of $\nu\bar P_{\mathrm L}^n$ is well-defined, and \cref{eq:c-clean-conditioning-transfer} bounds its total-variation distance from $\pi_0^\star$ by $(\star)/\Pi(j=0)$.
\end{corollary_ap}

\begin{proof}
For the proof, write $\bar h:=d\bar\Pi/d\Pi$ and $g_\rho:=\gamma_{\mathrm{CS}}-\Delta_{\rho,\mathrm{av}}>0$, and recall $\bar P=\bar P_{\mathrm{move}}+\bar S$, where $\bar S$ is the flip/stay operator from \cref{eq:c-uncorrected-local-block}, vanishing at the boundaries. We first bound the spectrum of $\bar P$, then obtain a strict bound for the lazy kernel on mean-zero functions.
Since $\Pi\otimes\Pi$ is a probability measure and $p>2$,
\begin{equation*}
\|\bar p_{\mathrm{move}}\|_{L^2(\Pi\otimes\Pi)}
\leq\|\bar p_{\mathrm{move}}\|_{L^p(\Pi\otimes\Pi)}<\infty.
\end{equation*}
To see compactness, approximate this density in $L^2(\Pi\otimes\Pi)$ by finite sums $\sum_{i=1}^N u_i(z)v_i(z')$ with real-valued $u_i,v_i$. The corresponding operators $A_N$ act as
\begin{equation*}
(A_Nf)(z):=\int\left[\sum_{i=1}^N u_i(z)v_i(z')\right]f(z')\,\Pi(dz')
=\sum_{i=1}^N u_i(z)\int v_i(z')f(z')\,\Pi(dz').
\end{equation*}
Their range lies in $\operatorname{span}\{u_1,\ldots,u_N\}$ and is therefore finite-dimensional. Cauchy--Schwarz gives
\begin{equation}
\left\|\bar P_{\mathrm{move}}-A_N\right\|_{2\to2}
\leq\left\|\bar p_{\mathrm{move}}-\sum_{i=1}^N u_i\otimes v_i\right\|_{L^2(\Pi\otimes\Pi)}\longrightarrow0.
\label{eq:c-practical-move-finite-rank-approximation}
\end{equation}
Thus $\bar P_{\mathrm{move}}$ is compact as an operator-norm limit of finite-rank operators, and so is $\bar P_{\mathrm{move}}^*$: its approximants satisfy $(A_N^*f)(z)=\sum_{i=1}^N v_i(z)\int u_i(z')f(z')\,\Pi(dz')$ and hence have finite rank, while $\|\bar P_{\mathrm{move}}^*-A_N^*\|_{2\to2}=\|\bar P_{\mathrm{move}}-A_N\|_{2\to2}\to0$ by \cref{eq:c-practical-move-finite-rank-approximation}.

We use the following consequence of the Fredholm alternative for compact operators~\citep[Theorem~8.7-2]{kreyszig1978functional}. If $A=B+K$, with $K$ compact, then every spectral value $\lambda$ of $A$ with $|\lambda|>\|B\|$ is an eigenvalue. Indeed, $\lambda I-B$ is invertible by its Neumann series, and
\begin{equation*}
\lambda I-A=(\lambda I-B)\left[I-(\lambda I-B)^{-1}K\right].
\end{equation*}
The operator $(\lambda I-B)^{-1}K$ is compact. By the Fredholm alternative, noninvertibility of the bracket implies a nonzero vector in its kernel, hence in the kernel of $\lambda I-A$.

The proof of \cref{thm:c-practical-bias} gives $T_t^*\bar h=\bar h$ and the exponential contraction on mean-zero functions in \cref{eq:c-average-rho-semigroup-contraction}. For any $f\in L^2(\Pi)$, write
\begin{equation*}
f=\Pi(f)\bar h+\bigl(f-\Pi(f)\bar h\bigr),\qquad
\Pi\bigl(f-\Pi(f)\bar h\bigr)=0.
\end{equation*}
Since $T_t^*\bar h=\bar h$, applying the contraction bound to $f-\Pi(f)\bar h$ gives
\begin{equation*}
\|T_t^*f\|_{2,\Pi}
\mkern-5mu\leq\mkern-5mu|\Pi(f)|\|\bar h\|_{2,\Pi}
+e^{-g_\rho t}\|f-\Pi(f)\bar h\|_{2,\Pi}
\mkern-5mu\leq\mkern-5mu\|f\|_{2,\Pi}+2|\Pi(f)|\|\bar h\|_{2,\Pi}
\mkern-5mu\leq\mkern-5mu(1+2\|\bar h\|_{2,\Pi})\|f\|_{2,\Pi},
\end{equation*}
where we used $e^{-g_\rho t}\leq1$, the triangle inequality, and $|\Pi(f)|\leq\|f\|_{2,\Pi}$. Taking the supremum over $\|f\|_{2,\Pi}=1$ and $t\geq0$ yields
\begin{equation}
\sup_{t\geq0}\|T_t\|_{2\to2}=\sup_{t\geq0}\|T_t^*\|_{2\to2}
\leq1+2\|\bar h\|_{2,\Pi}<\infty.
\label{eq:c-average-rho-semigroup-uniform-bound}
\end{equation}
Suppose $\bar P f=\lambda f$ for $f\ne0$ and $|\lambda|>1$. Positivity of $\bar{P}$ and the triangle inequality imply
\begin{equation*}
\bar P|f|\geq|\bar P f|=|\lambda||f|,
\qquad \bar P^n|f|\geq|\lambda|^n|f|.
\end{equation*}
Recalling $T_t=e^{-t}\sum_{n=0}^{\infty}(t^n/n!)\bar P^n$, we obtain
\begin{equation*}
T_t|f|=e^{-t}\sum_{n=0}^{\infty}\frac{t^n}{n!}\bar P^n|f|
\geq e^{-t}\sum_{n=0}^{\infty}\frac{(t|\lambda|)^n}{n!}|f|
=e^{(|\lambda|-1)t}|f|.
\end{equation*}
Consequently $\|T_t\|_{2\to2}\geq e^{(|\lambda|-1)t}$, contradicting the uniform bound in \cref{eq:c-average-rho-semigroup-uniform-bound} as $t\to\infty$. Since $\bar P=\bar S+\bar P_{\mathrm{move}}$, with $\|\bar S\|_{2\to2}\leq1$ and $\bar P_{\mathrm{move}}$ compact, the Fredholm argument above excludes all spectral values of $\bar P$ outside the unit disk:
\begin{equation*}
\sigma(\bar P)\subseteq\{\lambda\in\mathbb C:|\lambda|\leq1\},\qquad
\bar P\mathbf1=\mathbf1\ \Longrightarrow\ 1\in\sigma(\bar P),\qquad
r(\bar P):=\sup_{\lambda\in\sigma(\bar P)}|\lambda|=1=r(\bar P^*).
\end{equation*}
We next bound the part of the lazy kernel that does not move. By \cref{eq:c-positive-direction-block} and $c\geq1/2$,
\begin{equation*}
\operatorname{Re}\langle f,\bar S f\rangle_\Pi
\leq\|f\|_{2,\Pi}^2-\langle f,cf\rangle_\Pi
\leq\tfrac12\|f\|_{2,\Pi}^2.
\end{equation*}
Together with $\|\bar S\|_{2\to2}\leq1$, this yields
\begin{equation*}
\left\|\frac{I+\bar S}{2}f\right\|_{2,\Pi}^2
=\frac{\|f\|_{2,\Pi}^2+2\operatorname{Re}\langle f,\bar S f\rangle_\Pi+\|\bar S f\|_{2,\Pi}^2}{4}
\leq\frac34\|f\|_{2,\Pi}^2.
\end{equation*}
Let $L_0^2(\Pi):=\{f\in L^2(\Pi):\Pi(f)=0\}$. This subspace is invariant under $\bar P_{\mathrm L}^*$ because
\begin{equation*}
\Pi(\bar P_{\mathrm L}^*f)
=\langle\bar P_{\mathrm L}\mathbf1,f\rangle_\Pi=\Pi(f).
\end{equation*}
For its orthogonal projection $Qf:=f-\Pi(f)\mathbf1$, $\|Q\|_{2\to2}\leq1$, and hence on $L_0^2(\Pi)$,
\begin{equation*}
\bar P_{\mathrm L}^*\big|_{L_0^2(\Pi)}
=\left.Q\frac{I+\bar S^*}{2}\right|_{L_0^2(\Pi)}
+\left.\frac12Q\bar P_{\mathrm{move}}^*\right|_{L_0^2(\Pi)}.
\end{equation*}
The first term has norm at most $\sqrt3/2$ and the second is compact. The same Fredholm argument therefore shows that every spectral value of this restriction outside $|\lambda|\leq\sqrt3/2$ is an eigenvalue.
For such an eigenvalue, choose $0\ne f\in L_0^2(\Pi)$ with $\bar P_{\mathrm L}^*f=\lambda f$. Then
\begin{equation*}
\bar P^*f=(2\lambda-1)f,\qquad
\mu:=2\lambda-1,\qquad |\mu|\leq r(\bar P^*)=1.
\end{equation*}
Moreover, \cref{eq:c-average-rho-coercivity} gives
\begin{equation*}
g_\rho\|f\|_{2,\Pi}^2
\leq\operatorname{Re}\langle f,(I-\bar P^*)f\rangle_\Pi
=(1-\operatorname{Re}\mu)\|f\|_{2,\Pi}^2,
\qquad \operatorname{Re}\mu\leq1-g_\rho.
\end{equation*}
Thus
\begin{equation*}
|\lambda|^2=\frac{1+2\operatorname{Re}\mu+|\mu|^2}{4}
\leq\frac{1+2(1-g_\rho)+1}{4}=1-\frac{g_\rho}{2}.
\end{equation*}
Since $0<g_\rho\leq1/2$, we also have $3/4\leq1-g_\rho/2$. This bounds the entire spectrum of the restriction, giving
\begin{equation*}
r\left(\bar P_{\mathrm L}^*\big|_{L_0^2(\Pi)}\right)
\leq\sqrt{1-g_\rho/2}\leq e^{-g_\rho/4}<e^{-g_\rho/8}.
\end{equation*}
By \cref{lem:c-practical-operator-bound}, $\bar P$ is bounded on $L^2(\Pi)$, so $\bar P_{\mathrm L}^*\big|_{L_0^2(\Pi)}$ is a bounded linear operator on $L_0^2(\Pi)$. The spectral-radius formula~\citep[Theorem~7.5-5]{kreyszig1978functional} therefore gives
\begin{equation*}
\lim_{n\to\infty}\left\|(\bar P_{\mathrm L}^*)^n\big|_{L_0^2(\Pi)}\right\|_{2\to2}^{1/n}
=r\left(\bar P_{\mathrm L}^*\big|_{L_0^2(\Pi)}\right)<e^{-g_\rho/8}.
\end{equation*}
Hence the norm is at most $e^{-g_\rho n/8}$ for all sufficiently large $n$. Enlarging a finite constant $C_{\rho,\chi}$ to cover the remaining values of $n$ gives
\begin{equation*}
\left\|(\bar P_{\mathrm L}^*)^n f\right\|_{2,\Pi}
\leq C_{\rho,\chi}e^{-g_\rho n/8}\|f\|_{2,\Pi},
\qquad f\in L_0^2(\Pi),\quad n\geq0.
\end{equation*}
Finally, $d\nu/d\Pi-\bar h$ has zero $\Pi$ mean and $\bar P_{\mathrm L}^*\bar h=\bar h$. Therefore
\begin{align*}
\|\nu\bar P_{\mathrm L}^n-\bar\Pi\|_{\mathrm{TV}}
&=\frac12\Pi\left(\left|(\bar P_{\mathrm L}^*)^n\left(\frac{d\nu}{d\Pi}-\bar h\right)\right|\right)\\
&\leq\frac12\left\|(\bar P_{\mathrm L}^*)^n\left(\frac{d\nu}{d\Pi}-\bar h\right)\right\|_{2,\Pi}
\leq\frac{C_{\rho,\chi}}2e^{-g_\rho n/8}\left\|\frac{d\nu}{d\Pi}-\bar h\right\|_{2,\Pi}.
\end{align*}
The proof of \cref{thm:c-practical-bias}, together with the triangle inequality, also gives
\begin{equation*}
\left\|\frac{d\nu}{d\Pi}-\bar h\right\|_{2,\Pi}
\leq D_\nu+\|\bar h-1\|_{2,\Pi}
\leq D_\nu+\frac{2\Delta_{\rho,\mathrm{av}}}{g_\rho}.
\end{equation*}
Combining this with
$\|\nu\bar P_{\mathrm L}^n-\Pi\|_{\mathrm{TV}}
\leq\|\nu\bar P_{\mathrm L}^n-\bar\Pi\|_{\mathrm{TV}}+\|\bar\Pi-\Pi\|_{\mathrm{TV}}$
and \cref{eq:c-practical-bias} proves \cref{eq:c-practical-total-error}. The clean-conditional bound follows from \cref{eq:c-clean-conditioning-transfer}.
\end{proof}

Finally, we remark that the following condition on the reference edges $J_j^{\mathrm F}$ and $J_j^{\mathrm B}$ ensures $H_{\rho,p}<\infty$ (\cref{eq:c-edge-moving-density-moment}) without requiring $J_j^{\mathrm F}=J_j^{\mathrm B}$.

\begin{lemma_ap}[Sufficient conditions for finite movement-density moments]
\label{lem:c-edge-density-moment-condition}
Assume \cref{eq:c-quantitative-weight-bounds}. If $J_j^{\mathrm F},J_j^{\mathrm B}\ll\pi_j\otimes\pi_{j-1}$ and, for some finite $p>2$,
\begin{equation}
\max_{1\leq j\leq J}\left\{
\left\|\frac{dJ_j^{\mathrm F}}{d(\pi_j\otimes\pi_{j-1})}\right\|_{L^p(\pi_j\otimes\pi_{j-1})},\,
\left\|\frac{dJ_j^{\mathrm B}}{d(\pi_j\otimes\pi_{j-1})}\right\|_{L^p(\pi_j\otimes\pi_{j-1})}
\right\}<\infty,
\label{eq:c-reference-edge-density-moments}
\end{equation}
then the movement densities in \cref{eq:c-edge-moving-densities} exist and $H_{\rho,p}<\infty$.
\end{lemma_ap}

\begin{proof}
Set $b_*:=\min\{r_{\min},v_{\min}\}$. Applying \cref{eq:c-quantitative-weight-bounds} to the edge definitions in \cref{eq:cs-vgb-interior-conductance,eq:cs-vgb-clean-conductance} gives
\begin{equation*}
\widehat C_j(dx,dy)=\sqrt{\widehat V_j(x)\widehat V_{j-1}(y)}\,J_j^{\mathrm F}(dx,dy)
\quad\Longrightarrow\quad
b_*J_j^{\mathrm F}\leq\widehat C_j\leq J_j^{\mathrm F},\qquad j\geq2,
\end{equation*}
where the same bound holds for $j=1$ since $\widehat C_1(dx,dy)=r(y)J_1^{\mathrm F}(dx,dy)$ and $b_*\leq r(y)\leq1$.
Since the reference edge has marginals
\begin{equation*}
J_j^{\mathrm F}(dx,\mathcal X)=\pi_j(dx),\qquad
J_j^{\mathrm F}(\mathcal X,dy)=\pi_{j-1}(dy),
\end{equation*}
each of the two terms $\widehat C_j(dx,\mathcal X)$ and $\widehat C_{j+1}(\mathcal X,dx)$ defining $\mu_j(dx)$ at an interior checkpoint lies between $b_*\pi_j(dx)$ and $\pi_j(dx)$. Thus \cref{eq:c-incident-measures} gives
\begin{equation*}
\mu_j(dx)=\widehat C_j(dx,\mathcal X)+\widehat C_{j+1}(\mathcal X,dx)
\quad\Longrightarrow\quad
2b_*\pi_j\leq\mu_j\leq2\pi_j.
\end{equation*}
At either boundary only one edge contributes, giving $b_*\pi_j\leq\mu_j\leq\pi_j$. Hence, for every admissible $(j,d)$, \cref{eq:c-stationary-law} implies
\begin{equation*}
b_*\pi_j\leq\mu_j\leq2\pi_j,\qquad
\frac{b_*}{\calZ}\pi_j(dx)
\leq\Pi(j,dx,d)=\frac{\mu_j(dx)}{\calZ}
\leq\frac2{\calZ}\pi_j(dx).
\end{equation*}
The second bound above and the assumption $J_j^{\mathrm F},J_j^{\mathrm B}\ll\pi_j\otimes\pi_{j-1}$, which gives $K_j^{\mathrm F}(\cdot\mid x)\ll\pi_{j-1}$ and $K_{j-1}^{\mathrm B}(\cdot\mid y)\ll\pi_j$ for $\pi_j$-almost every $x$ and $\pi_{j-1}$-almost every $y$, respectively (also for the corrected B kernel), imply
\begin{equation*}
P_{\mathrm{move}}(z,\cdot)\ll\Pi,\qquad
\bar P_{\mathrm{move}}(z,\cdot)\ll\Pi
\quad\text{for }\Pi\text{-almost every }z,
\end{equation*}
so the two densities in \cref{eq:c-edge-moving-densities} exist.
For an F move with $2\leq j<J$, multiplying the move probability by the sampling transition gives
\begin{equation*}
\begin{aligned}
P_{\mathrm{move}}((j,x,\mathrm F),(j-1,dy,\mathrm F))
&=p_j^{\mathrm F}(x)Q_j^{\mathrm F}(dy\mid x) =\frac{\sqrt{\widehat V_j(x)\widehat V_{j-1}(y)}}{a_j^{\mathrm F}(x)+a_j^{\mathrm B}(x)}K_j^{\mathrm F}(dy\mid x).
\end{aligned}
\end{equation*}
Each directional mass $a_j^d$ integrates a weight in $[b_*,1]$ against a probability kernel, so $a_j^d\geq b_*$. Hence $a_j^{\mathrm F}+a_j^{\mathrm B}\geq2b_*$ and the numerator weight is at most $1$, giving $P_{\mathrm{move}}((j,x,\mathrm F),(j-1,dy,\mathrm F))\leq K_j^{\mathrm F}(dy\mid x)/b_*$. For $z=(j,x,\mathrm F)$ and $z'=(j-1,y,\mathrm F)$, taking densities and using $J_j^{\mathrm F}(dx,dy)=\pi_j(dx)K_j^{\mathrm F}(dy\mid x)$ yields
\begin{equation*}
p_{\mathrm{move}}(z,z')
\underbrace{\frac{d\Pi(j-1,\cdot,\mathrm F)}{d\pi_{j-1}}(y)}_{\displaystyle\geq b_*/\calZ}
\leq\frac1{b_*}\frac{dK_j^{\mathrm F}(\cdot\mid x)}{d\pi_{j-1}}(y)
=\frac1{b_*}\frac{dJ_j^{\mathrm F}}{d(\pi_j\otimes\pi_{j-1})}(x,y).
\end{equation*}
For the uncorrected F move, whose density relative to $\Pi$ is $\bar p_{\mathrm{move}}$, omitting $\rho$ leaves the numerator unchanged and may change the B mass in the denominator, which remains at least $b_*$. Dividing both movement measures by the same stationary reference measure $\Pi$ therefore gives
\begin{equation*}
p_{\mathrm{move}}(z,z')
\leq\frac{\calZ}{b_*^2}
\frac{dJ_j^{\mathrm F}}{d(\pi_j\otimes\pi_{j-1})}(x,y),\qquad
\bar p_{\mathrm{move}}(z,z')
\leq\frac{\calZ}{b_*^2}
\frac{dJ_j^{\mathrm F}}{d(\pi_j\otimes\pi_{j-1})}(x,y).
\end{equation*}
The same bound holds at boundaries because the forced transition $Q_j^d$ has denominator $a_j^d\geq b_*$, and on the clean edge because $r\leq1$, with directions reset inward. For B moves, including the correction $\rho$ gives $\pi_{j-1}(dy)K_{j-1}^{\mathrm B,\rho}(dx\mid y)=J_j^{\mathrm F}(dx,dy)$, whereas omitting it gives $\pi_{j-1}(dy)K_{j-1}^{\mathrm B}(dx\mid y)=J_j^{\mathrm B}(dx,dy)$. The same calculation therefore gives, for $z=(j-1,y,\mathrm B)$ and $z'=(j,x,\mathrm B)$ when $j<J$, with $z'=(J,x,\mathrm F)$ when $j=J$,
\begin{equation*}
p_{\mathrm{move}}(z,z')\leq\frac{\calZ}{b_*^2}
\frac{dJ_j^{\mathrm F}}{d(\pi_j\otimes\pi_{j-1})}(x,y),\qquad
\bar p_{\mathrm{move}}(z,z')\leq\frac{\calZ}{b_*^2}
\frac{dJ_j^{\mathrm B}}{d(\pi_j\otimes\pi_{j-1})}(x,y).
\end{equation*}
Adding the bounds for $p_{\mathrm{move}}$ and $\bar p_{\mathrm{move}}$ in each direction yields
\begin{equation*}
p_{\mathrm{move}}(z,z')+\bar p_{\mathrm{move}}(z,z')
\leq\frac{\calZ}{b_*^2}
\begin{cases}
2\dfrac{dJ_j^{\mathrm F}}{d(\pi_j\otimes\pi_{j-1})}(x,y),&\mathrm F\text{ move},\\[6pt]
\dfrac{dJ_j^{\mathrm F}}{d(\pi_j\otimes\pi_{j-1})}(x,y)
+\dfrac{dJ_j^{\mathrm B}}{d(\pi_j\otimes\pi_{j-1})}(x,y),&\mathrm B\text{ move}.
\end{cases}
\end{equation*}
For each allowed move between checkpoints $j$ and $j-1$, the stationary upper bound gives $\Pi(dz)\Pi(dz')\leq4\pi_j(dx)\pi_{j-1}(dy)/\calZ^2$. Using these bounds and $(a+b)^p\leq2^{p-1}(a^p+b^p)$, we integrate over $x,y$ and sum over all edges and both move directions to obtain
\begin{equation*}
H_{\rho,p}^p\leq C\sum_{j=1}^J\sum_{d\in\{\mathrm F,\mathrm B\}}
\left\|\frac{dJ_j^d}{d(\pi_j\otimes\pi_{j-1})}\right\|_{L^p(\pi_j\otimes\pi_{j-1})}^p<\infty
\end{equation*}
for some finite constant $C$, by \cref{eq:c-reference-edge-density-moments}.
\end{proof}

}

\section{Finite-Population Synchronized CS-VGB}
\label{app:subsec:synchronized-cs-vgb}

Running $N$ independent single-chain CS-VGB samplers explores several trajectories in parallel, but the chains generally occupy different checkpoints and directions. In this case, process-verifier values at different noise levels need not be directly comparable, making it difficult to allocate search effort across chains. We therefore synchronize the checkpoint and direction while retaining distinct particle states, placing all particles at a common noise level where their verifier values can be compared without differences in checkpoint calibration.

We first construct \textbf{Multi-chain Synchronized CS-VGB (MCS-CS-VGB)}, which preserves the product clean target $(\pi_0^\star)^{\otimes N}$ for every finite $N$ under the assumptions below. To adapt this synchronized structure to finite-budget search with one-step posterior maps, we share the highest-reward endpoint among $N m_e$ posterior draws, together with its parent, across particles; each particle then generates its own local candidate pool and selects the highest-weight candidate (\cref{eq:search-best-endpoint-sharing,eq:search-local-argmax}). This best-endpoint sharing and local selection heuristic combines parallel exploration with recurrent search, but does not inherit the stationary-target guarantee of the exact sampler developed below.

\paragraph{Exact MCS-CS-VGB kernel.}
For $N\geq1$, write $\mathbf x=(x^{(1)},\ldots,x^{(N)})$ and $\boldsymbol\pi_j:=\pi_j^{\otimes N}$. The population state is $(j,\mathbf x,d)$, where all particles share $j$ and $d$. Define the \emph{product edge measure}
\begin{equation}
\widehat C_j^{(N)}(d\mathbf x,d\mathbf y)
:=\bigotimes_{n=1}^{N}\widehat C_j(dx^{(n)},dy^{(n)}).
\label{eq:sync-product-conductance}
\end{equation}
Recalling the single-chain directional masses $a_j^{\mathrm F},a_j^{\mathrm B}$ from \cref{eq:cs-vgb-directional-masses}, the population directional masses at $1\leq j<J$ are
\begin{equation}
A_j^{\mathrm F}(\mathbf x):=\prod_{n=1}^{N}a_j^{\mathrm F}(x^{(n)}),
\qquad
A_j^{\mathrm B}(\mathbf x):=\prod_{n=1}^{N}a_j^{\mathrm B}(x^{(n)}),
\label{eq:sync-directional-masses}
\end{equation}
and
\begin{equation*}
p_j^{\mathrm F}:=\frac{A_j^{\mathrm F}}{A_j^{\mathrm F}+A_j^{\mathrm B}},
\qquad
p_j^{\mathrm B}:=\frac{A_j^{\mathrm B}}{A_j^{\mathrm F}+A_j^{\mathrm B}},
\qquad
\ell_j:=\min\{p_j^{\mathrm F},p_j^{\mathrm B}\}.
\end{equation*}
All particles share the action, but conditional on a move each draws from its own transition:
\begin{equation}
Q_j^{\mathrm F,(N)}(d\mathbf y\mid\mathbf x)
:=\prod_{n=1}^{N}Q_j^{\mathrm F}(dy^{(n)}\mid x^{(n)}),
\qquad
Q_j^{\mathrm B,\rho,(N)}(d\mathbf y\mid\mathbf x)
:=\prod_{n=1}^{N}Q_j^{\mathrm B,\rho}(dy^{(n)}\mid x^{(n)}).
\label{eq:sync-product-transitions}
\end{equation}
The action probabilities are those of \cref{tab:cs-vgb-momentum-update} with the population values $p_j^{\mathrm F},p_j^{\mathrm B},\ell_j$. These shared direction probabilities and the particlewise transitions above are induced by the same product edge measure in \cref{eq:sync-product-conductance}, whose marginals are
\begin{equation}
\widehat C_j^{(N)}(d\mathbf x,\mathcal X^{N})
=\boldsymbol\pi_j(d\mathbf x)A_j^{\mathrm F}(\mathbf x),
\qquad
\widehat C_{j+1}^{(N)}(\mathcal X^{N},d\mathbf x)
=\boldsymbol\pi_j(d\mathbf x)A_j^{\mathrm B}(\mathbf x).
\label{eq:sync-edge-marginals}
\end{equation}
Define
\begin{equation}
\begin{aligned}
\mu_0^{(N)}(d\mathbf x)&:=\widehat C_1^{(N)}(\mathcal X^{N},d\mathbf x),
&\mu_J^{(N)}(d\mathbf x)&:=\widehat C_J^{(N)}(d\mathbf x,\mathcal X^{N}),\\
\mu_j^{(N)}(d\mathbf x)&:=\boldsymbol\pi_j(d\mathbf x)\sbra{A_j^{\mathrm F}(\mathbf x)+A_j^{\mathrm B}(\mathbf x)},
&1&\leq j<J.
\end{aligned}
\label{eq:sync-incident-measures}
\end{equation}
and
\begin{equation}
\mathcal Z_{N}
:=\mu_0^{(N)}(\mathcal X^{N})
+2\sum_{j=1}^{J-1}\mu_j^{(N)}(\mathcal X^{N})
+\mu_J^{(N)}(\mathcal X^{N})\in(0,\infty).
\label{eq:sync-normalizer}
\end{equation}
\Needspace{6\baselineskip}
\begin{theorem_ap}[Stationarity of MCS-CS-VGB]
\label{thm:sync-stationarity}
Under the assumptions of \cref{thm:c-exact-stationarity}, the MCS-CS-VGB kernel has stationary law
\begin{equation}
\Pi_{N}(0,d\mathbf x,\mathrm B):=\frac{\mu_0^{(N)}(d\mathbf x)}{\mathcal Z_{N}},\quad
\Pi_{N}(J,d\mathbf x,\mathrm F):=\frac{\mu_J^{(N)}(d\mathbf x)}{\mathcal Z_{N}},\quad
\Pi_{N}(j,d\mathbf x,d):=\frac{\mu_j^{(N)}(d\mathbf x)}{\mathcal Z_{N}},
\label{eq:sync-stationary-law}
\end{equation}
where $1\leq j<J$ and $d\in\{\mathrm F,\mathrm B\}$. Its clean conditional is
\begin{equation}
\Pi_{N}(d\mathbf x\mid j=0)
=\bigl(\pi_0^\star\bigr)^{\otimes N}(d\mathbf x).
\label{eq:sync-clean-product-target}
\end{equation}
\end{theorem_ap}

\begin{proof}
Write $P_N$ for the MCS-CS-VGB kernel. We verify $\Pi_NP_N=\Pi_N$ by summing the one-update measures entering each state. At an interior state $(j,\mathbf x,\mathrm F)$, an F move arrives from $j+1$. Since $Q_{j+1}^{\mathrm F,(N)}$ is the conditional transition of the product edge measure, \cref{eq:sync-edge-marginals,eq:sync-incident-measures} gives
\begin{equation}
\begin{aligned}
\int_{\mathcal X^N}\Pi_N(j+1,d\mathbf x',\mathrm F)
p_{j+1}^{\mathrm F}(\mathbf x')
Q_{j+1}^{\mathrm F,(N)}(d\mathbf x\mid\mathbf x')=\frac1{\mathcal Z_N}\int_{\mathcal X^N}
\widehat C_{j+1}^{(N)}(d\mathbf x',d\mathbf x)
=\frac{\mu_j^{(N)}(d\mathbf x)p_j^{\mathrm B}(\mathbf x)}{\mathcal Z_N}.
\end{aligned}
\label{eq:sync-incoming-move}
\end{equation}
Here $p_J^{\mathrm F}=1$ for the forced move at the noise boundary. A B-to-F flip at the same $(j,\mathbf x)$ contributes $\Pi_N(j,d\mathbf x,\mathrm B)[p_j^{\mathrm F}(\mathbf x)-\chi\ell_j(\mathbf x)]$, and an F stay contributes $\Pi_N(j,d\mathbf x,\mathrm F)\chi\ell_j(\mathbf x)$. Thus
\begin{equation*}
\begin{aligned}
(\Pi_NP_N)(j,d\mathbf x,\mathrm F)
=\frac{\mu_j^{(N)}(d\mathbf x)}{\mathcal Z_N}
\sbra{p_j^{\mathrm B}+p_j^{\mathrm F}-\chi\ell_j+\chi\ell_j}=\Pi_N(j,d\mathbf x,\mathrm F).
\end{aligned}
\end{equation*}
The B case is symmetric. At the boundaries, the incoming measures are $\widehat C_1^{(N)}(\mathcal X^N,d\mathbf x)/\mathcal Z_N=\Pi_N(0,d\mathbf x,\mathrm B)$ and $\widehat C_J^{(N)}(d\mathbf x,\mathcal X^N)/\mathcal Z_N=\Pi_N(J,d\mathbf x,\mathrm F)$, respectively, proving stationarity. Finally,
\begin{equation*}
\mu_0^{(N)}(d\mathbf x)
=\bigotimes_{n=1}^{N}\sbra{r(x^{(n)})\pi_0(dx^{(n)})}
=Z^{N}\bigl(\pi_0^\star\bigr)^{\otimes N}(d\mathbf x),
\end{equation*}
which proves \cref{eq:sync-clean-product-target} after conditioning on $j=0$.
\end{proof}

\paragraph{Finite-candidate implementation of MCS-CS-VGB.}
Use the local pools and coordinatewise reverse insertion of \cref{app:subsec:cs-vgb-calibrated-lazy}. For fixed $\alpha>0$, denote the depth-weighted candidate contributions in \cref{eq:cs-vgb-depth-weighted-candidate-contributions} for particle $n$ by $e_{j,n,i}^d$, with $d\in\{\mathrm F,\mathrm B\}$. Define
\begin{equation}
\begin{aligned}
\widehat A_j^{\mathrm F}&:=\prod_{n=1}^N\left(\frac1{L_F}\sum_{i=1}^{L_F}e_{j,n,i}^{\mathrm F}\right),
& E_{j,\mathbf i}^{\mathrm F}&:=\prod_{n=1}^Ne_{j,n,i_n}^{\mathrm F},\\
\widehat A_j^{\mathrm B}&:=\prod_{n=1}^N\left(\frac1{L_B}\sum_{k=1}^{L_B}e_{j,n,k}^{\mathrm B}\right),
& E_{j,\mathbf k}^{\mathrm B}&:=\prod_{n=1}^Ne_{j,n,k_n}^{\mathrm B}.
\end{aligned}
\label{eq:sync-finite-pool-masses}
\end{equation}
Without lazy pool evaluation, substitute $\widehat A_j^{\mathrm F},\widehat A_j^{\mathrm B}$ into the synchronized action rule. On a move in direction $d\in\{\mathrm F,\mathrm B\}$, independently select each particle's candidate with probability
\begin{equation}
\bbP(I_n=i\mid d\text{ move})
=\frac{e_{j,n,i}^d}{\sum_{h=1}^{L_d}e_{j,n,h}^d},
\qquad
\bbP(\mathbf I=\mathbf i\mid d\text{ move})
=\prod_{n=1}^N\frac{e_{j,n,i_n}^d}{\sum_{h=1}^{L_d}e_{j,n,h}^d}
=\frac{E_{j,\mathbf i}^d}{L_d^N\widehat A_j^d},
\label{eq:sync-factorized-index-selection}
\end{equation}
and reverse-insert the previous states into independent uniform opposite-pool slots.

With \emph{lazy pool evaluation}, an F update uses the full F pools and B mark $\mathbf k$; a B update uses the full B pools and F mark $\mathbf i$. In direction F, draw $I_n\sim\operatorname{Unif}([L_F])$ independently and write $\mathbf I=(I_1,\ldots,I_N)$; in direction B, draw $K_n\sim\operatorname{Unif}([L_B])$ independently and write $\mathbf K=(K_1,\ldots,K_N)$. Conditional on these indices, the probabilities of $(\mathsf{move},\mathsf{stay},\mathsf{flip})$ are
\begin{equation}
\begin{aligned}
\mathrm F:\quad&
\frac{\bigl(\widehat A_j^{\mathrm F},\,
\chi\min\{E_{j,\mathbf I}^{\mathrm F},E_{j,\mathbf k}^{\mathrm B}\},\,
E_{j,\mathbf k}^{\mathrm B}-\chi\min\{E_{j,\mathbf I}^{\mathrm F},E_{j,\mathbf k}^{\mathrm B}\}\bigr)}
{\widehat A_j^{\mathrm F}+E_{j,\mathbf k}^{\mathrm B}},\\
\mathrm B:\quad&
\frac{\bigl(\widehat A_j^{\mathrm B},\,
\chi\min\{E_{j,\mathbf K}^{\mathrm B},E_{j,\mathbf i}^{\mathrm F}\},\,
E_{j,\mathbf i}^{\mathrm F}-\chi\min\{E_{j,\mathbf K}^{\mathrm B},E_{j,\mathbf i}^{\mathrm F}\}\bigr)}
{\widehat A_j^{\mathrm B}+E_{j,\mathbf i}^{\mathrm F}}.
\end{aligned}
\label{eq:sync-lazy-auxiliary-gate}
\end{equation}
On a move, discard the uniform indices and redraw by \cref{eq:sync-factorized-index-selection}, and mark the reverse-inserted slots at an interior destination. On an F-to-B flip, mark the preselected indices $\mathbf I$ in the retained F pools and complete only missing B entries; the B-to-F case uses the preselected indices $\mathbf K$ symmetrically. A stay changes nothing; generated entries are retained across flips and stays. Boundary moves are forced without marks.

Write $\mathsf G_j^{(N)}:=\prod_n\mathsf G_j$. Let $\mathcal Z_{N,\alpha}$ be the normalizer in \cref{eq:sync-normalizer} with each edge $\widehat C_j^{(N)}$ replaced by $\alpha^{N(j-1)}\widehat C_j^{(N)}$.
\Needspace{5\baselineskip}
\begin{theorem_ap}[Stationary clean target with finite local pools]
\label{thm:sync-finite-pool-stationarity}
\label{thm:sync-lazy-finite-pool-stationarity}
Without lazy evaluation, the interior stationary law is
\begin{equation}
\Pi_{\mathrm{pool},N}(j,d\mathbf x,d\mathbf u,d)
=\frac{\boldsymbol\pi_j(d\mathbf x)\mathsf G_j^{(N)}(d\mathbf u\mid\mathbf x)}{\mathcal Z_{N,\alpha}}
(\widehat A_j^{\mathrm F}+\widehat A_j^{\mathrm B}).
\label{eq:sync-pool-stationary-law}
\end{equation}
With lazy evaluation, the F law is
\begin{equation}
\Pi_{\mathrm{mark},N}(j,d\mathbf x,d\mathbf u,\mathrm F,\mathbf k)
=\frac{\boldsymbol\pi_j(d\mathbf x)\mathsf G_j^{(N)}(d\mathbf u\mid\mathbf x)}{\mathcal Z_{N,\alpha}}
\frac{\widehat A_j^{\mathrm F}+E_{j,\mathbf k}^{\mathrm B}}{L_B^N},
\label{eq:sync-marked-stationary-law}
\end{equation}
and the B law is symmetric. At a boundary, both laws use only the available directional mass. Summing over the mark recovers $\Pi_{\mathrm{pool},N}$, and
\begin{equation}
\Pi_{\mathrm{mark},N}(d\mathbf x\mid j=0)
=\Pi_{\mathrm{pool},N}(d\mathbf x\mid j=0)
=(\pi_0^\star)^{\otimes N}(d\mathbf x).
\label{eq:sync-pool-clean-target}
\end{equation}
\end{theorem_ap}

\begin{proof}

The reverse-insertion identity in \cref{eq:c-indexed-edge-criterion} applies independently to each particle, so taking products gives the corresponding identity for the population.
With the factorized selection rule in \cref{eq:sync-factorized-index-selection} and independent uniform insertion slots, the proof of \cref{thm:c-joint-pool-invariance} therefore applies with
\begin{equation*}
(L_d,e_{j,i}^d,\widehat a_j^d)
\longmapsto
(L_d^N,E_{j,\mathbf i}^d,\widehat A_j^d).
\end{equation*}
The uniform tuple draws and conditional action probabilities in \cref{eq:sync-lazy-auxiliary-gate} give the same flip--stay cancellation.
Thus both stated laws are stationary, and summing over the mark recovers the unmarked law.
At the clean boundary, the edge weight is one and integrating out the pools gives a measure proportional to
$\bigotimes_{n=1}^N[r(x^{(n)})\pi_0(dx^{(n)})]$,
whose normalization is $(\pi_0^\star)^{\otimes N}$.

\end{proof}

\section{Additional Related Works}
\label{sec:main-related}

\paragraph{Guidance and particle sampling.}
Like CS-VGB, guidance and candidate-selection methods use intermediate information to favor desirable outputs, either by modifying denoising predictions~\citep{main:dhariwal2021guided,main:ho2022cfg,main:song2023lgd} or selecting candidates using soft values~\citep{li2025svdd}. TDS and Feynman--Kac steering address the related target-sampling problem through particle weighting and resampling~\citep{wu2023tds,singhal2025fk}. TDS retains asymptotic exactness through importance corrections as the particle count increases, whereas our corrected construction preserves the stationary clean target at fixed finite pool sizes. 

\paragraph{Search and iterative refinement.}
Our use of additional compute to revisit generation decisions connects CS-VGB to verifier-based search~\citep{rw:ma2025scaling,zhang2025classical} and the noising--denoising refinement of RERD and ABCD~\citep{uehara2025refinement,lee2025cyclic}. Distributional guarantees also appear in DTS, which improves tree-based values as rollouts accumulate~\citep{jain2025dts}, and IterRef, which applies multiple-try Metropolis to intermediate distributions in discrete diffusion~\citep{rw:lee2025iterref}. Our focus is instead on preserving the continuous-state clean target despite process-verifier error, without requiring those estimates to become exact.

\paragraph{Verifier error and backtracking.}
The closest conceptual predecessors of CS-VGB are AR-VGB and MDM-VGB, which separate imperfect process verifiers from the terminal target in discrete generation~\citep{rohatgi2026vgb,jeon2026mdmvgb}. CS-VGB extends this principle to continuous states by pairing denoising and renoising transitions. A related approach, DOIT, uses Monte Carlo estimates of a Doob $h$-correction to steer pretrained diffusion models toward a reward-tilted target and analyzes the resulting approximation error~\citep{rw:zhu2026doit}. CS-VGB instead uses process verifiers within a bidirectional stationary construction, where verifier error affects convergence rates but not the clean target under correct transition pairing. We separately quantify the bias caused by uncorrected transition mismatch.

\paragraph{Stochastic transitions and posterior maps.}
We instantiate CS-VGB with diffusion models and one-step posterior maps. For diffusion models, reverse denoising steps and forward noising transitions provide the checkpoint kernels~\citep{ho2020ddpm,song2021score}, whereas Meta Flow Maps and Posterior Diamond Maps generate clean posterior samples that we combine with Gaussian bridges to reach intermediate checkpoints~\citep{potaptchik2026meta,holderrieth2026diamond}. Related approaches include Noise Conditional Variational Score Distillation (NCVSD), which distills diffusion models into stochastic denoisers that approximate clean posterior samples across noise levels~\citep{rw:peng2025ncvsd}, and Strong Stochastic Flow Maps (SSFM), which learn the pathwise solution map of an additive-noise SDE and thus approximate its conditional transitions~\citep{rw:mccallum2026ssfm}. These methods provide transition samplers that can be used within CS-VGB, which determines how to traverse checkpoint kernels in both directions while preserving the stationary clean target under correct transition pairing. In particular, as conditional samplers of clean states, NCVSD and SSFM are also compatible with CS-VGB through the Gaussian-bridge construction (\cref{prop:cs-vgb-posterior-bridge-pairing}).

\section{Experiment Settings and Supplementary Results}
\label{app:sec:experiment-settings}

\paragraph{Settings.}
All experiments were implemented in Python using PyTorch~\citep{paszke2019pytorch} and conducted on Ubuntu Linux. We used NVIDIA H100 GPUs with 80\,GB of memory, NVIDIA H200 GPUs with 141\,GB of memory, and NVIDIA RTX PRO 6000 Blackwell GPUs with 96\,GB of memory.

\subsection{Shared Search Protocol}
\label{app:subsec:practical-implementation-details}
Unlike the analytic sampling experiments, practical search sets $\rho_j=1$ and proposal importance weights $dK/dG=1$. It archives clean outputs in $\mathcal S$ and, unless otherwise specified, returns $x^\star\in\arg\max_{x\in\mathcal S}r(x)$, with task-specific endpoint archiving as described below. These choices prioritize reward rather than exact target sampling. We check the NFE budget before each search update and report the actual NFE, which may exceed the requested budget when completing the forward-only fallback.

\paragraph{Gradient guidance for posterior bridges.}
For ImageNet and CelebA, we use normalized proposal-mean shifts before candidate selection, rather than the covariance-scaled bridge guidance in \cref{eq:posterior-bridge-covariance-guidance}. Following the Monte Carlo reward-gradient estimation used in MFM and PDM~\citep{potaptchik2026meta,holderrieth2026diamond}, we differentiate the log mean reward of posterior samples. Write $t_j=t_{\kappa_j}$ and use $m_j,a_j,\sigma_j$ from \cref{eq:cs-vgb-flowmap-noise-variances}. Let $\zeta_\ell\overset{\mathrm{iid}}{\sim}\mathcal N(0,I)$, $\ell=1,\ldots,m_g$, be the noise inputs to the posterior sampler. Define
\begin{equation}
Z_\ell(x)=\Phi_{t_j,0}^{\theta}(x,\zeta_\ell),\qquad
h_j(x)=\nabla_x\log\!\left[\frac{1}{m_g}\sum_{\ell=1}^{m_g}r(Z_\ell(x))\right],\qquad
\bar Z=\frac{1}{m_g}\sum_{\ell=1}^{m_g}Z_\ell(x).
\label{eq:search-posterior-gradient}
\end{equation}
The gradient is taken through the posterior sampler; at $j=0$, $h_0=\nabla\log r$. The implemented guided proposals are
\begin{align}
G_j^{\mathrm F}(dy\mid x,Z)
&=\mathcal N\!\left(dy;m_j(x,Z)+g_F(t_j-t_{j-1})\left\|\frac{\bar Z-x}{t_j}\right\|_2\frac{h_j(x)}{\|h_j(x)\|_2},\frac{t_{j-1}^2\sigma_j^2}{t_j^2}I\right),
\label{eq:search-guided-forward}\\
G_j^{\mathrm B}(dy\mid x)
&=\mathcal N\!\left(dy;a_{j+1}x+g_B\sigma_{j+1}\frac{h_j(x)}{\operatorname{RMS}(h_j(x))},\sigma_{j+1}^2I\right),
\label{eq:search-guided-backward}
\end{align}
These kernels are conditional on the guidance draws, with $Z$ a clean posterior draw and $\operatorname{RMS}(h)=\|h\|_2/\sqrt{d}$ for state dimension $d$. Both guidance shifts are defined to be zero when $h_j(x)=0$. Averaging over the posterior draws gives the proposal kernels conditional on $x$ alone.

\paragraph{Best-endpoint sharing and local candidate selection.}
For unguided search with one-step posterior maps, we follow MFM-Search~\citep{potaptchik2026meta} in sharing the highest-reward clean endpoint across particles. From each parent $x_j^{(n)}$, draw $m_e$ posterior endpoints $Z_{n,\ell}$ and select
\begin{equation}
(n^\star,\ell^\star)\in\arg\max_{1\leq n\leq N,\,1\leq\ell\leq m_e}r(Z_{n,\ell}),\qquad
(x^\star,Z^\star):=(x_j^{(n^\star)},Z_{n^\star,\ell^\star}).
\label{eq:search-best-endpoint-sharing}
\end{equation}
Each particle then forms its own $L_F$ bridge candidates from this shared parent--endpoint pair and selects the highest-weight local child:
\begin{equation}
Y_{n,i}=m_j(x^\star,Z^\star)+\frac{t_{j-1}\sigma_j}{t_j}\xi_{n,i},\qquad
\xi_{n,i}\overset{\mathrm{iid}}{\sim}\mathcal N(0,I),\qquad
i_n^\star\in\arg\max_{1\leq i\leq L_F}e_{j,n,i}^{\mathrm F}.
\label{eq:search-local-argmax}
\end{equation}
Here $e_{j,n,i}^{\mathrm F}$ is the depth-weighted candidate contribution in \cref{eq:cs-vgb-depth-weighted-candidate-contributions}, with the search corrections specified above. After moving to $Y_{n,i_n^\star}$, reverse insertion uses the shared parent $x^\star$, not the particle's previous parent, and the reverse-insertion mark is updated accordingly at interior destinations. In unguided search, backward moves likewise select the highest-weight candidate from the local pool.

\paragraph{Renoising warp.}
Finite-budget search favors frequent clean outputs, motivating a smaller depth weight $\alpha$. However, this can cause the chain to switch back to denoising after only shallow renoising, spending NFE on repeated local refinements near the clean endpoint. To encourage broader exploration within the same budget, we optionally force renoising to depth $j=\lceil\beta J\rceil$ after a clean visit; $\beta=0$ disables this warp. For $\beta>0$, we force successive backward moves to $j$, subject to the remaining budget, retaining the usual proposal generation, candidate selection, and reverse insertion. We then resume ordinary action selection without resetting the direction or pools, subject to the standard boundary rules. This renoising warp is a search-only heuristic, distinct from ordinary backtracking.

\paragraph{Skipping stays.}
Since stays leave the state, pools, and mark unchanged, we sample the next non-stay action directly:
\begin{equation*}
\bbP(\mathsf{move}\mid\text{not stay})
=\frac{p_{\mathrm{move}}}{1-p_{\mathrm{stay}}},\qquad
\bbP(\mathsf{flip}\mid\text{not stay})
=\frac{p_{\mathrm{flip}}}{1-p_{\mathrm{stay}}}.
\end{equation*}
Candidate and mark selection are unchanged. Skipping stays alone preserves the chain's stationary clean conditional because clean-boundary states have no stays.

\subsection{Target Recovery in Two Dimensions}
\label{app:subsec:toy-verifier-error}
\label{app:subsec:sampling-additional-results}

We use a 24-component Gaussian mixture as the reference and define the reward by
\[
\pi_0(x)=\sum_{k=1}^{24}w_k\mathcal N(x;\mu_k,\Sigma_k),
\qquad
r(x)=10^{-3}+\sum_{\ell=1}^{6}a_\ell
\exp\!\left[-\tfrac12(x-c_\ell)^\top S_\ell^{-1}(x-c_\ell)\right].
\]
Here $w_k$, $\mu_k$, and $\Sigma_k$ are the mixture weights, means, and covariance matrices, while $a_\ell$, $c_\ell$, and $S_\ell$ specify the strength, center, and spatial spread of each reward term.
We use 32 intervals uniformly spaced in $-\log(1-t)$, with $t_j=1-100^{-j/32}$ for $j=0,\ldots,32$.
Renoising is Gaussian and denoising uses its exact conditional reverse, so $\rho_j=1$.
The initial state follows the reference marginal at $t_{32}=0.99$.

\paragraph{Verifier error.}
The ideal verifier $V_j(x)$ is evaluated analytically using Gaussian-mixture posterior expectations of the terminal reward. To model a spatially and temporally coherent error rather than fresh evaluation noise, we use the fixed random Fourier field
\begin{equation}
\eta(x,t)=\sqrt{\frac{2}{M}}\sum_{m=1}^{M}c_m\cos(\omega_m^\top x+\nu_m t+\phi_m),
\qquad s_\eta^2=\frac1M\sum_{m=1}^{M}c_m^2,
\label{eq:app-toy-verifier-field}
\end{equation}
with $M=16$, $\omega_m\sim\mathcal N(0,I_2)$, $\nu_m,c_m\sim\mathcal N(0,1)$, and $\phi_m\sim\operatorname{Unif}[0,2\pi)$. The field is fixed across methods and repetitions; $s_\eta^2$ is its variance over random phases conditional on the amplitudes.

The perturbed verifier is defined by
\begin{equation}
\log\widehat V_j(x)=\log V_j(x)+\lambda t\,\eta(x,t)-\tfrac12(\lambda t)^2s_\eta^2,\qquad t=t_{\kappa_j}.
\label{eq:exp-verifier-perturbation}
\end{equation}
We use $\lambda\in\{0,3,10,15,20\}$, and both perturbation terms vanish at $t=0$, preserving the exact terminal reward.

\paragraph{Baseline protocols.}
Noisy guidance uses 16,384 Euler--Maruyama reverse-SDE steps, the analytic reference score, and unit-strength $\nabla\log\widehat V_t$ guidance. SMC uses 512 particles, 32 exact denoising intervals, incremental weights $\widehat V_{j-1}/\widehat V_j$, and resampling at every interval. RERD uses 512 samples, forward candidate pools of four, and 65,536 refinement updates with uniform renoising depths. 

\subsection{ImageNet Prompt-Guided Generation}
\label{app:subsec:imagenet-settings}

\paragraph{Generator and evaluation prompts.}
We use the pretrained class-conditional MFM checkpoint~\citep{potaptchik2026meta}\footnote{\url{https://github.com/adh1s/mfm}} for ImageNet $256\times256$~\citep{main:deng2009imagenet}, with Stability AI's fine-tuned VAE decoder\footnote{\url{https://huggingface.co/stabilityai/sd-vae-ft-mse}}, classifier-free guidance scale $2.5$, and one model step per posterior draw. Class guidance is shared by all methods.

We construct 100 fixed evaluation prompts by pairing ImageNet classes with additional conditions: 21 actions, 20 object interactions, 20 visual states, 20 contexts, and 19 spatial relations. The generator is conditioned on the class label, while the full prompt is used to compute the reward and evaluate prompt alignment. \Cref{fig:imagenet-prompt-examples} provides ten example prompts.

\begin{figure}[!htb]
\centering
\begin{mdframed}[linewidth=0.6pt,linecolor=black!65,backgroundcolor=black!3,innerleftmargin=11pt,innerrightmargin=11pt,innertopmargin=10pt,innerbottommargin=10pt,skipabove=0pt,skipbelow=0pt]

\textbf{ImageNet: example evaluation prompts}\par\smallskip
\hrule height 0.4pt\relax\medskip
\small\raggedright
\begin{enumerate}[label=\textbf{\arabic*.},start=0,leftmargin=1.8em,itemsep=6pt,topsep=0pt,parsep=0pt]
\item \textbf{volcano (980)}\\\texttt{majestic volcano erupting with lava flowing down its sides}
\item \textbf{space shuttle (812)}\\\texttt{space shuttle launching into space, fiery fumes}
\item \textbf{hammerhead (4)}\\\texttt{a hammerhead nudging a floating buoy with its snout}
\item \textbf{stingray (6)}\\\texttt{a stingray partially buried in pale sand}
\item \textbf{hen (8)}\\\texttt{a hen in a sunlit farmyard}
\item \textbf{brambling (10)}\\\texttt{a brambling perched on a frosted branch}
\item \textbf{house finch (12)}\\\texttt{a house finch pulling a seed from a sunflower head}
\item \textbf{indigo bunting (14)}\\\texttt{an indigo bunting with rain-soaked plumage}
\item \textbf{bulbul (16)}\\\texttt{a bulbul taking flight}
\item \textbf{magpie (18)}\\\texttt{a magpie in a rain-soaked city park}
\end{enumerate}
\end{mdframed}
\caption{Example prompts for ImageNet generation.}
\label{fig:imagenet-prompt-examples}
\end{figure}

\paragraph{Reward and independent evaluation.}
We use HPSv2~\citep{exp:wu2023hpsv2} with v2.1 weights. The reported score is $s(x,c)=100\langle f_I(x),f_T(c)\rangle$, where $f_I(x)$ and $f_T(c)$ are the unit-normalized embeddings of image $x$ and prompt $c$ produced by the HPSv2 image and text encoders, respectively. For search, we use the calibrated log reward
\begin{equation*}
\log r(x;c)=\frac{s(x,c)-\mu_c}{T\sigma_c},\qquad T=0.5,
\end{equation*}
where $\mu_c$ and $\sigma_c$ are the mean and standard deviation of HPSv2 scores for prompt $c$, computed from 64 images generated by the base MFM under the corresponding class label without reward guidance or search. The same normalization statistics are used for all methods. We remark that the plotted HPSv2 values are raw scores, not standardized log rewards. We evaluate prompt alignment with VQAScore~\citep{exp:lin2024vqascore}, using the 3B-parameter instruction-tuned Qwen2.5-VL model~\citep{exp:bai2025qwen25vl} as the evaluation backbone\footnote{\url{https://huggingface.co/Qwen/Qwen2.5-VL-3B-Instruct}} (\cref{fig:main-imagenet-mfm}). For the question \emph{Does this figure show \{prompt\}? Please answer Yes or No.}, VQAScore reports the full-vocabulary next-token probability of \emph{Yes}. CLIP-IQA~\citep{exp:wang2023clipiqa} is shown for both ImageNet and CelebA in \cref{fig:app-imagenet-heldout}. Neither metric is used for optimization.

\paragraph{Guided CS-VGB.}
CS-VGB uses the posterior bridges in \cref{app:subsec:model-checkpoint-transitions}, with a uniform grid $t_j=j/J$, initialized from Gaussian noise at $t=1$. We set $\alpha=0.3$, $\chi=1$, and one posterior sample per value estimate, with lazy evaluation, skipped stays, and reverse insertion. We set $\beta=0.2$. If the search ends without any clean visits, a forward-only fallback completes denoising from the current state to the clean level. Edge and proposal corrections are omitted as specified in \cref{app:subsec:practical-implementation-details}.

Guided runs use one chain; F guides forward proposals only, while F+B also guides backward proposals. \Cref{tab:imagenet-guided-settings} lists F+B settings: posterior draws $m_g$, guidance scales $g_F,g_B$. In our implementation, the guided method returns the highest-reward image among clean visits and clean endpoints evaluated for value estimation or gradient guidance.

\begin{table}[!htb]
\centering\small
\caption{\textbf{ImageNet guided CS-VGB configurations.} $B$ is the requested NFE budget; $J$ counts checkpoint intervals.}
\label{tab:imagenet-guided-settings}
\renewcommand{\arraystretch}{0.9}
\begin{tabular}{rrrrrrr}
\toprule
$B$ & $J$ & $L_F$ & $L_B$ & $m_g$ & $g_F$ & $g_B$\\
\midrule
$2^9$ &95&1&1&1&3.5&1.0\\
$2^{10}$ &79&1&1&2&4.0&1.5\\
$2^{11}$ &64&1&1&4&2.0&1.5\\
$2^{12}$ &64&1&1&4&1.75&1.5\\
$2^{13}$ &64&4&4&8&2&1.5\\
$10{,}000$ &64&4&4&8&2&1.5\\
$19{,}609$ &128&8&8&16&3&2\\
$42{,}362$ &129&16&8&16&3.5&2\\
$2^{16}$ &129&16&8&16&3.5&2\\
\bottomrule
\end{tabular}
\end{table}

\paragraph{Unguided CS-VGB search.}
Unguided runs use $L_F=L_B=4$ and a budget-dependent population size $N$, drawing $m_e$ clean endpoint candidates per particle for each denoising proposal (\cref{tab:imagenet-search-settings}). Following MFM-Search~\citep{potaptchik2026meta}, we select the highest-reward clean endpoint across particles. We then select the highest-weight child among the $L_F$ local candidates for each particle and reverse-insert its parent. Following the search protocol of Posterior Diamond Maps~\citep{holderrieth2026diamond}, we also archive clean endpoints generated for verifier evaluation and return the highest-reward image in the archive.

\begin{table}[!htb]
\centering\small
\caption{\textbf{ImageNet unguided CS-VGB configurations.} Columns give requested NFE budget $B$, intervals $J$, particles $N$, and clean endpoint candidates per particle $m_e$.}
\label{tab:imagenet-search-settings}
\renewcommand{\arraystretch}{0.9}
\begin{tabular}{rrrr@{\qquad\qquad}rrrr}
\toprule
$B$&$J$&$N$&$m_e$&$B$&$J$&$N$&$m_e$\\
\midrule
1,433&16&4&8&10,623&32&8&24\\
1,854&16&4&16&12,637&32&8&24\\
1,902&16&4&16&12,767&32&8&24\\
1,953&16&4&16&16,827&32&8&24\\
2,037&16&4&16&21,716&64&8&24\\
3,975&16&4&32&25,859&64&8&24\\
6,230&16&8&24&34,333&64&8&24\\
7,956&16&8&32&64,954&64&8&24\\
10,393&16&8&24&131,029&64&16&52\\
\bottomrule
\end{tabular}
\end{table}

\noindent\begin{minipage}{\linewidth}
\subsubsection{Qualitative comparisons on ImageNet}
\label{app:subsec:imagenet-qualitative}
At approximately 65k NFE, we compare outputs for five of the first 20 evaluation prompts. 

\begin{table}[H]
\centering
\setlength{\tabcolsep}{2pt}
\caption{\textbf{ImageNet prompt-aligned outputs near 65k NFE.} Generated images with HPSv2 scores (reported on the $\times100$ scale) and actual NFE. Bold marks the largest HPSv2 score in each row.}
\label{tab:imagenet-qualitative}
\begin{tabular*}{\linewidth}{@{\extracolsep{\fill}}p{0.23\linewidth}ccccc@{}}
\toprule
Prompt & BoN & \shortstack{MFM\\Search} & \shortstack{CS-VGB\\Search} & MFM-G & CS-VGB-G \\
\midrule
\parbox[c]{\linewidth}{\footnotesize\raggedright\textbf{ID 004}\\a hen in a sunlit farmyard.} & \raisebox{-0.5\height}{\shortstack{\includegraphics[width=0.135\linewidth]{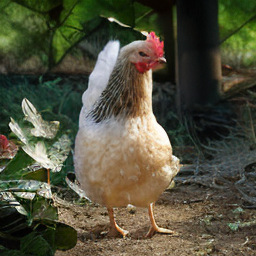}\\30.25\\\scriptsize NFE 65,536}} & \raisebox{-0.5\height}{\shortstack{\includegraphics[width=0.135\linewidth]{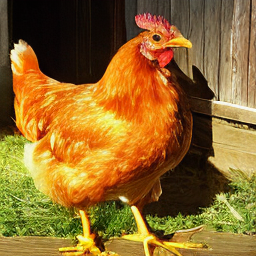}\\35.81\\\scriptsize NFE 65,536}} & \raisebox{-0.5\height}{\shortstack{\includegraphics[width=0.135\linewidth]{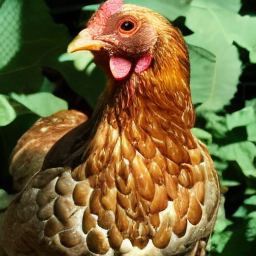}\\36.61\\\scriptsize NFE 64,768}} & \raisebox{-0.5\height}{\shortstack{\includegraphics[width=0.135\linewidth]{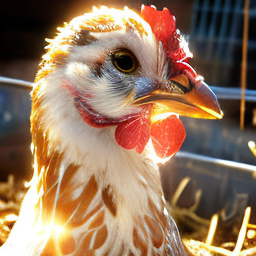}\\46.96\\\scriptsize NFE 64,990}} & \raisebox{-0.5\height}{\shortstack{\includegraphics[width=0.135\linewidth]{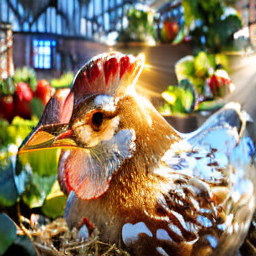}\\\textbf{55.29}\\\scriptsize NFE 65,548}} \\[3pt]
\parbox[c]{\linewidth}{\footnotesize\raggedright\textbf{ID 005}\\a brambling perched on a frosted branch.} & \raisebox{-0.5\height}{\shortstack{\includegraphics[width=0.135\linewidth]{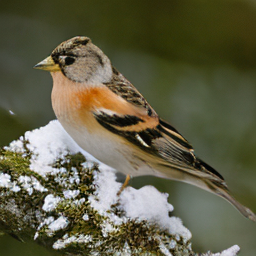}\\32.26\\\scriptsize NFE 65,536}} & \raisebox{-0.5\height}{\shortstack{\includegraphics[width=0.135\linewidth]{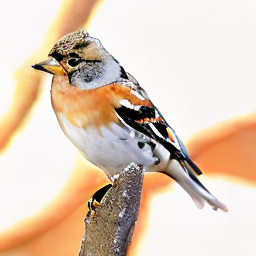}\\31.56\\\scriptsize NFE 65,536}} & \raisebox{-0.5\height}{\shortstack{\includegraphics[width=0.135\linewidth]{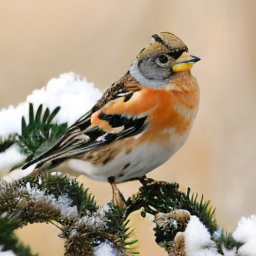}\\38.72\\\scriptsize NFE 64,712}} & \raisebox{-0.5\height}{\shortstack{\includegraphics[width=0.135\linewidth]{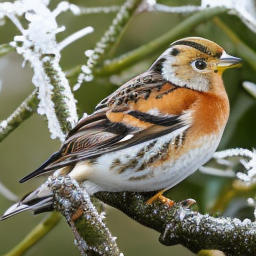}\\48.05\\\scriptsize NFE 64,990}} & \raisebox{-0.5\height}{\shortstack{\includegraphics[width=0.135\linewidth]{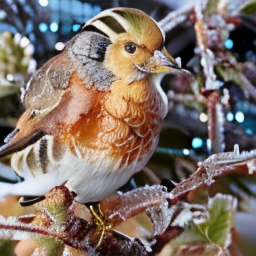}\\\textbf{51.73}\\\scriptsize NFE 65,536}} \\[3pt]
\parbox[c]{\linewidth}{\footnotesize\raggedright\textbf{ID 009}\\a magpie in a rain-soaked city park.} & \raisebox{-0.5\height}{\shortstack{\includegraphics[width=0.135\linewidth]{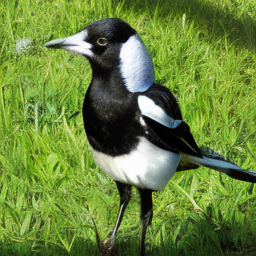}\\30.97\\\scriptsize NFE 65,536}} & \raisebox{-0.5\height}{\shortstack{\includegraphics[width=0.135\linewidth]{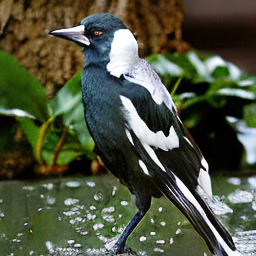}\\35.86\\\scriptsize NFE 65,536}} & \raisebox{-0.5\height}{\shortstack{\includegraphics[width=0.135\linewidth]{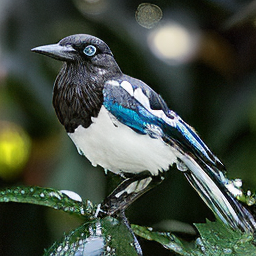}\\39.32\\\scriptsize NFE 64,904}} & \raisebox{-0.5\height}{\shortstack{\includegraphics[width=0.135\linewidth]{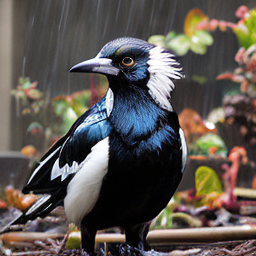}\\47.33\\\scriptsize NFE 64,990}} & \raisebox{-0.5\height}{\shortstack{\includegraphics[width=0.135\linewidth]{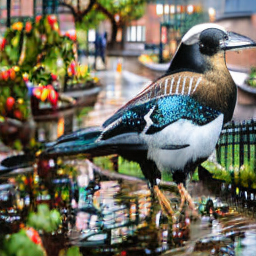}\\\textbf{56.08}\\\scriptsize NFE 65,567}} \\[3pt]
\parbox[c]{\linewidth}{\footnotesize\raggedright\textbf{ID 014}\\a spotted salamander on a damp forest floor.} & \raisebox{-0.5\height}{\shortstack{\includegraphics[width=0.135\linewidth]{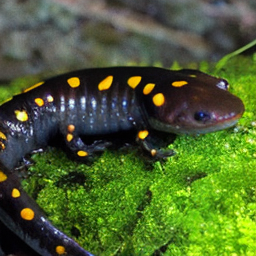}\\31.61\\\scriptsize NFE 65,536}} & \raisebox{-0.5\height}{\shortstack{\includegraphics[width=0.135\linewidth]{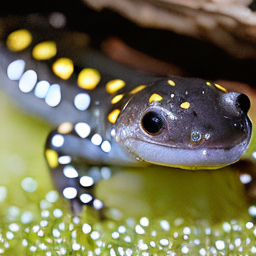}\\31.96\\\scriptsize NFE 65,536}} & \raisebox{-0.5\height}{\shortstack{\includegraphics[width=0.135\linewidth]{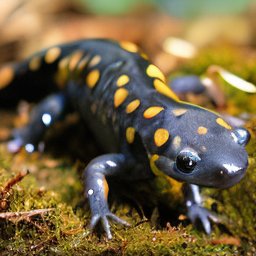}\\36.58\\\scriptsize NFE 64,864}} & \raisebox{-0.5\height}{\shortstack{\includegraphics[width=0.135\linewidth]{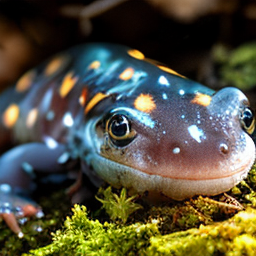}\\45.95\\\scriptsize NFE 64,990}} & \raisebox{-0.5\height}{\shortstack{\includegraphics[width=0.135\linewidth]{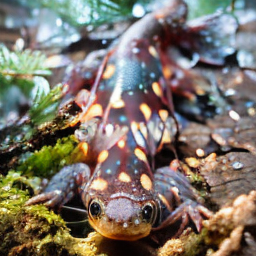}\\\textbf{51.71}\\\scriptsize NFE 65,573}} \\[3pt]
\parbox[c]{\linewidth}{\footnotesize\raggedright\textbf{ID 016}\\a tailed frog leaping from the water.} & \raisebox{-0.5\height}{\shortstack{\includegraphics[width=0.135\linewidth]{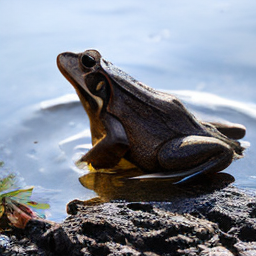}\\29.55\\\scriptsize NFE 65,536}} & \raisebox{-0.5\height}{\shortstack{\includegraphics[width=0.135\linewidth]{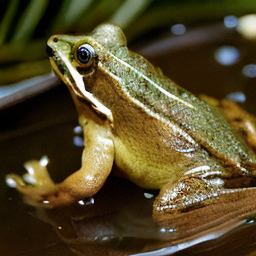}\\34.16\\\scriptsize NFE 65,536}} & \raisebox{-0.5\height}{\shortstack{\includegraphics[width=0.135\linewidth]{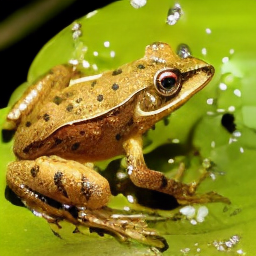}\\37.49\\\scriptsize NFE 64,736}} & \raisebox{-0.5\height}{\shortstack{\includegraphics[width=0.135\linewidth]{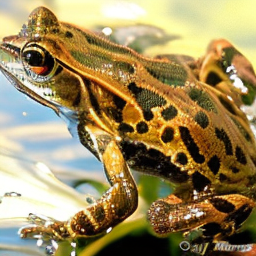}\\49.16\\\scriptsize NFE 64,990}} & \raisebox{-0.5\height}{\shortstack{\includegraphics[width=0.135\linewidth]{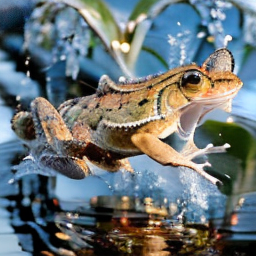}\\\textbf{51.96}\\\scriptsize NFE 65,575}} \\[3pt]
\bottomrule
\end{tabular*}
\end{table}
\end{minipage}

CS-VGB-G depicts a wet urban setting for ID 009 and an airborne frog with splashes for ID 016, with HPSv2 gains of $2.80$--$8.75$ over MFM-G across these cases. Higher scores do not verify every detail; aggregate reward and alignment results appear in \cref{fig:main-imagenet-mfm}.
\FloatBarrier

\subsection{CelebA Facial-Attribute Generation}
\label{app:subsec:celeba-settings}

\paragraph{Generator and conditions.}
We use the pretrained PDM checkpoint~\citep{holderrieth2026diamond}\footnote{\url{https://github.com/PeterHolderrieth/diamond_maps}} for CelebA $64\times64$~\citep{liu2015faceattributes} and one model step per posterior draw or transition. PDM-G additionally uses the flow-map checkpoint from the same release for its drift.

\paragraph{Attribute conditions.}
We construct 100 fixed attribute conditions, each requesting presence or absence of ten attributes: Smiling, Eyeglasses, Bangs, Wearing Hat, Mouth Slightly Open, No Beard, Bushy Eyebrows, Wearing Lipstick, Wearing Earrings, and one of Black Hair, Brown Hair, or Blond Hair. Negative hair-color targets exclude that color. \Cref{fig:celeba-prompt-examples} shows ten examples.

\begin{figure}[!htb]
\centering
\begin{mdframed}[linewidth=0.6pt,linecolor=black!65,backgroundcolor=black!3,innerleftmargin=11pt,innerrightmargin=11pt,innertopmargin=10pt,innerbottommargin=10pt,skipabove=0pt,skipbelow=0pt]

\textbf{CelebA: example attribute conditions}\par\smallskip
\hrule height 0.4pt\relax\medskip
\small\raggedright
{\ttfamily a close-up portrait photo of a person who}\par\medskip
\begin{enumerate}[label=\textbf{\arabic*.},start=0,leftmargin=1.8em,itemsep=7pt,topsep=0pt,parsep=0pt]
\ttfamily
\item is not smiling, is not wearing eyeglasses, has bangs, is not wearing a hat, has their mouth slightly open, has blond hair, has a beard, has bushy eyebrows, is wearing lipstick, and is not wearing earrings
\item is smiling, is wearing eyeglasses, does not have bangs, is wearing a hat, has their mouth closed, does not have black hair, does not have a beard, does not have bushy eyebrows, is not wearing lipstick, and is wearing earrings
\item is not smiling, is wearing eyeglasses, has bangs, is not wearing a hat, has their mouth closed, does not have brown hair, has a beard, has bushy eyebrows, is wearing lipstick, and is wearing earrings
\item is smiling, is not wearing eyeglasses, does not have bangs, is wearing a hat, has their mouth slightly open, has black hair, does not have a beard, does not have bushy eyebrows, is not wearing lipstick, and is not wearing earrings
\item is smiling, is wearing eyeglasses, does not have bangs, is wearing a hat, has their mouth slightly open, has brown hair, has a beard, does not have bushy eyebrows, is wearing lipstick, and is wearing earrings
\item is not smiling, is not wearing eyeglasses, has bangs, is wearing a hat, has their mouth closed, does not have blond hair, does not have a beard, has bushy eyebrows, is not wearing lipstick, and is not wearing earrings
\item is smiling, is not wearing eyeglasses, has bangs, is not wearing a hat, has their mouth closed, does not have brown hair, has a beard, does not have bushy eyebrows, is not wearing lipstick, and is not wearing earrings
\item is not smiling, is wearing eyeglasses, does not have bangs, is not wearing a hat, has their mouth slightly open, has blond hair, does not have a beard, has bushy eyebrows, is wearing lipstick, and is wearing earrings
\item is smiling, is not wearing eyeglasses, has bangs, is not wearing a hat, has their mouth slightly open, has brown hair, does not have a beard, has bushy eyebrows, is not wearing lipstick, and is wearing earrings
\item is not smiling, is wearing eyeglasses, does not have bangs, is wearing a hat, has their mouth closed, does not have black hair, has a beard, does not have bushy eyebrows, is wearing lipstick, and is not wearing earrings
\end{enumerate}
\end{mdframed}
\caption{Example attribute conditions for CelebA generation.}
\label{fig:celeba-prompt-examples}
\end{figure}

\paragraph{Reward and evaluation metrics.}
We construct the attribute reward using the CelebA-trained FaRL attribute classifier~\citep{exp:zheng2022farl} and a MobileNet-based RetinaFace detector from FACER\footnote{\url{https://github.com/FacePerceiver/facer}}, upscaling images to $256\times256$ before evaluation. Let $p_a(x)$ denote the classifier probability of the requested presence or absence of attribute $a$, clipped to $[10^{-6},1-10^{-6}]$. With reward temperature one, the optimization reward and the reported \emph{attribute geometric mean} are
\begin{equation*}
\log r(x)=\sum_{a=1}^{10}\log p_a(x),
\qquad
\operatorname{GM}(x)=\exp\!\left(\frac1{10}\sum_{a=1}^{10}\log p_a(x)\right).
\end{equation*}
If no face is detected, optimization assigns $p_a=10^{-4}$ with zero image gradient, while evaluation assigns GM zero. 

Held-out metrics are CLIP-IQA~\citep{exp:wang2023clipiqa} and VQAScore~\citep{exp:lin2024vqascore}. VQAScore uses the same evaluation procedure and Qwen2.5-VL 3B backbone as in the ImageNet experiments.

\paragraph{Guided CS-VGB.}
CS-VGB uses the posterior bridges in \cref{app:subsec:model-checkpoint-transitions}, a uniform grid $t_j=j/J$, $\alpha=0.3$, $\chi=1$, lazy evaluation, skipped stays, and one posterior draw per value estimate. Edge and proposal corrections are omitted. We set $\beta=0.2$ and use the same forward-only fallback as in the ImageNet experiments if the search ends without any clean visits. Guided variants use one chain and $L_F=L_B=2$, with $g_F=0.1$ and $g_B=0$ (F) or $0.1$ (F+B).

\paragraph{Unguided CS-VGB search.}
Unguided runs use the same posterior bridges, depth weight, persistence, lazy evaluation, renoising, fallback, and archive policy described above, without gradient guidance. Budget-dependent settings appear in \cref{tab:celeba-cs-settings}. Each proposal draws $m_e$ endpoints per particle and builds bridges from the highest-reward endpoint's parent.

\begin{table}[H]
\centering\small
\caption{\textbf{CS-VGB settings for CelebA.} Guided columns apply to both F-only and F+B variants. Here, $m_g$ denotes the number of posterior endpoint samples used for guidance, $J$ the number of checkpoint intervals, $N$ the number of particles, $L=L_F=L_B$ the candidate-pool size, and $m_e$ the number of posterior endpoint samples per particle used for unguided search.}
\label{tab:celeba-cs-settings}
\begin{tabular}{c cc cccc}
\toprule
& \multicolumn{2}{c}{Guided ($N=1$, $L=2$)} & \multicolumn{4}{c}{Unguided} \\
\cmidrule(lr){2-3}\cmidrule(lr){4-7}
Target NFE & $J$ & $m_g$ & $J$ & $N$ & $L$ & $m_e$ \\
\midrule
$2^6$  & 4  & 1 & 2  & 2 & 2 & 4 \\
$2^7$  & 4  & 1 & 4  & 2 & 4 & 4 \\
$2^8$  & 4  & 1 & 4  & 4 & 8 & 8 \\
$2^9$  & 8  & 1 & 8  & 4 & 4 & 8 \\
$2^{10}$ & 16 & 2 & 16 & 8 & 4 & 4 \\
$2^{11}$ & 16 & 2 & 8  & 8 & 8 & 16 \\
$2^{12}$ & 32 & 4 & 16 & 8 & 8 & 8 \\
$2^{13}$ & 32 & 4 & 16 & 8 & 8 & 32 \\
\bottomrule
\end{tabular}
\end{table}

\FloatBarrier

\clearpage
\subsubsection{Qualitative comparisons on CelebA}
\label{app:subsec:celeba-qualitative}
At around 4,096-NFE target, we compare outputs of baselines for five attribute conditions. 

\begin{table}[H]
\centering
\setlength{\tabcolsep}{2pt}
\caption{\textbf{CelebA attribute-conditioned outputs at a 4,096-NFE target budget.}  $+$/$-$ denote required presence/absence; beard is expressed directly. Bold marks the largest GM in each row.}
\label{tab:celeba-qualitative}
\begin{tabular*}{\linewidth}{@{\extracolsep{\fill}}p{0.23\linewidth}ccccc@{}}
\toprule
Condition & BoN & PDM-SMC & \shortstack{CS-VGB\\Search} & PDM-G & CS-VGB-G \\
\midrule
\parbox[c]{\linewidth}{\footnotesize\raggedright\textbf{ID 011}\\$-$ smiling, $+$ eyeglasses, $-$ bangs, $+$ hat, $-$ mouth open, $+$ black hair, Beard, $+$ bushy brows, $-$ lipstick, $+$ earrings.} & \raisebox{-0.5\height}{\shortstack{\includegraphics[width=0.135\linewidth]{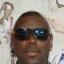}\\0.2688\\\scriptsize NFE 4,096}} & \raisebox{-0.5\height}{\shortstack{\includegraphics[width=0.135\linewidth]{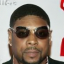}\\0.7476\\\scriptsize NFE 4,068}} & \raisebox{-0.5\height}{\shortstack{\includegraphics[width=0.135\linewidth]{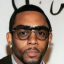}\\0.7634\\\scriptsize NFE 4,152}} & \raisebox{-0.5\height}{\shortstack{\includegraphics[width=0.135\linewidth]{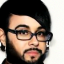}\\0.9453\\\scriptsize NFE 4,068}} & \raisebox{-0.5\height}{\shortstack{\includegraphics[width=0.135\linewidth]{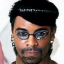}\\\textbf{0.9805}\\\scriptsize NFE 4,101}} \\[5pt]
\parbox[c]{\linewidth}{\footnotesize\raggedright\textbf{ID 013}\\$-$ smiling, $+$ eyeglasses, $+$ bangs, $-$ hat, $+$ mouth open, $-$ blond hair, No beard, $+$ bushy brows, $+$ lipstick, $-$ earrings.} & \raisebox{-0.5\height}{\shortstack{\includegraphics[width=0.135\linewidth]{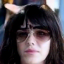}\\0.4232\\\scriptsize NFE 4,096}} & \raisebox{-0.5\height}{\shortstack{\includegraphics[width=0.135\linewidth]{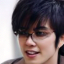}\\0.7046\\\scriptsize NFE 4,068}} & \raisebox{-0.5\height}{\shortstack{\includegraphics[width=0.135\linewidth]{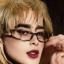}\\0.9511\\\scriptsize NFE 4,184}} & \raisebox{-0.5\height}{\shortstack{\includegraphics[width=0.135\linewidth]{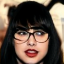}\\0.9668\\\scriptsize NFE 4,068}} & \raisebox{-0.5\height}{\shortstack{\includegraphics[width=0.135\linewidth]{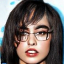}\\\textbf{0.9923}\\\scriptsize NFE 4,103}} \\[5pt]
\parbox[c]{\linewidth}{\footnotesize\raggedright\textbf{ID 001}\\$+$ smiling, $+$ eyeglasses, $-$ bangs, $+$ hat, $-$ mouth open, $-$ black hair, No beard, $-$ bushy brows, $-$ lipstick, $+$ earrings.} & \raisebox{-0.5\height}{\shortstack{\includegraphics[width=0.135\linewidth]{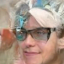}\\0.5476\\\scriptsize NFE 4,096}} & \raisebox{-0.5\height}{\shortstack{\includegraphics[width=0.135\linewidth]{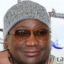}\\0.9456\\\scriptsize NFE 4,068}} & \raisebox{-0.5\height}{\shortstack{\includegraphics[width=0.135\linewidth]{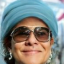}\\0.9775\\\scriptsize NFE 4,104}} & \raisebox{-0.5\height}{\shortstack{\includegraphics[width=0.135\linewidth]{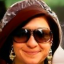}\\0.9678\\\scriptsize NFE 4,068}} & \raisebox{-0.5\height}{\shortstack{\includegraphics[width=0.135\linewidth]{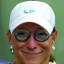}\\\textbf{0.9900}\\\scriptsize NFE 4,101}} \\[5pt]
\parbox[c]{\linewidth}{\footnotesize\raggedright\textbf{ID 003}\\$+$ smiling, $-$ eyeglasses, $-$ bangs, $+$ hat, $+$ mouth open, $+$ black hair, No beard, $-$ bushy brows, $-$ lipstick, $-$ earrings.} & \raisebox{-0.5\height}{\shortstack{\includegraphics[width=0.135\linewidth]{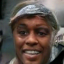}\\0.6736\\\scriptsize NFE 4,096}} & \raisebox{-0.5\height}{\shortstack{\includegraphics[width=0.135\linewidth]{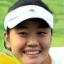}\\0.7743\\\scriptsize NFE 4,068}} & \raisebox{-0.5\height}{\shortstack{\includegraphics[width=0.135\linewidth]{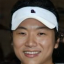}\\0.9759\\\scriptsize NFE 4,184}} & \raisebox{-0.5\height}{\shortstack{\includegraphics[width=0.135\linewidth]{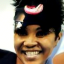}\\0.9670\\\scriptsize NFE 4,068}} & \raisebox{-0.5\height}{\shortstack{\includegraphics[width=0.135\linewidth]{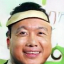}\\\textbf{0.9889}\\\scriptsize NFE 4,107}} \\[5pt]
\parbox[c]{\linewidth}{\footnotesize\raggedright\textbf{ID 006}\\$+$ smiling, $-$ eyeglasses, $+$ bangs, $-$ hat, $-$ mouth open, $-$ brown hair, Beard, $-$ bushy brows, $-$ lipstick, $-$ earrings.} & \raisebox{-0.5\height}{\shortstack{\includegraphics[width=0.135\linewidth]{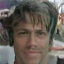}\\0.7856\\\scriptsize NFE 4,096}} & \raisebox{-0.5\height}{\shortstack{\includegraphics[width=0.135\linewidth]{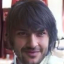}\\0.9525\\\scriptsize NFE 4,068}} & \raisebox{-0.5\height}{\shortstack{\includegraphics[width=0.135\linewidth]{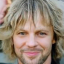}\\0.9921\\\scriptsize NFE 4,168}} & \raisebox{-0.5\height}{\shortstack{\includegraphics[width=0.135\linewidth]{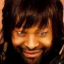}\\0.9782\\\scriptsize NFE 4,068}} & \raisebox{-0.5\height}{\shortstack{\includegraphics[width=0.135\linewidth]{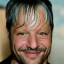}\\\textbf{0.9969}\\\scriptsize NFE 4,106}} \\[5pt]
\bottomrule
\end{tabular*}
\end{table}
\FloatBarrier

\begin{figure}[!htb]
\centering
\vspace{-4pt}
\captionsetup{skip=3pt}
\includegraphics[width=0.76\linewidth]{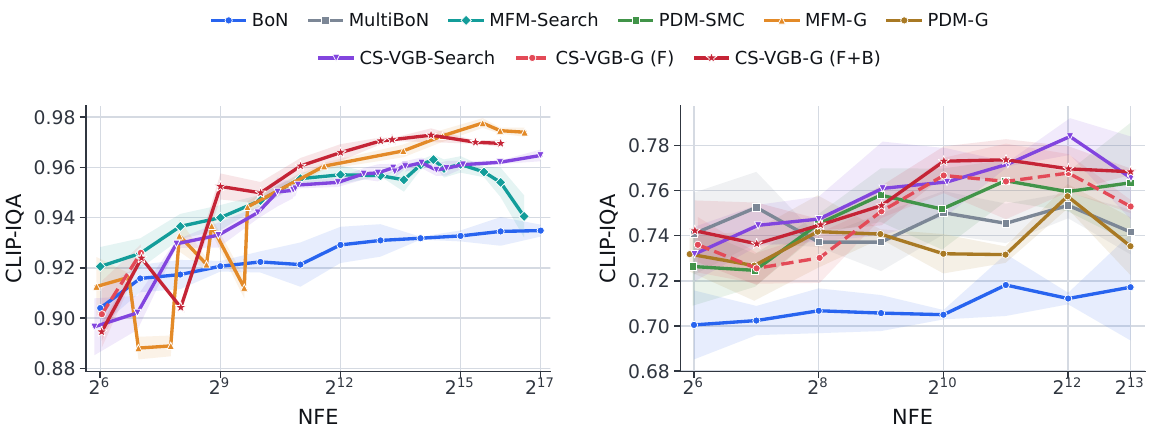}
\par
\makebox[0.38\linewidth][c]{\small (a) ImageNet with MFM}%
\makebox[0.38\linewidth][c]{\small (b) CelebA with PDM}
\par
\caption{\textbf{Independent image-quality evaluation.} CLIP-IQA of images selected using HPSv2 on ImageNet and attribute GM on CelebA. Curves show mean $\pm1$ SD across three seeds, each evaluated on 100 conditions, against actual mean NFE. CLIP-IQA is not used for reward optimization.}
\label{fig:app-imagenet-heldout}
\end{figure}

\subsection{Multi-Robot Planning with Diffusion Models}
\label{app:subsec:planning-settings}

\paragraph{Tasks and reference models.}
We use MMD's pretrained diffusion models ~\citep{exp:shaoul2025mmd}\footnote{\url{https://github.com/yoraish/mmd}}. Highways and Circle involve 16 robots; Conveyor Moderate and Hard involve four. Circle has a central obstacle of radius $0.20$, and the Conveyor tasks differ in blocking geometry (\cref{fig:planning-success-trajectories}). Joint success requires all robots to satisfy their prescribed start and goal positions and remain collision-free within the workspace. Collision checks account for robot radii and cover the linearly interpolated segments between consecutive waypoints. Each task has 64 fixed instances generated by perturbing starts, goals, and obstacles.

\begin{figure}[!htbp]
\centering
\begin{minipage}{0.245\linewidth}
\centering
\includegraphics[width=\linewidth,trim=45 63.2 1125 85,clip]{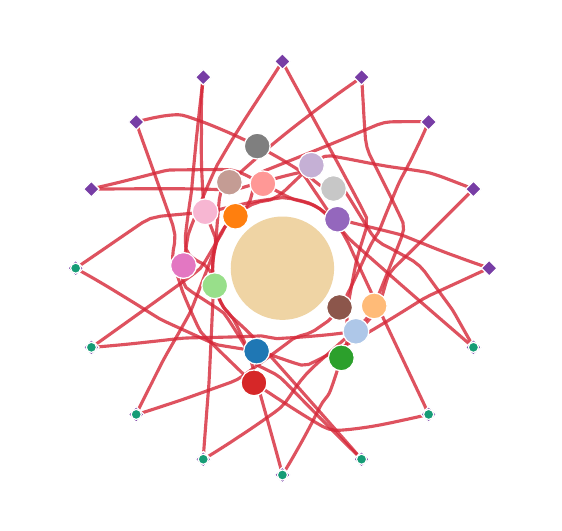}
\par\smallskip
{\small (a) Circle\par}
\end{minipage}\hfill
\begin{minipage}{0.245\linewidth}
\centering
\includegraphics[width=\linewidth,trim=405 63.2 765 85,clip]{figures/planning_success_row_source.pdf}
\par\smallskip
{\small (b) Conveyor Moderate\par}
\end{minipage}\hfill
\begin{minipage}{0.245\linewidth}
\centering
\includegraphics[width=\linewidth,trim=765 63.2 405 85,clip]{figures/planning_success_row_source.pdf}
\par\smallskip
{\small (c) Conveyor Hard\par}
\end{minipage}\hfill
\begin{minipage}{0.245\linewidth}
\centering
\includegraphics[width=\linewidth,trim=1125 63.2 45 85,clip]{figures/planning_success_row_source.pdf}
\par\smallskip
{\small (d) Highways\par}
\end{minipage}
\caption{\textbf{Successful multi-robot trajectories from CS-VGB-G.} Examples in Circle, Conveyor Moderate, Conveyor Hard, and Highways. Red curves show trajectories, green dots mark starts, purple diamonds mark goals, and shaded regions indicate obstacles. The Highways example uses eight robots, whereas the quantitative Highways experiments use 16.}
\label{fig:planning-success-trajectories}
\end{figure}

\paragraph{Reward and process verifier.}
Following the cost-based guidance formulation of MMD~\citep{exp:shaoul2025mmd}, we use $C=\lambda_{\mathrm{col}}C_{\mathrm{col}}+\lambda_{\mathrm{smooth}}C_{\mathrm{smooth}}+\lambda_{\mathrm{obj}}C_{\mathrm{obj}}$, combining penalties for inter-robot collisions, trajectory smoothness, and obstacle avoidance, respectively. Costs are evaluated in physical coordinates, with smoothness measured by a constant-velocity Gaussian-process energy. For a joint trajectory $x$, we define $r(x)=\exp(-C(x)/2)$. At a noisy checkpoint, we set $\widehat V_j(x_j)=r(\widehat x_0(x_j))$ using a single denoiser clean prediction, clipped to the valid range of the model's normalized trajectory coordinates and conditioned on the prescribed start and goal states; at the clean level, $\widehat V_0=r$.

\paragraph{CS-VGB updates and guidance.}
CS-VGB uses one chain, checkpoint grid $[6,2,0]$, and guided DDPM rollouts between checkpoints. After each forward DDPM step, we apply iterative cost guidance while keeping the start and goal states fixed. Backward proposals are unguided. We set $L_F=L_B=2$, $\alpha=0.3$, and $\chi=1$, with reverse insertion, lazy pool evaluation, and skipped stays. We use the search protocol in \cref{app:subsec:practical-implementation-details}, archive evaluated clean trajectories, and stop upon feasibility or the NFE/update limit (12,288 updates). We disable renoising warp for this task. If the search ends with neither a clean visit nor a feasible candidate, we complete denoising to the clean level using a forward-only fallback, which may exceed the requested NFE budget. For output selection, we prioritize feasible trajectories and select the lowest-cost one among those evaluated.

\paragraph{Evaluation and computational cost.}
Guided NFE caps are $2^9,2^{10},2^{11},2^{12},2^{13},12{,}288,2^{14}$. We plot actual mean model NFE per joint trial, including early stopping, without dividing by robot count. \Cref{tab:planning-final-results} reports overall and task-wise success rates alongside actual mean NFE at the largest cap. CS-VGB-G achieves the highest overall success with the fewest model evaluations on average; its advantage is concentrated in Highways and Circle, while all methods are near ceiling on the Conveyor tasks.

\begin{table}[!htb]
\centering\small
\caption{\textbf{Planning results at the largest evaluated setting.} All methods use NFE cap $2^{14}$. Success entries are percentages.}
\label{tab:planning-final-results}
\begin{tabular}{lrrrrrr}
\toprule
Method & Overall $\pm$ SD & Mean NFE & Highways & Circle & Conv. Mod. & Conv. Hard\\
\midrule
BoN-G&$67.58\pm1.56$&6,355.5&64.69&5.62&100.00&100.00\\
DFS-G&$60.08\pm1.25$&8,029.7&40.00&0.31&100.00&100.00\\
ABCD-G&$71.17\pm1.52$&5,765.2&73.12&11.56&100.00&100.00\\
CS-VGB-G&$84.84\pm1.98$&5,044.8&80.00&60.00&99.69&99.69\\
\bottomrule
\end{tabular}
\end{table}
\par
\FloatBarrier

\subsection{Additional Ablations on Search Dynamics}
\label{app:subsec:search-efficiency-ablations}
We study lazy pool evaluation and directional persistence on 100 CelebA conditions (examples in \cref{fig:celeba-prompt-examples}). To measure evaluation cost, we compare lazy ON and OFF at $\chi=1$ with the same number of updates within each search variant: 313 for guided CS-VGB and 45 for unguided search. Unlike a fixed-NFE comparison, this protocol measures how much computation each variant consumes for a fixed number of search updates.

\begin{table}[!htbp]
\centering\small
\setlength{\tabcolsep}{4pt}
\caption{\textbf{Lazy pool evaluation at matched update counts on CelebA.} NFE averages all 100 runs.}
\label{tab:celeba-search-efficiency-ablations}
\begin{tabular}{@{}lrrrrrr@{}}
\toprule
& & \multicolumn{2}{c}{Mean NFE} & & \multicolumn{2}{c}{Attribute GM}\\
\cmidrule(lr){3-4}\cmidrule(lr){6-7}
Variant & Updates & \multicolumn{1}{c}{ON} & \multicolumn{1}{c}{OFF} & Saved & \multicolumn{1}{c}{ON} & \multicolumn{1}{c@{}}{OFF}\\
\midrule
Guided   &313&4,103.89&5,955.02&31.1\%&0.98138&0.98182\\
Unguided &45 &4,174.24&5,738.72&27.3\%&0.83989&0.83953\\
\bottomrule
\end{tabular}
\end{table}

\paragraph{Lazy pool evaluation.}
At matched update counts, lazy evaluation reduces mean NFE by $31.1\%$ with guidance (\cref{tab:celeba-search-efficiency-ablations}). Guided search retains nearly the same mean GM, while mean runtime decreases from $73.27$ to $48.63$ seconds, a $33.6\%$ reduction. A similar computational saving is observed in unguided search, where lazy evaluation reduces mean NFE by $27.3\%$ with similar mean GM among runs that reach a clean output. These results illustrate how deferring opposite-pool evaluations reduces the computational cost of continued search.

\paragraph{Directional persistence.}
We also vary directional persistence with lazy ON and an initial target budget of 4,096 NFE (\cref{tab:celeba-persistence-ablation}). Direction reversals and clean visits are measured at this initial budget over all 100 runs, including zero visits for runs without a clean output. Reward averages use all 100 conditions after supplementing missing outputs at higher NFE caps and are therefore not compute-matched across $\chi$.

\begin{table}[!htbp]
\centering\small
\caption{\textbf{Directional persistence on CelebA with lazy ON.} G/U denote guided/unguided search. Panel (a) reports supplemented attribute GM; panel (b) reports mean counts per run at the initial 4,096-NFE target, comparing $\chi=0$ with $\chi=1$.}
\label{tab:celeba-persistence-ablation}
\begin{minipage}[t]{0.32\linewidth}
\vspace{0pt}\centering
\setlength{\tabcolsep}{3pt}
\begin{tabular*}{\linewidth}{@{\extracolsep{\fill}}crr@{}}
\toprule
\multicolumn{3}{c}{(a) Attribute GM}\\
\midrule
$\chi$ & \multicolumn{1}{c}{G} & \multicolumn{1}{c@{}}{U}\\
\midrule
0   &0.9787&0.8214\\
0.3 &0.9803&0.8296\\
0.5 &0.9816&0.8233\\
0.7 &0.9804&0.8282\\
1   &0.9804&0.8264\\
\bottomrule
\end{tabular*}
\end{minipage}\hfill
\begin{minipage}[t]{0.65\linewidth}
\vspace{0pt}\centering
\setlength{\tabcolsep}{3pt}
\begin{tabular*}{\linewidth}{@{\extracolsep{\fill}}lccc@{}}
\toprule
\multicolumn{4}{c}{(b) Search dynamics: $\chi=0\ \to\ \chi=1$}\\
\midrule
Variant & F-to-B flips & B-to-F flips & Clean visits\\
\midrule
G & $45.81\to8.85$ & $55.19\to20.53$ & $12.47\to15.99$\\
U & $7.78\to4.35$ & $9.10\to6.17$ & $1.81\to2.21$\\
\bottomrule
\end{tabular*}
\end{minipage}
\end{table}

Increasing $\chi$ from 0 to 1 reduces flips in both directions and increases clean visits for both variants. This pattern is qualitatively consistent with the role of persistence in \cref{prop:main-hitting}, which establishes $\Theta(J)$ expected updates to the first clean visit under its stated assumptions. Reward changes are small and nonmonotonic, with guided GM varying by 0.0029 across the sweep. Persistence therefore primarily changes checkpoint traversal in these experiments, rather than uniformly improving reward.

\end{document}